\documentclass[12pt]{article}
\usepackage{amsmath}
\usepackage{graphicx}
\usepackage{enumerate}
\usepackage{url}

\usepackage[a4paper, lmargin=2.7cm, rmargin=2.5cm, bottom=2.5cm, top=2.3cm]{geometry}
\usepackage{amsmath, amsthm, amssymb}
\usepackage{graphicx}
\graphicspath{{./figures/}}
\usepackage[round]{natbib}
\usepackage{float}

\usepackage{subcaption}
\newtheorem{theorem}{Theorem}[section]

\newtheorem{remark}[theorem]{Remark}
\newtheorem{corollary}[theorem]{Corollary}

\theoremstyle{definition}

\newtheorem{definition}[theorem]{Definition}
\newtheorem{assumption}{Assumption}

\usepackage{subcaption}

\usepackage[hidelinks]{hyperref}

\usepackage{xr}
\usepackage{xcolor}

\title{Conformal confidence sets for biomedical image segmentation}
\author{Samuel Davenport}
\date{\today}

\begin{document}
	
\maketitle

%========================================================================
% Abstract
%========================================================================
\begin{abstract}
	We develop confidence sets which provide spatial uncertainty guarantees for the output of a black-box machine learning model designed for image segmentation. To do so we adapt conformal inference to the imaging setting, obtaining thresholds on a calibration dataset based on the distribution of the maximum of the transformed logit scores within and outside of the ground truth masks. We prove that these confidence sets, when applied to new predictions of the model, are guaranteed to contain the true unknown segmented mask with desired probability. We show that learning appropriate score transformations on a learning dataset before performing calibration is crucial for optimizing performance. We illustrate and validate our approach on a polpys tumor dataset. To do so we obtain the logit scores from a deep neural network trained for polpys segmentation and show that using distance transformed scores to obtain outer confidence sets and the original scores for inner confidence sets enables tight bounds on tumor location whilst controlling the false coverage rate. 
\end{abstract}

%========================================================================
% Introduction
%========================================================================
\section{Introduction}

Deep neural networks promise to significantly enhance a wide range of important tasks in biomedical imaging. However these models, as typically used, lack formal uncertainty guarantees on their output which can lead to overconfident predictions and critical errors \citep{Guo2017, Gupta2020}. Misclassifications or inaccurate segmentations can lead to serious consequences, including misdiagnosis, inappropriate treatment decisions, or missed opportunities for early intervention \citep{Topol2019}. Without uncertainty quantification, medical professionals cannot rely on deep learning models to provide accurate information and predictions which can limit their use in practical applications \citep{Jungo2020}. 
%As these models are increasingly deployed in critical real-world scenarios quantifying the uncertainty associated with their predictions is a significant challenge.

In order to address this problem, conformal inference, a robust framework for uncertainty quantification, has become increasingly used as a means of providing prediction guarantees, offering reliable, distribution-free confidence sets for the output of neural networks which have finite sample validity. This approach, originally introduced in \cite{Papadopoulos2002, Vovk2005}, has become increasingly popular due to its ability to provide rigorous statistical guarantees without making strong assumptions about the underlying data distribution or model architecture. Conformal prediction methods, in their most commonly used form - split conformal inference - work by calibrating the predictions of the model on a held-out dataset in order to provide sets which contain the output with a given probability, see \cite{Shafer2008} and \cite{Angelopoulos2021} for good introductions.

In the context of image segmentation, we have a decision to make at each pixel/voxel of an image which can lead to a large multiple testing problem. Traditional conformal methods, typically designed for scalar outputs, require adaptation to handle multiple tests and their inherent spatial dependencies. To do so \cite{Angelopoulos2021LTT} applied conformal inference pixelwise and performed multiple testing correction on the resulting $p$-values, however this approach does not account for the complex dependence structure inherent in the images. To take advantage of this structure, in an approach analogous to the FDR control of \citep{Benjamini1995}, \cite{Bates2021} and \cite{Angelopoulos2022} sought to control the expected risk of a given loss function over the image and used a conformal approach to produce outer confidence sets for segmented images which control the expected proportion of false negatives. Other work considering conformal inference in the context of multiple dependent hypotheses includes \cite{Marandon2024} and \cite{Blanchard2024} who established conformal FDR control when testing for the presence of missing links in graphs.

In this work we argue that bounding the segmented outcome with guarantees in probability rather than on the proportion of discoveries is more informative, avoiding errors at the borders of potential tumors. This is analogous to the tradeoff between FWER and FDR/FDP control in the multiple testing literature in which there is a balance between power and coverage rate, the distinction being that in medical image segmentation making mistakes can have potentially serious consequences. Under-segmentation might cause part of the tumor to be missed, potentially leading to inadequate treatment \citep{Jalalifar2022}. Over-segmentation, on the other hand, could result in unnecessary interventions, increasing patient risk and healthcare costs \citep{Gupta2020, Patz2014}. Confidence sets are instead guaranteed to contain the outcome with a given level of certainty. Since the guarantees are more meaningful the problem is more difficult and existing work on conformal uncertainty quantification for images has thus often focused on producing sets with guarantees on the proportions of discoveries or pixel level inference rather than coverage (\cite{Bates2021}, \cite{Wieslander2020}, \cite{Mossina2024}) which is a stricter error criterion. 
%Existing work on conformal confidence sets which aim to provide coverage of the entire ground truth mask with a given probability has primarily focused on bounding boxes, see e.g. \citep{De2022, Andeol2023, Mukama2024}. 
%In order to address this, we use a held out learning dataset to learn the score transformations which provide the most informative confidence regions.

In order to obtain confidence sets we use a split-conformal inference approach in which we learn appropriate cutoffs, with which to threshold the output of an image segmenter, from a calibration dataset. These thresholds are obtained by considering the distribution of the maximum logit (transformed) scores provided by the model within and outside of the ground truth masks. This approach allows us to capture the spatial nature of the uncertainty in segmentation tasks, going beyond simple pixel-wise confidence measures. By applying these learned thresholds to new predictions, we can generate inner and outer confidence sets that are guaranteed to contain the true, unknown segmented mask with a desired probability. As we shall see, naively using the original model scores to do so can lead to rather large and uninformative outer confidence sets but these can be greatly improved using distance transformations. 
%In the following sections, we will explore the technical details of our method, present our theoretical results, and illustrate and validate our approach on a polpys tumor dataset. In particular Section XXX provides the theory for constructing joint and marginal conformal confidence sets and includes an extension to full conformal inference. We provide theoretical guarantees on the coverage properties of our confidence sets, ensuring their reliability across different datasets and segmentation models. In Section XXX, we learn appropriate score transformations on a set aside learning dataset. We perform conformal inference on a clibrati

%========================================================================
% Main document
%========================================================================
\section{Theory}
\subsection{Set up}
Let $\mathcal{V} \subset \mathbb{R}^m$, for some dimension $m \in \mathbb{N}$, be a finite set corresponding to the domain which represents the pixels/voxels/points at which we observe imaging data. Let $\mathcal{X} = \lbrace g: \mathcal{V} \rightarrow \mathbb{R}\rbrace$ be the set of real functions on $\mathcal{V}$ and let $\mathcal{Y} = \lbrace g: \mathcal{V} \rightarrow \lbrace 0,1 \rbrace \rbrace$ be the set of all functions on $\mathcal{V}$ taking the values 0 or 1. We shall refer to elements of $\mathcal{X}$ and $\mathcal{Y}$ as images. Suppose that we observe a calibration dataset $(X_i, Y_i)_{i = 1}^n$ of random images, where $X_i: \mathcal{V} \rightarrow \mathbb{R}$ represents the $i$th observed calibration image and $Y_i:\mathcal{V} \rightarrow \lbrace 0, 1\rbrace$ outputs labels at each $v \in \mathcal{V}$ giving 1s at the true location of the objects in the image $X_i$ that we wish to identify and 0s elsewhere. Let $\mathcal{P}(\mathcal{V})$ be the set of all subsets of $\mathcal{V}$. Given a function $f:\mathcal{X} \rightarrow \mathcal{X}$, we shall write $f(X,v)$ to denote $f(X)(v)$ for all $v \in \mathcal{V}$. 

Let $s:\mathcal{X}  \rightarrow \mathcal{X} $ be a score function - trained on an independent dataset - such that given an image pair $(X,Y) \in \mathcal{X}\times \mathcal{Y}$, $s(X)$ is a score image in which $s(X,v) $ is intended to be higher at the $v \in \mathcal{V}$ for which $Y(v) = 1$. The score function can for instance be the logit scores obtained from applying a deep neural network image segmentation method to the image $X$. Given $X \in \mathcal{X}$, let $\hat{M}(X) \in \mathcal{Y}$ be the predicted mask given by the model which is assumed to be obtained using the scores $s(X)$.
%Let $T(Y) = \left\lbrace v\in \mathcal{V}: Y(v) = 1 \right\rbrace$ correspond to the true location of the objects in the image $X$. 

In what follows we will use the calibration dataset to construct confidence functions $I,O:  \mathcal{X}  \rightarrow \mathcal{P}(\mathcal{V})$ such that for a new image pair $(X,Y)$, given error rates $\alpha_1, \alpha_2 \in (0,1)$ we have
\begin{equation}\label{eq:probstat1}
	\mathbb{P}\left( I(X) \subseteq \lbrace v\in \mathcal{V}: Y(v) = 1 \rbrace  \right) \geq 1 - \alpha_1, 
\end{equation}
\begin{equation}\label{eq:probstat3}
	\text{ and } 	\mathbb{P}\left( \lbrace v\in \mathcal{V}: Y(v) = 1 \rbrace \subseteq O(X)  \right) \geq 1 - \alpha_2.
\end{equation}
Here $I(X)$ and $O(X)$ serve as inner and outer confidence sets for the location of the true segmented mask. Their interpretation is that, up to the guarantees provided by the probabilistic statements \eqref{eq:probstat1} and \eqref{eq:probstat3}, we can be sure that for each $v\in I(X)$, $Y(v) = 1$ or that for each $v \not\in O(X)$, $Y(v) = 0$. Joint control over the events can also be guaranteed, either via sensible choices of $\alpha_1$ and $\alpha_2$ or by using the joint distribution of the maxima of the logit scores - see Section \ref{SS:joint}. 

In order to establish conformal confidence results we shall require the following exchangeablity assumption. 
\begin{assumption}\label{ass:ex}
		Given a new random image pair, $(X_{n+1},Y_{n+1})$, suppose that $(X_i, Y_i)_{i = 1}^{n+1}$ is an exchangeable sequence of random image pairs in the sense that 
	\begin{equation*}
		\left\lbrace (X_1,Y_1), \dots, (X_{n+1}, Y_{n+1}) \right\rbrace =_d \left\lbrace (X_{\sigma(1)}, Y_{\sigma(1)}), \dots, (X_{\sigma(n+1)}, Y_{\sigma(n+1)}) \right\rbrace
	\end{equation*}
	for all permutations $\sigma \in S_{n+1}$. Here $=_d$ denotes equality in distribution and $S_{n+1} $ is the group of permutations of the integers $\lbrace1, \dots, n+1\rbrace$.
\end{assumption}
Exchangeability or a variant is a standard assumption in the conformal inference literature \citep{Angelopoulos2021} and facilitates coverage guarantees. It holds for instance if we assume that the collection $(X_i, Y_i)_{i = 1}^{n+1}$ is an i.i.d. sequence of image pairs but is more general and in principle allows for other dependence structures. 

\subsection{Marginal confidence sets}\label{SS:MCS}
In order to construct conformal confidence sets let $f_I, f_O:\mathcal{X} \rightarrow \mathcal{X}$ be inner and outer transformation functions and for each $1\leq i \leq n +1 $, let $\tau_i = \max_{v \in \mathcal{V}: Y_i(v) = 0} f_I(s(X_i), v)$ and $\gamma_i = \max_{v \in \mathcal{V}: Y_i(v) = 1} -f_O(s(X_i), v)$  be the maxima of the function transformed scores over the areas at which the true labels equal 0 and 1 respectively. We will require the following assumption on the scores and the transformation functions.
\begin{assumption}\label{ass:indep}
	(Independence of scores) $(X_i, Y_i)_{i = 1}^{n+1}$ is independent of the functions $s, f_O, f_I$. 
\end{assumption}

Given this we construct confidence sets as follows.
\begin{theorem}\label{thm:inner}
	(Marginal inner set)
	Under Assumptions \ref{ass:ex} and \ref{ass:indep}, given $\alpha_1 \in (0,1)$, let 
	\begin{equation*}
		\lambda_I(\alpha_1) = \inf\left\lbrace \lambda: \frac{1}{n} \sum_{i = 1}^n 1\left[ \tau_i\leq \lambda \right] \geq \frac{\lceil (1-\alpha_1)(n+1) \rceil}{n}\right\rbrace,
	\end{equation*}
%	be the upper $\alpha_1$ quantile of $(\tau_i)_{i = 1}^n$
	and define $I(X) = \lbrace v \in \mathcal{V}: f_I(s(X), v) >\lambda_I(\alpha_1)  \rbrace $. Then,
	\begin{equation}\label{eq:probstat}
		\mathbb{P}\left( I(X_{n+1}) \subseteq\lbrace v\in \mathcal{V}: Y_{n+1}(v) = 1 \rbrace \right) \geq 1 - \alpha_1.
	\end{equation}
\end{theorem}
\begin{proof}
	Under Assumptions \ref{ass:ex} and \ref{ass:indep}, exchangeability of the image pairs implies exchangeability of the sequence $(\tau_i)_{i = 1}^{n+1}$. In particular, $\lambda_I(\alpha_1)$ is the upper $\alpha_1$ quantile of the distribution of $(\tau_i)_{i = 1}^{n} \cup \lbrace \infty \rbrace $ and so, by Lemma 1 of \cite{Tibshirani2019}, it follows that 
	\begin{equation*}
	\mathbb{P}\left(\tau_{n+1} \leq \lambda_I(\alpha_1) \right) \geq 1 - \alpha_1. 
	\end{equation*}
	Now consider the event that $\tau_{n+1}\leq \lambda_I(\alpha_1)$. On this event, $ f_I(s(X_{n+1}),v) \leq \lambda_I(\alpha_1) $
	for all $v \in \mathcal{V}$ such that $Y_{n+1}(v) = 0$. As such, given $u \in \mathcal{V}$ such that $ f_I(s(X_{n+1}), u) > \lambda_I(\alpha_1) $, we must have $Y_{n+1}(u) = 1$ and so $I(X_{n+1}) \subseteq \lbrace v\in \mathcal{V}: Y_{n+1}(v) = 1 \rbrace  $. It thus follows that
	\begin{equation*}
	\mathbb{P}\left( I(X_{n+1}) \subseteq \lbrace v\in \mathcal{V}: Y_{n+1}(v) = 1 \rbrace  \right) \geq \mathbb{P}\left(\tau_{n+1} \leq \lambda_I(\alpha_1) \right) \geq 1 - \alpha_1. 
\end{equation*}
\end{proof}
\noindent For the outer set we have the following analogous result.
\begin{theorem}\label{thm:outer}
	(Marginal outer set)
	Under Assumptions \ref{ass:ex} and \ref{ass:indep}, given $\alpha_2 \in (0,1)$, let 
	\begin{equation*}
		\lambda_O({\alpha_2})= \inf\left\lbrace \lambda: \frac{1}{n} \sum_{i = 1}^n 1\left[ \gamma_i\leq \lambda \right] \geq \frac{\lceil (1-\alpha_2)(n+1) \rceil}{n} \right\rbrace,
	\end{equation*}
%	be the upper $\alpha_2$ quantile of $(\gamma_i)_{i = 1}^n$
	and define $O(X) = \lbrace v \in \mathcal{V}: -f_O(s(X), v) \leq \lambda_O(\alpha_2)  \rbrace $. Then,
	\begin{equation}\label{eq:probstat}
		\mathbb{P}\left( \lbrace v\in \mathcal{V}: Y_{n+1}(v) = 1 \rbrace \subseteq O(X_{n+1}) \right) \geq 1 - \alpha_2.
	\end{equation}
\end{theorem}
\begin{proof}
	Arguing as in the proof of Theorem \ref{thm:inner}, it follows that $\mathbb{P}\left(\gamma_{n+1} \leq \lambda_O(\alpha_2) \right) \geq 1 - \alpha_2.$
	Now on the event that $\gamma_{n+1}\leq \lambda_O(\alpha_2)$ we have $ -f_O(s(X_{n+1}),v) \leq \lambda_O(\alpha_2) $ for all $v \in \mathcal{V}$ such that $Y_{n+1}(v) = 1$. As such, given $u \in \mathcal{V}$ such that $ -f_O(s(X_{n+1}),u) > \lambda_I(\alpha) $, we must have $Y_{n+1}(u) = 0$ and so $	O(X)^C  \subseteq \lbrace v\in \mathcal{V}: Y_{n+1}(v) = 0 \rbrace  $. The result then follows as above.
\end{proof}
%\noindent The proof of Theorem \ref{thm:outer} follows that of Theorem \ref{thm:inner} and is thus omitted. 
\begin{remark}\label{rmk:max}
	We have used the maximum over the transformed scores in order to combine score information on and off the ground truth masks. The maximum is a natural combination function in imaging and is commonly used in the context of multiple testing \citep{Worsley1992}. However the theory above is valid for any increasing combination function. We show this in Appendix \ref{A:CF} where we establish generalized versions of these results.
\end{remark}
\begin{remark}
	Inner and outer coverage can also be viewed as a special case of conformal risk control with an appropriate choice of loss function. We can thus instead establish coverage results as a corollary to risk control, see Appendix \ref{risk2con} for details. This amounts to an alternative proof of the results as the proof of the validity of risk control is different though still strongly relies on exchangeability.
\end{remark}
%\begin{remark}
%	Importantly the coverage of the sets $U_M(X)$ and $V_M(X)$ is not jointly valid and so when using these results the choice of inner versus outer set must be made in advance.
%\end{remark}

%\subsection{Confidence sets for connected components}

%\subsection{Full conformal confidence sets}
%We have so far assumed that we have a calibration dataset available, separate from the training data used to contruct the score function, on which we can learn cutoffs and use them to provide conformal confidence sets, using split conformal prediction. As an alternative, we could instead use full conformal prediction in which the entire dataset is used to both train the model and to provide conformal uncertainty. 
%
%To do so let $s_{}$
%
%\begin{remark}
%	Full conformal confidence sets come with the same drawbacks as full conformal inference. In particular they can be very computationally expensive to generate because they require retraining the model for each. As a result, this approach does not scale well when the dataset is large and will often not be practical.
%\end{remark}

\subsection{Joint confidence sets}\label{SS:joint}
Instead of focusing on marginal control one can instead spend all of the $\alpha$ available to construct sets which have a joint probabilistic guarantees. This gain comes at the expense of a loss of precision. The simplest means of constructing jointly valid confidence sets is via the marginal sets themselves.
\begin{corollary}\label{cor:weighting}
	(Joint from marginal) Assume Assumptions \ref{ass:ex} and \ref{ass:indep} hold and given $\alpha \in (0,1)$ and $\alpha_1, \alpha_2 \in (0,1)$ such that $\alpha_1 + \alpha_2 \leq \alpha$, define $I(X)$ and  $O(X)$ as in Theorems \ref{thm:inner} and \ref{thm:outer}. Then 
	\begin{equation}
		\mathbb{P}\left( I(X_{n+1}) \subseteq \lbrace v\in \mathcal{V}: Y_{n+1}(v) = 1 \rbrace \subseteq O(X_{n+1})  \right) \geq  1-\alpha. 
	\end{equation}
\end{corollary}
Alternatively joint control can be obtained using the joint distribution of the maxima of the transformed logit scores as follows.
\begin{theorem}\label{thm:joint}
	(Joint coverage) Assume that Assumption \ref{ass:ex} and \ref{ass:indep}  hold. Given $\alpha \in (0,1)$, define 
	\begin{equation*}
		\lambda(\alpha) = \inf\left\lbrace \lambda: \frac{1}{n} \sum_{i = 1}^n 1\left[ \max(\tau_i, \gamma_i) \leq \lambda \right] \geq \frac{\lceil (1-\alpha)(n+1) \rceil}{n} \right\rbrace.
	\end{equation*}
% 	be the upper $\alpha$-quantile of the distribution of $\max(\tau_i, \gamma_i)$ over $1 \leq i \leq n$.\\
 Let $O(X) = \lbrace v \in \mathcal{V}: -f_O(s(X),v) \leq \lambda(\alpha) \rbrace $ and $I(X) = \lbrace v \in \mathcal{V}: f_I(s(X),v) >	\lambda(\alpha) \rbrace $. Then,
\begin{equation}\label{eq:probstat}
	\mathbb{P}\left( I(X_{n+1}) \subseteq \lbrace v\in \mathcal{V}: Y_{n+1}(v) = 1 \rbrace \subseteq O(X_{n+1}) \right) \geq 1 - \alpha.
\end{equation}
\end{theorem}
\begin{proof}
	Exchangeability of the image pairs implies exchangeability of the sequence $(\tau_i, \gamma_i)_{i = 1}^{n+1}$. Moreover on the event that $\max(\tau_{n+1}, \gamma_{n+1}) \leq \lambda(\alpha)$ we have $\tau_{n+1} \leq \lambda(\alpha)$ and $\gamma_{n+1} \leq \lambda(\alpha)$ so the result follows via a proof similar to that of Theorems \ref{thm:inner} and \ref{thm:outer}.
\end{proof}

\begin{remark}
	The advantage of Corollary \ref{cor:weighting} is that the resulting inner and outer sets provide pivotal inference - not favouring one side or the other - which can be important when the distribution of the score function is asymmetric. Moreover the levels $\alpha_1$ and $\alpha_2$ can be used to provide a greater weight to either inner or outer sets whilst maintaining joint coverage. Theorem \ref{thm:joint} may instead be useful when there is strong dependence between $\tau_{n+1}$ and $\gamma_{n+1}$. However, when this dependence is weak, scale differences in the scores can lead to a lack of pivotality. This can be improved by appropriate choices of the score transformations $f_I$ and $f_O$ however in practice it may be simpler to construct joint sets using Corollary \ref{cor:weighting}. 
\end{remark}

\subsection{Optimizing score transformations}
%\subsubsection{Setting aside a learning dataset}
The choice of score transformations $f_I$ and $f_O$ is extremely important and can have a large impact on the size of the conformal confidence sets. The best choice depends on both the distribution of the data and on the nature of the output of the image segmentor used to calculate the scores. We thus recommend setting aside a learning dataset independent from both the calibration dataset, used to compute the conformal thresholds, and the test dataset. This approach was used in \cite{Sun2024} to learn the best copula transformation for combining dependent data streams.

In order to make efficient use of the data available, the learning dataset can in fact contain some or all of the data used to train the image segmentor. This data is assumed to be independent of the calibration and test data and so can be used to learn the best score transformations without compromising subsequent validity. The advantage of doing so is that less additional data needs to be set aside or collected for the purposes of learning a score function. Moreover it allows for additional data to be used to train the model resulting in better segmentation performance. The disadvantage is that machine learning models typically overfit their training data meaning that certain score functions may appear to perform better on this data than they do in practice. The choice of whether to include training data in the learning dataset thus depends on the quantity of data available and the quality of the segmentation model.

A score transformation that we will make particular use of in Section \ref{SS:res} is based on the distance transformation which we define as follows. Given $\mathcal{A} \subseteq \mathcal{V}$, let $E(\mathcal{A})$ be the set of points on the boundary of $\mathcal{A}$ obtained using the marching squares algorithm \citep{Maple2003}. Given a distance metric $\rho$ define the distance transformation $d_{\rho}: \mathcal{P}(\mathcal{V}) \times \mathcal{V}\rightarrow \mathbb{R}$, which sends $\mathcal{A} \in \mathcal{P}(\mathcal{V})$ and $v\in \mathcal{V}$ to
\begin{equation*}
	d_{\rho}(\mathcal{A}, v) = \text{sign}(\mathcal{A}, v)\min\lbrace \rho(v, e): e \in E(\mathcal{A})\rbrace, 
\end{equation*}
where $ \text{sign}(\mathcal{A}, v) = 1 $ if $v\in \mathcal{A}$ and equals $-1$ otherwise. The function $d_{\rho}$ is an adapation of the distance transform of \cite{Borgefors1986} which provides positive values within the set $\mathcal{A}$ and negative values outside of $\mathcal{A}$.

\subsection{Constructing confidence sets from bounding boxes}
%Inner and outer confidence sets can instead be provided using bounding boxes \citep{De2022, Andeol2023, Mukama2024}. 
Existing work on conformal inner and outer confidence sets, which aim to provide coverage of the entire ground truth mask with a given probability, has primarily focused on bounding boxes \citep{De2022, Andeol2023, Mukama2024}. These papers adjust for multiple comparisons over the 4 edges of the bounding box, doing so conformally by comparing the distance between the predicted bounding box and the bounding box of the ground truth mask. These approaches provide box-wise coverage by aggregrating the predictions over all objects within all of the calibration images, often combining multiple bounding boxes per image. However, as observed in Section 5 of \cite{De2022}, doing so violates exchangeability which is needed for valid conformal inference, as there is dependence between the objects within each image. Instead image-wise coverage can be provided without violating exchangeability by treating the union of the boxes as the ground truth image \citep{De2022, Andeol2023}.

% particular CITE relies on the results of CITE to establish validity however these results do not directly apply in the bounding box setting. This is because, while CITE shows the validity of conformal inference over multiple depenedent inferences but it assumes a fixed number of these inferences. Instead the number of true and predicted bounding boxes in a given image can vary so the result of CITE does not apply.
We establish the validity of a version of the image-wise max-additive method of \cite{Andeol2023} (adapted to provide confidence sets) as a corollary to our results, see Appendix \ref{AA:BBtheory}. In this approach we define bounding box scores based on the chessboard distance transformation to the inner and outer predicted masks and use these scores to provide conformal confidence sets. Validity then follows as a consequence of the results above as we show in Corollaries \ref{thm:boxinnergen} and \ref{thm:boxgenouter}. We compare to this approach in our experiments below. Targeting bounding boxes does not directly target the mask itself and so the resulting confidence sets are typically conservative.

\section{Application to Polpys tumor segmentation}\label{SS:res}
In order to illustrate and validate our approach we consider the problem of polpys tumor segmentation. To do so we use the same dataset as in \cite{Angelopoulos2022} in which 1798 poplys images, with available ground truth masks were combined from 5 open-source datasets (\cite{KVASIR2017}, \cite{Hyperkvasir2020} \cite{Bernal2012}, \cite{Silva2014}). Logit scores were obtained for these images using the parallel reverse attention network (PraNet) model \citep{PraNet2020}.

\subsection{Choosing a score transformation}\label{SS:learn}
\begin{figure}
	\centering
	\includegraphics[width=0.32\textwidth]{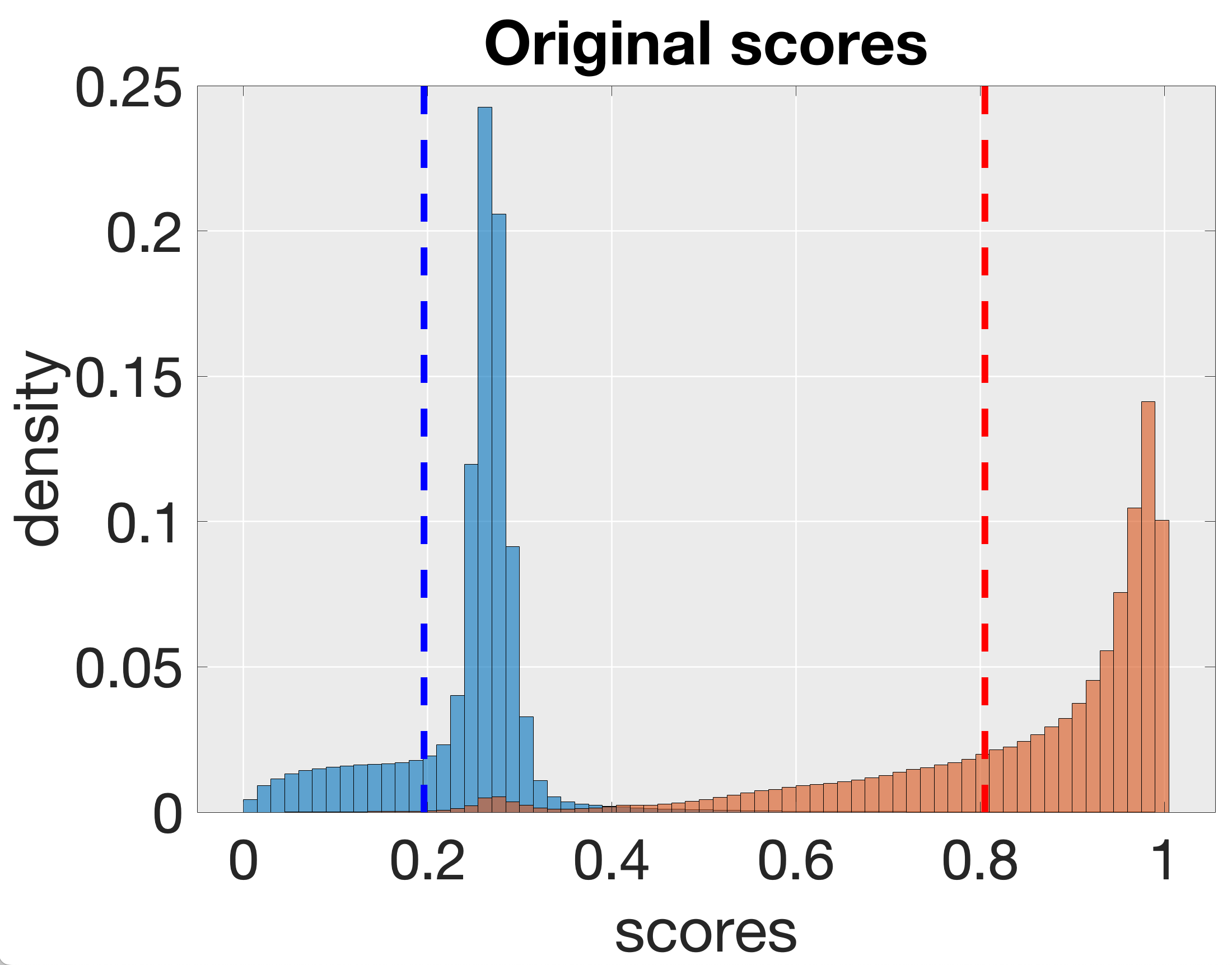}
	\includegraphics[width=0.32\textwidth]{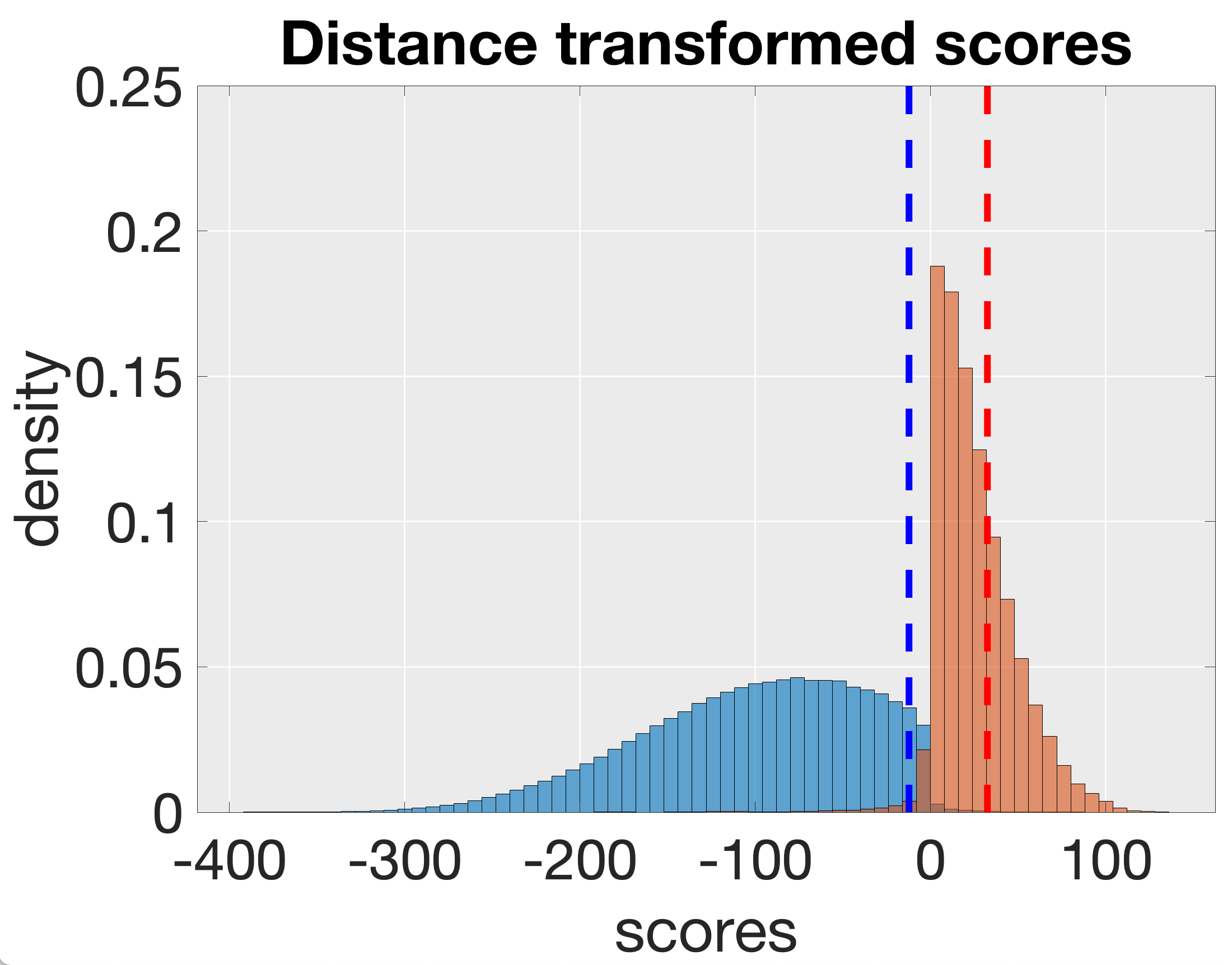}
	\includegraphics[width=0.32\textwidth]{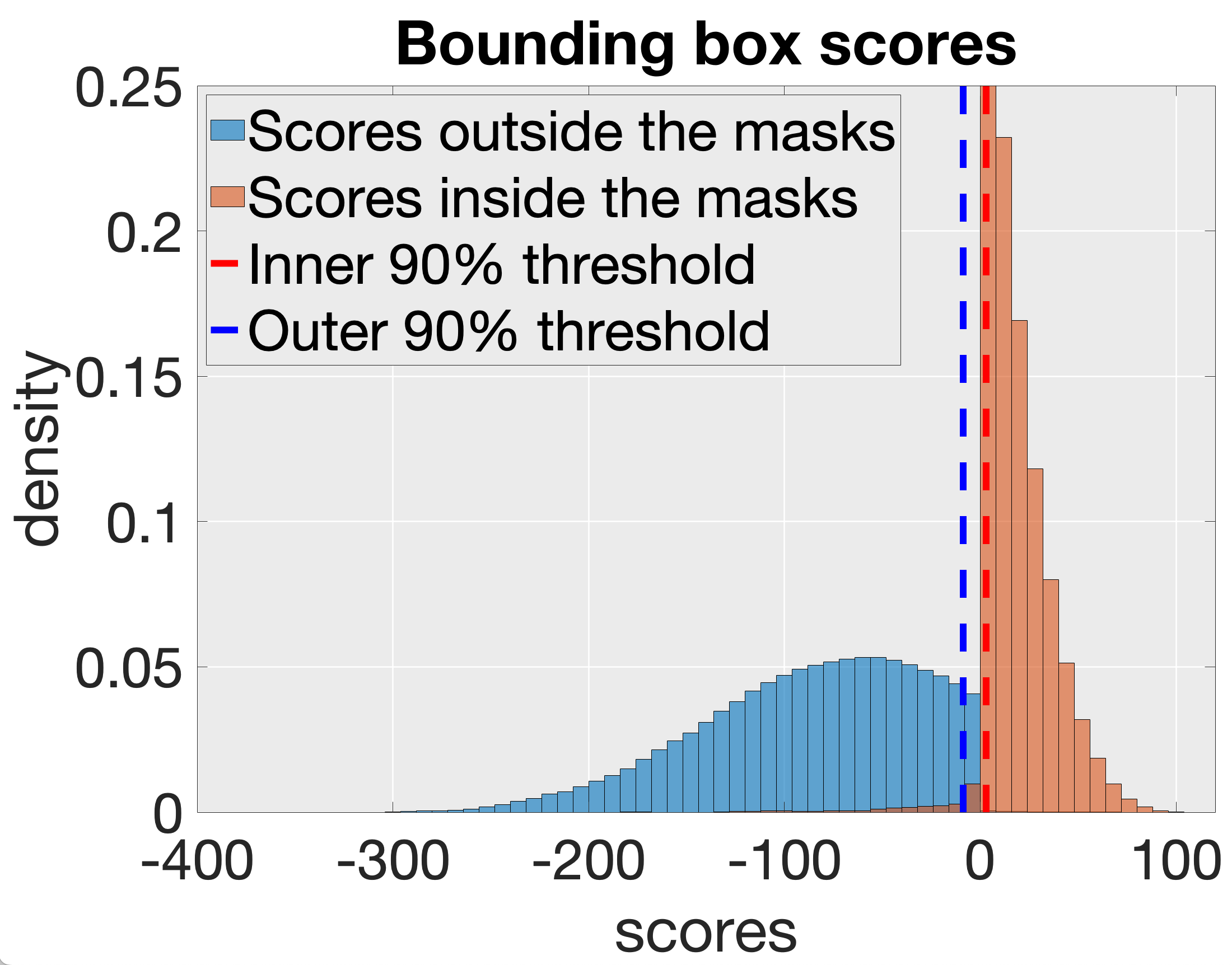}
	\caption{Histograms of the distribution of the scores over the whole image within and outside the ground truth masks. Thresholds obtained for the marginal $90\%$ inner and outer confidence sets, obtained based on quantiles of the distribution of $(\tau_i)_{i = 1}^n$ and $(\gamma_i)_{i = 1}^n$, are displayed in red and blue.}
	\label{scorehists}
\end{figure}
In order to optimize the size of our confidence sets we set aside 298 of the 1798 polpys images to form a learning dataset on which to choose the best score transformations. Importantly as the learning dataset is independent of the remaining 1500 images set-aside, we can study it as much as we like without compromising the validity of the follow-up analyses in Sections \ref{SS:val}. In particular in this section we shall use the learning dataset to both calibrate and study the results, in order to maximize the amount of important information we can learn from it.

The score transformations we considered were the identity (after softmax transformation) and distance transformations of the predicted masks:  taking $f_I(s(X), v) = f_O(s(X), v) = d_\rho(\hat{M}(X), v)$, where $\rho$ is the Euclidean metric. We also compare to the results of using the bounding box transformations $f_I = b_I$ and $f_O = b_O$ which correspond to transforming the predicted bounding box using a distance transformation based on the chessboard metric and are defined formally in Appendix \ref{AA:BBtheory}. For the purposes of plotting we used the combined bounding box scores defined in Definition \ref{dfn:BBS}.

From the histograms in Figure \ref{scorehists} we can see that thresholding the original scores at the inner threshold well separates the data. However this is not the case for the outer threshold for which the data is better separated using the distance transformed and bounding box scores. Figure \ref{fig:learning} shows PraNet scores for 2 typical examples, along with surface plots of the transformed scores and corresponding $90\%$ marginal confidence regions (with thresholds obtained from calibrating over the learning dataset). From these we see that PraNet typically assigns a high softmax score to the polpys regions which decreases in the regions directly around the  boundary of the tumor before returning to a higher level away from the polpys. This results in tight inner sets but large outer sets as the model struggles to identify where the tumor ends. Instead the distance transformed and bounding box scores are much better at providing outer bounds on the tumor, with distance transformed scores providing a tighter outside fit. Additional examples are shown in Figures \ref{fig:learning2} and \ref{fig:learning3} and have the same conclusion.

Based on the results of the learning dataset we decided to combine the best of the approaches for the inner and outer sets respectively for the inference in Section \ref{SS:val}, taking $f_I$ to be the identity and $f_O$ to be the distance transformation of the predicted mask in order to optimize performance. We can also use the learning dataset to determine how to weight the $\alpha$ used to obtain joint confidence sets. A ratio of 4 to 1 seems appropriate here in light of the fact that in this dataset identifying where a given tumor ends appears to be more challenging than identifying pixels where we are sure that there is a tumor. To achieve joint coverage of $90\%$ this involves taking $\alpha_1 = 0.02$ and $\alpha_2 = 0.08$.

\begin{figure}
%	\centering
\begin{center}
	\includegraphics[width=0.24\textwidth]{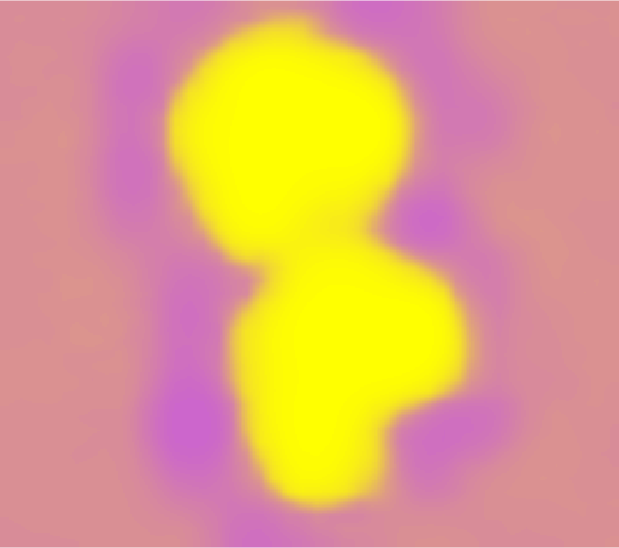}
		\includegraphics[width=0.24\textwidth]{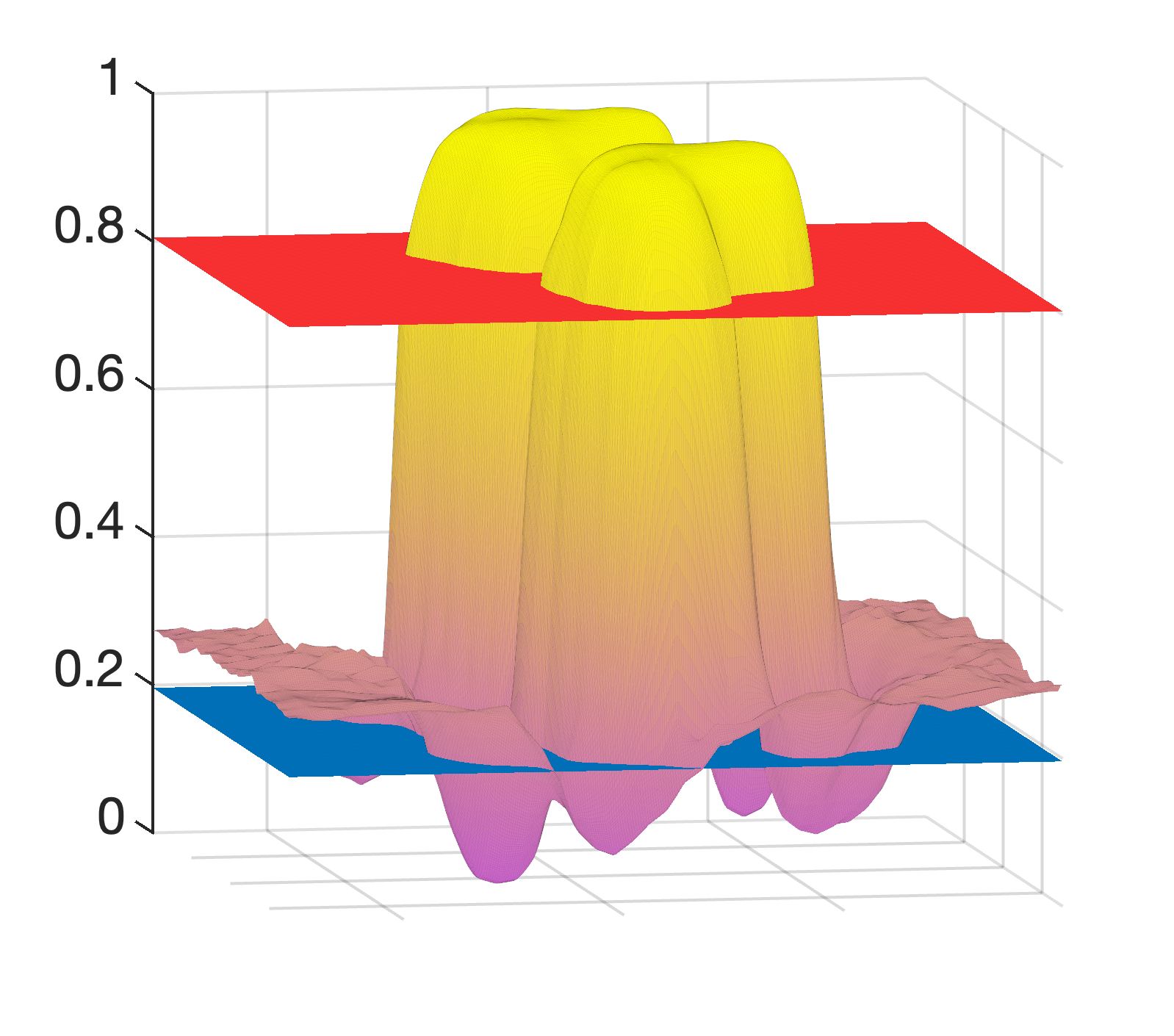}	\includegraphics[width=0.24\textwidth]{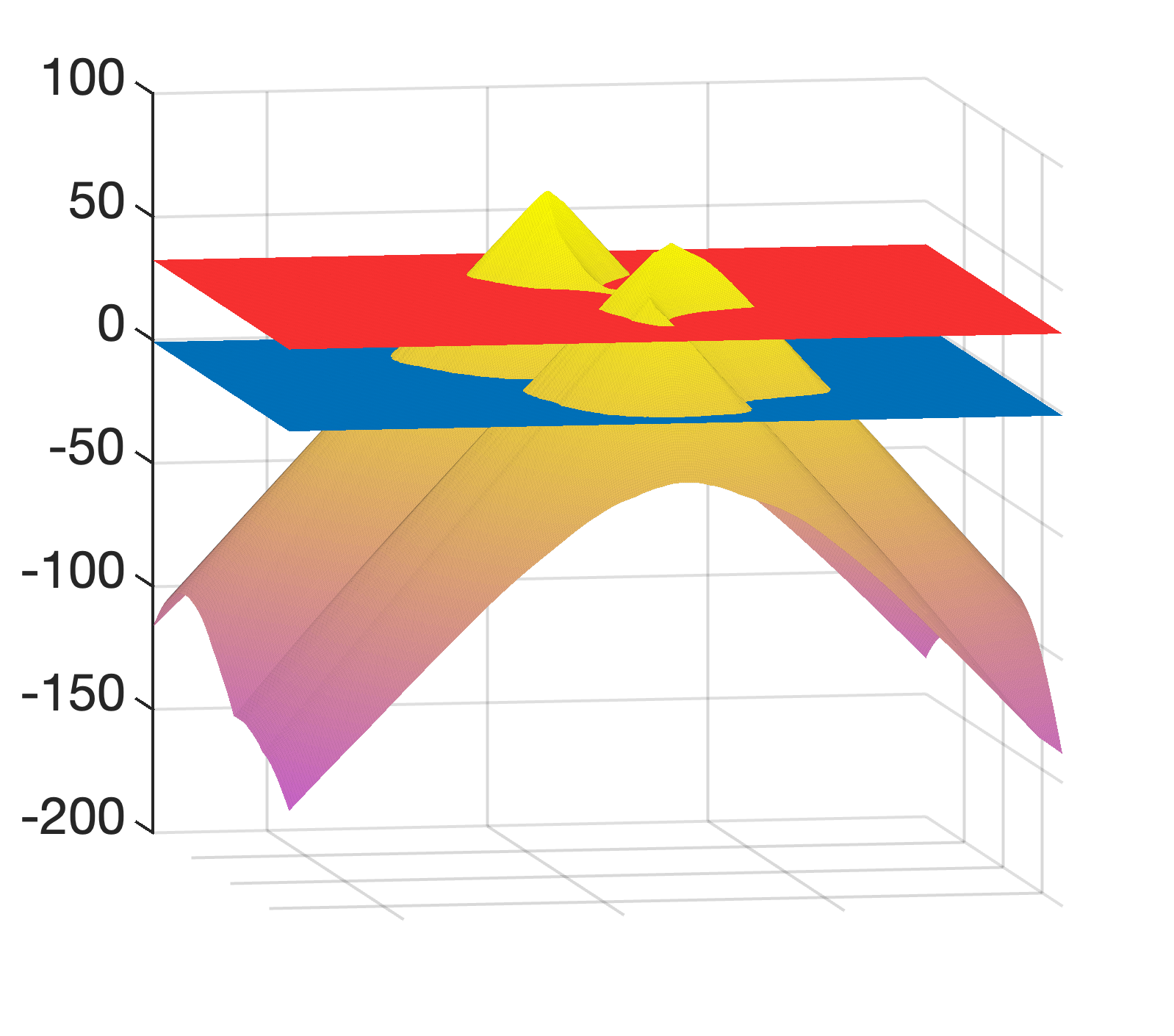}
		\includegraphics[width=0.24\textwidth]{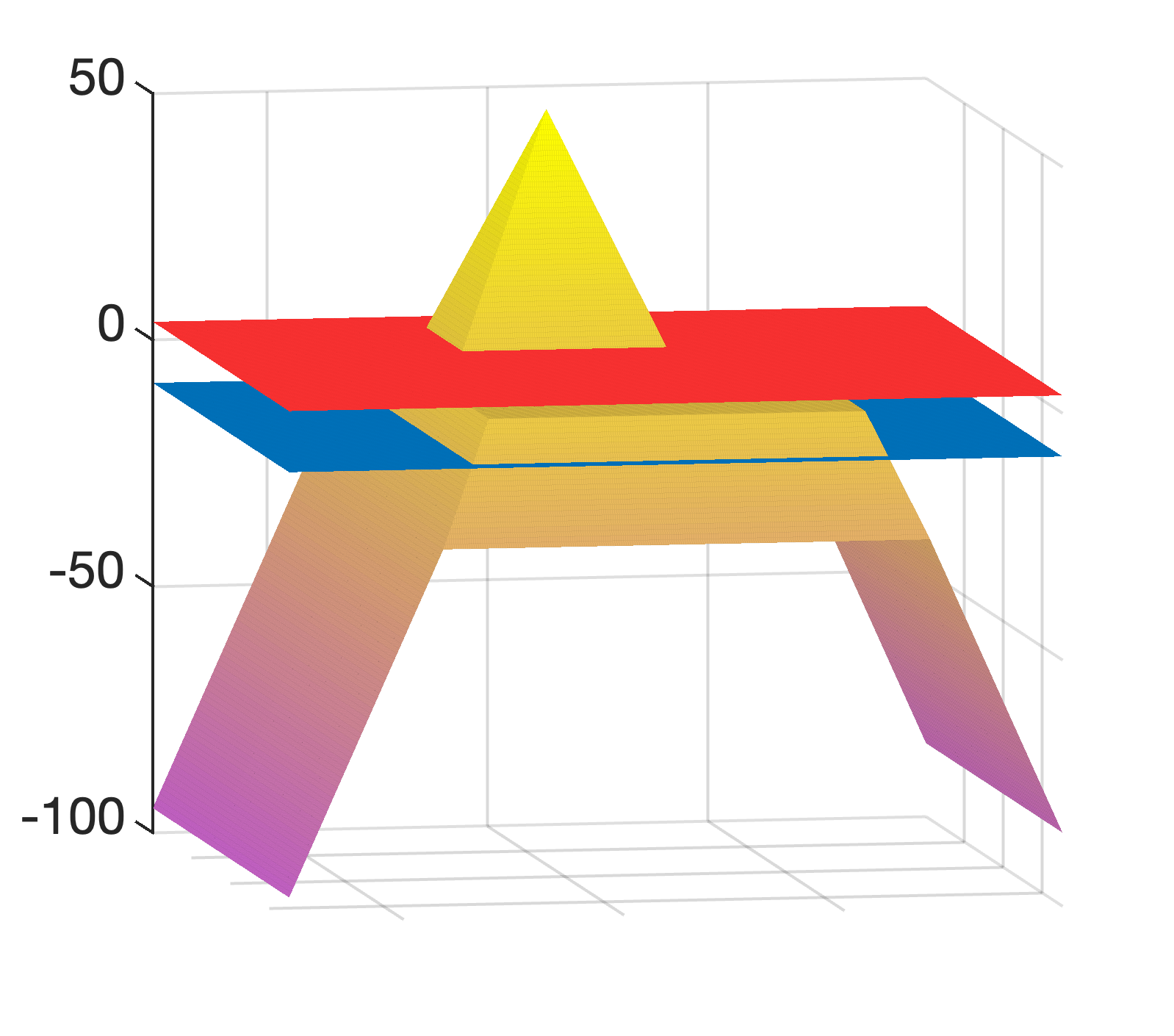}\\
		\includegraphics[width=0.24\textwidth]{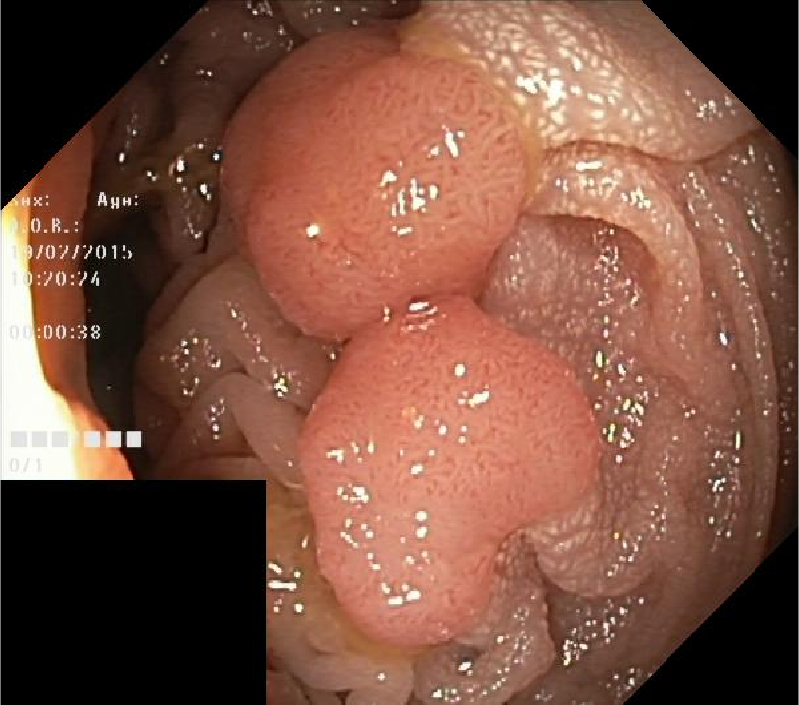}
	\includegraphics[width=0.24\textwidth]{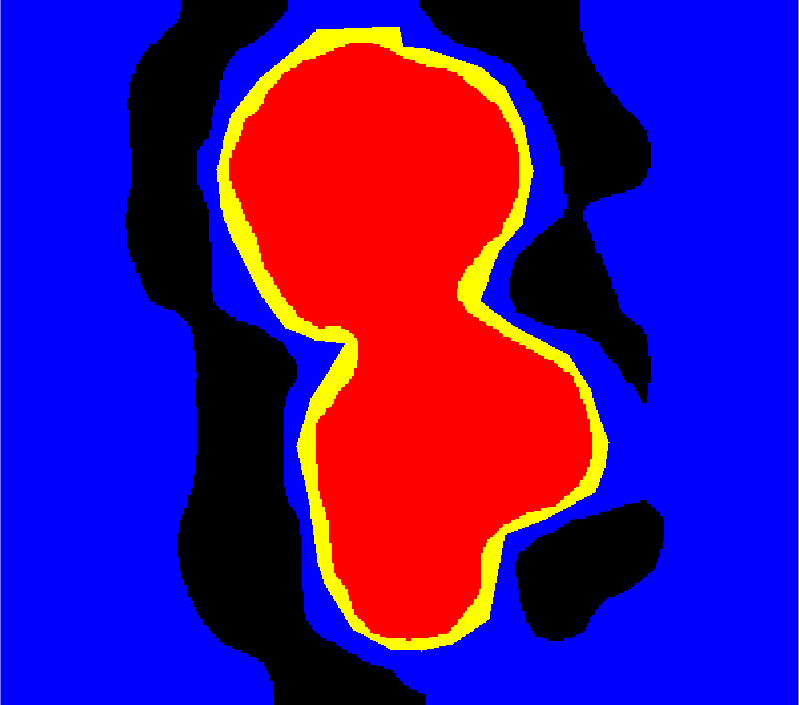}
	\includegraphics[width=0.24\textwidth]{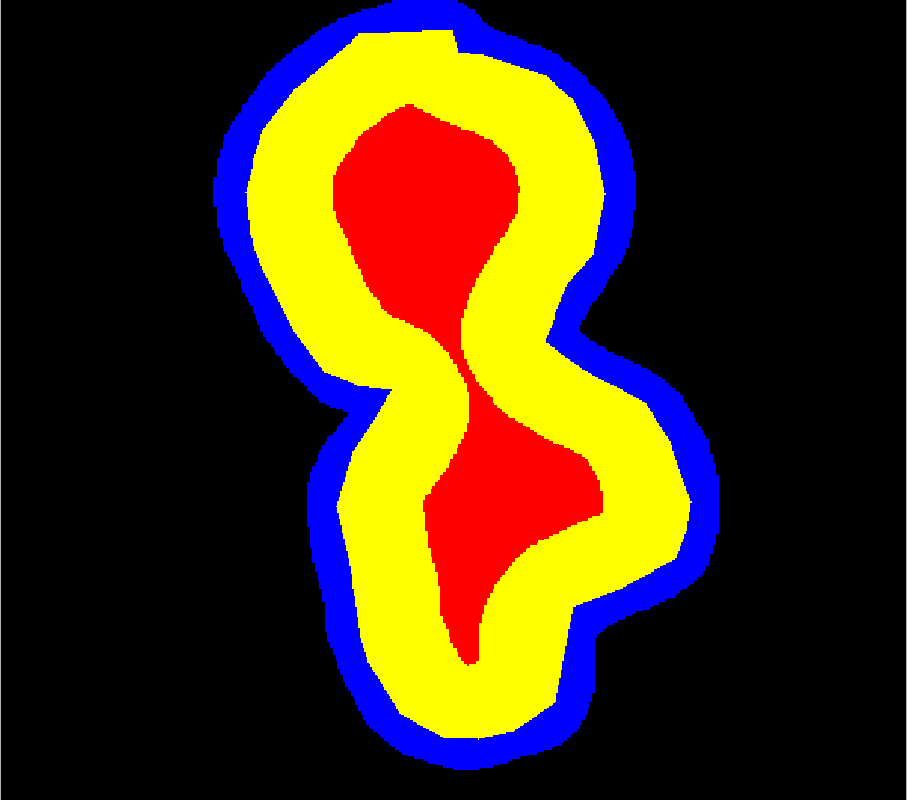}
	\includegraphics[width=0.24\textwidth]{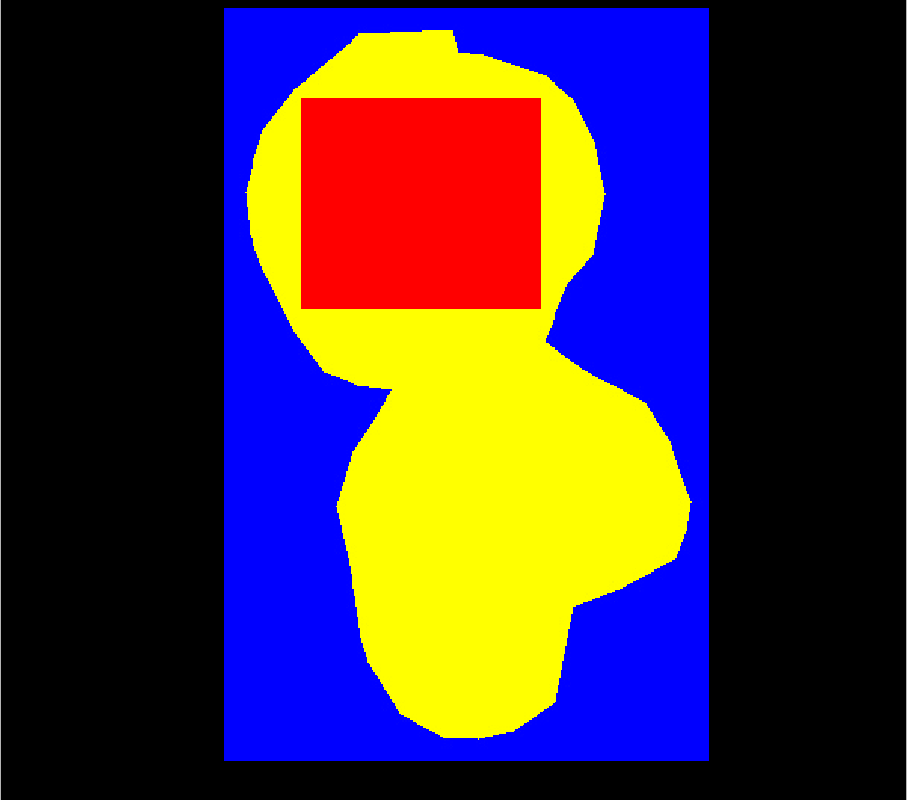}\\
\vspace{0.5cm}
		\includegraphics[width=0.24\textwidth]{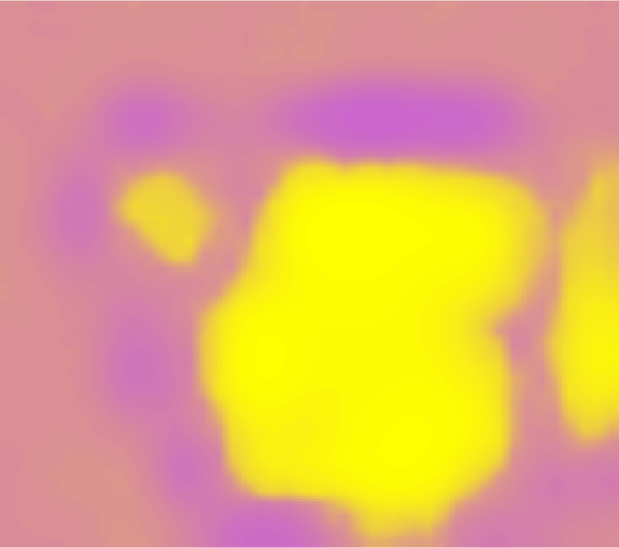}
	\includegraphics[width=0.24\textwidth]{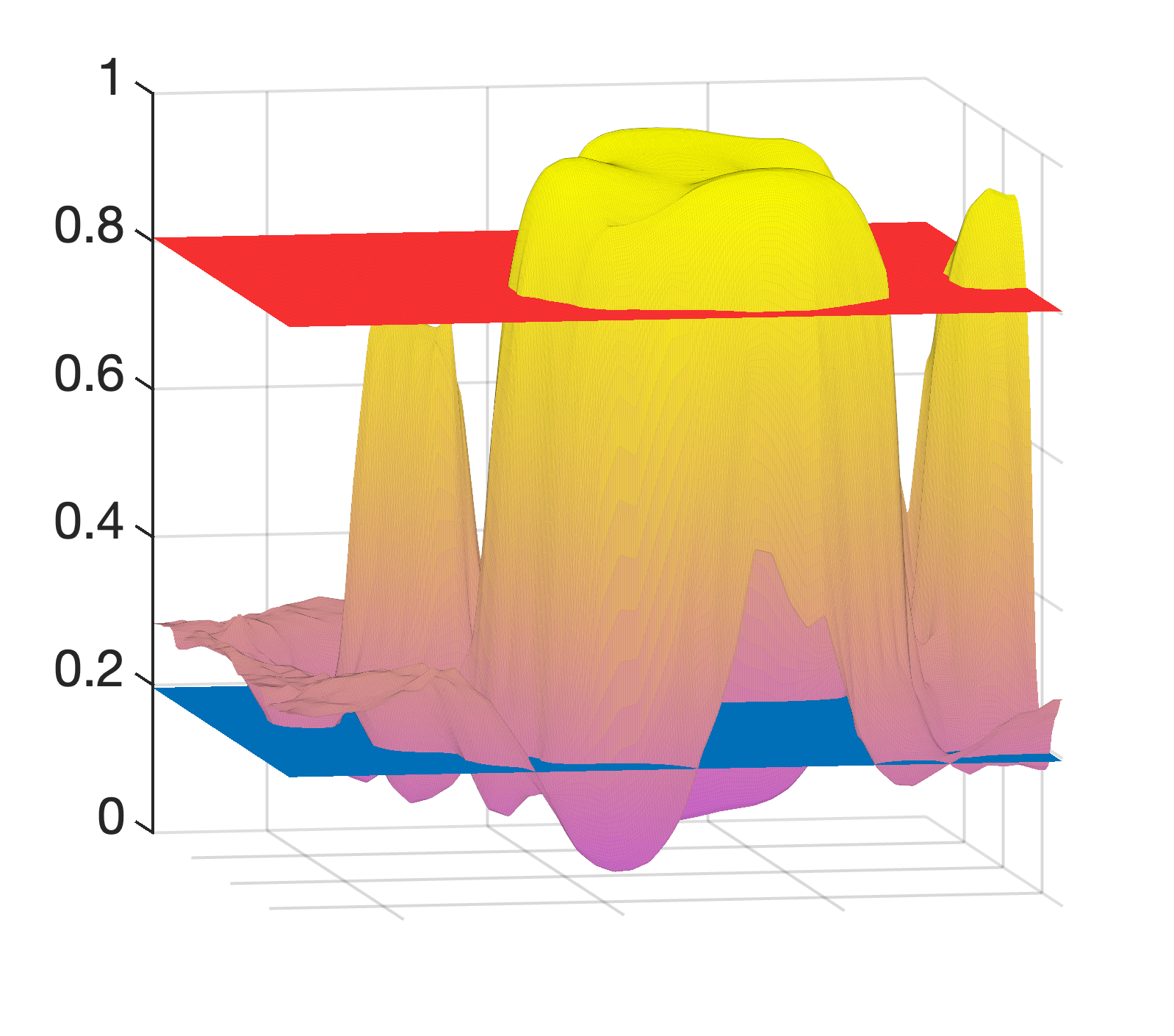}	\includegraphics[width=0.24\textwidth]{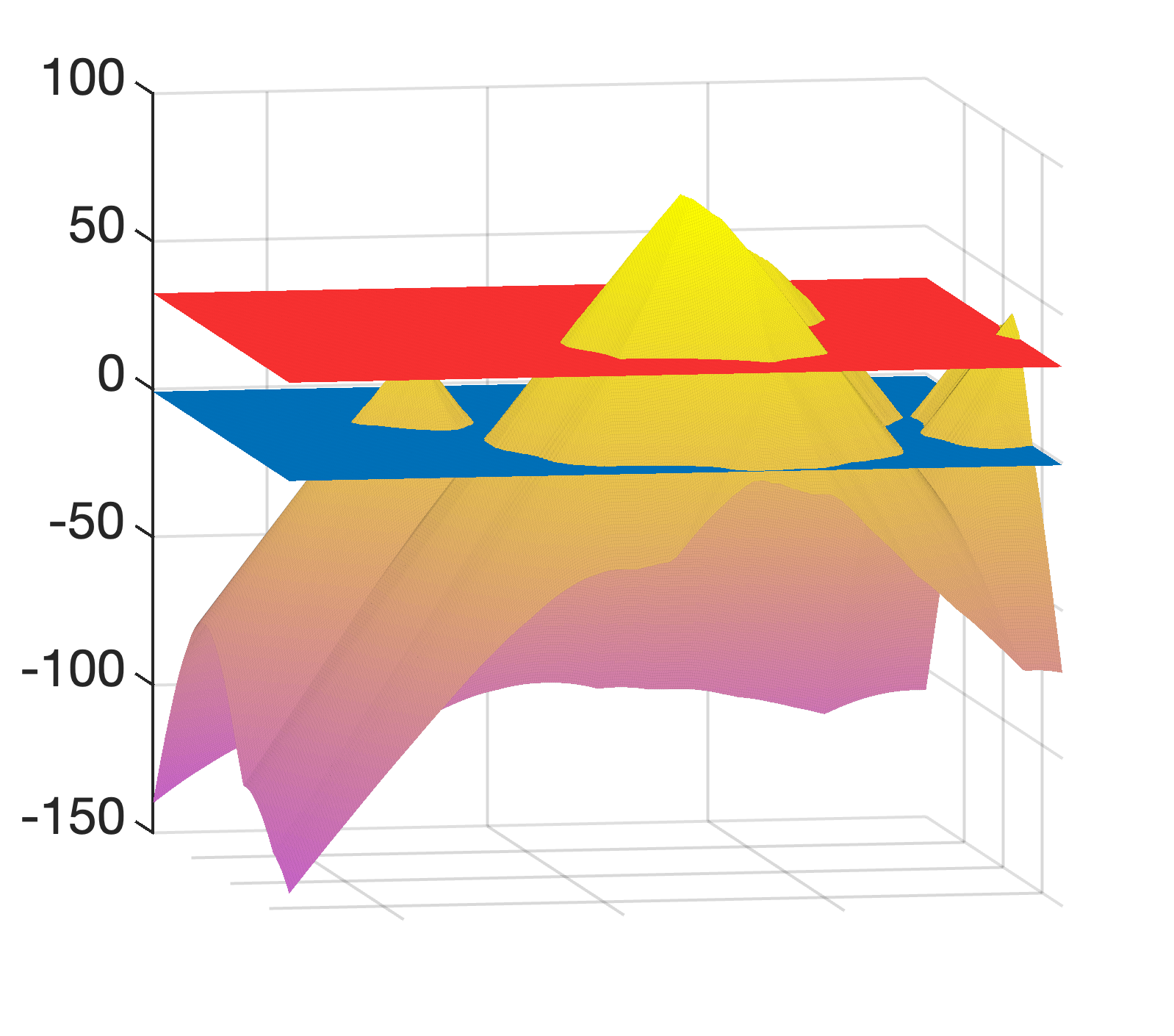}
	\includegraphics[width=0.24\textwidth]{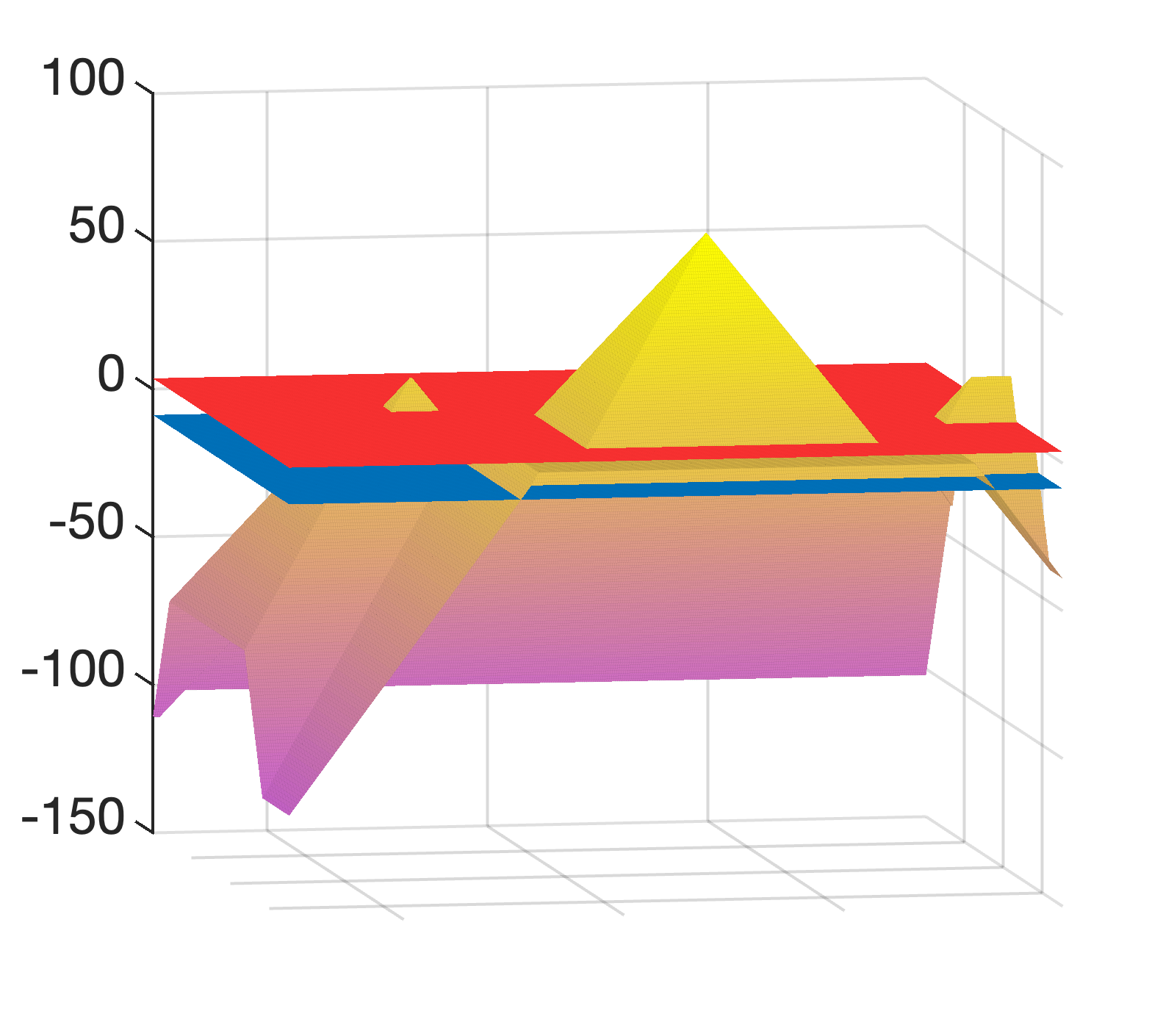}\\
	\includegraphics[width=0.24\textwidth]{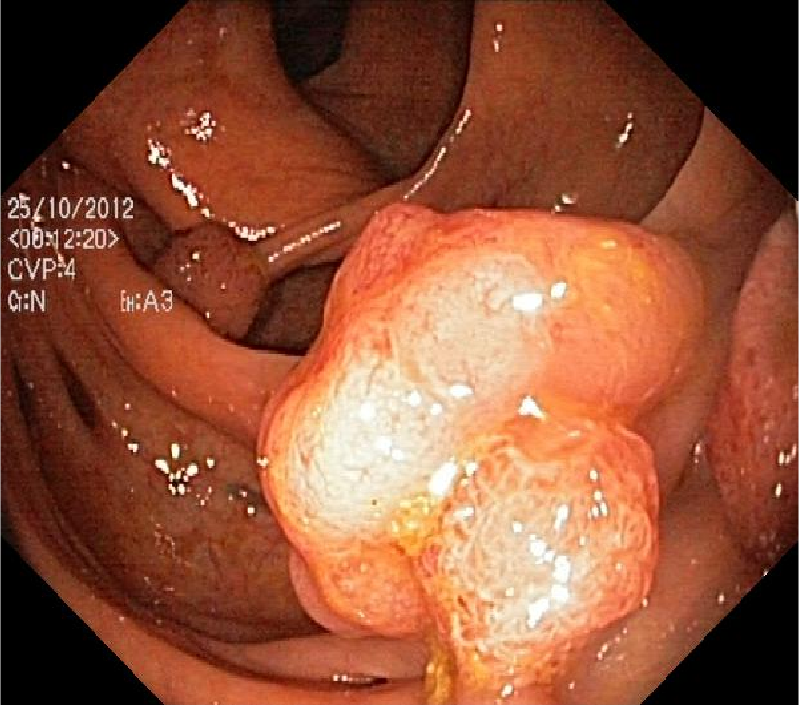}
	\includegraphics[width=0.24\textwidth]{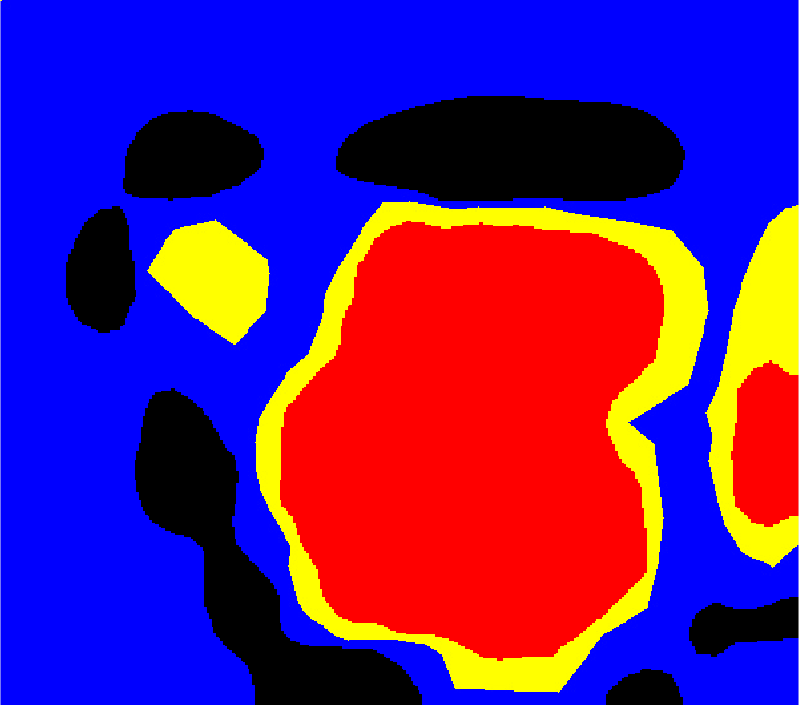}
	\includegraphics[width=0.24\textwidth]{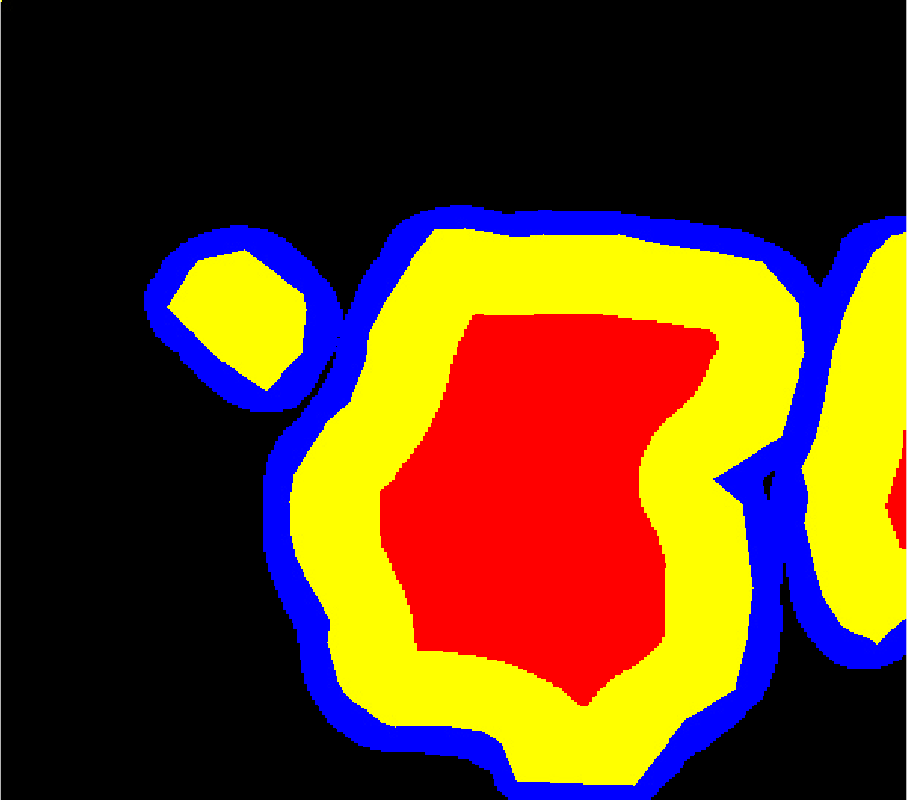}
	\includegraphics[width=0.24\textwidth]{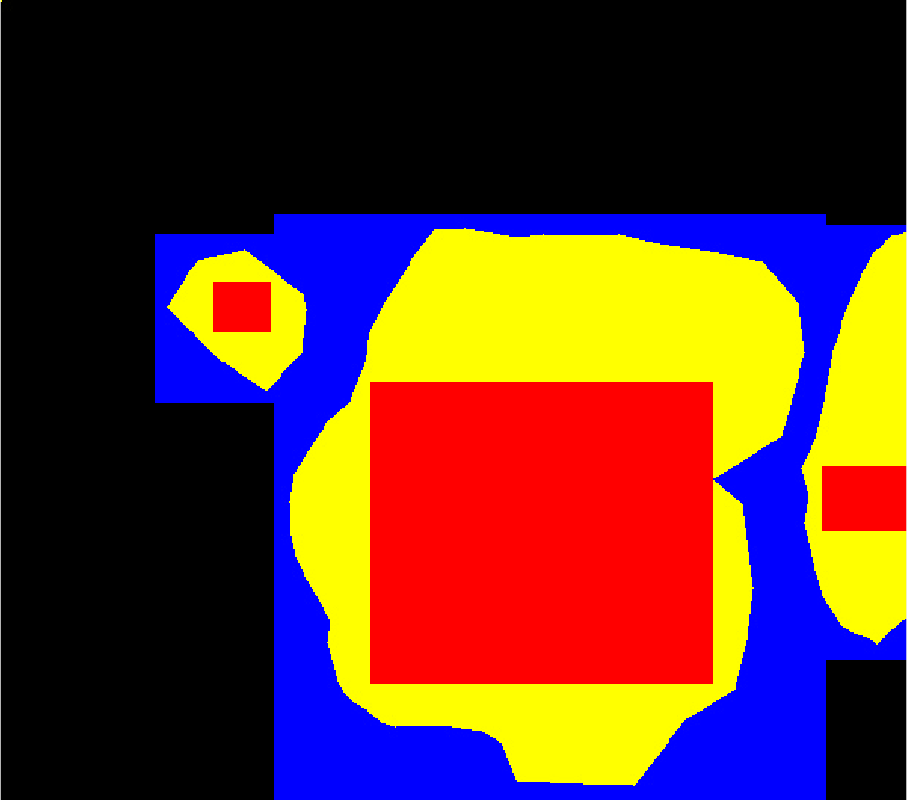}
\end{center}
	\caption{Illustrating the performance of the different score transformations on the learning dataset. We display 2 example tumors and present the results of each in 8 panels. These panels are as follows. Bottom left: the original image of the polpys tumor. Top Left: an intensity plot of the scores obtained from PraNet with purple/yellow indicating areas of lower/higher assigned probability. For the remaining panels, 3 different score transformations are shown which from left to right are the original scores, distance transformed scores $d_\rho(\hat{M}(X), v)$ and bounding box scores (obtained using the combined bounding box score $b_M$ defined in Definition \ref{dfn:BBS}). In each of the panels on the top row a surface plot of the transformed PraNet scores is shown, along with the conformal thresholds which are used to obtain the marginal 90\% inner and outer confidence sets.  These thresholds are illustrated via red and blue planes respectively and are obtained over the learning dataset. The panels on the bottom row of each example show the corresponding conformal confidence sets. Here the inner set is shown in red, plotted over the ground truth mask of the polyps, shown in yellow, plotted over the outer set which is shown in blue. The outer set contains the ground truth mask which contains the inner set in all examples. From these figures we see that the original scores provide tight inner confidence sets and the distance transformed scores instead provide tight outer confidence sets. The conclusion from the learning dataset is therefore that it makes sense to combine these two score transformations.}
	\label{fig:learning}
\end{figure}

%\subsection{Tumor detection}
%\begin{figure}[h!]
%	\centering
%	\includegraphics[width=\textwidth]{tumorfwerimage.png}
%	\caption{Examples}
%	\label{fig:enter-label}
%\end{figure

%\subsection{}
\subsection{Illustrating the performance of conformal confidence sets}\label{SS:val}
In order to illustrate the full extent of our methods in practice we divide the set aside 1500 images at random into 1000 for conformal calibration, and 500 for testing. The resulting conformal confidence sets for 10 example images from the test dataset are shown in Figure \ref{fig:res}, with inner sets obtained using the original scores and outer sets using the distance transformed scores. The inner sets are shown in red and represent regions where we can have high confidence of the presence of polyps. The outer sets are shown in blue and represent regions in which the polpys may be. The ground truth mask for each polpys is shown in yellow and can be compared to the original images. In each of the examples considered the ground truth is bounded from within by the inner set and from without by the outer set. 
\begin{figure}[h!]
	\begin{subfigure}{0.19\textwidth}
		\centering
		\includegraphics[width=\textwidth]{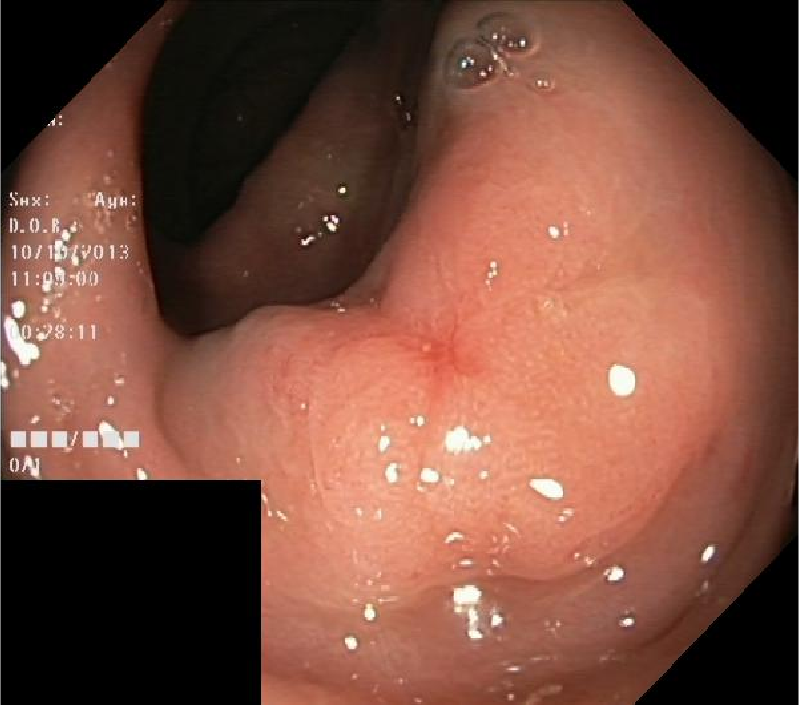}
		\label{fig:1}
	\end{subfigure}
	\begin{subfigure}{0.19\textwidth}
		\centering
		\includegraphics[width=\textwidth]{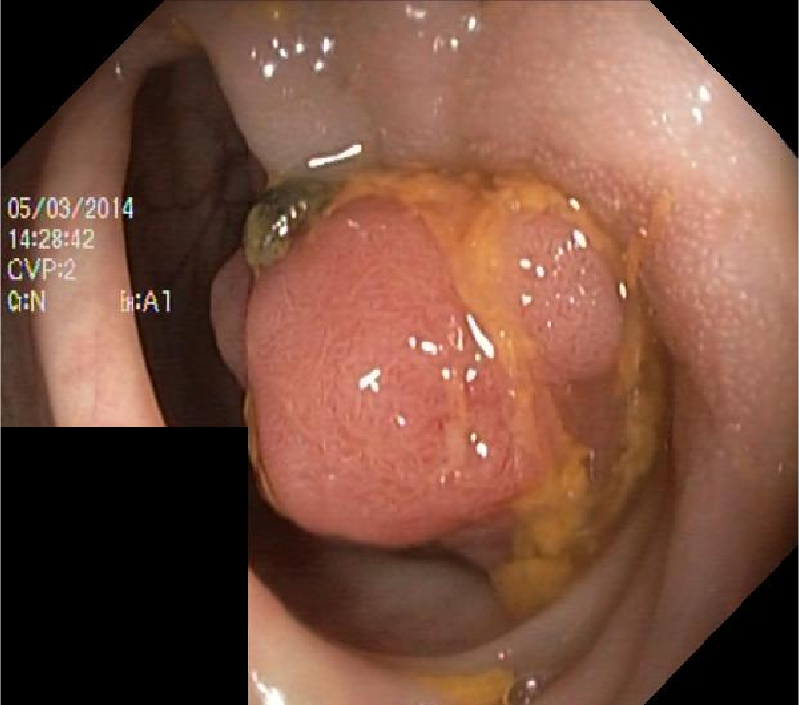}
		\label{fig:1}
	\end{subfigure}
	\begin{subfigure}{0.19\textwidth}
		\centering
		\includegraphics[width=\textwidth]{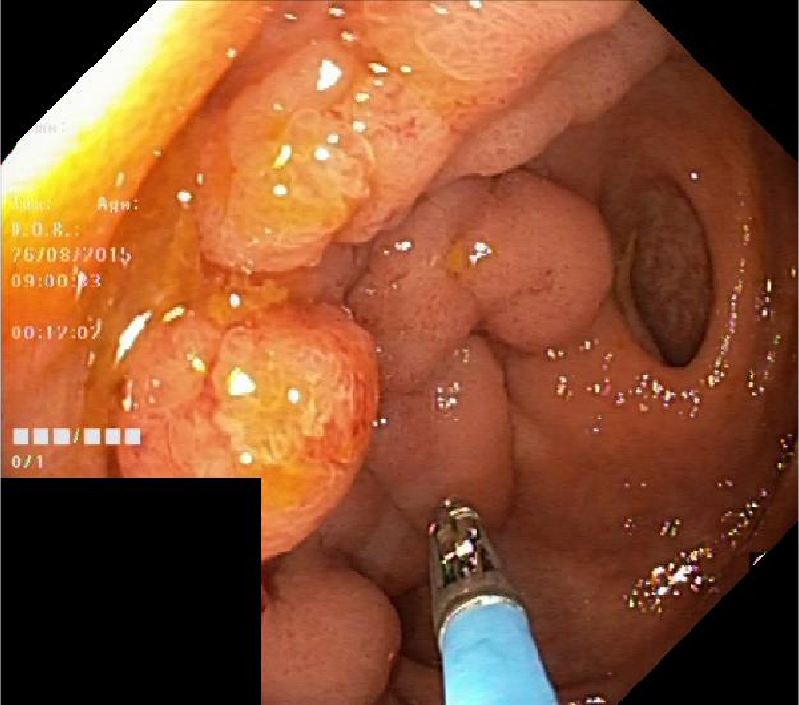}
		\label{fig:1}
	\end{subfigure}
	\begin{subfigure}{0.19\textwidth}
		\centering
		\includegraphics[width=\textwidth]{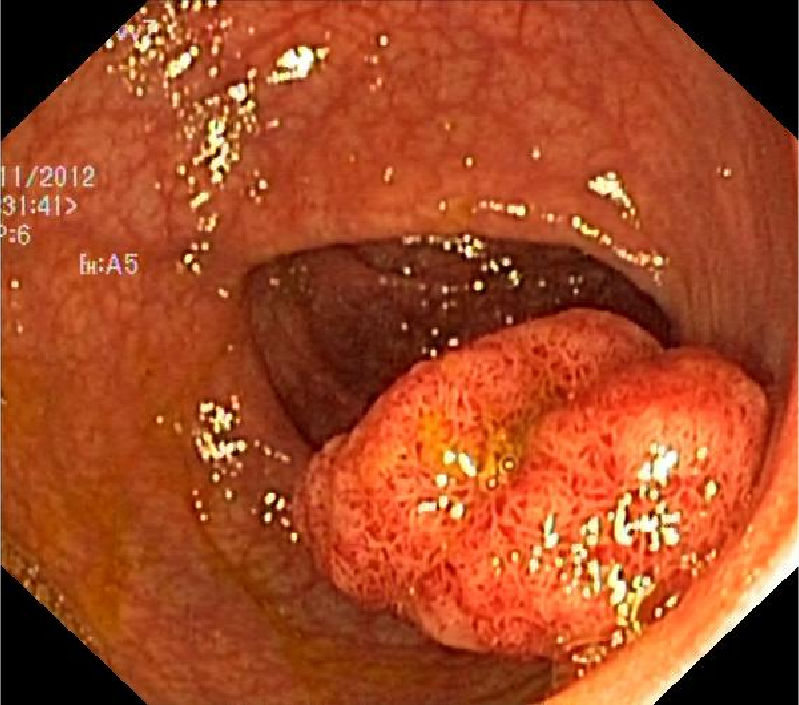}
		\label{fig:1}
	\end{subfigure}
	\begin{subfigure}{0.19\textwidth}
		\centering
		\includegraphics[width=\textwidth]{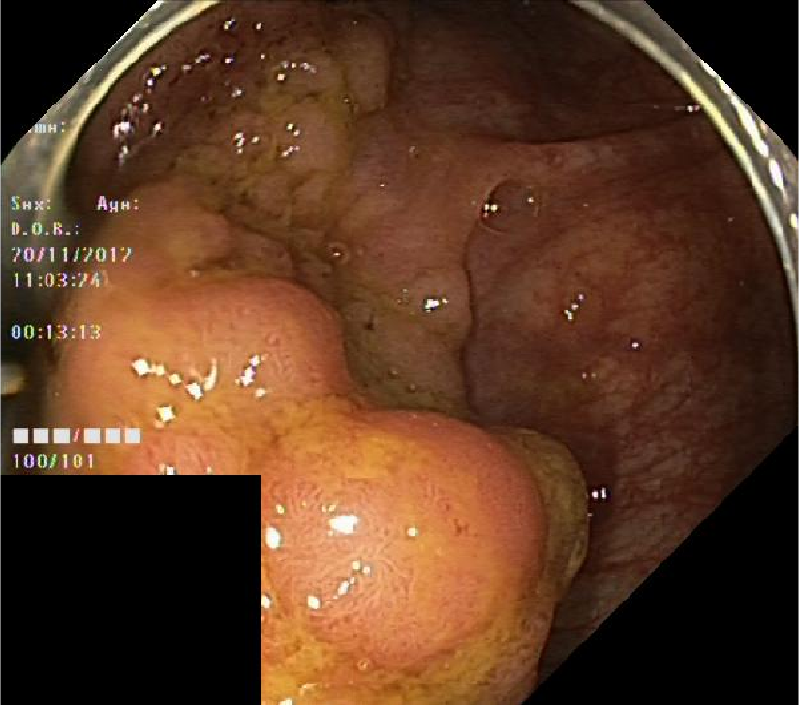}
		\label{fig:1}
	\end{subfigure}
	\vspace{-0.35cm}
	\\
	\begin{subfigure}{0.19\textwidth}
		\centering
		\includegraphics[width=\textwidth]{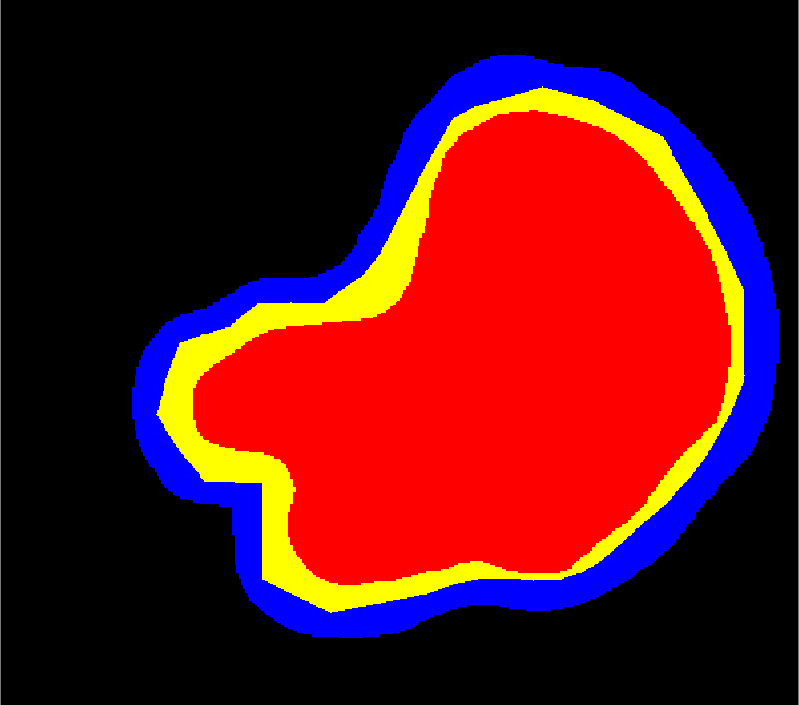}
		\label{fig:1}
	\end{subfigure}
	\begin{subfigure}{0.19\textwidth}
		\centering
		\includegraphics[width=\textwidth]{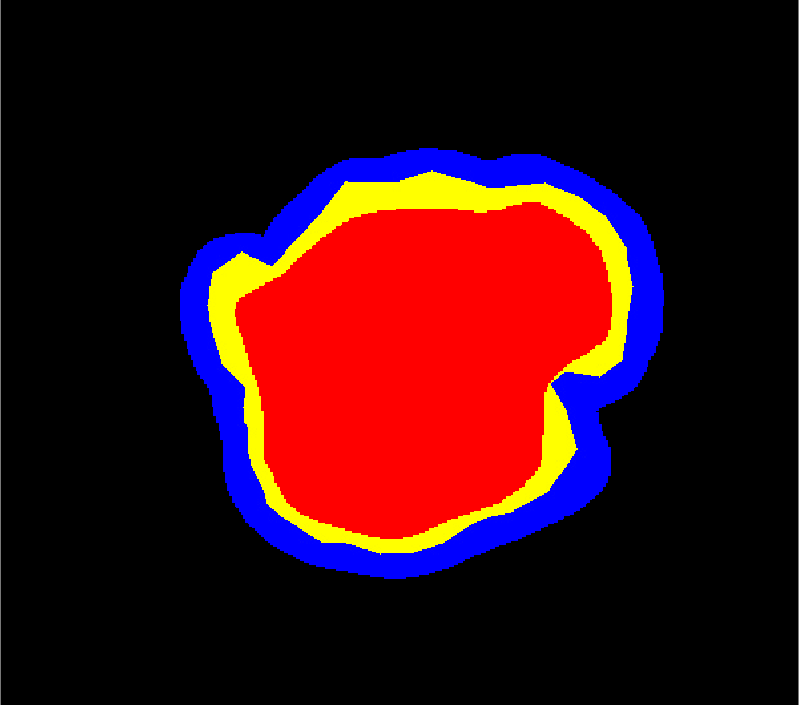}
		\label{fig:1}
	\end{subfigure}
	\begin{subfigure}{0.19\textwidth}
		\centering
		\includegraphics[width=\textwidth]{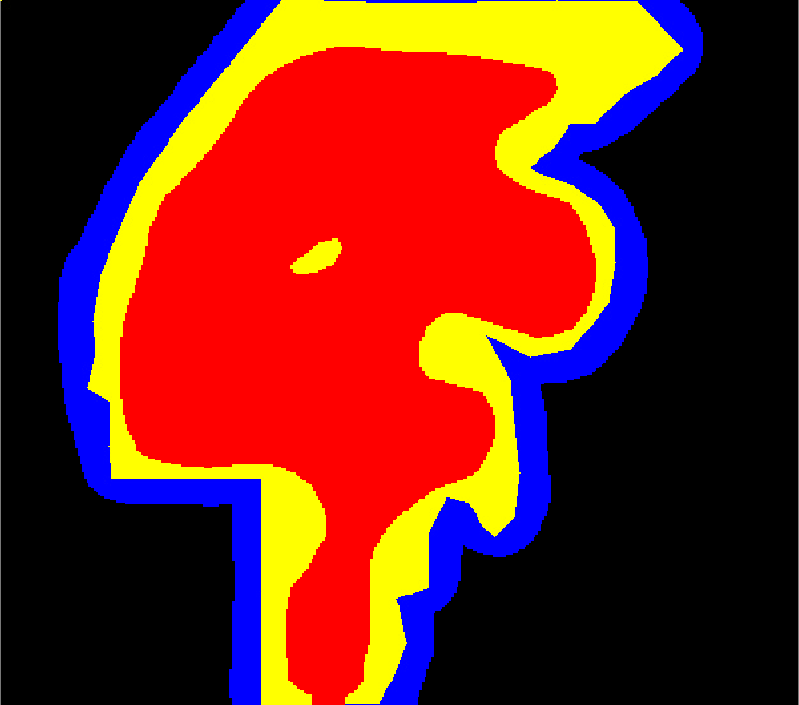}
		\label{fig:1}
	\end{subfigure}
	\begin{subfigure}{0.19\textwidth}
		\centering
		\includegraphics[width=\textwidth]{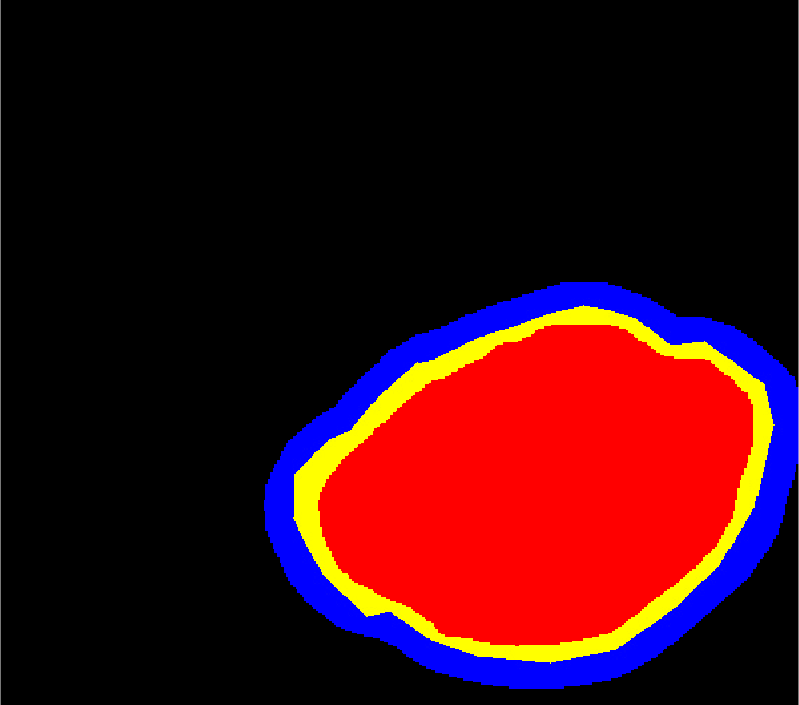}
		\label{fig:1}
	\end{subfigure}
	\begin{subfigure}{0.19\textwidth}
		\centering
		\includegraphics[width=\textwidth]{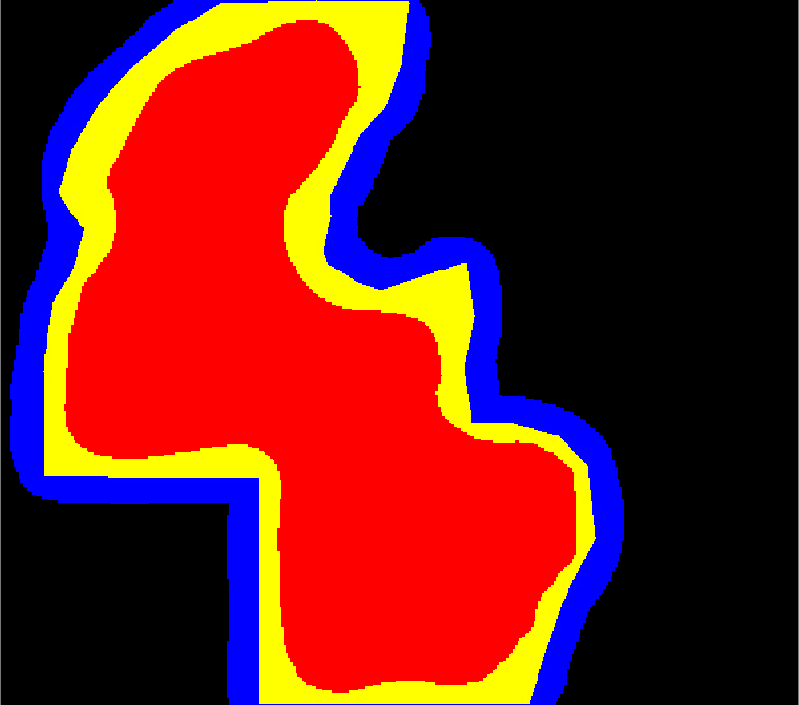}
		\label{fig:1}
	\end{subfigure}
	\vspace{-0.35cm}
	\\
		\begin{subfigure}{0.19\textwidth}
		\centering
		\includegraphics[width=\textwidth]{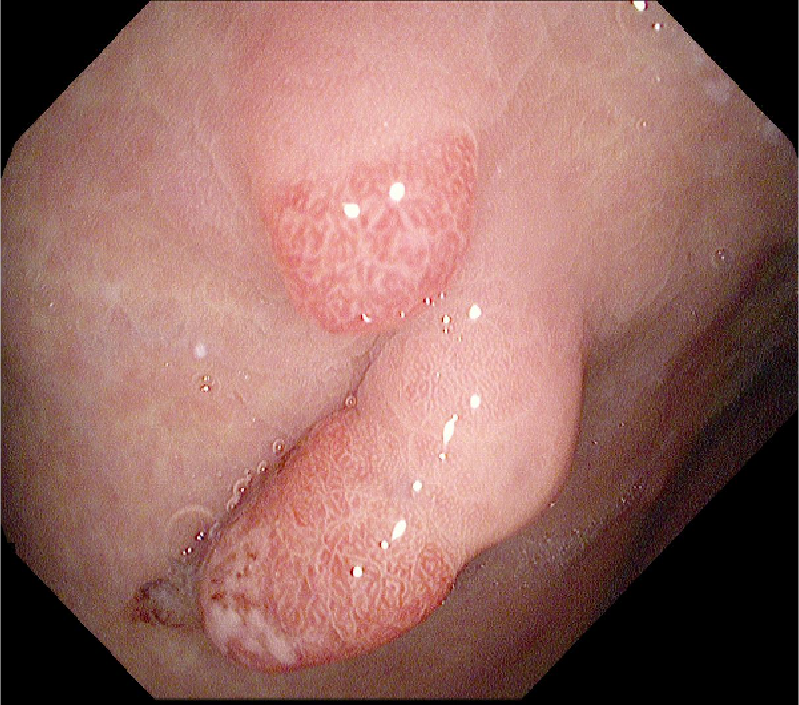}
		\label{fig:1}
	\end{subfigure}
	\begin{subfigure}{0.19\textwidth}
		\centering
		\includegraphics[width=\textwidth]{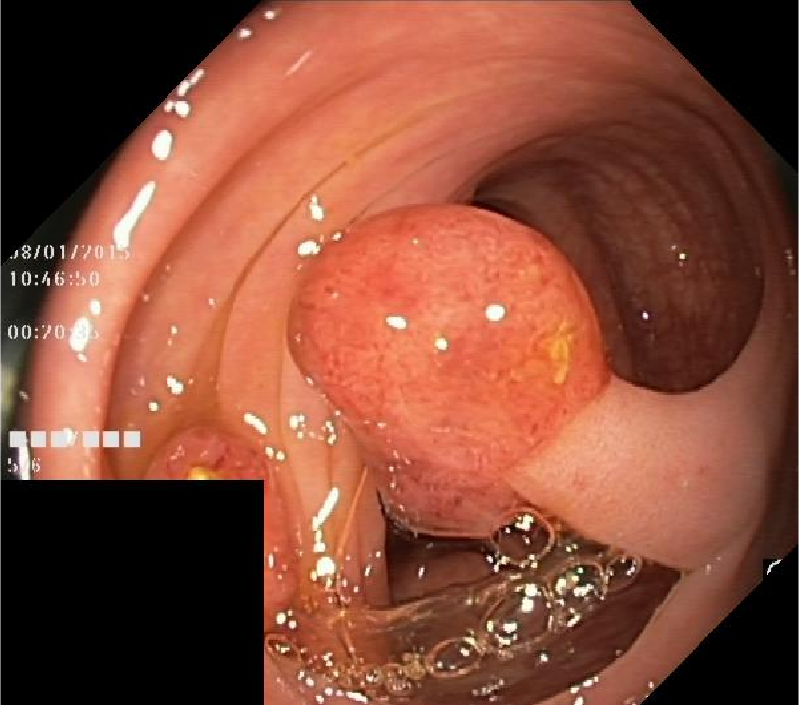}
		\label{fig:1}
	\end{subfigure}
	\begin{subfigure}{0.19\textwidth}
		\centering
		\includegraphics[width=\textwidth]{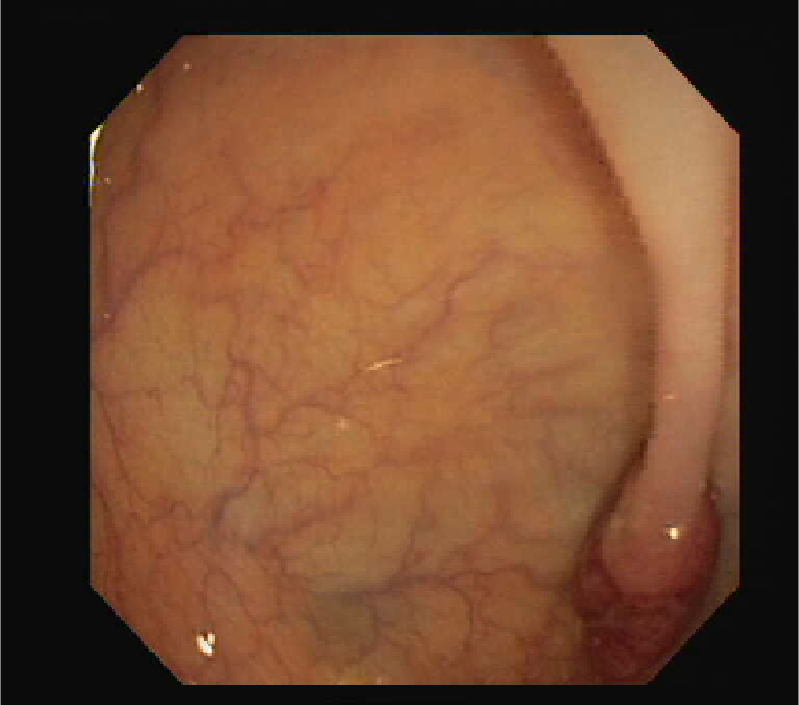}
		\label{fig:1}
	\end{subfigure}
	\begin{subfigure}{0.19\textwidth}
		\centering
		\includegraphics[width=\textwidth]{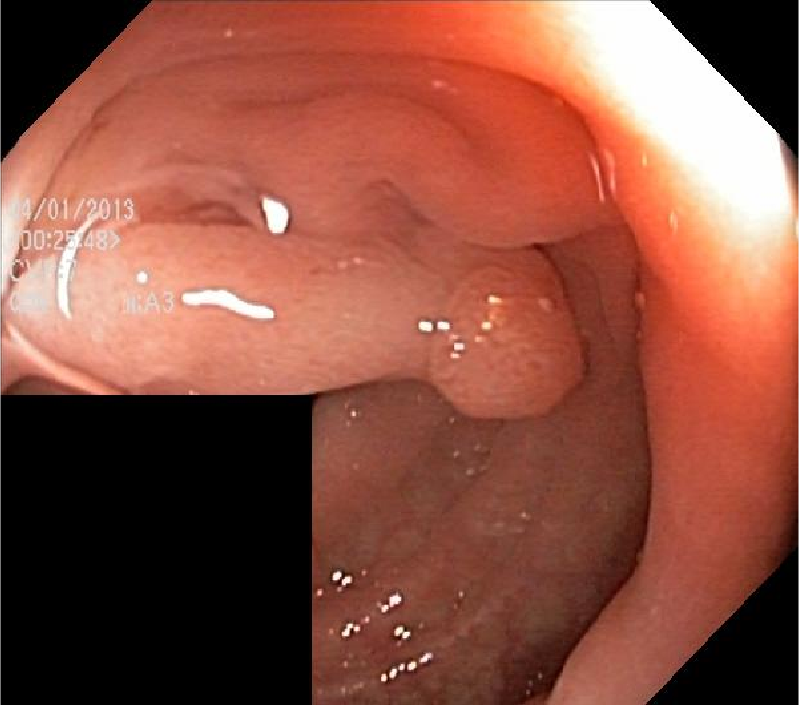}
		\label{fig:1}
	\end{subfigure}
	\begin{subfigure}{0.19\textwidth}
		\centering
		\includegraphics[width=\textwidth]{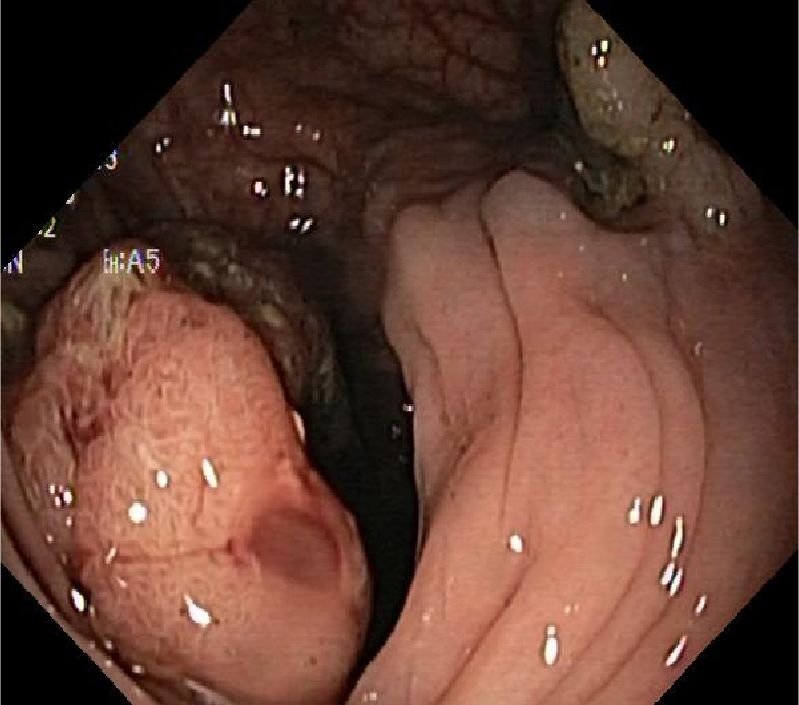}
		\label{fig:1}
	\end{subfigure}
	\vspace{-0.35cm}
	\\
	\begin{subfigure}{0.19\textwidth}
		\centering
		\includegraphics[width=\textwidth]{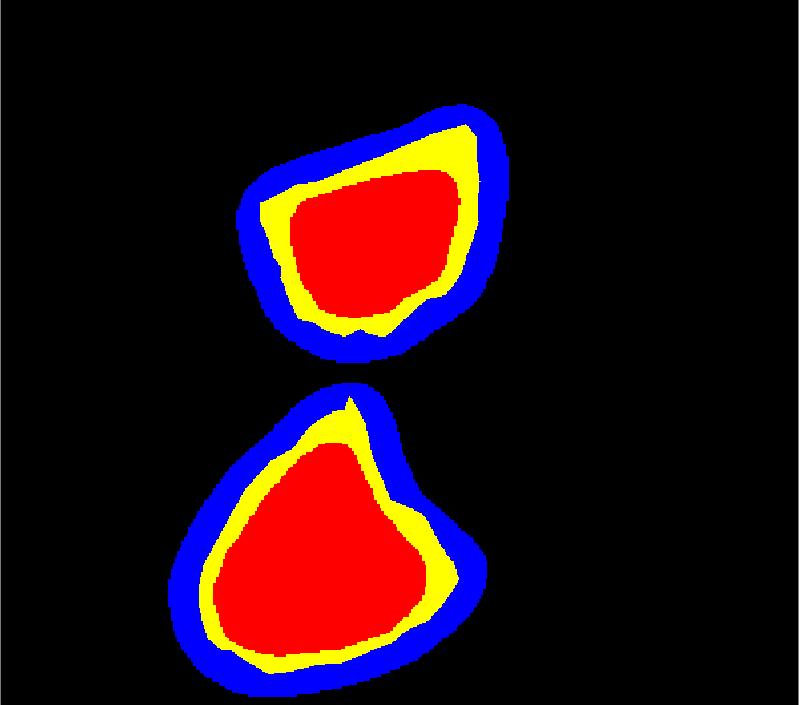}
		\label{fig:1}
	\end{subfigure}
	\begin{subfigure}{0.19\textwidth}
		\centering
		\includegraphics[width=\textwidth]{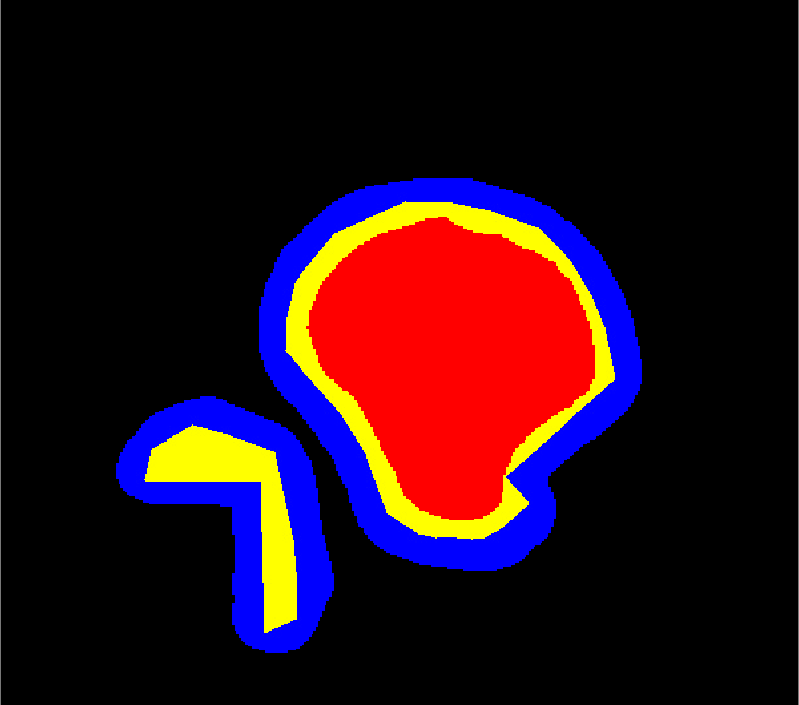}
		\label{fig:1}
	\end{subfigure}
	\begin{subfigure}{0.19\textwidth}
		\centering
		\includegraphics[width=\textwidth]{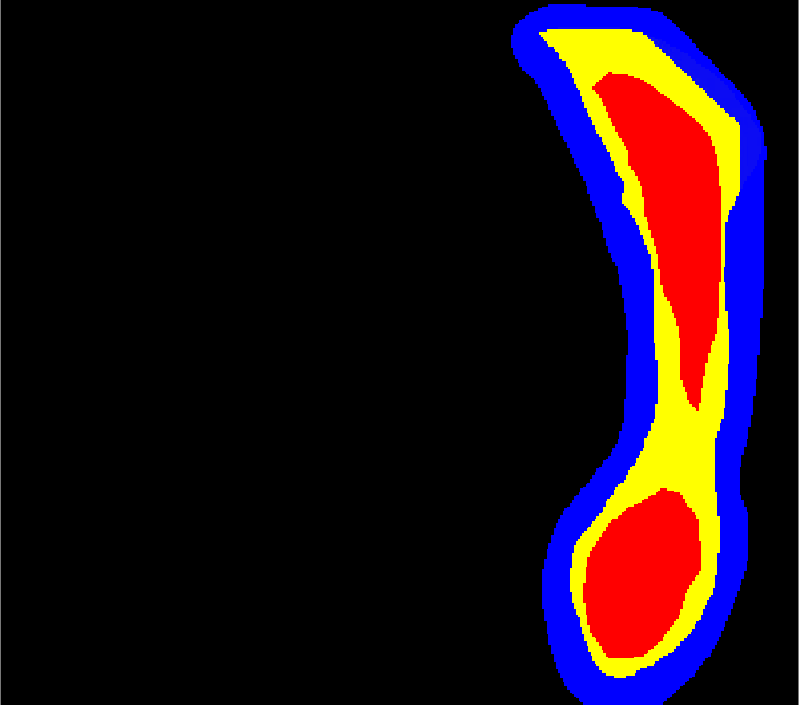}
		\label{fig:1}
	\end{subfigure}
	\begin{subfigure}{0.19\textwidth}
		\centering
		\includegraphics[width=\textwidth]{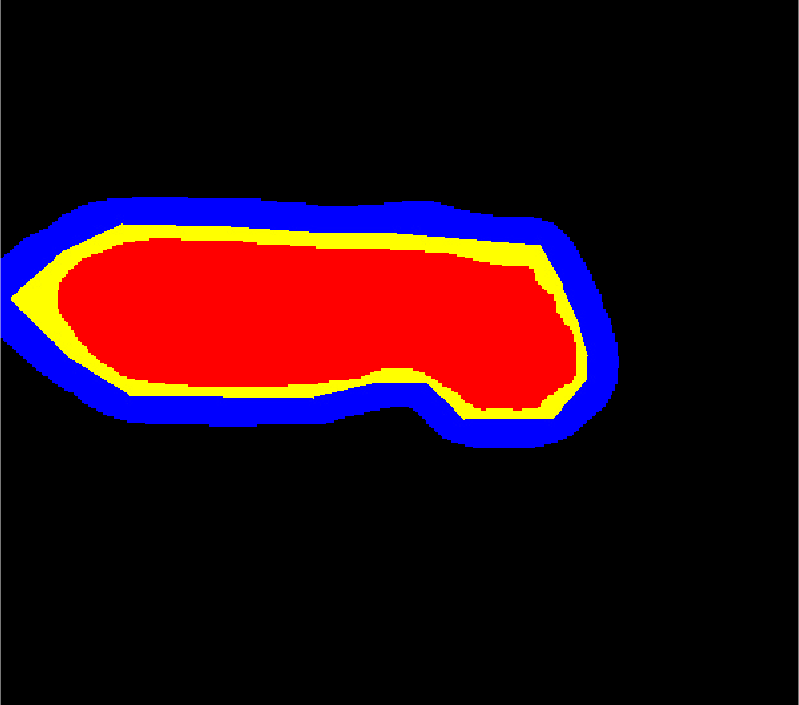}
		\label{fig:1}
	\end{subfigure}
	\begin{subfigure}{0.19\textwidth}
		\centering
		\includegraphics[width=\textwidth]{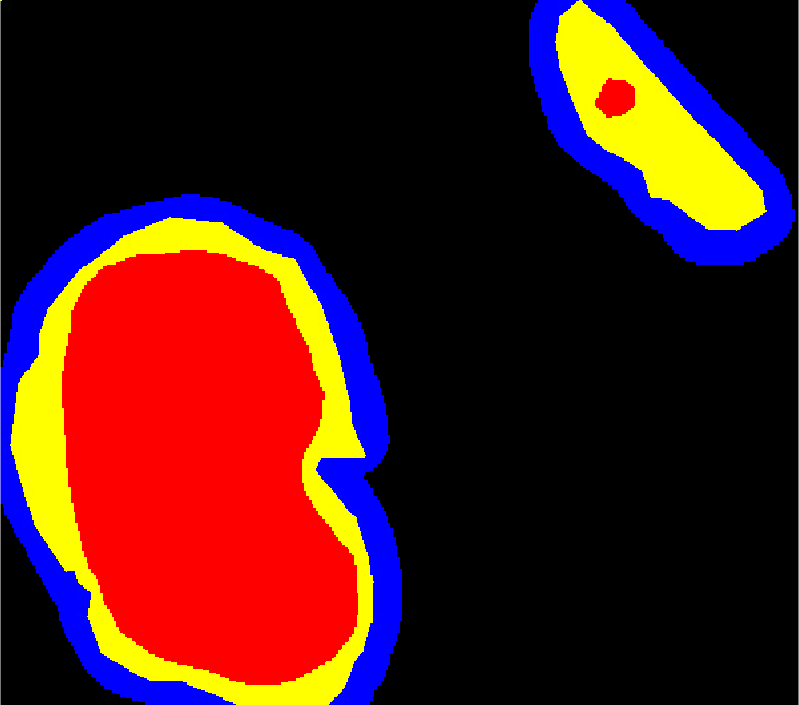}
		\label{fig:1}
	\end{subfigure}
	\label{fig:grid}
	\caption{Conformal confidence sets for the polyps data. For each set of polpys images the top row shows the original endoscopic images with visible polyps and the second row presents the marginal 90\% confidence sets, with ground truth masks shown in yellow. The inner sets and outer sets are shown in red and blue, obtained using the identity and distance transforms respectively. The figure shows the benefits of combining different score transformations for the inner and outer sets and illustrates the method's effectiveness in accurately identifying polyp regions whilst providing informative spatial uncertainty bounds.}\label{fig:res}
\end{figure}
Results for confidence sets based on the original and bounding box scores as well as additional examples are available in Figures \ref{fig:polpysex} and \ref{fig:polpysex2}. Confidence sets can also be provided for the bounding boxes themselves if that is the object of interest, see Figure \ref{fig:resbb}. Joint $90\%$ confidence sets are displayed in Figure \ref{fig:joint}, from which we can see that with alpha-weighting (i.e. taking $\alpha_1 = 0.02$ and $\alpha_2 = 0.08$) we are able to obtain joint confidence sets which are still relatively tight.

These results collectively show that we can provide informative confidence bounds for the location of the polpys and allow us to use the PraNet segmentation model with uncertainty guarantees. From Figure \ref{fig:res} we can see that the method, which combines the original and the transformed scores, effectively delineates polyps regions. These results also help to make us aware of the limitations of the model, allowing medical practioners to follow up on outer sets which do not contain inner sets  in order to determine whether a tumor is present. Improved uncertainty quantification would require an improved segmentation model. 
%Larger uncertainty bounds may require specialist follow-up in order to be certain about the true extent of the observed tumor. 

More precise results can be obtained at the expense of probabilistic guarantees, see Figures \ref{fig:joint2} and \ref{fig:joint3}. A trade off must be made between precision and confidence. The most informative confidence level can be determined in advance based on the learning dataset and the desired type of coverage.

%The approach of CITE controls the empirical false negative risk yielding additional precision but at the cost of coverage as shown in Figure

\subsection{Measuring the coverge rate}\label{SS:cov}
In this section we run validations to evaluate the false coverage rate of our approach. To do so we take the set aside 1500 images and run 1000 validations, in each validation dividing the data into 1000 calibration and 500 test images. In each division we calculate the conformal confidence sets using the different score transformations, based on thresholds derived from the calibration dataset, and evaluate the coverage rate on the test dataset. We average over all 1000 validations and present the results in Figure \ref{fig:coverage}. Histograms for the $90\%$ coverage obtained over all validation runs are shown in Figure \ref{fig:valhist}. From these results we can see that for all the approaches the coverage rate is controlled at or above the nominal level as desired. Using the bounding box scores results in slight over coverage at lower confidence levels. This is likely due to the discontinuities in the score functions $b_I$ and $b_O$. 
%In this Figure we also compare to the coverage attained by using Conformal Risk control \cite{}. We can see that conformal risk control can have highly inflated error rates - this is because it is designed to control the expected proportion of discoveries not cover the tumors. The results indicate the trade-off that must be made when choosing between the methodss, i.e. whilst risk control can provide meaningful inference CITE it comes with a cost in terms of under coverage. Instead, in this setting, conformal confidence sets provide informative segmentation bounds (as illustrated in Section \ref{SS:val}) and come with strong coverage guarantees. 
\begin{figure}
	\begin{center}
		\includegraphics[width=0.4\textwidth]{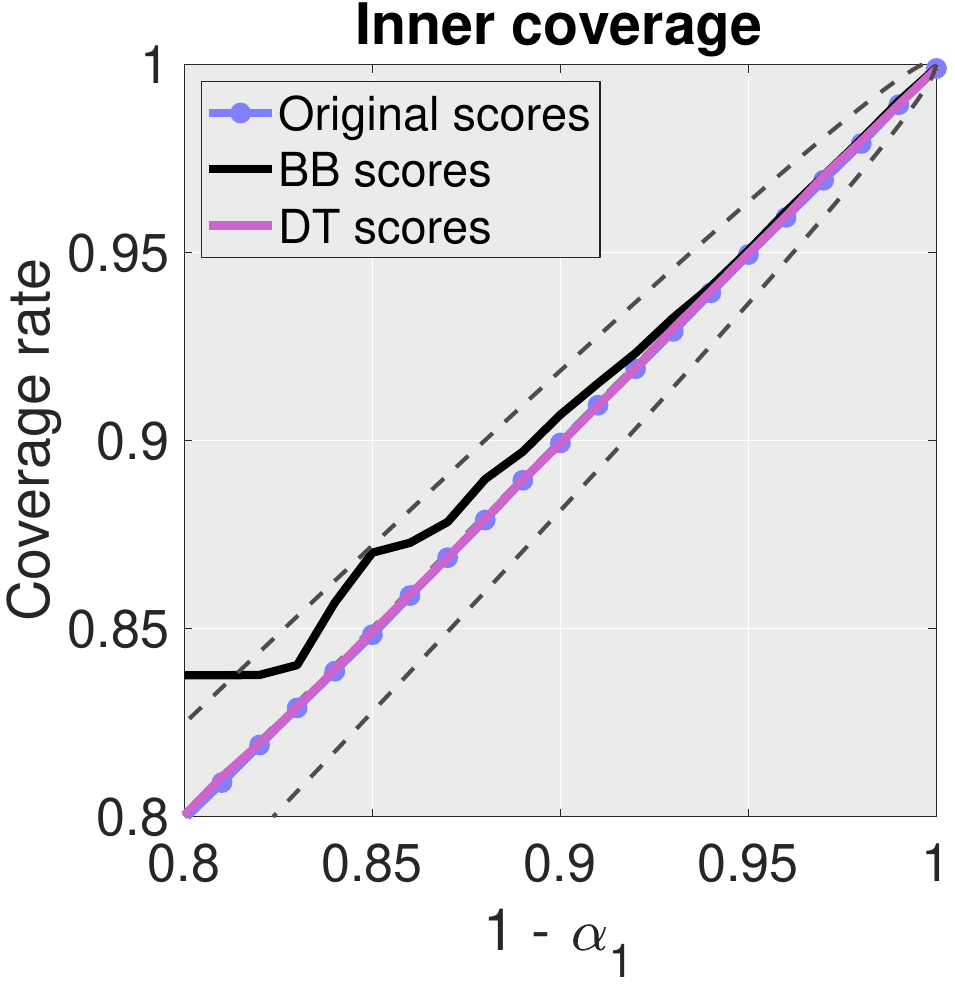}
		\quad\quad
		\includegraphics[width=0.4\textwidth]{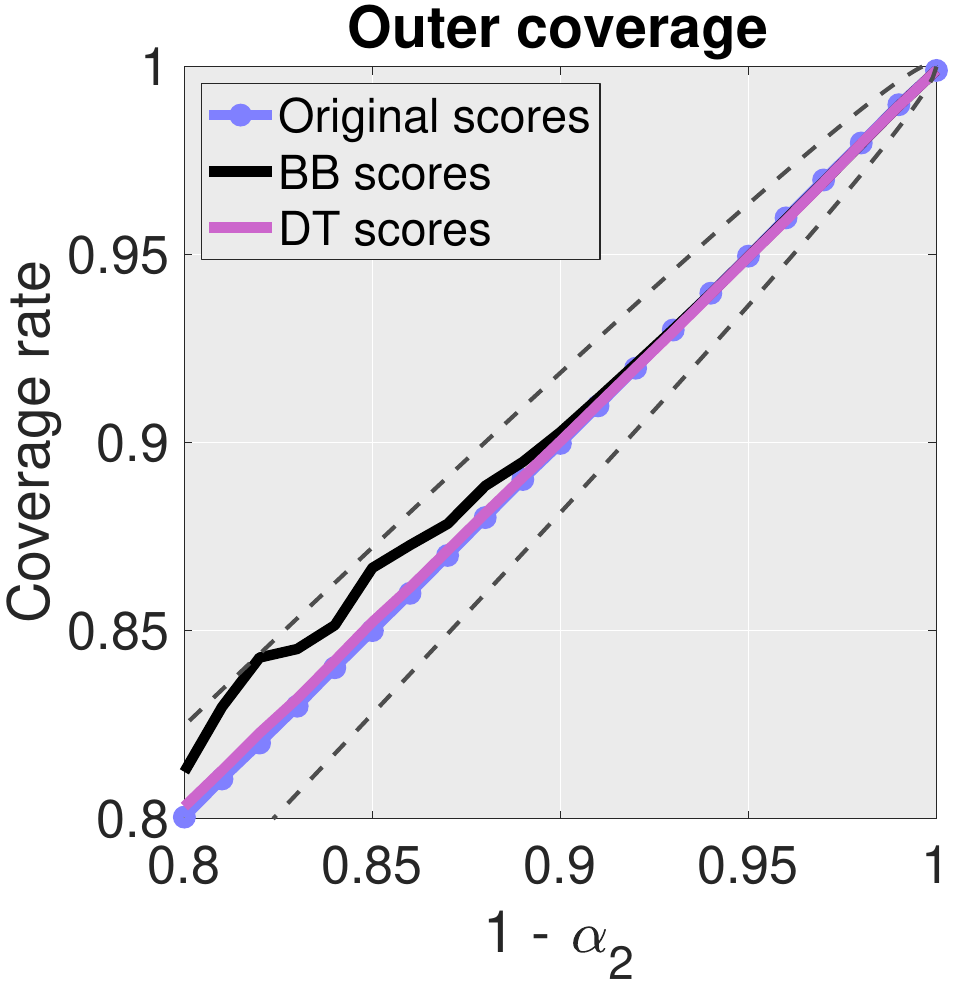}
	\end{center}
	\caption{Coverage levels of the inner and outer sets averaged over 1000 validations for the original, distance transformed (DT) and bounding box (BB) scores.}\label{fig:coverage}
\end{figure}

%\begin{figure}
%	\includegraphics[width=0.32\textwidth]{../figures/validation/inner_coverage.pdf}
%	\includegraphics[width=0.32\textwidth]{../figures/validation/outer_coverage.pdf}
%	\caption{False coverage levels of the inner and outer sets averaged over 1000 validations for the original, distance transformed (DT) and bounding box (BB) scores.}\label{fig:coverage}
%\end{figure}
%	\quad
%\includegraphics[width=0.28\textwidth]{../figures/validation/val_hist.pdf}

\subsection{Comparing the efficiency of the bounds}
In this section we compare the efficiency of the confidence sets based on the different score transformations. To do so we run 1000 validations in each dividing and calibrating as in Section \ref{SS:cov}. For each run we compute the ratio between the diameter of the inner set and the diameter of the ground truth mask and average this ratio over the 500 test images. In order to make a smooth curve we average this quantity over all 1000 runs. A similar calculation is performed for the outer set. The results are shown in Figure \ref{fig:efficiency}. They show that the inner confidence sets produced by using the original scores are the most efficient. Instead, for the outer set, the distance transformed scores perform best. These results match the observations made on the learning dataset in Section \ref{SS:learn} and the results found in Section \ref{SS:val}.

We repeat this procedure instead targeting the proportion of the entire image which is under/over covered by the respective confidence sets. The results are shown in Figure \ref{fig:efficiency2} and can be interpreted similarly. 

\begin{figure}
	\begin{center}
			\includegraphics[width=0.45\textwidth]{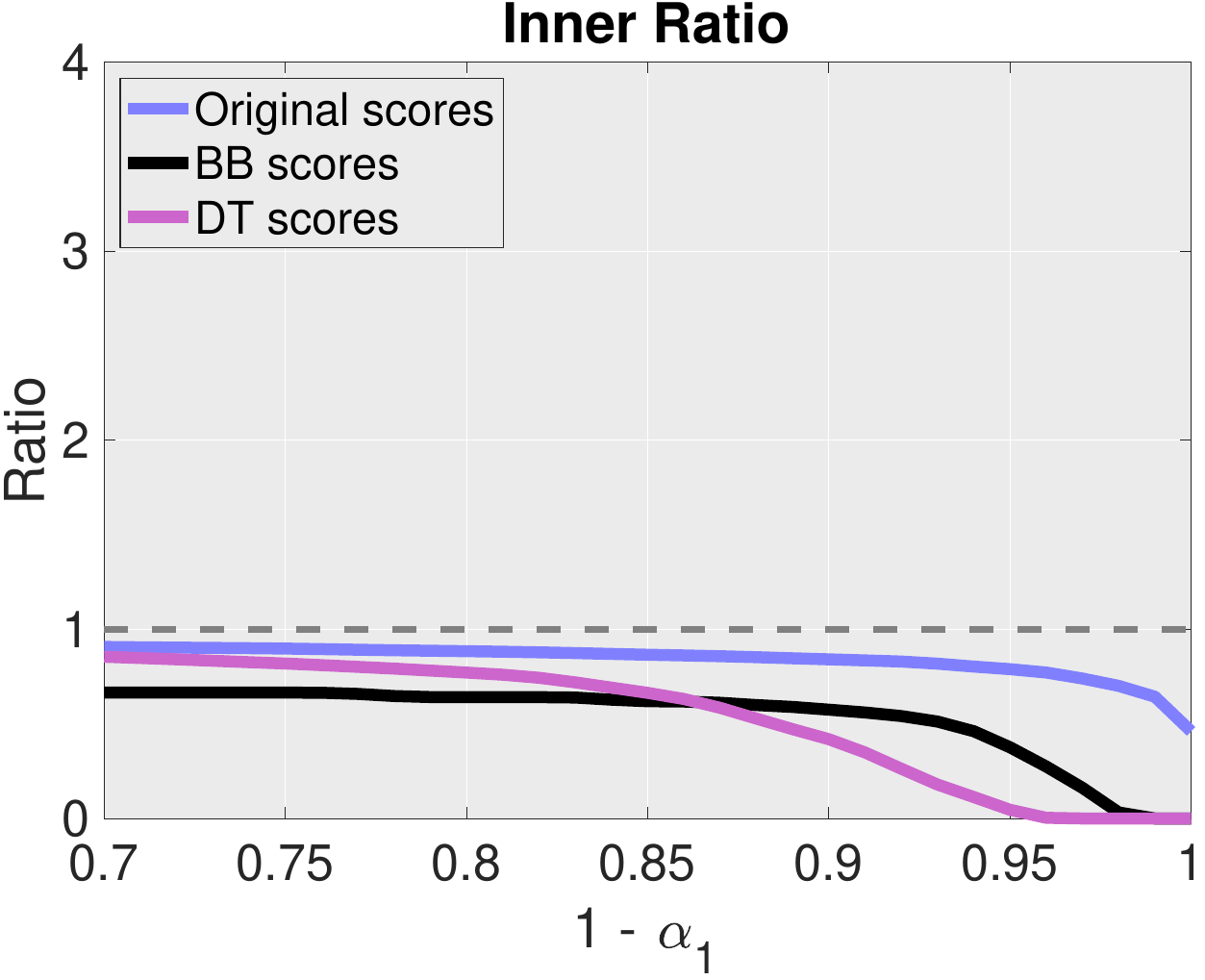}
			\quad\quad
		\includegraphics[width=0.45\textwidth]{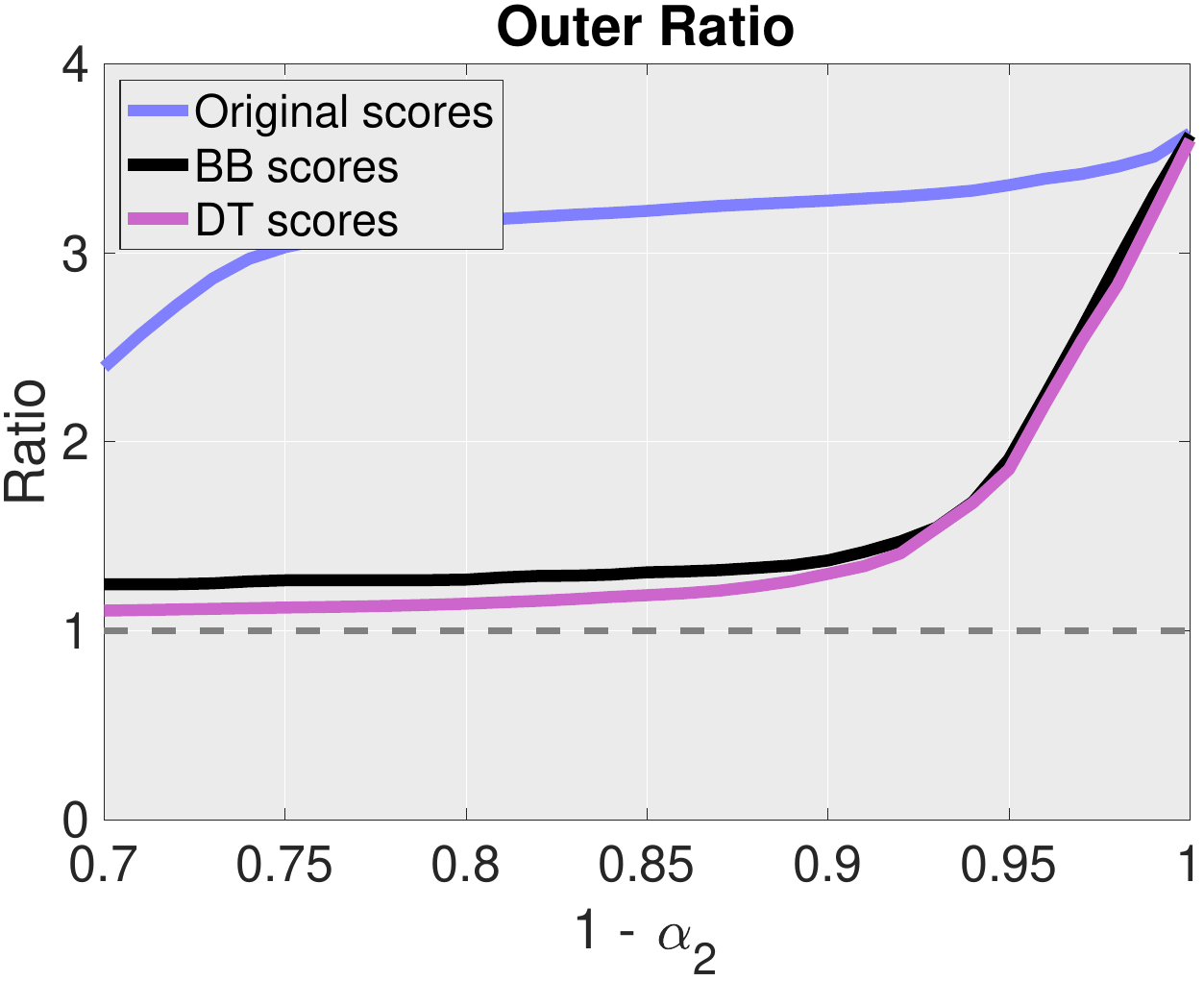}
	\end{center}
	\caption{Measuring the efficiency of the bound using the ratio of the diameter of the coverage set to the diameter of the true tumor mask. The closer the ratio is to one the better. Higher coverage rates lead to a lower efficiency. The original scores provide the most efficient inner sets and the distance transformed scores provide the most efficient outer sets.}\label{fig:efficiency}
\end{figure}

\begin{figure}
	\begin{center}
		\includegraphics[width=0.46\textwidth]{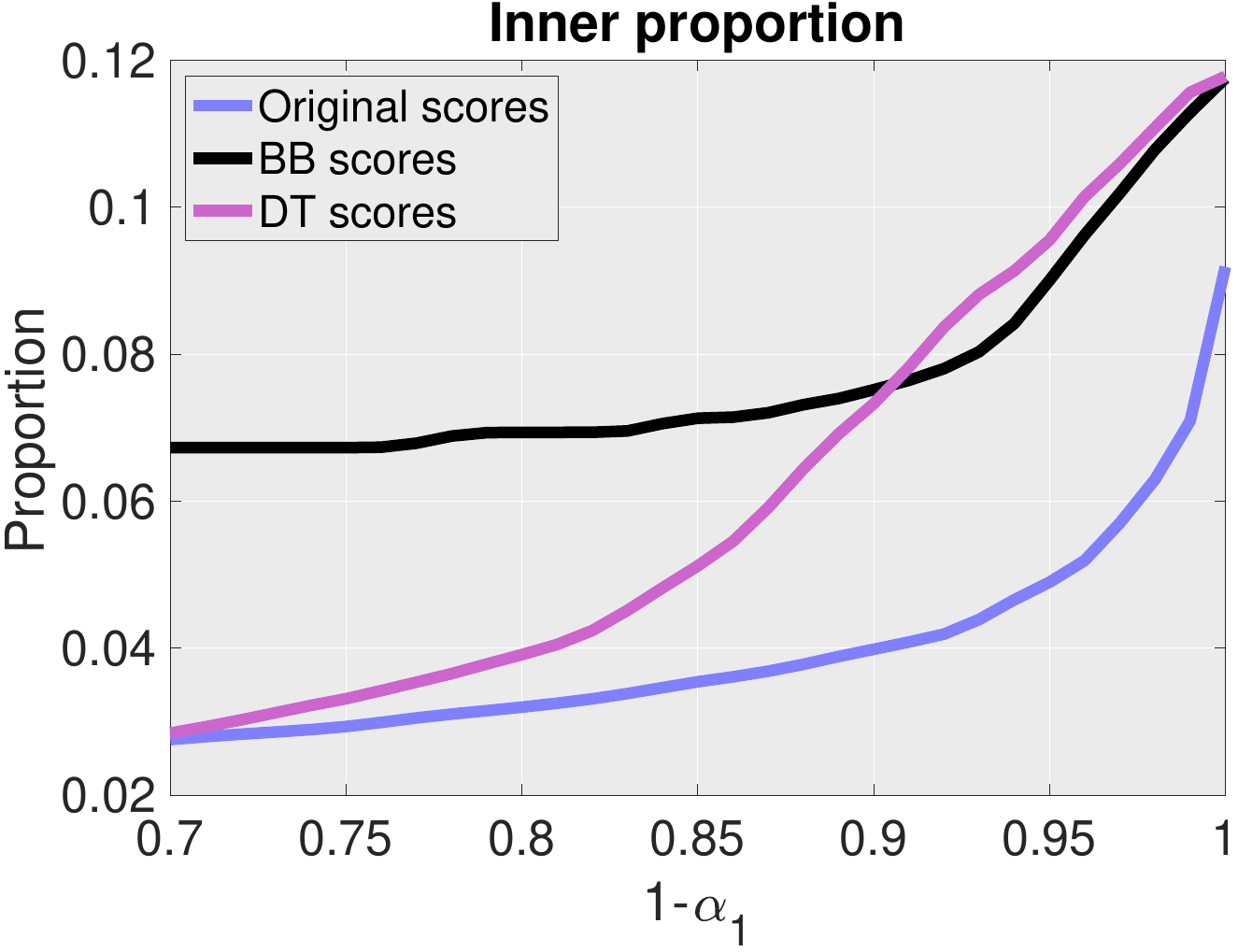}
		\quad\quad
		\includegraphics[width=0.45\textwidth]{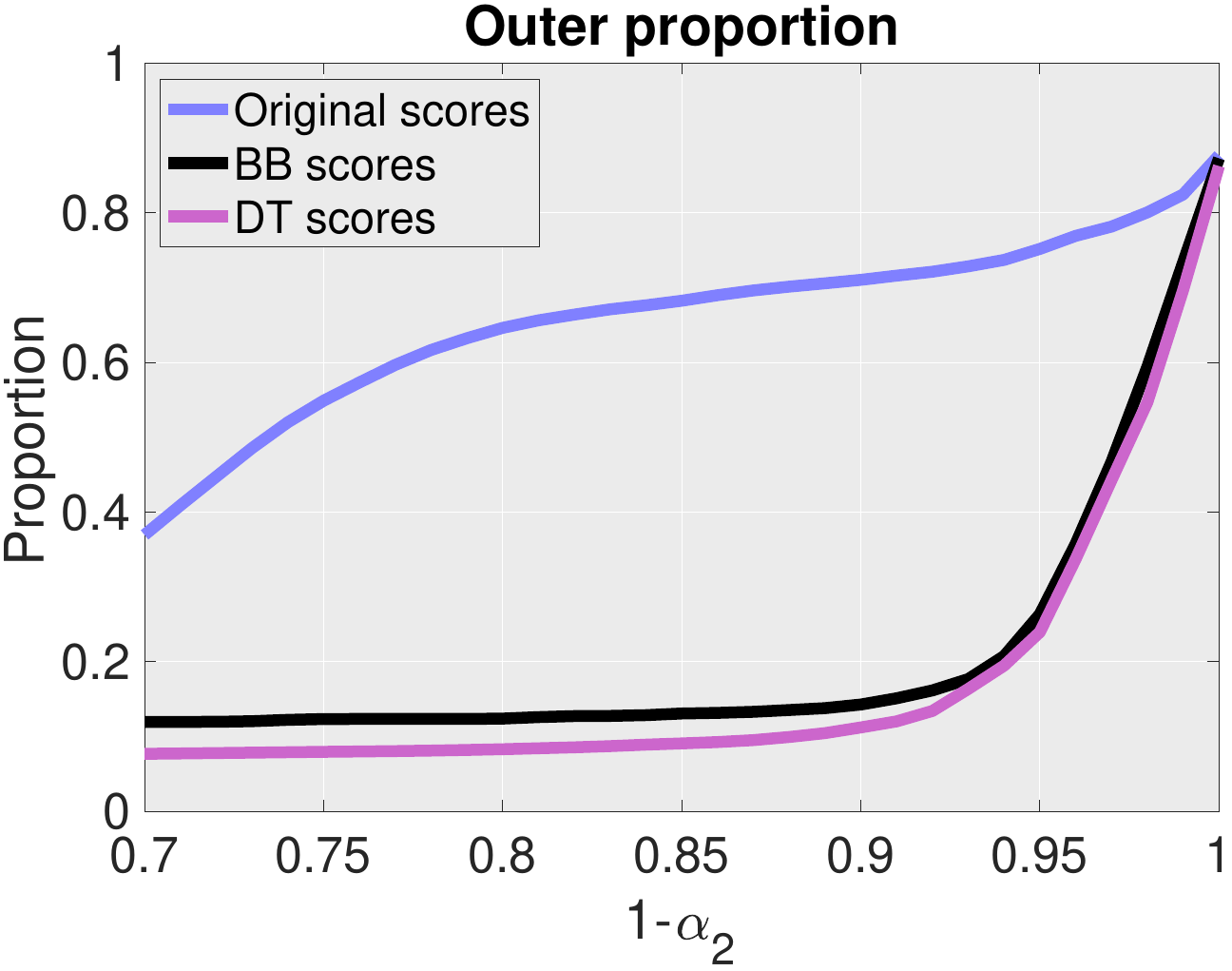}
	\end{center}
	\caption{Measuring the proportion of the entire image which is under/over covered by the respective confidence sets. Left: proportion of the image which lies within the true mask but outside of the inner set. Right: proportion of the image which lies within the confidence set but outside of the true mask. For both a lower proportion corresponds to increased precision. }\label{fig:efficiency2}
\end{figure}
%\subsection{Improving risk control using transformed scores}
%Risk control can also benefit  

%\subsection{Application to Melanoma segmentation}

%\subsection{Melanoma Lesion Segmentation}

%\subsection{Melanoma Segmentation}

%========================================================================
% Discussion
%========================================================================
\section{Discussion}
In this work, we have developed conformal confidence sets which offer probabilistic guarantees for the output of a black box image segmentation model and provide tight bounds. Our work helps to address the lack of formal uncertainty quantification in the application of deep neural networks to medical imaging which has limited the reliability and adoption of these models in practice. The use of improved neural networks which can better separate the scores within and outside the ground truth masks would lead to more precise confidence sets and optimizing this is an important area of research. We have here established validity guarantees and additionally showed that these can be used to theoretically justify a modified version of the max-additive bounding box based method of \cite{Andeol2023}. 

The use of the distance transformed scores was crucial in providing tight outer confidence bounds as the original neural network is by itself unable to reliably determine where the tumors end with certainty.  The distance transformation penalizes regions away from the predicted mask, allowing the tumors to be distinguished from the background. In other datasets and model settings, other transformations may be appropriate. As such we strongly recommend the use of a learning dataset in order to calibrate the transformations and maximize precision of the resulting confidence bounds.

The confidence sets we develop in this paper are related in spirit to work on uncertainty quantification for spatial excursion sets (\cite{chen2017density}, \cite{Bowring2019}, \cite{Mejia2020}). These approaches instead assume that multiple observations from a signal plus noise model are observed and perform inference on the underlying signal rather than prediction. Unlike conformal inference these approaches rely on central limit theorems or distributional assumptions in order to provide spatial confidence regions with asymptotic coverage guarantees. 

\section*{Availability of code}
\vspace{-0.1cm}
Matlab code to implement the methods of this paper and a demo on a downscaled version of the data is available in the supplementary material. The code is very fast: calculating inner and outer thresholds (over the 1000 images in the calibration set) requires approximately 0.03 seconds on the downscaled data on a standard laptop (Apple M3 chip with 16 GB RAM) and 2.64 seconds for the original dataset.

%========================================================================
% Acknowledgments
%========================================================================

\section*{Acknowledgements}
I'm grateful to Habib Ganjgahi at the Big Data Institute at the University of Oxford for useful conversations on this topic. I'm also very grateful to Armin Schartzman at the University of San Diego, California for generous funding and support via NIH grant R01EB026859.

%========================================================================
% Bibliography
%========================================================================
\bibliographystyle{plainnat}
\bibliography{./bibfiles/RFT,./bibfiles/MachineLearning,./bibfiles/MHT,./bibfiles/Statistics,./bibfiles/TomsPapers,./bibfiles/Genetics,./bibfiles/fMRI,./bibfiles/Books}

%%========================================================================
%% Appendix
%%========================================================================
\newgeometry{top=1.5cm, bottom=1.5cm, left=2.7cm, right=2.5cm}  % New margins
\appendix
\section{Appendix}
\renewcommand{\thefigure}{A\arabic{figure}}
\subsection{Obtaining conformal confidence sets with increasing combination functions}\label{A:CF}
As discussed in Remark \ref{rmk:max} the results of Sections \ref{SS:MCS} and \ref{SS:joint} can be generalized to a wider class of combination functions. 
\begin{definition}
	We define a suitable combination function to be a function $C: \mathcal{P}(\mathcal{V}) \times \mathcal{X} \rightarrow \mathbb{R}$  which is increasing in the sense that for all sets $\mathcal{A} \subseteq \mathcal{V}$ and each $v \in \mathcal{A} $, $C(v, X) \leq C(\mathcal{A}, X)$ for all $X \in \mathcal{X}$. 
\end{definition}
The maximum is a suitable combination function since $X(v) = \max_{v \in \lbrace v \rbrace } X(v) \leq \max_{v \in \mathcal{A}} X(v)$. As such this framework directly generalizes the results of the main text. 
%Moreover for each $1\leq i \leq n + 1$, let $$\tau'_i = C(\lbrace v \in \mathcal{V}: Y_i(v) = 0\rbrace, f_O(s(X_i))) $$and $$\gamma'_i= C(\lbrace v \in \mathcal{V}: Y_i(v) = 1\rbrace, f_I(-s(X_i))).$$

We can construct generalized marginal confidence sets as follows.
\begin{theorem}\label{thm:innergen}
	(Marginal inner set)
	Under Assumptions \ref{ass:ex} and \ref{ass:indep}, given $\alpha_1 \in (0,1)$, define 
	\begin{equation*}
		\lambda_I(\alpha_1) = \inf\left\lbrace \lambda: \frac{1}{n} \sum_{i = 1}^n 1\left[ C(\lbrace v \in \mathcal{V}: Y_i(v) = 1\rbrace, f_I(s(X_i))) \leq \lambda \right] \geq \frac{\lceil (1-\alpha_1)(n+1) \rceil}{n} \right\rbrace,
	\end{equation*}
 	for a suitable combination function $C$, and define $I(X) = \lbrace v \in \mathcal{V}: C(v, f_I(s(X))) >\lambda_I(\alpha_1)  \rbrace $. Then,
	\begin{equation}\label{eq:probstat}
		\mathbb{P}\left( I(X_{n+1}) \subseteq\lbrace v\in \mathcal{V}: Y_{n+1}(v) = 1 \rbrace \right) \geq 1 - \alpha_1.
	\end{equation}
\end{theorem}
The proof follows that of Theorem \ref{thm:inner}. The key observation is that for any suitable combination function $C$,  given $\lambda \in \mathbb{R}$, $\mathcal{A} \subseteq \mathcal{V} $ and $X \in \mathcal{X}$, $C(\mathcal{A}, X) \leq \lambda$ implies that $C(v, X) \leq \lambda$. This is the relevant property of the maximum which we used for the results in the main text. For the outer set we similarly have the following.
\begin{theorem}\label{thm:genouter}
	(Marginal outer set)
	Under Assumptions \ref{ass:ex} and \ref{ass:indep}, given $\alpha_2 \in (0,1)$, define 
	\begin{equation*}
		\lambda_O({\alpha_2})= \inf\left\lbrace \lambda: \frac{1}{n} \sum_{i = 1}^n 1\left[ C(\lbrace v \in \mathcal{V}: Y_i(v) = 0\rbrace, -f_O(s(X_i))) \leq \lambda \right] \geq \frac{\lceil (1-\alpha_2)(n+1) \rceil}{n} \right\rbrace.
	\end{equation*}
	for a suitable combination function $C$, and let $O(X) = \lbrace v \in \mathcal{V}: C(v, -f_O(s(X))) \leq \lambda_O(\alpha_2)  \rbrace $. Then,
	\begin{equation}\label{eq:probstat}
		\mathbb{P}\left( \lbrace v\in \mathcal{V}: Y_{n+1}(v) = 1 \rbrace \subseteq O(X_{n+1}) \right) \geq 1 - \alpha_2.
	\end{equation}
\end{theorem}
Joint results can be analogously obtained. 

\subsection{Obtaining confidence sets from risk control}\label{risk2con}
We can alternatively establish Theorems \ref{thm:inner} and \ref{thm:innergen} using an argument from risk control \citep{Angelopoulos2022}. In particular, given an image pair $(X,Y)$ and $\lambda \in \mathbb{R}$, let $$I_\lambda(X) =  \lbrace v \in \mathcal{V}: f_I(s(X), v) > \lambda \rbrace.$$ Define a loss function, $L:\mathcal{P}(\mathcal{V}) \times \mathcal{Y} \rightarrow \mathbb{R}$ which sends $(X,Y)$ to 
\begin{equation*}
	L(I_\lambda(X), Y) = 1\left[ I_\lambda(X) \not \subseteq\lbrace v\in \mathcal{V}: Y(v) = 1 \rbrace \right].
\end{equation*}
For $i = 1, \dots, n + 1$, let 	$L_i(\lambda) = 	L(I_\lambda(X_i), Y_i)$. Arguing as in the proof of Theorem \ref{thm:inner} it follows that $L_i(\lambda) = 1[\tau_i > \lambda]$. Then applying Theorem 1 of \cite{Angelopoulos2022} it follows that 
\begin{equation*}
	\mathbb{E}\left[ L_{n+1}(\hat{\lambda})\right] \leq \alpha_1,
\end{equation*}
where $\hat{\lambda} = \inf\left\lbrace \lambda: \frac{1}{n}\sum_{i = 1}^n L_i(\lambda) \leq \alpha_1 - \frac{1-\alpha_1}{n}\right\rbrace$. Arguing as in Appendix A of \citep{Angelopoulos2022} it follows that 
	\begin{equation*}
	\hat{\lambda} = \inf\left\lbrace \lambda: \frac{1}{n} \sum_{i = 1}^n 1\left[ \tau_i\leq \lambda \right] \geq \frac{\lceil (1-\alpha_1)(n+1) \rceil}{n}\right\rbrace = \lambda_I(\alpha_1),
\end{equation*}
and so $I(X) = I_{\hat{\lambda}}(X)$. As such 	
\begin{equation}\label{eq:probstat2}
	\mathbb{P}\left( I(X_{n+1}) \subseteq\lbrace v\in \mathcal{V}: Y_{n+1}(v) = 1 \rbrace \right) = 1 - \mathbb{E}\left[ L_{n+1}(\hat{\lambda})\right]  \geq 1 - \alpha_1, 
\end{equation}
and we recover the desired result. Arguing similarly it is possible to establish a proof of Theorem \ref{thm:outer}.

\subsection{Deriving confidence sets from bounding boxes}\label{AA:BBtheory}
We can use our results in order to provide valid inference for bounding boxes via an adaption of the approach of \cite{Andeol2023}. In particular given $Z \in \mathcal{Y}$, let $B_{I, \max}(Z)$ be the largest box which can be contained within the set $\lbrace v\in \mathcal{V}: Z(v) = 1 \rbrace$ and let $ B_{O, \min}(Z)$ be the smallest box which contains the set $\lbrace v\in \mathcal{V}: Z(v) = 1 \rbrace$. Given $Y \in \mathcal{Y}, $ let $cc(Y) \subseteq \mathcal{P}(\mathcal{V})$ denote the set of connected components of the set $\lbrace v\in \mathcal{V}: Y(v) = 1 \rbrace$ for a given connectivity criterion (which we take to be $4$ in our examples), and note that these components can themselves be identifed as elements of $\mathcal{Y}$. Define 
$$B_I(Y) = \cup_{c \in cc(Y)} B_{I, \max}(c) \text{ and } B_O(Y) = \cup_{c \in cc(Y)} B_{O, \min}(c)$$
to be the unions of the largest inner and smallest outer boxes of the connected components of the image $Y$, respectively. Then define
$$\hat{B}_I(s(X)) = \cup_{c \in cc(\hat{M}(X)) } B_{I, \max}(c) \text{ and } \hat{B}_O(s(X)) = \cup_{c \in cc(\hat{M}(X))} B_{O, \min}(c)$$
to be the unions of the largest inner and smallest outer boxes of the connected components of the predicted mask $\hat{M}(X)$, respectively. Note that this is well-defined as $\hat{M}(X)$ is a function of $s(X)$.

For the remainder of this section we shall assume that $\mathcal{V} \subset \mathbb{R}^2$, this is not strictly necessary but will help to simplify notation. Given $u,v \in \mathcal{V}$, write $u = (u_1, u_2)$ and $v = (v_1, v_2)$ and let $\rho(u,v) = \max \left( |u_1 - v_1|, |u_2 - v_2| \right)$ be the chessboard metric. 
\begin{definition}\label{dfn:BBS}
	(Bounding box scores) For each $X \in \mathcal{X}$ and $v \in \mathcal{V}$, let
	\begin{equation*}
		b_I(s(X), v) = d_{\rho}(\hat{B}_I(s(X)), v) \text{ and } b_O(s(X), v) = d_{\rho}(\hat{B}_O(s(X)), v) 
	\end{equation*}
be the distance transformed scores based on the chessboard distance to the predicted inner and outer box collections $\hat{B}_I(s(X))$ and $\hat{B}_O(s(X))$, respectively. We also define a combination of these $b_M$, primarily for the purposes of plotting in Figure \ref{fig:learning}, as follows. Let $b_M(s(X), v) = b_O(s(X),v)$ for each $v \not\in \hat{B}_O(s(X))$ and let $b_M(s(X), v) = \max(b_I(s(X),v), 0) $ for $v \in \hat{B}_O(s(X))$. We shall write $b_I(s(X)) \in \mathcal{X}$ to denote the image which has $b_I(s(X))(v) = 	b_I(s(X), v)$ and similarly for $b_O(s(X))$ and $b_M(s(X))$. 
%An illustration of these scores for two example tumors is shown in Figure XXX.
\end{definition}
Now consider the sequences of image pairs $(X_i, B_i^I)_{i = 1}^n$ and $(X_i, B_i^O)_{i = 1}^n$. These both satisfy exchangeability and so, applying Theorems \ref{thm:innergen} and \ref{thm:genouter}, we obtain the following bounding box validity results.
\begin{corollary}\label{thm:boxinnergen}
	(Marginal inner bounding boxes)
	Suppose Assumption \ref{ass:ex} holds and that $(X_i, Y_i)_{i = 1}^{n+1}$ is independent of the functions $s$ and $b_I$.  Given $\alpha_1 \in (0,1)$, define 
	\begin{equation}
		\lambda_I(\alpha_1) = \inf\left\lbrace \lambda: \frac{1}{n} \sum_{i = 1}^n 1\left[ C(B^I_i, b_I(s(X_i))) \leq \lambda \right] \geq  \frac{\lceil (1-\alpha_1)(n+1) \rceil}{n} \right\rbrace,
	\end{equation}
	for a suitable combination function $C$, and define $I(X) = \lbrace v \in \mathcal{V}: C(v, b_I(s(X))) >\lambda_I(\alpha_1)  \rbrace $. Then,
	\begin{equation*}\label{eq:probstat}
		\mathbb{P}\left( I(X_{n+1}) \subseteq B^I_{n+1} \subseteq\lbrace v\in \mathcal{V}: Y_{n+1}(v) = 1 \rbrace \right) \geq 1 - \alpha_1.
	\end{equation*}
\end{corollary}
\begin{corollary}\label{thm:boxgenouter}
	(Marginal outer bounding boxes)
	Suppose Assumption \ref{ass:ex} holds and that $(X_i, Y_i)_{i = 1}^{n+1}$ is independent of the functions $s$ and $b_O$. Given $\alpha_2 \in (0,1)$, define 
	\begin{equation}
		\lambda_O({\alpha_2})= \inf\left\lbrace \lambda: \frac{1}{n} \sum_{i = 1}^n 1\left[ C(B^O_i, -b_O(s(X_i))) \leq \lambda \right] \geq  \frac{\lceil (1-\alpha_2)(n+1) \rceil}{n} \right\rbrace.
	\end{equation}
	for a suitable combination function $C$, and let $O(X) = \lbrace v \in \mathcal{V}: C(v, -b_O(s(X))) \leq \lambda_O(\alpha_2)  \rbrace $. Then,
	\begin{equation*}\label{eq:probstat}
		\mathbb{P}\left( \lbrace v\in \mathcal{V}: Y_{n+1}(v) = 1 \rbrace \subseteq B^O_{n+1} \subseteq O(X_{n+1}) \right) \geq 1 - \alpha_2.
	\end{equation*}
\end{corollary}
Joint results can be obtained in a similar manner to those in Section \ref{SS:joint}.

\newpage
\subsection{Additional examples from the learning dataset}

\begin{figure}[h!]
	%	\centering
	\begin{center}
		\includegraphics[width=0.24\textwidth]{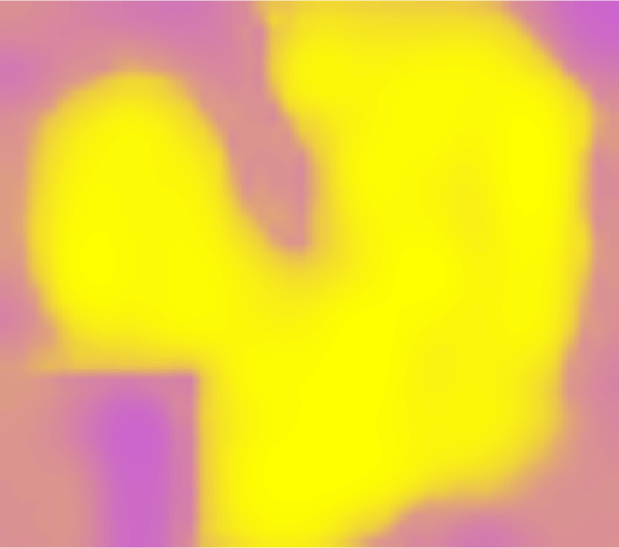}
		\includegraphics[width=0.24\textwidth]{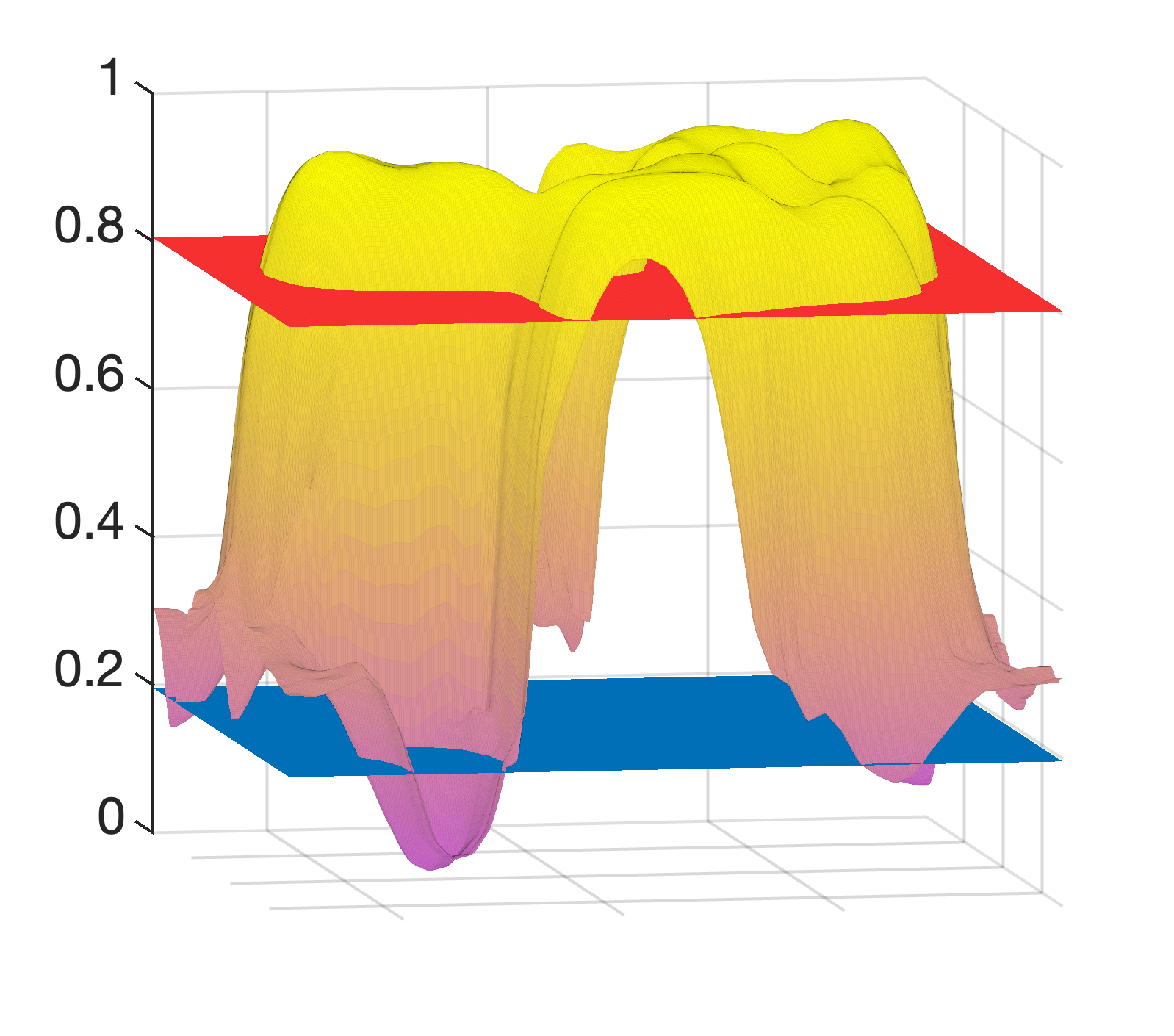}	\includegraphics[width=0.24\textwidth]{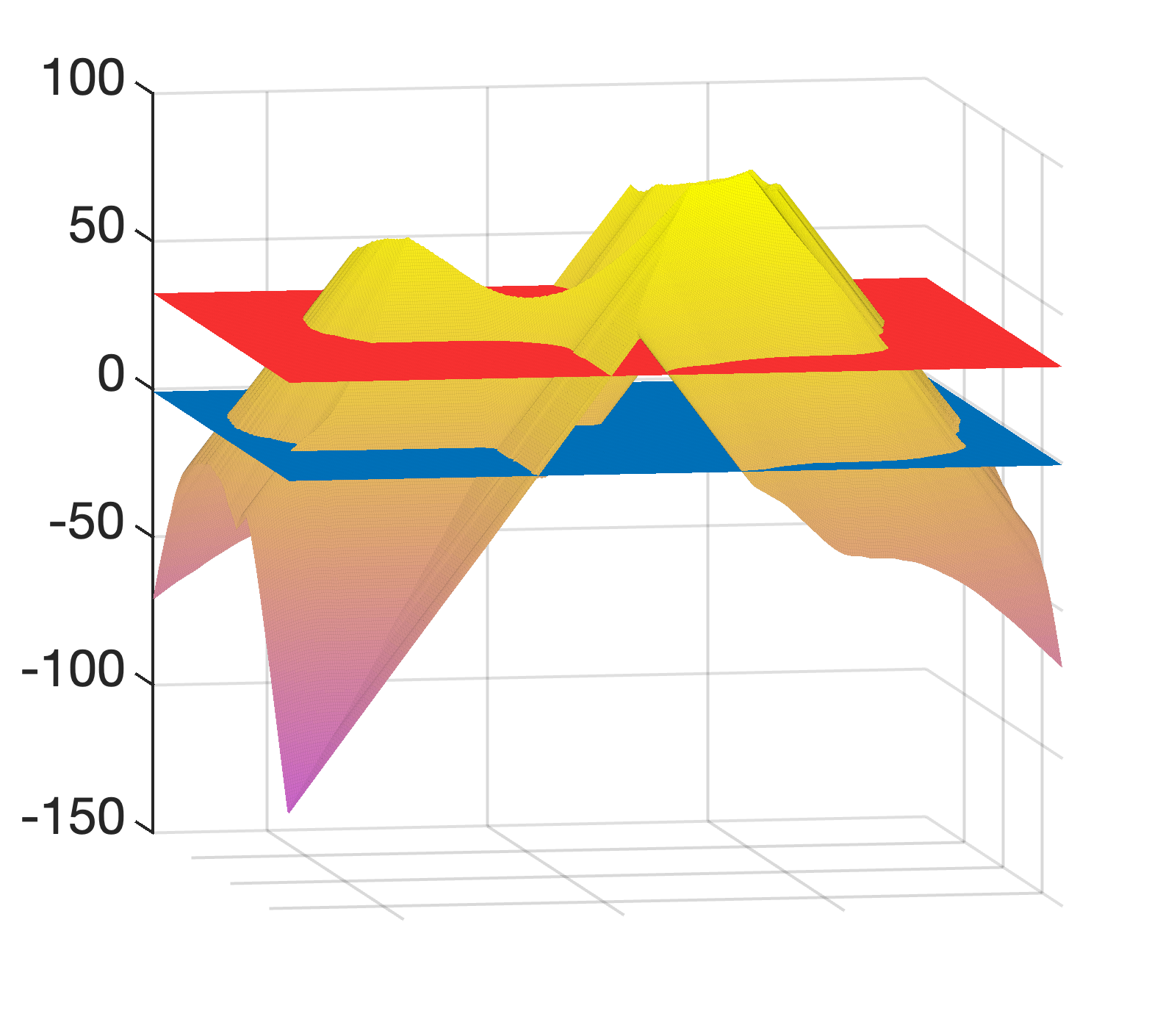}
		\includegraphics[width=0.24\textwidth]{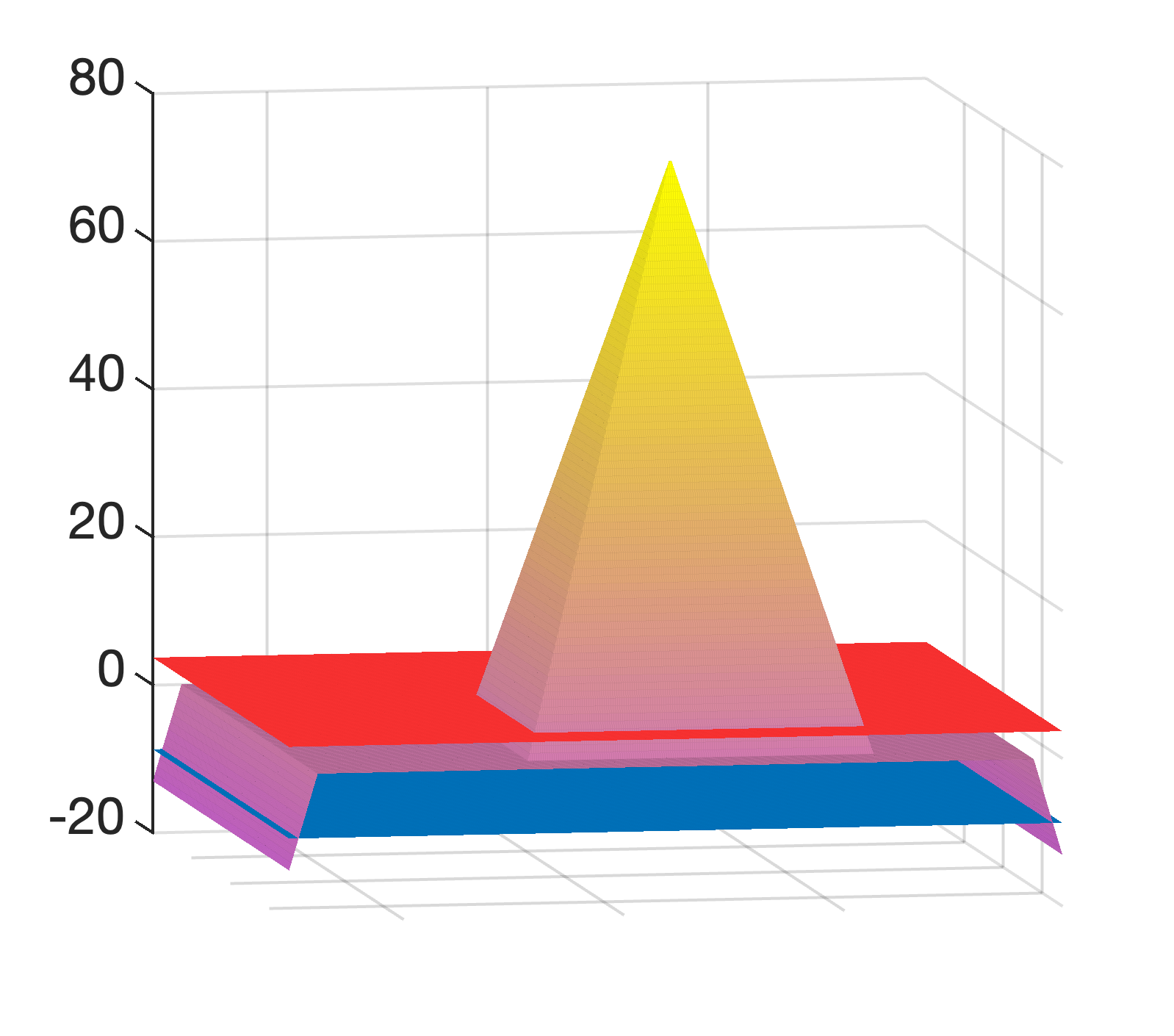}\\
		\includegraphics[width=0.24\textwidth]{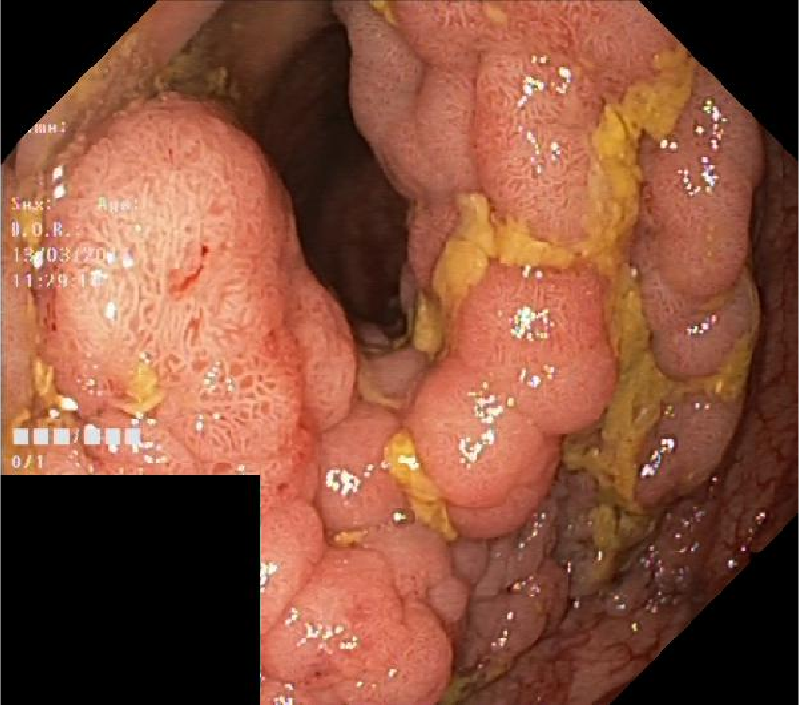}
		\includegraphics[width=0.24\textwidth]{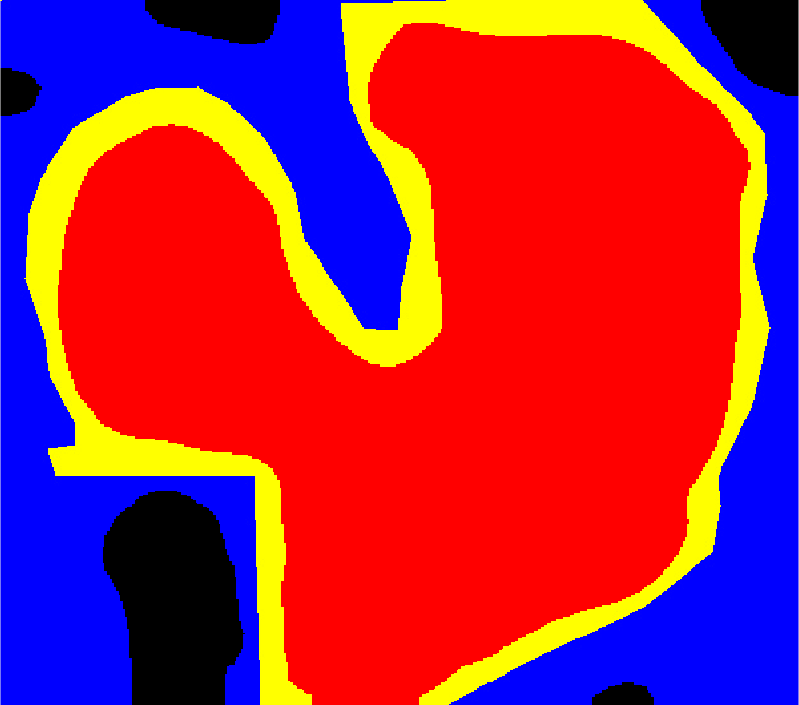}
		\includegraphics[width=0.24\textwidth]{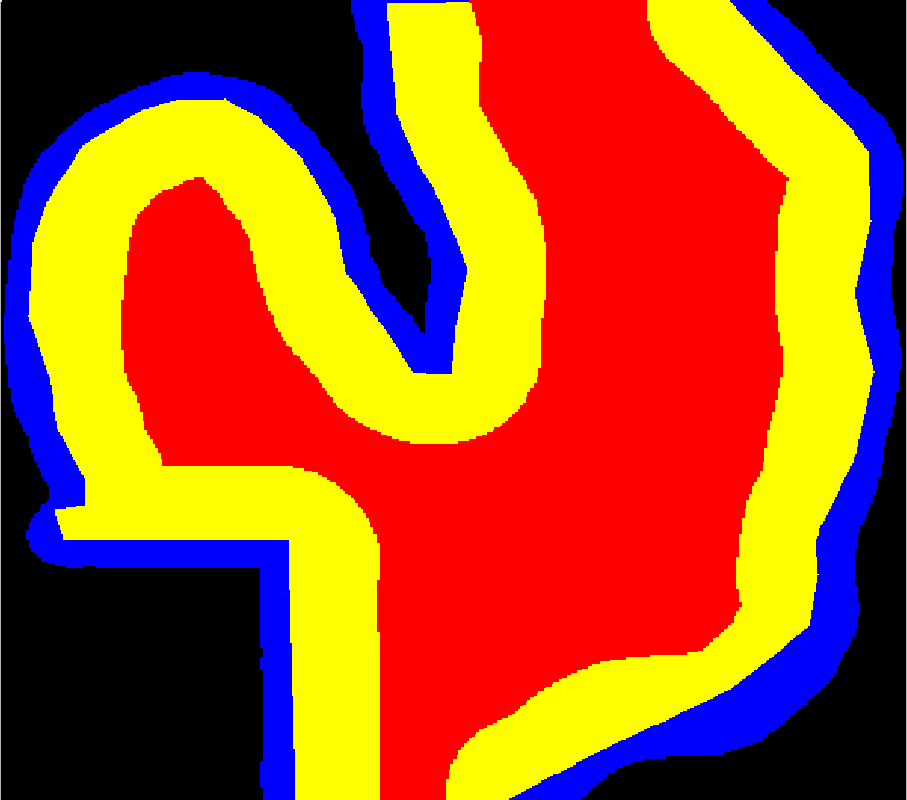}
		\includegraphics[width=0.24\textwidth]{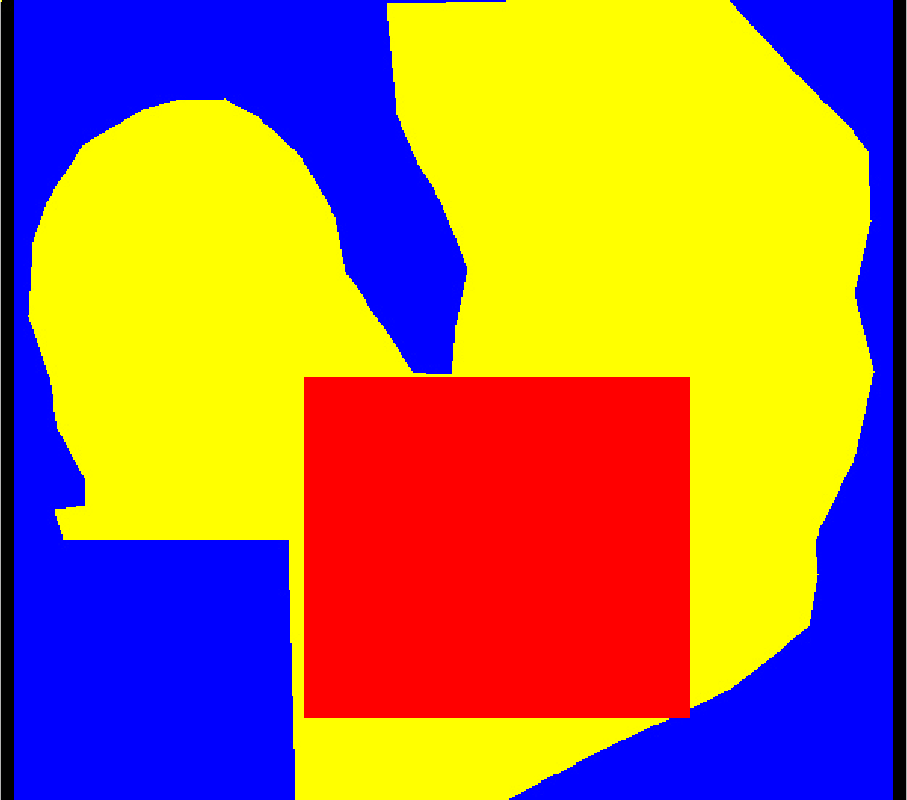}\\
		\vspace{0.5cm}
		\includegraphics[width=0.24\textwidth]{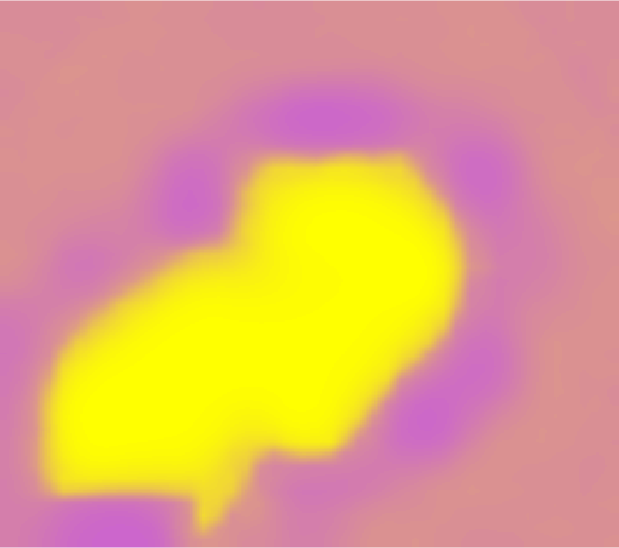}
		\includegraphics[width=0.24\textwidth]{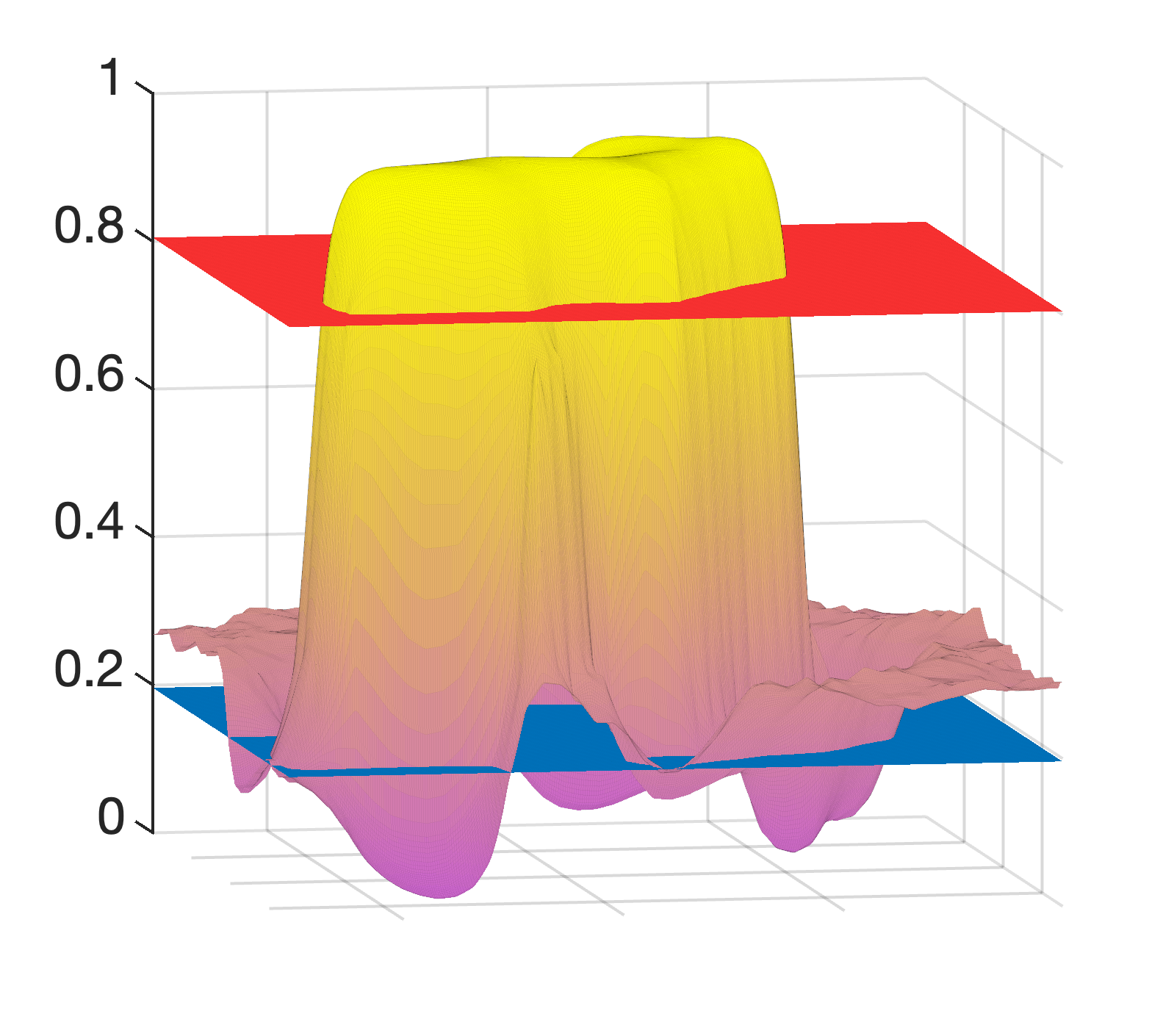}	\includegraphics[width=0.24\textwidth]{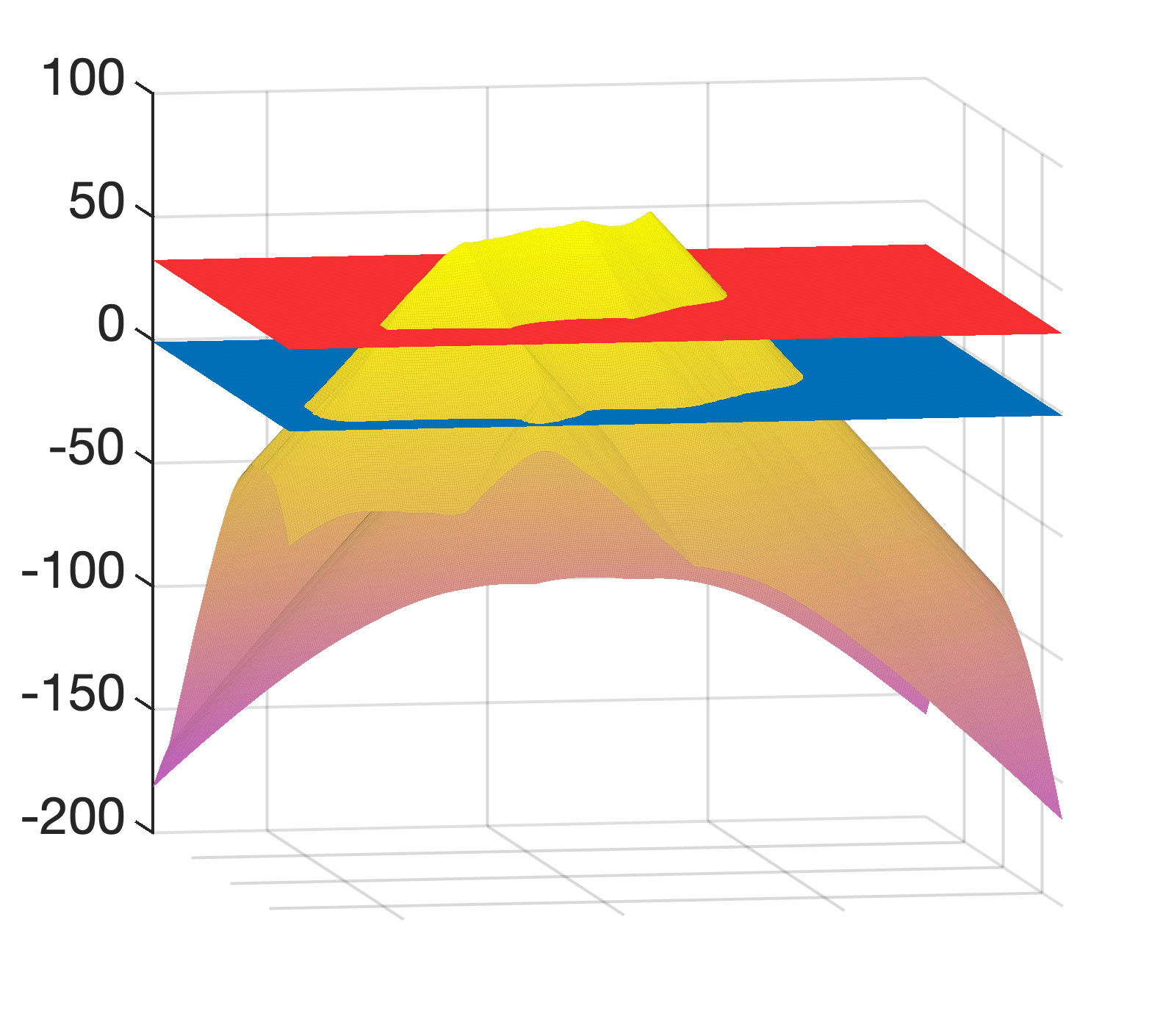}
		\includegraphics[width=0.24\textwidth]{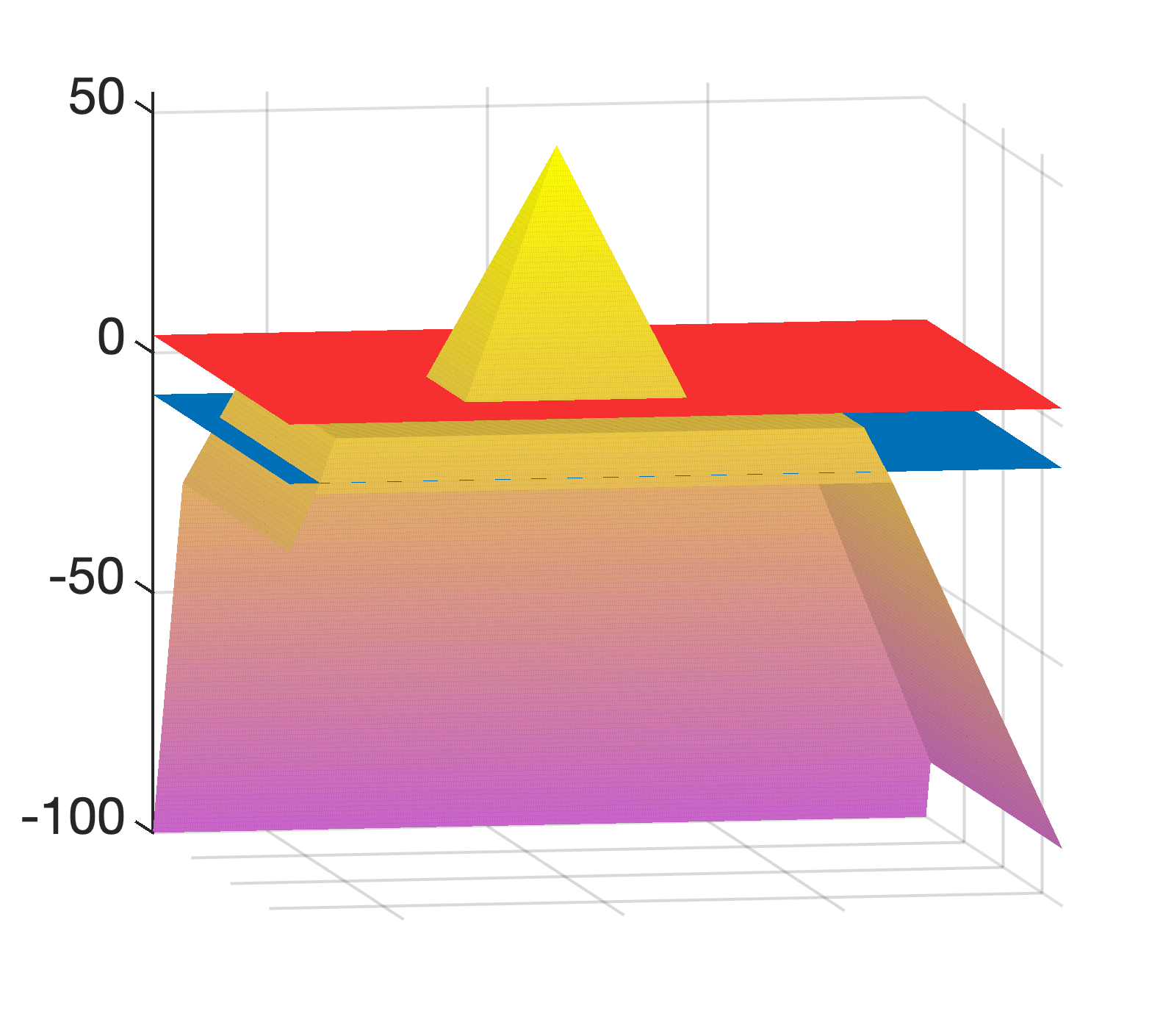}\\
		\includegraphics[width=0.24\textwidth]{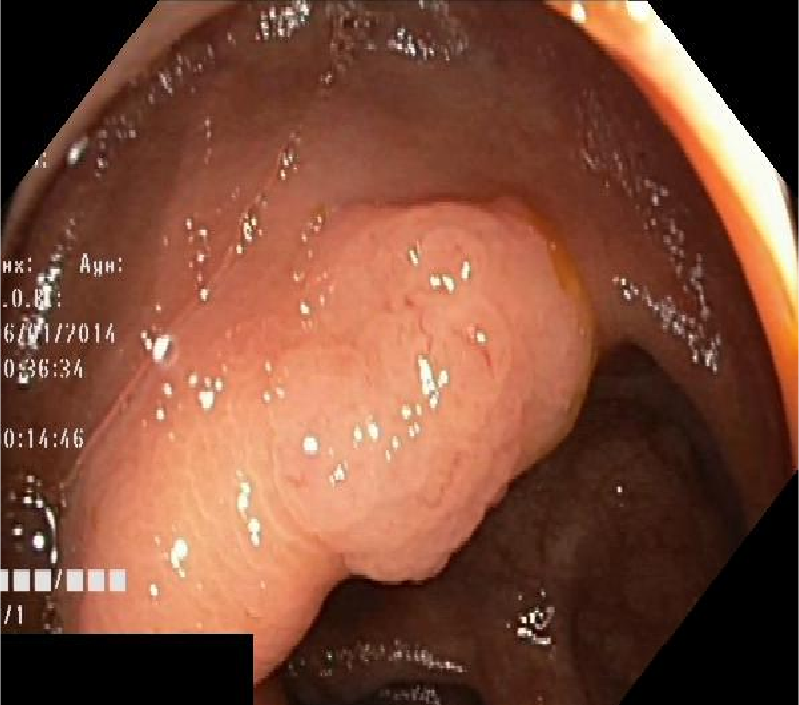}
		\includegraphics[width=0.24\textwidth]{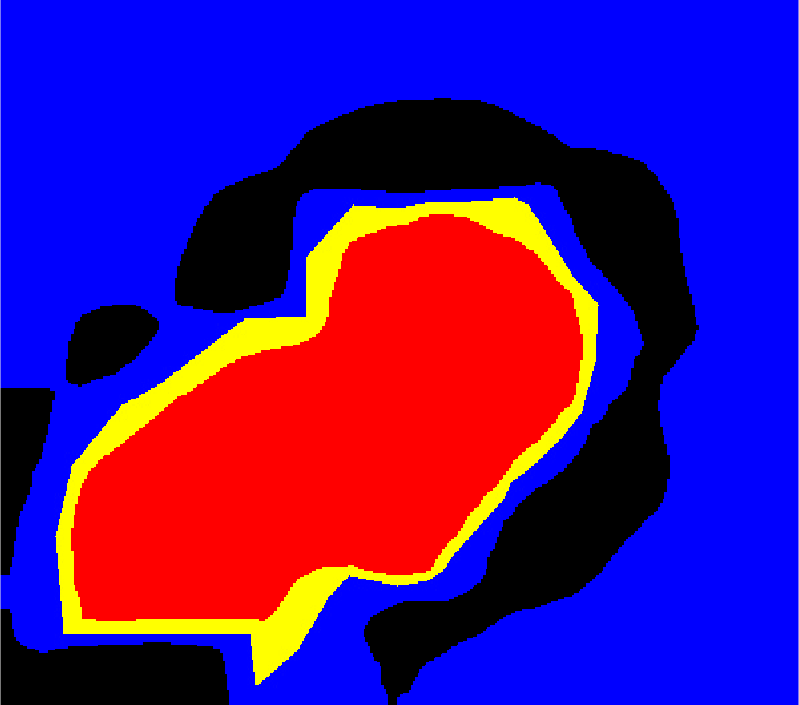}
		\includegraphics[width=0.24\textwidth]{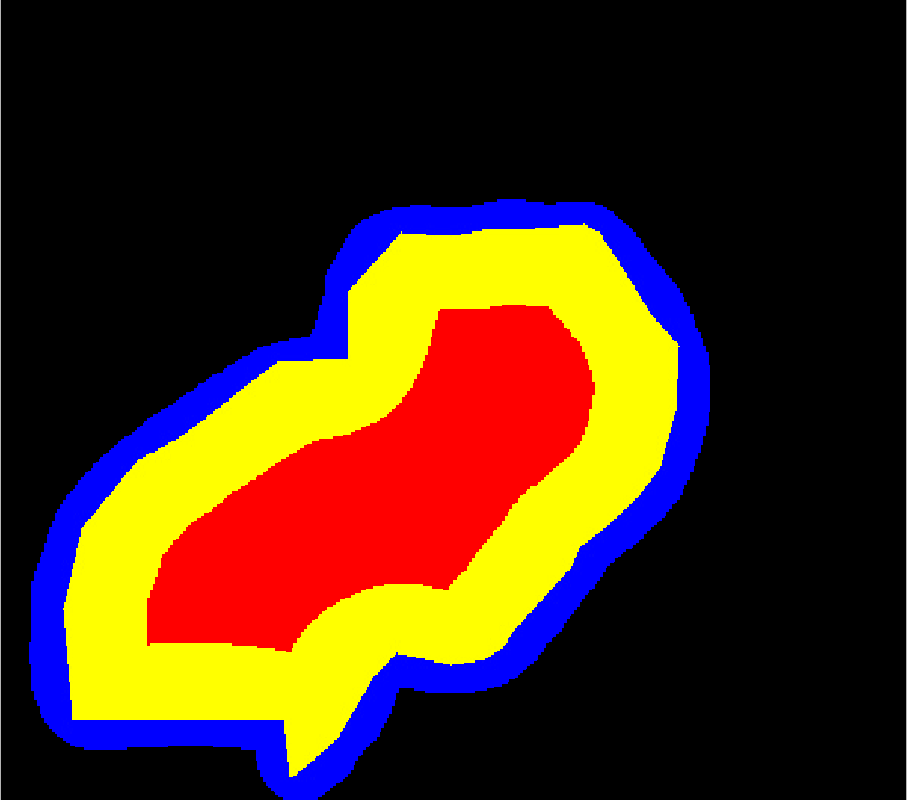}
		\includegraphics[width=0.24\textwidth]{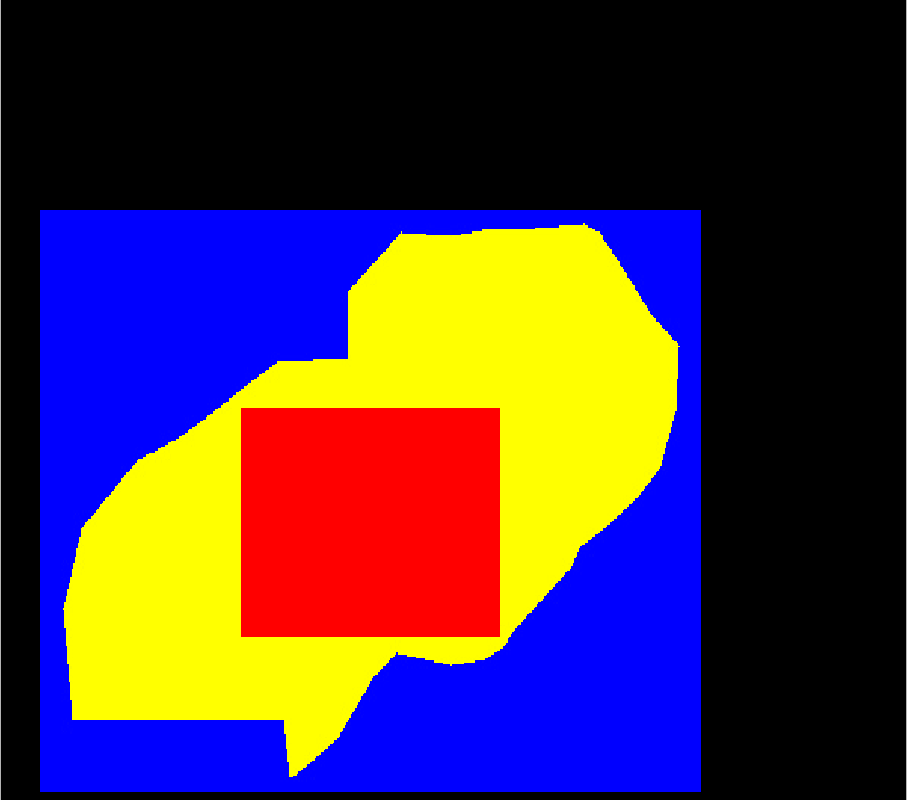}\\
			\vspace{0.5cm}
		\includegraphics[width=0.24\textwidth]{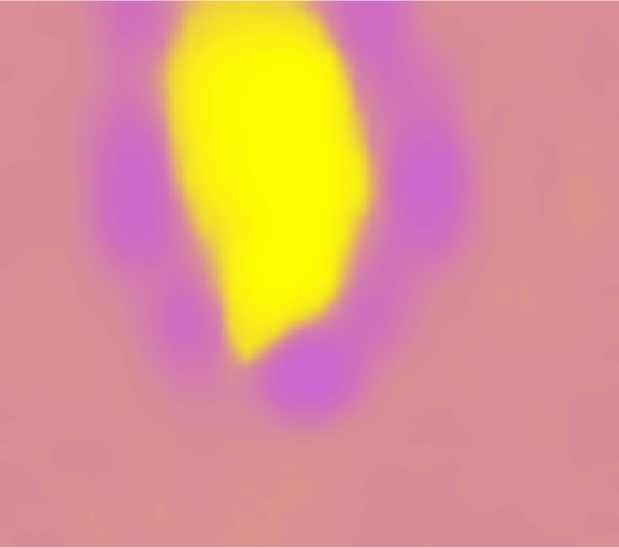}
		\includegraphics[width=0.24\textwidth]{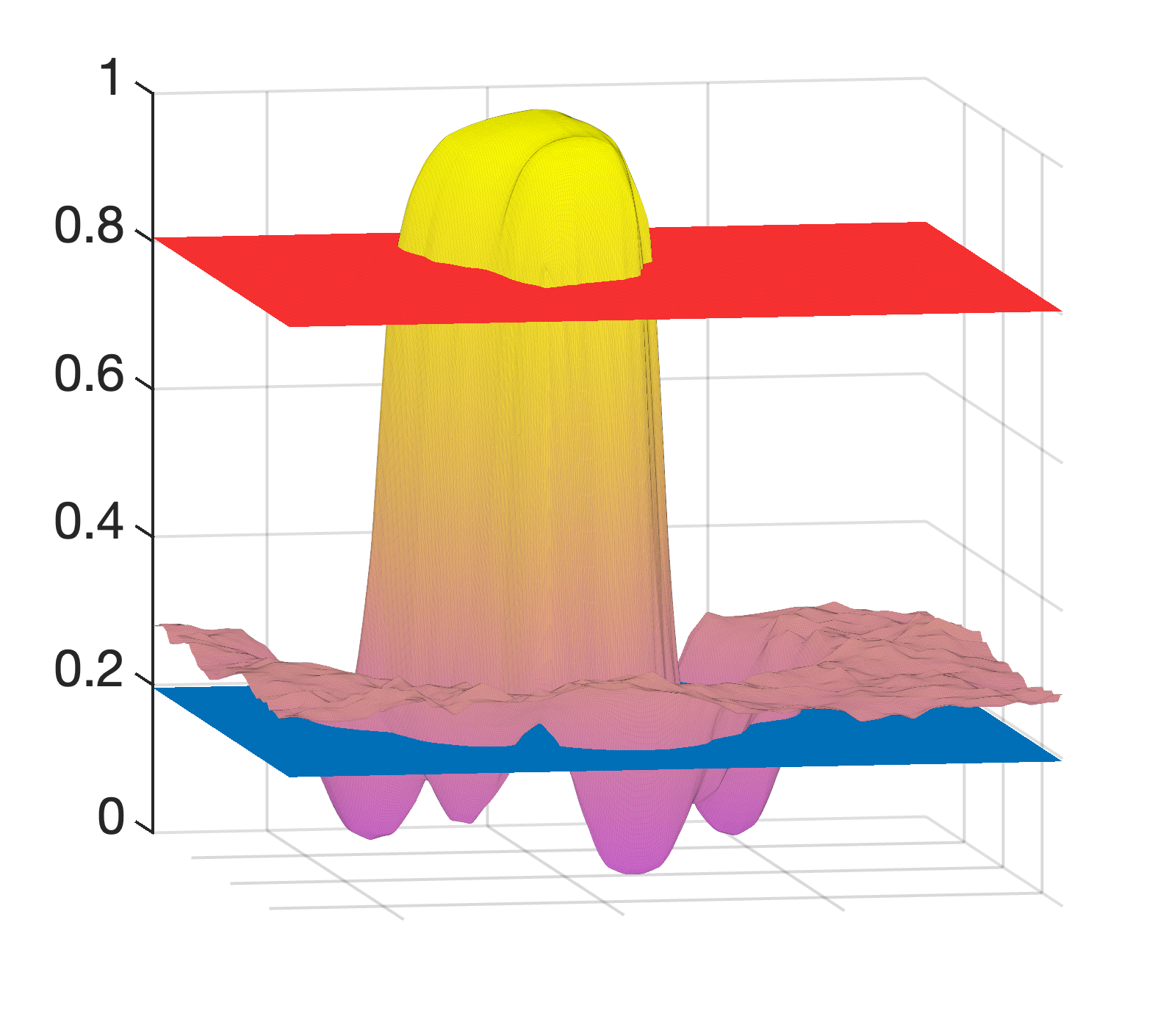}	\includegraphics[width=0.24\textwidth]{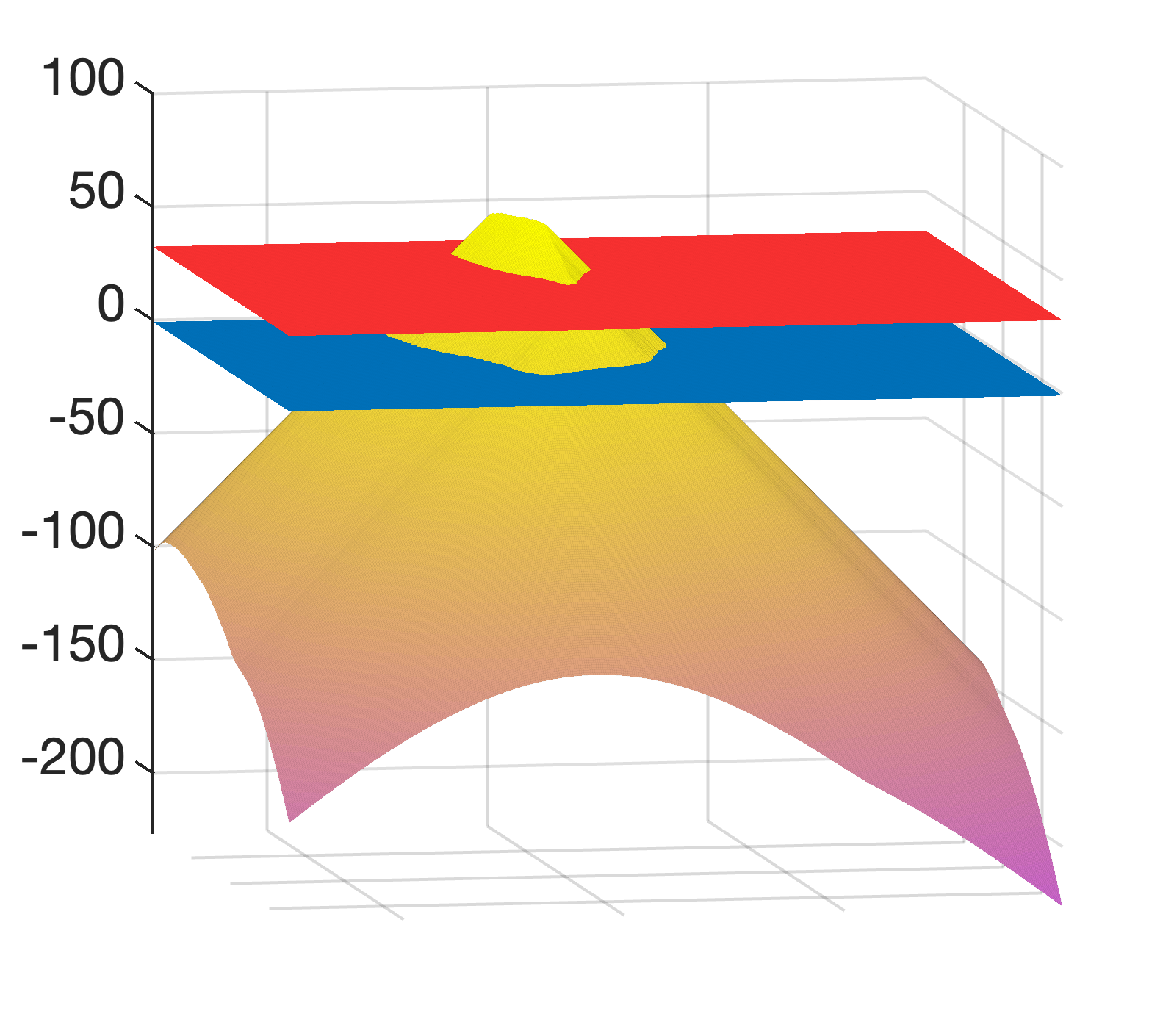}
		\includegraphics[width=0.24\textwidth]{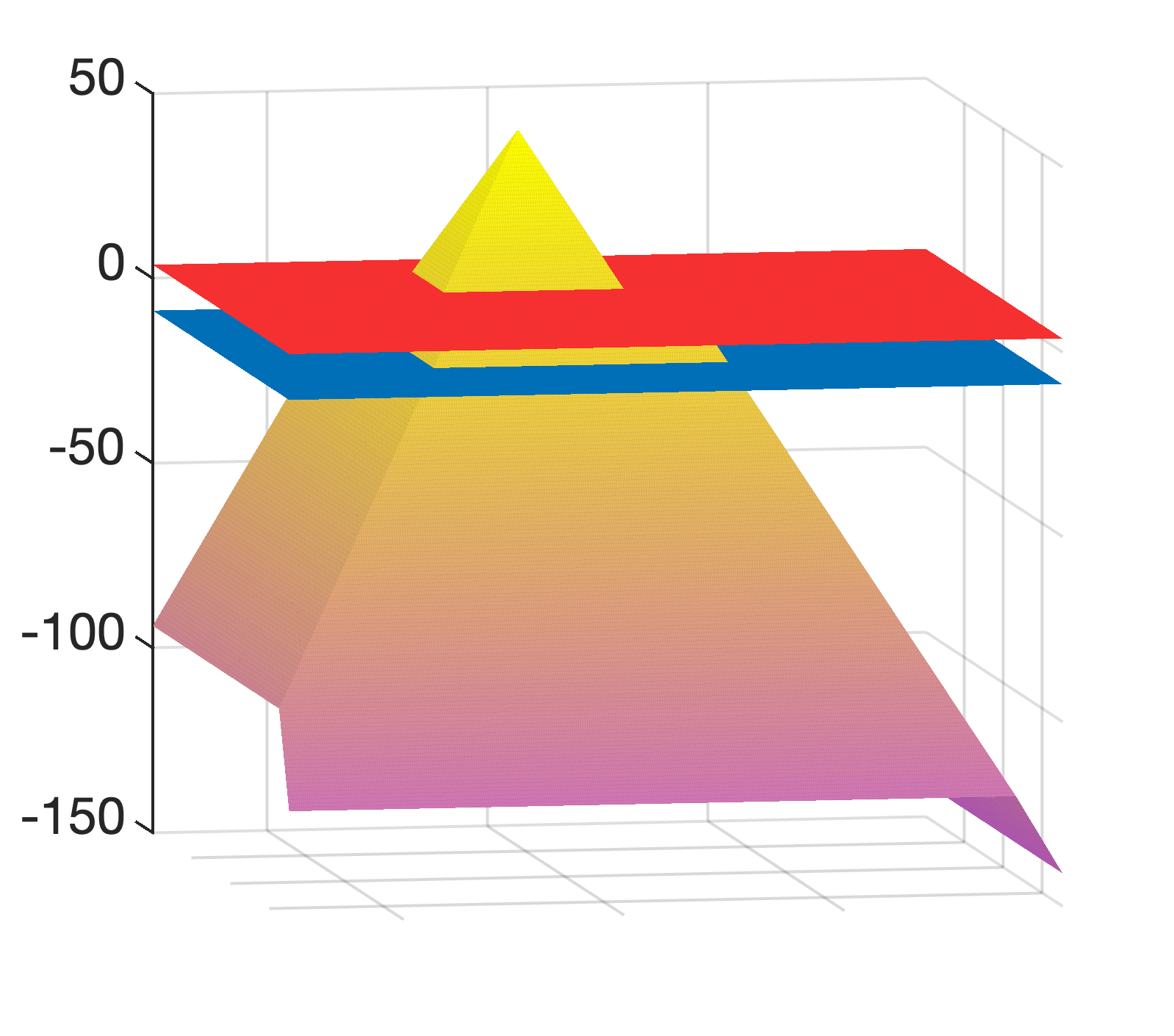}\\
		\includegraphics[width=0.24\textwidth]{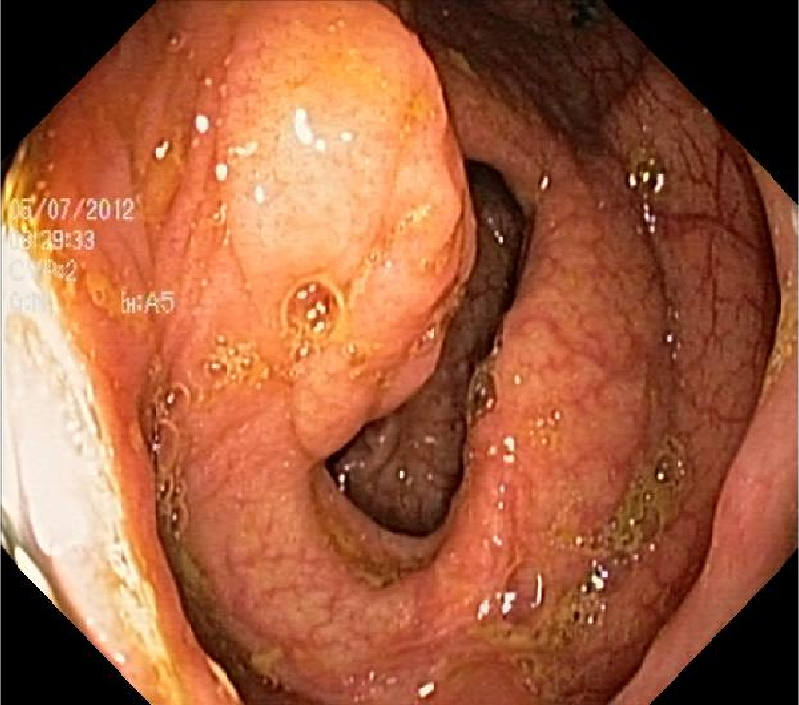}
		\includegraphics[width=0.24\textwidth]{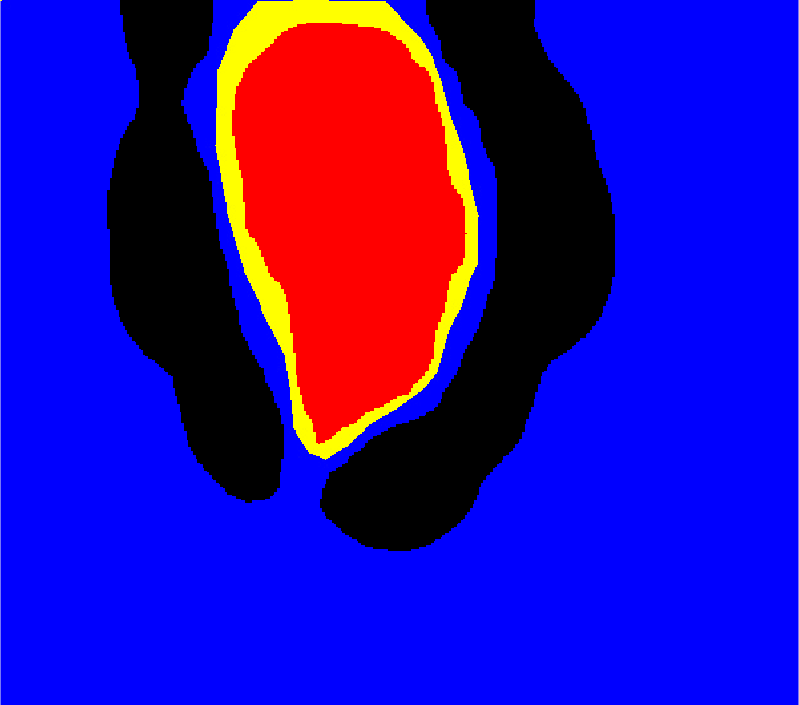}
		\includegraphics[width=0.24\textwidth]{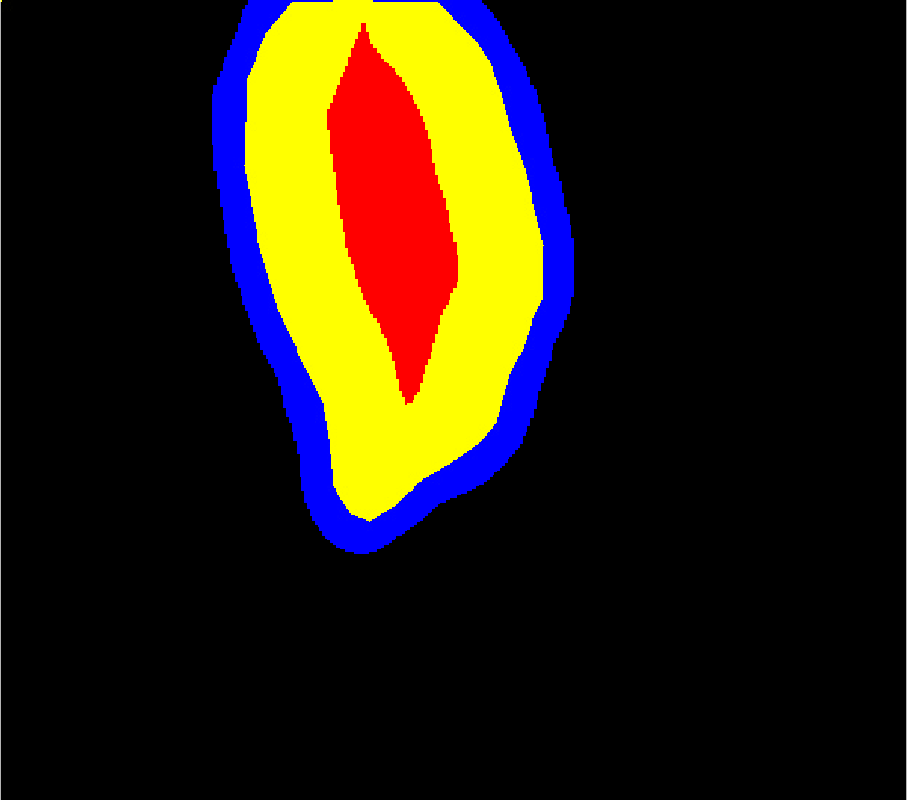}
		\includegraphics[width=0.24\textwidth]{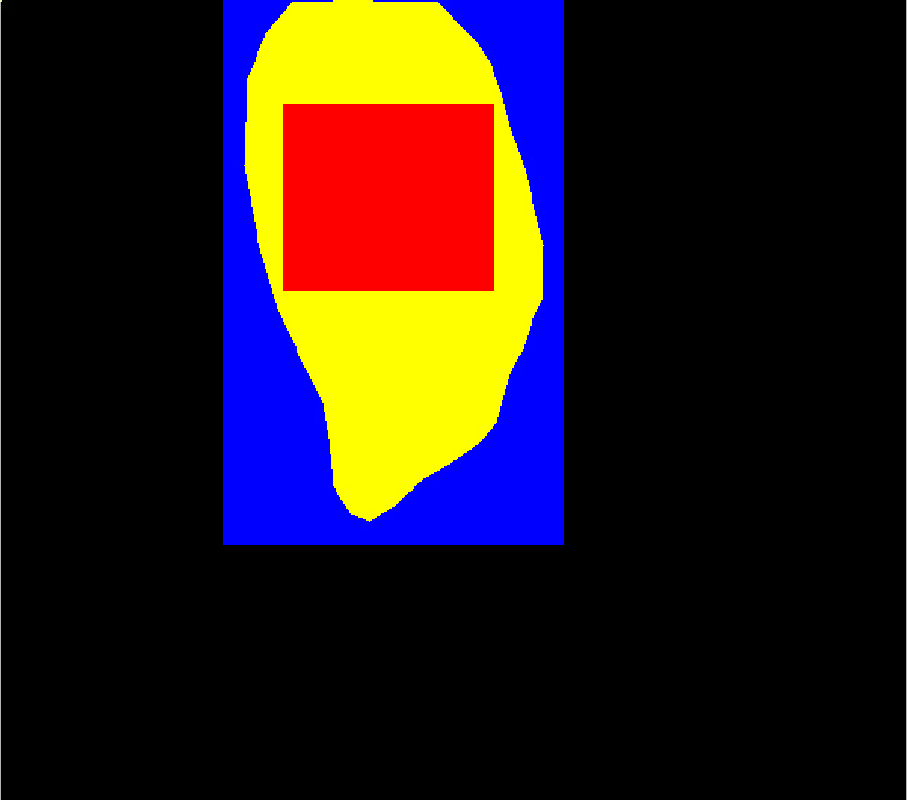}
	\end{center}
	\caption{Additional examples from the learning dataset. The layout of these figures is the same as for Figure \ref{fig:learning}.}
	\label{fig:learning2}
\end{figure}

\begin{figure}
	%	\centering
	\begin{center}
		\includegraphics[width=0.24\textwidth]{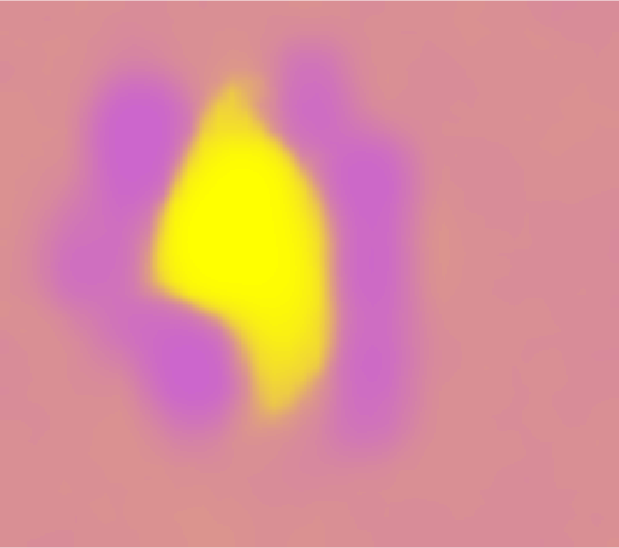}
		\includegraphics[width=0.24\textwidth]{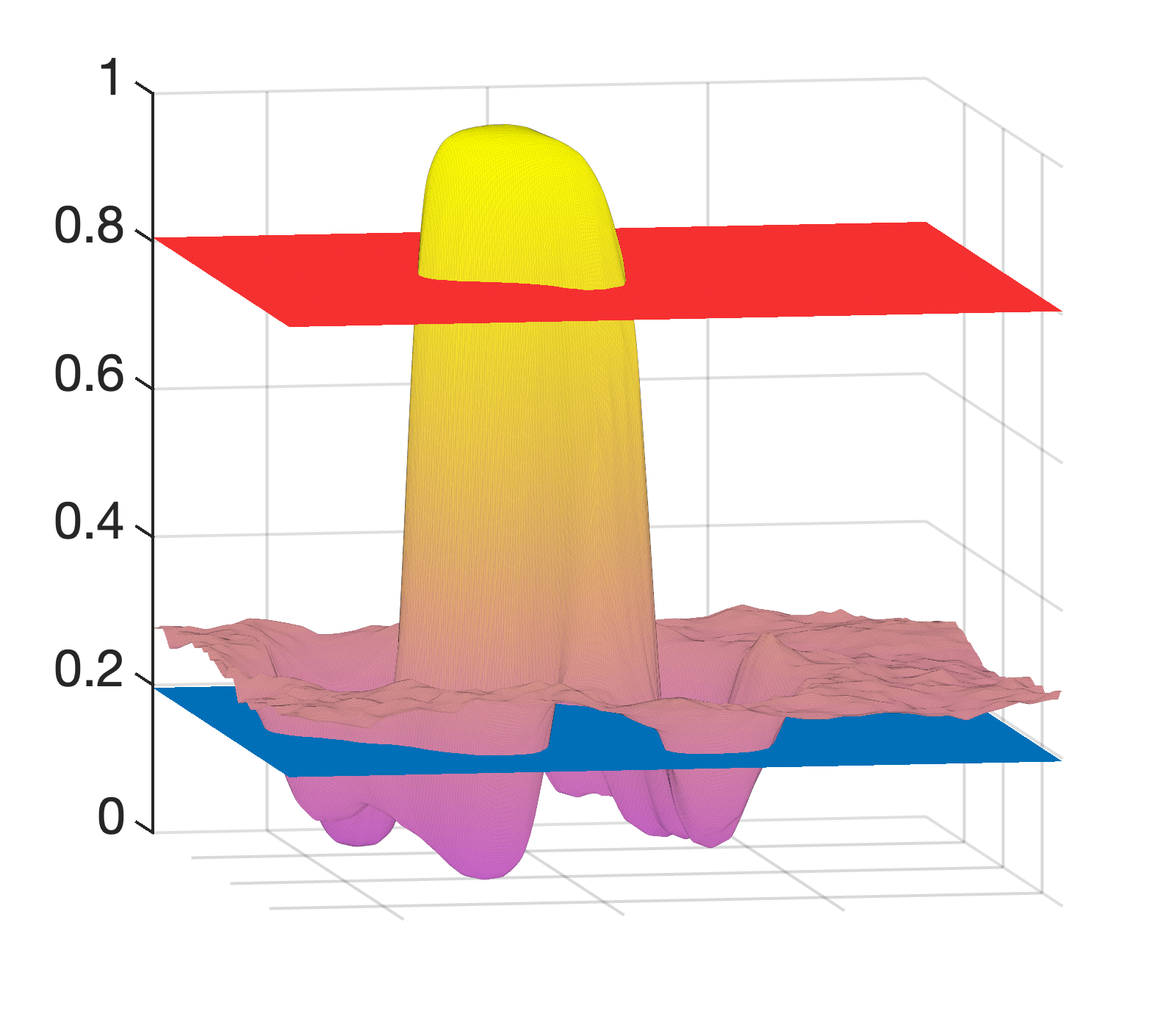}	\includegraphics[width=0.24\textwidth]{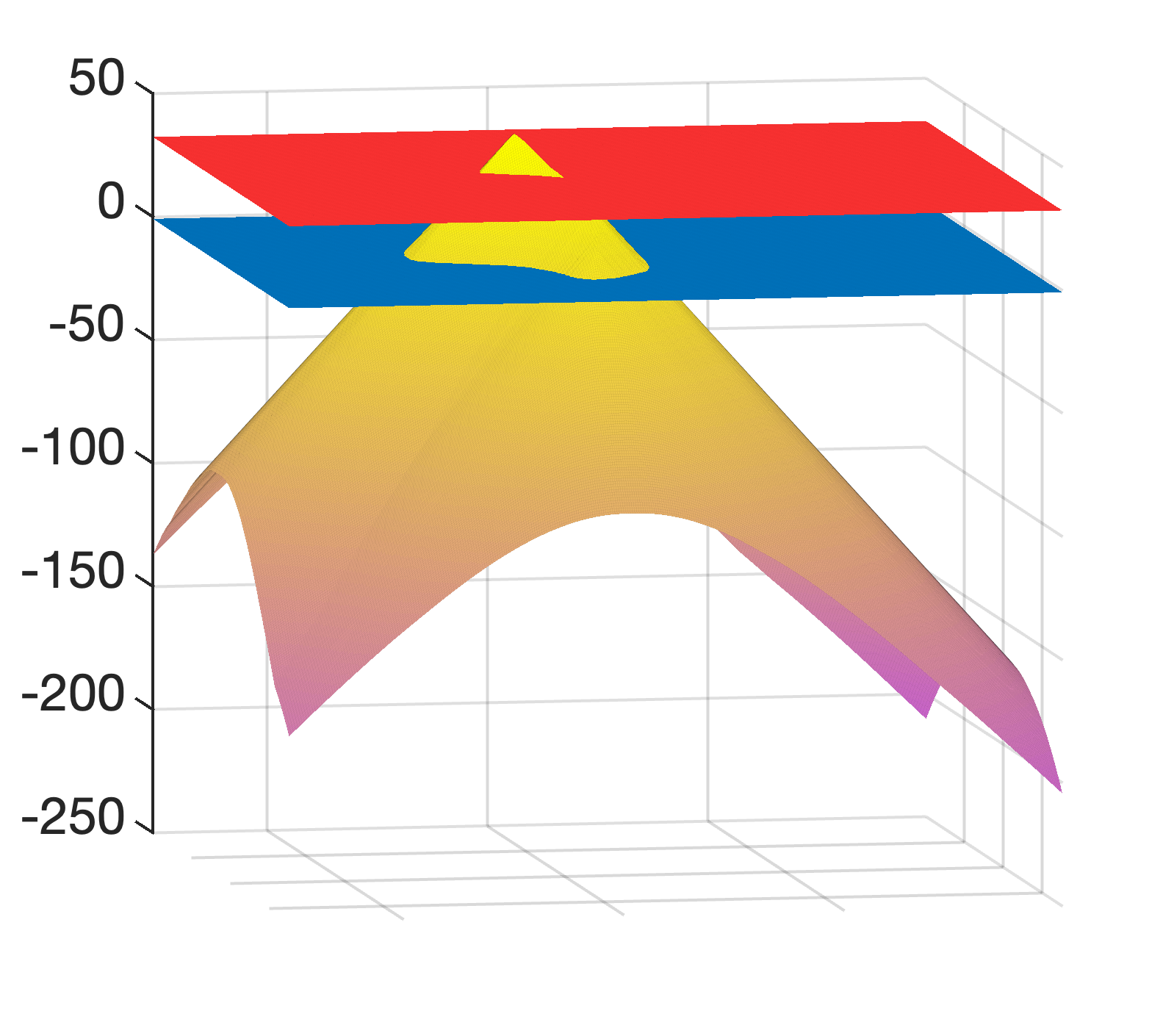}
		\includegraphics[width=0.24\textwidth]{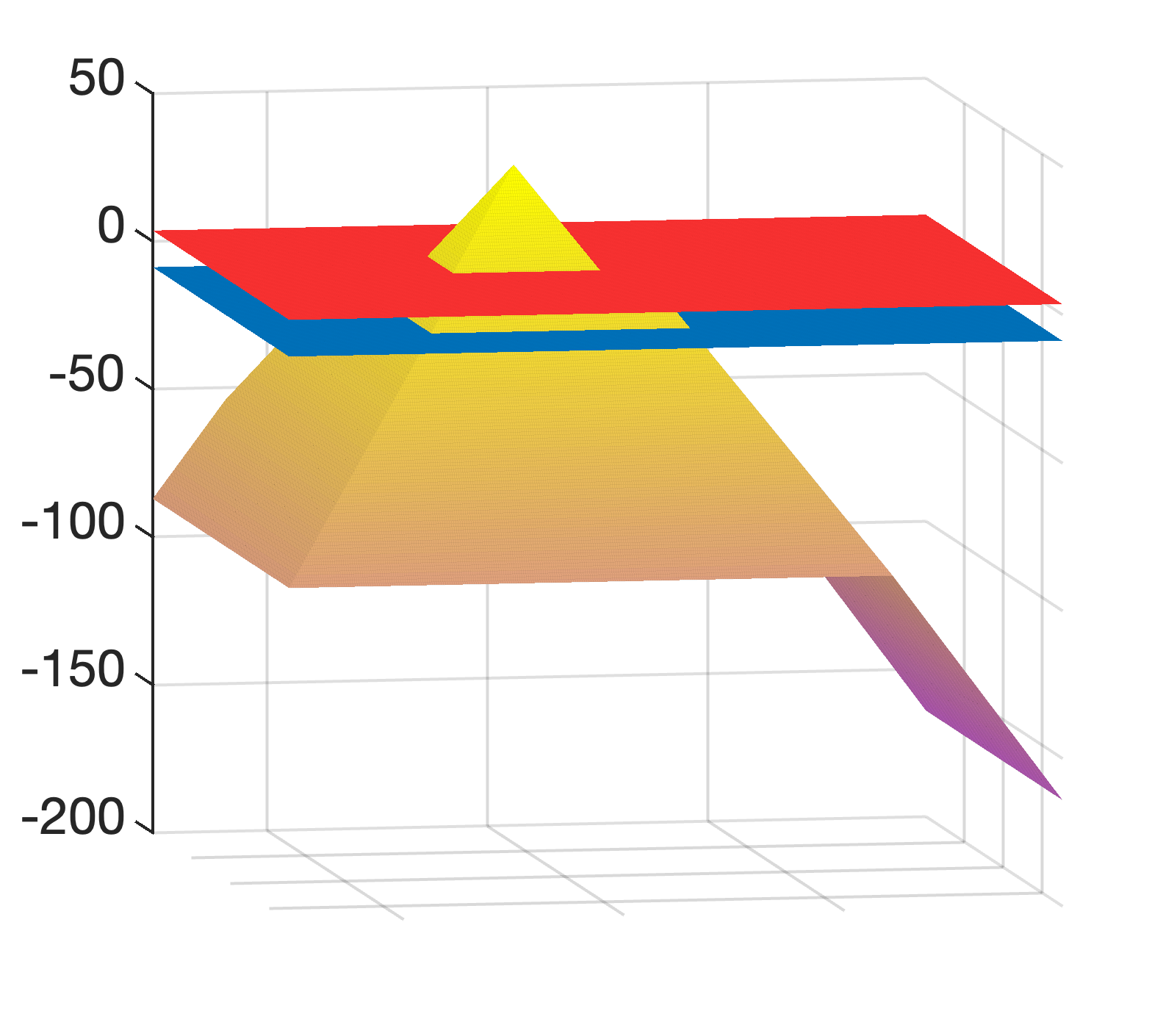}\\
		\includegraphics[width=0.24\textwidth]{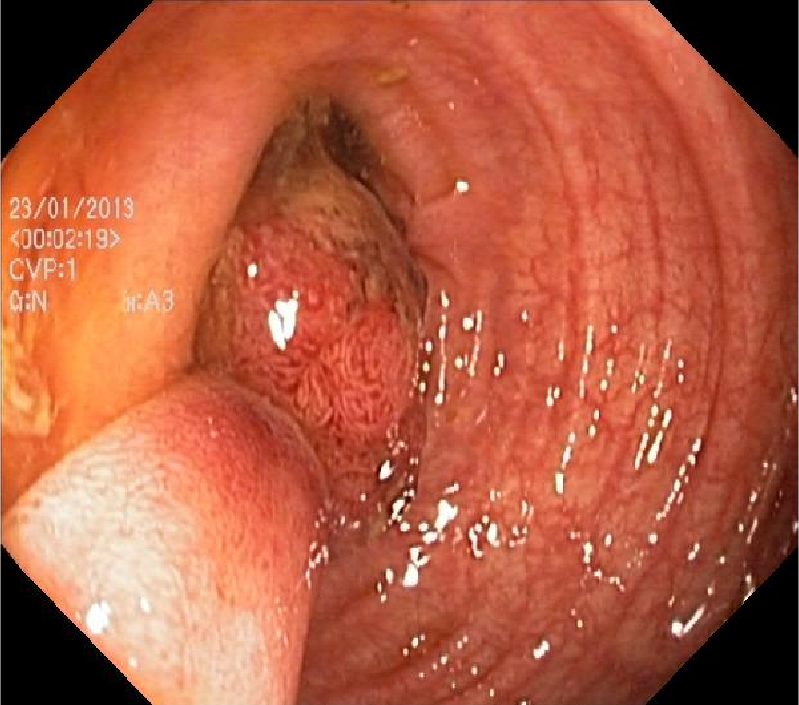}
		\includegraphics[width=0.24\textwidth]{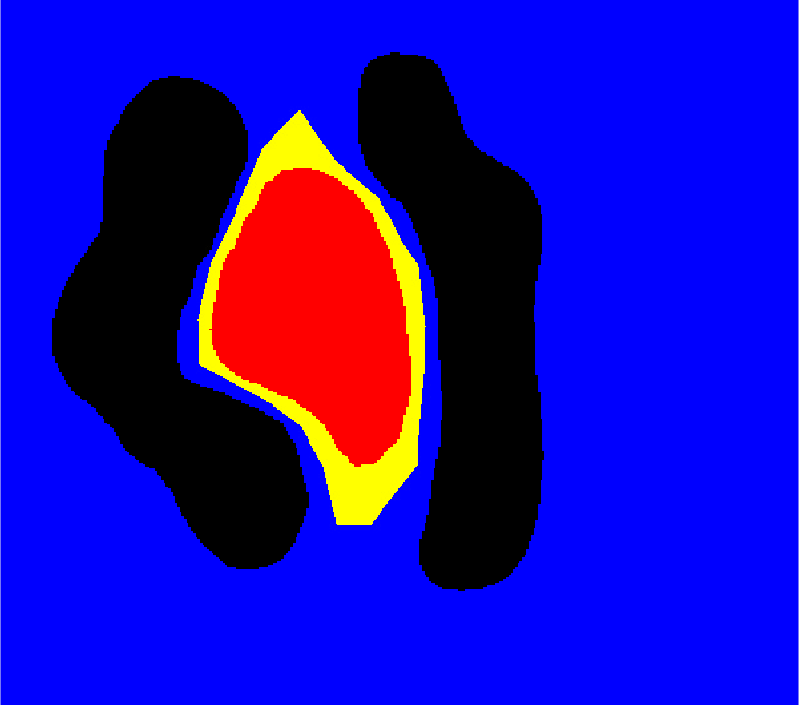}
		\includegraphics[width=0.24\textwidth]{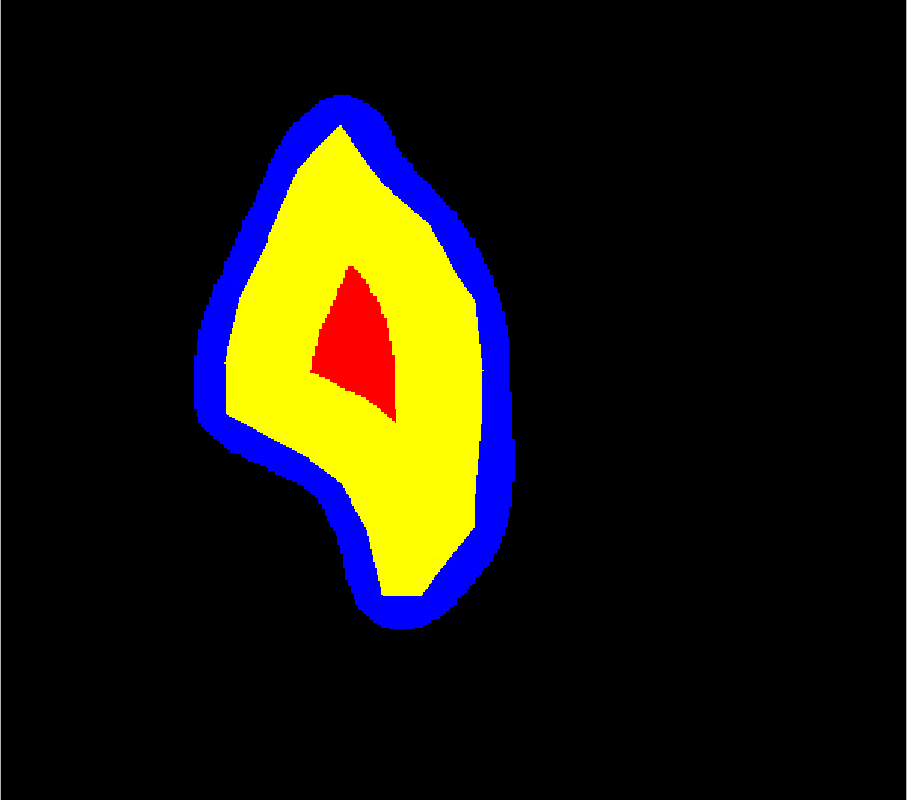}
		\includegraphics[width=0.24\textwidth]{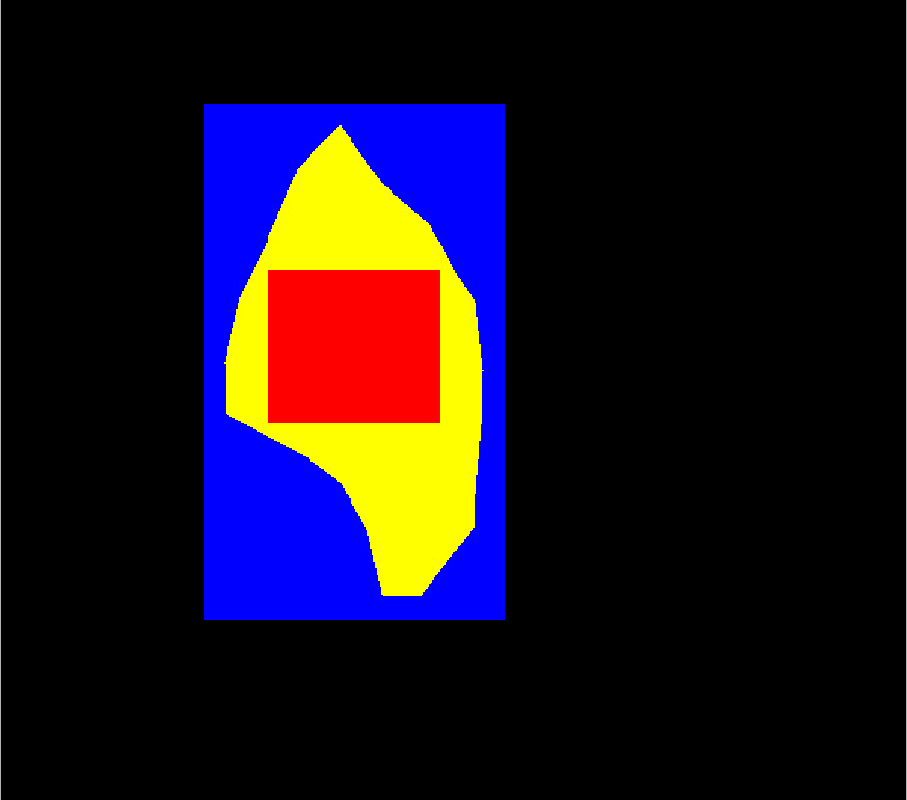}\\
		\vspace{0.5cm}
		\includegraphics[width=0.24\textwidth]{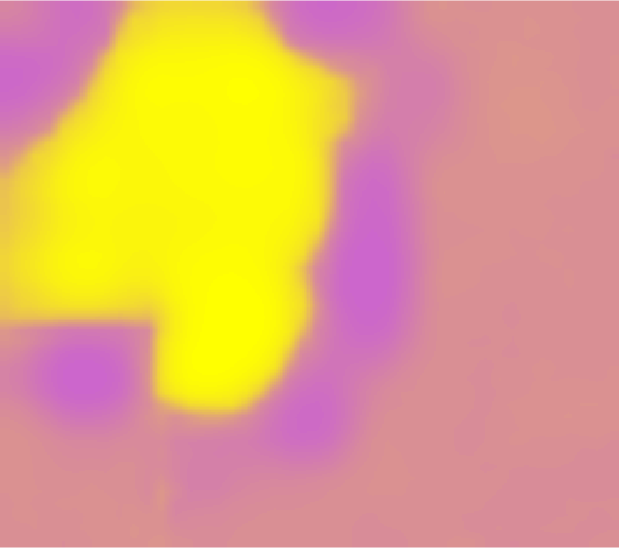}
		\includegraphics[width=0.24\textwidth]{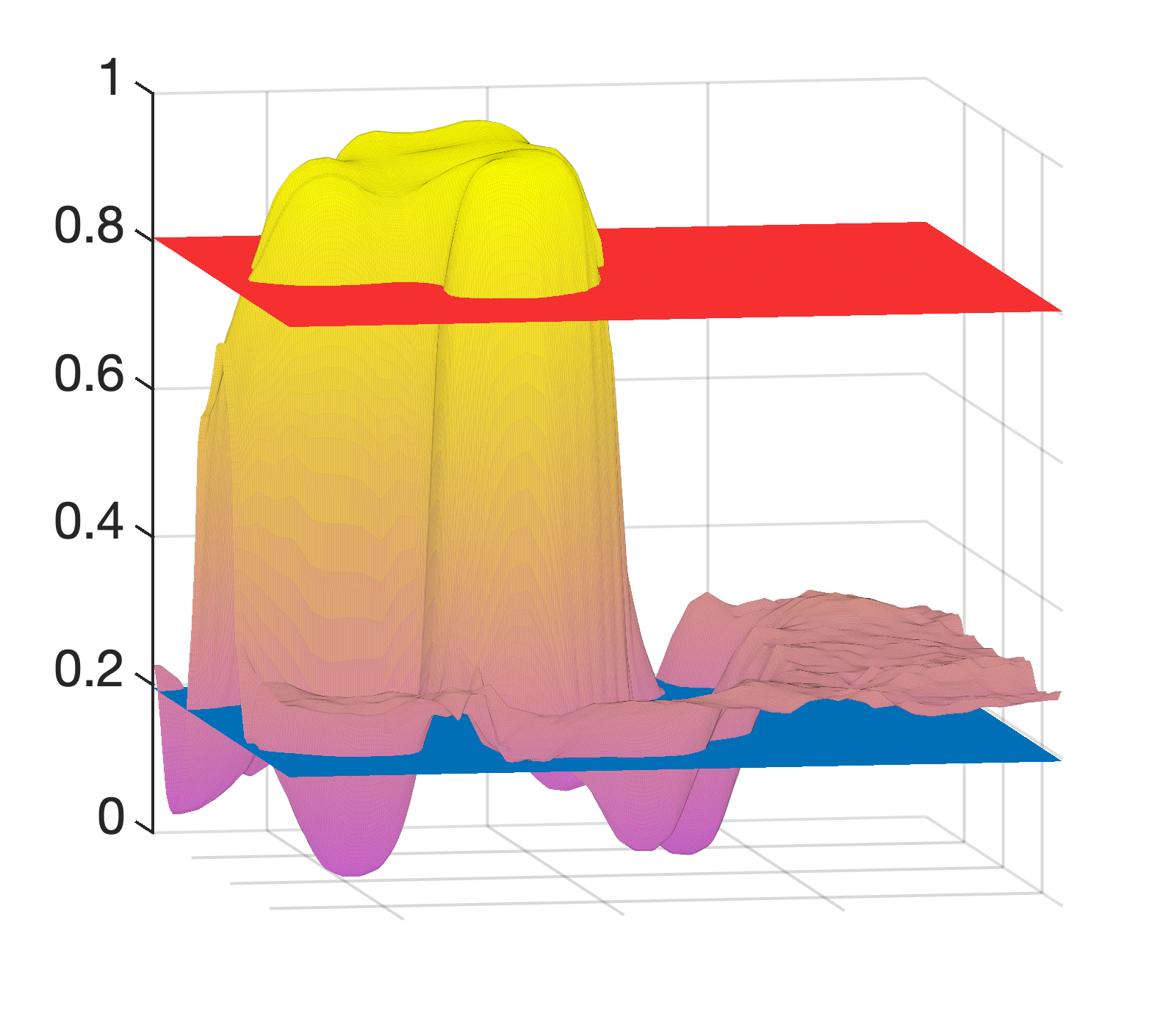}	\includegraphics[width=0.24\textwidth]{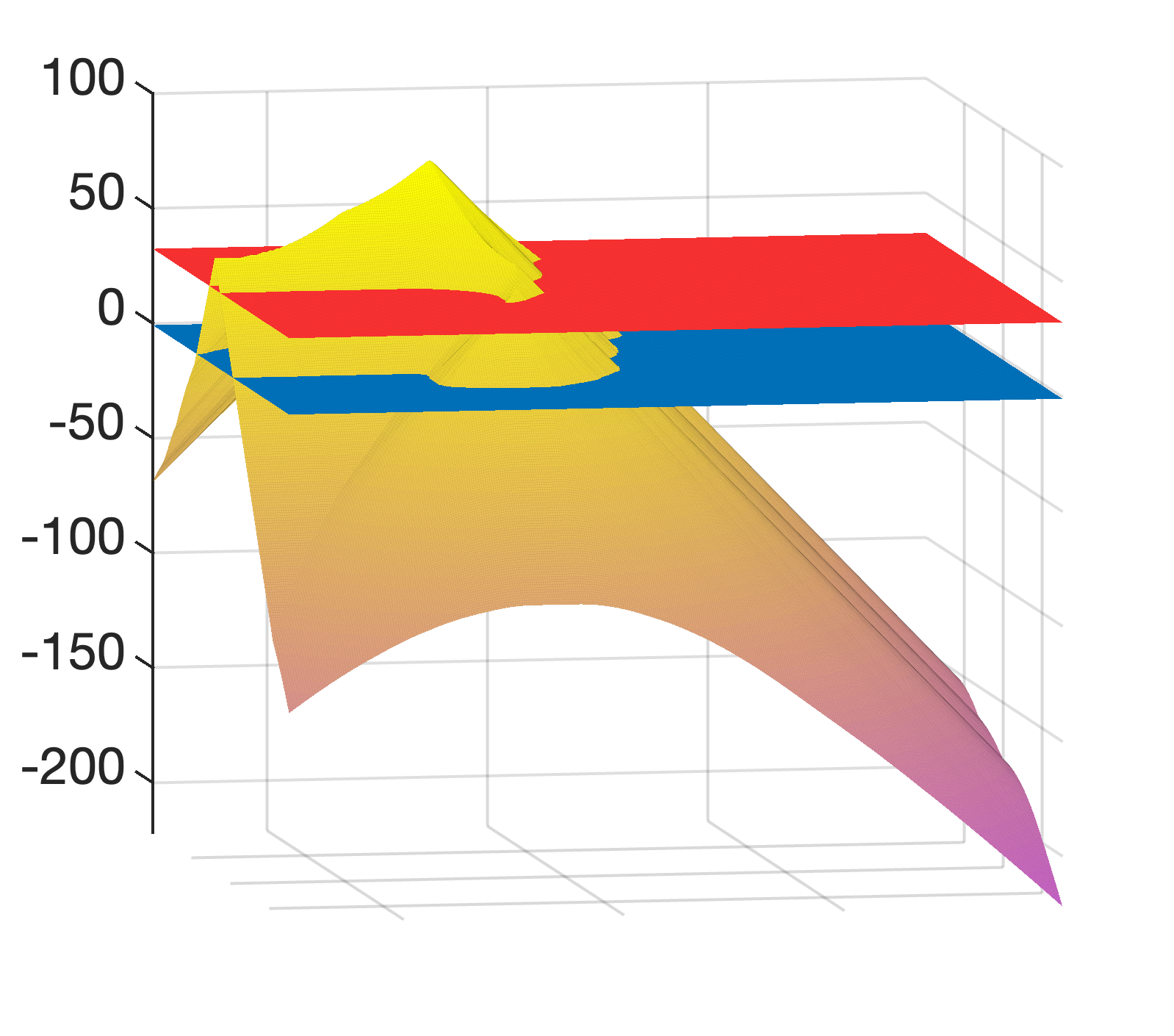}
		\includegraphics[width=0.24\textwidth]{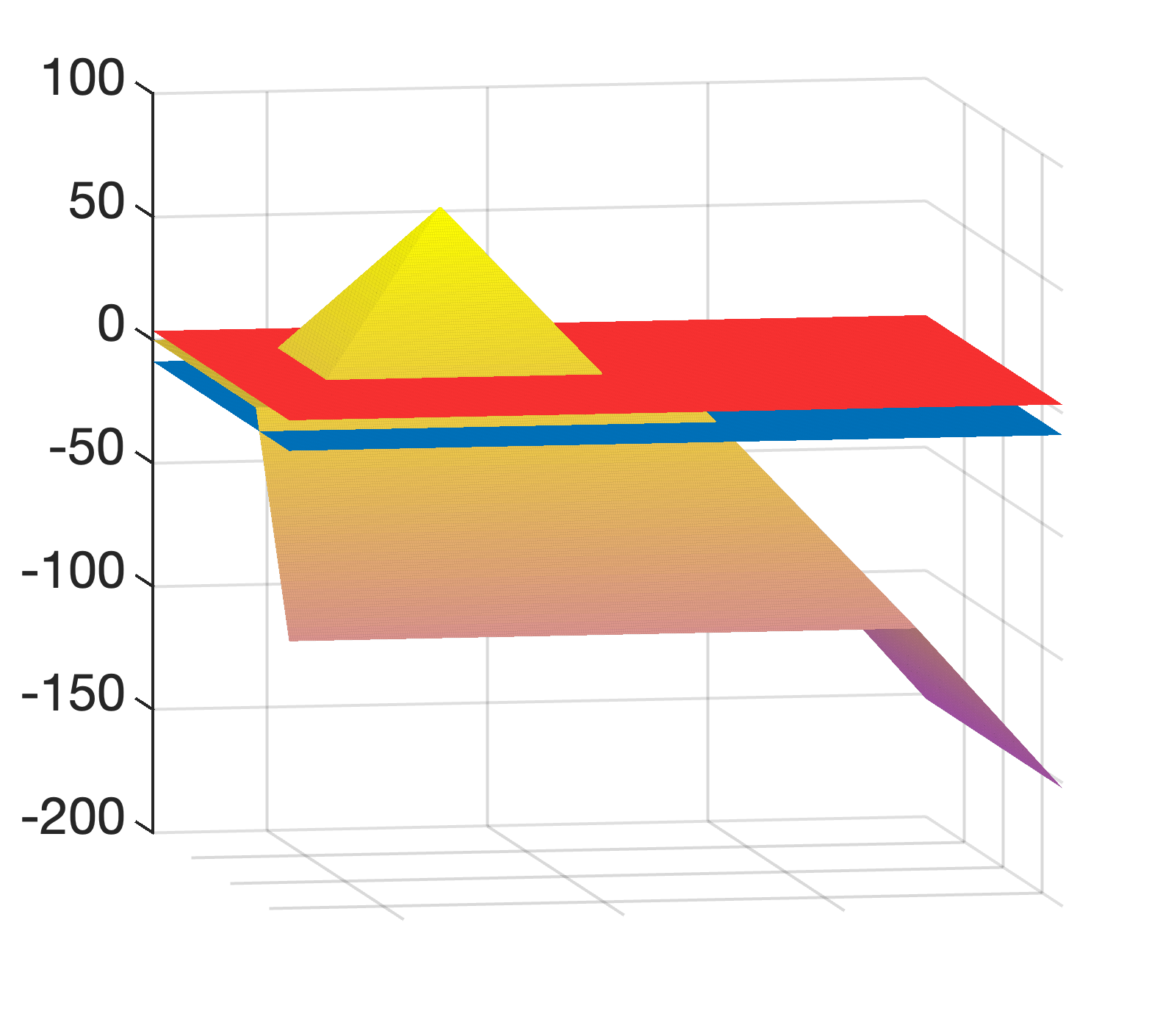}\\
		\includegraphics[width=0.24\textwidth]{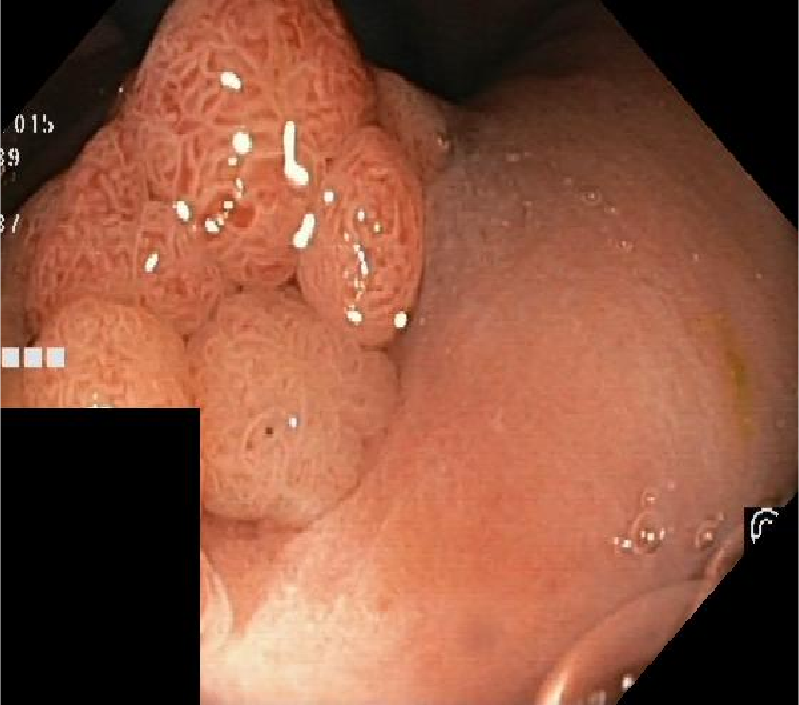}
		\includegraphics[width=0.24\textwidth]{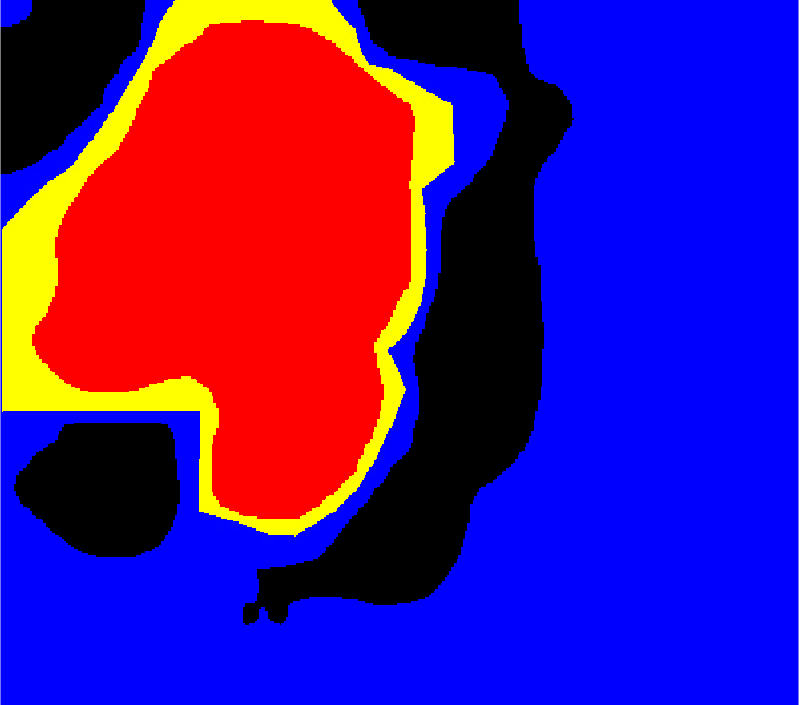}
		\includegraphics[width=0.24\textwidth]{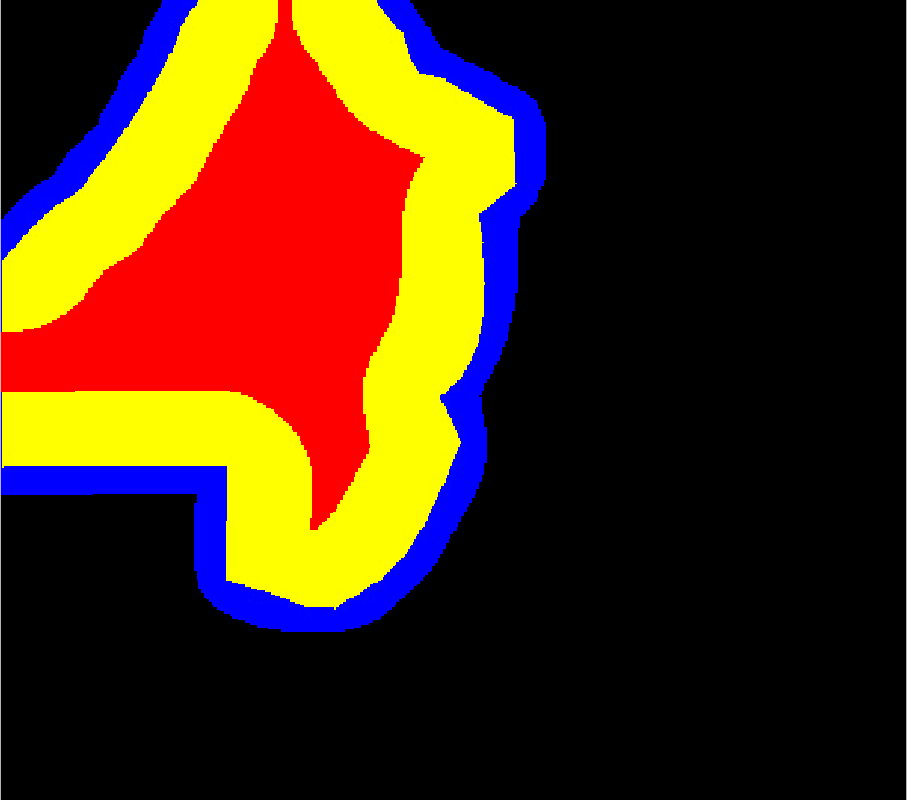}
		\includegraphics[width=0.24\textwidth]{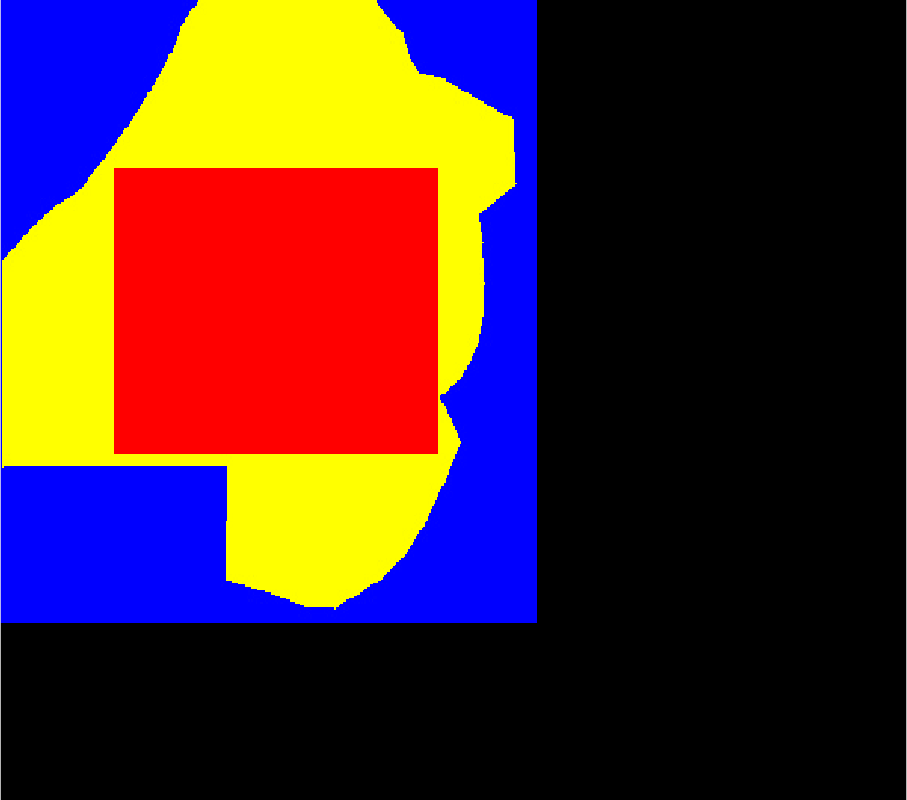}\\
		\vspace{0.5cm}
		\includegraphics[width=0.24\textwidth]{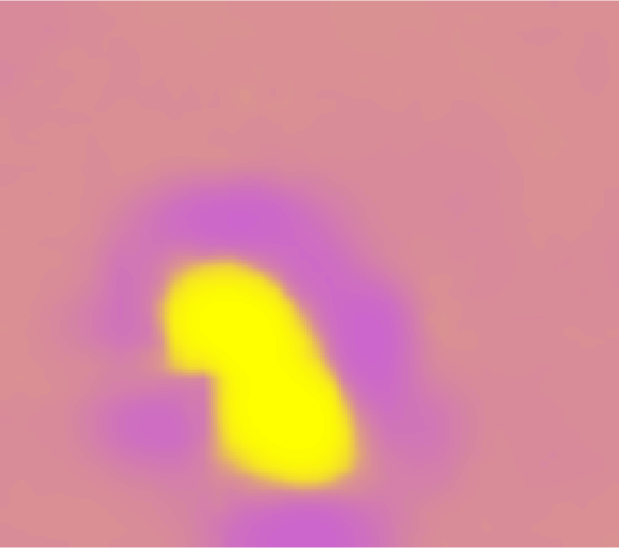}
		\includegraphics[width=0.24\textwidth]{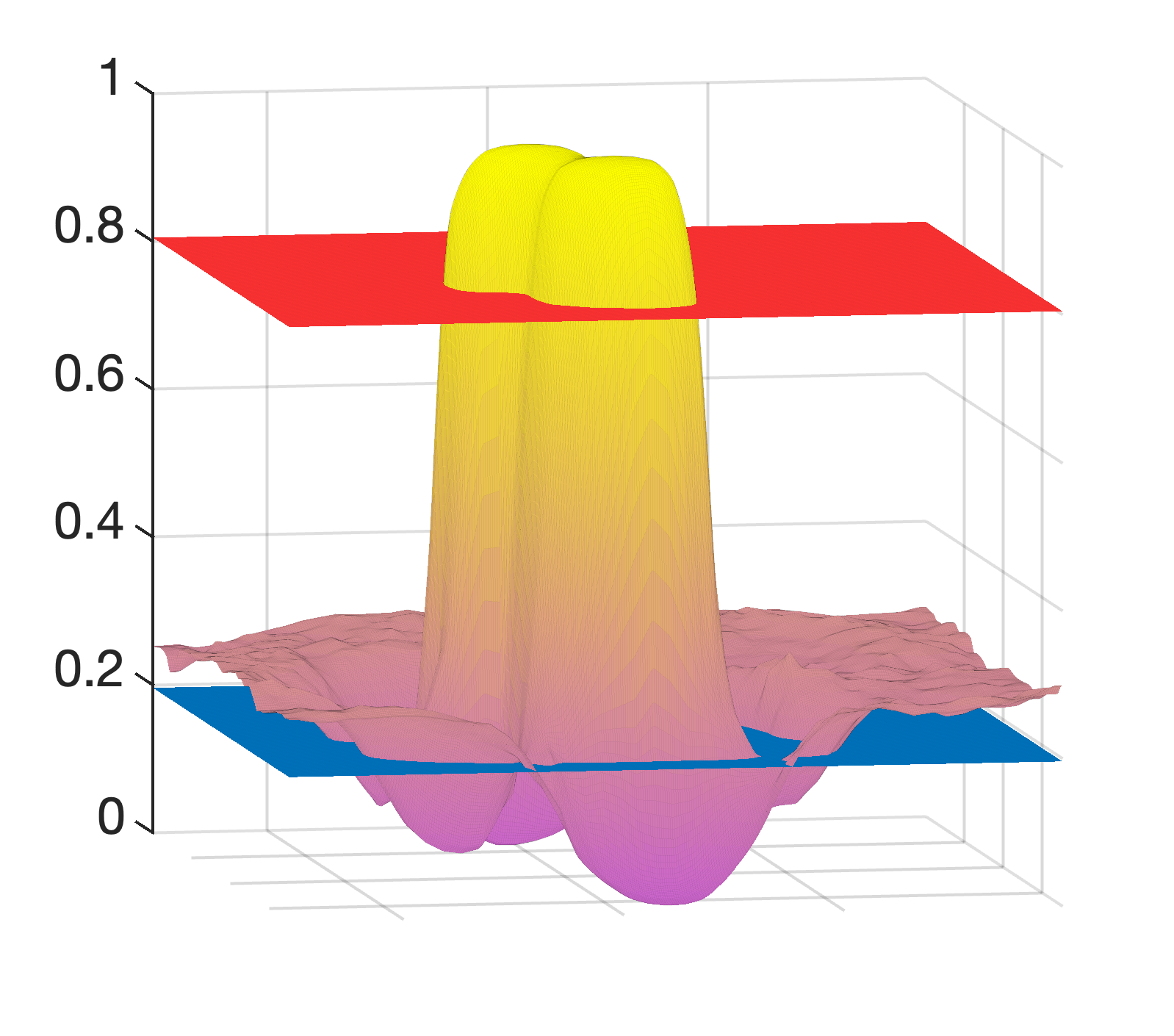}	\includegraphics[width=0.24\textwidth]{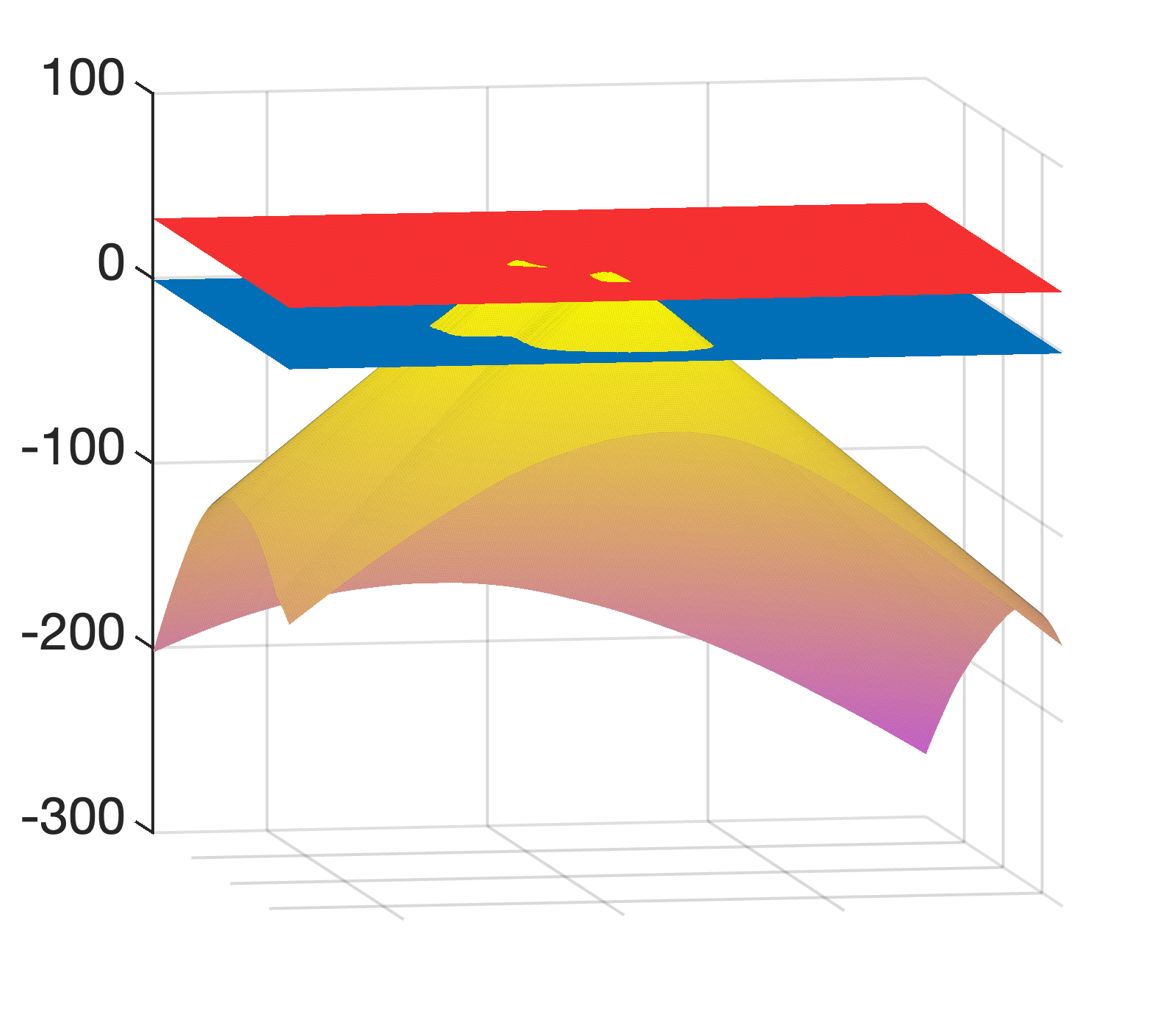}
		\includegraphics[width=0.24\textwidth]{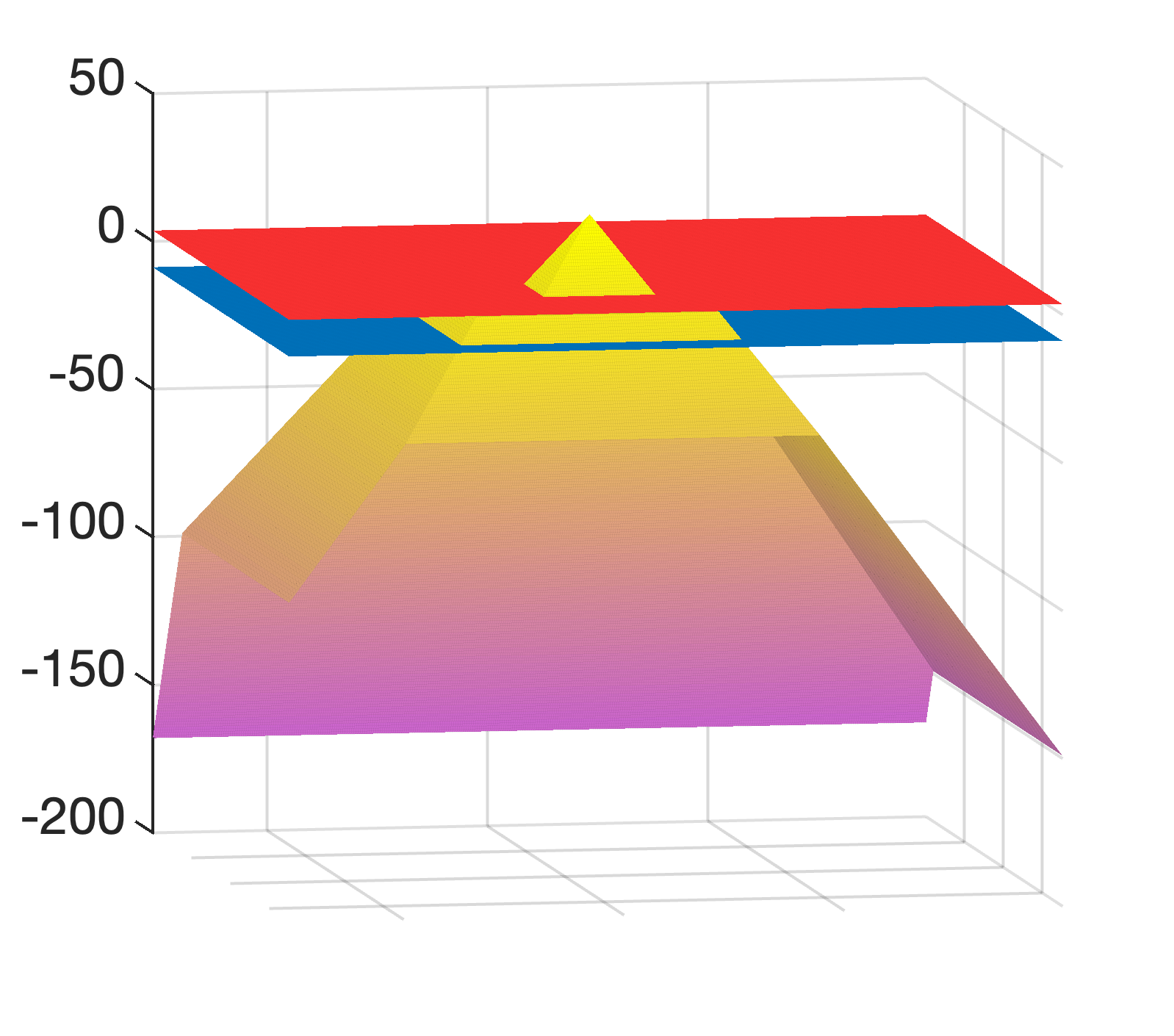}\\
		\includegraphics[width=0.24\textwidth]{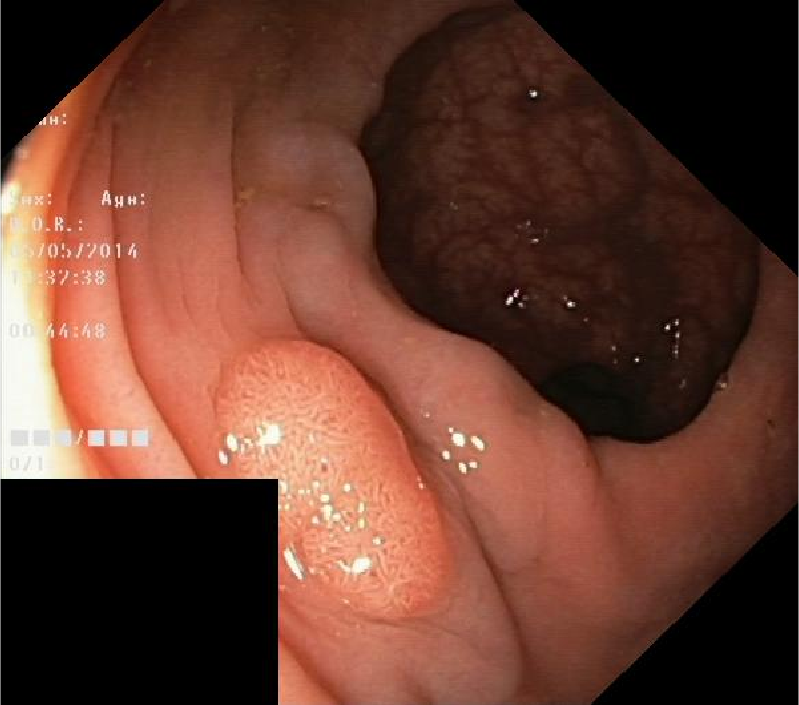}
		\includegraphics[width=0.24\textwidth]{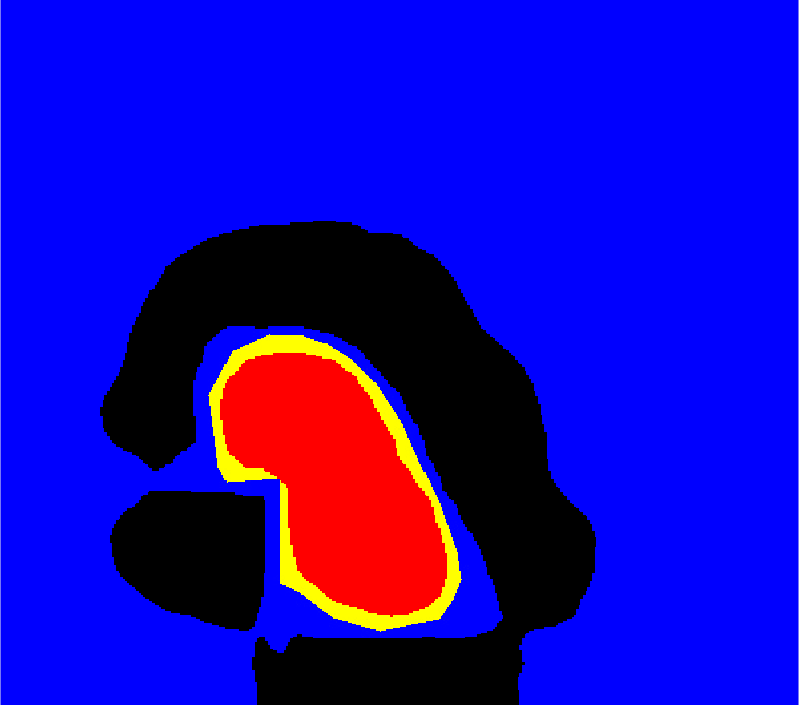}
		\includegraphics[width=0.24\textwidth]{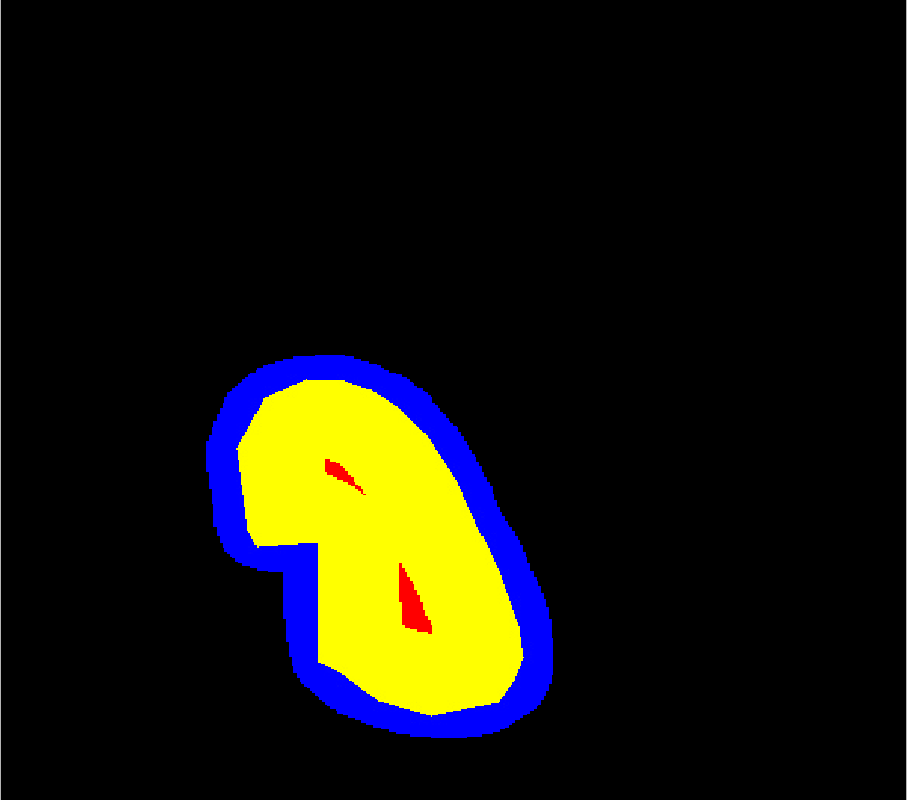}
		\includegraphics[width=0.24\textwidth]{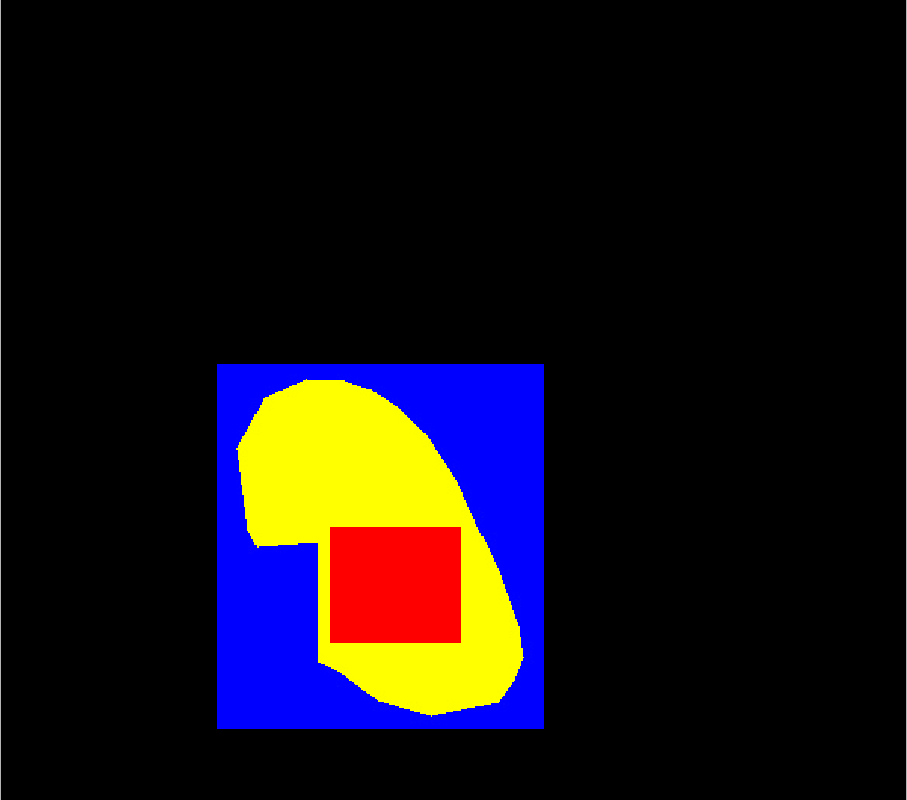}
	\end{center}
	\caption{Futher examples from the learning dataset. The layout of these figures is the same as for Figure \ref{fig:learning}.}
	\label{fig:learning3}
\end{figure}

\newpage
\subsection{Validition figures for the original and bounding box scores}\label{SS:furtherval}
\begin{figure}[h!]
	\begin{subfigure}{0.19\textwidth}
		\centering
		\includegraphics[width=\textwidth]{../figures/all_images/61.png}
		\label{fig:1}
	\end{subfigure}
	\begin{subfigure}{0.19\textwidth}
		\centering
		\includegraphics[width=\textwidth]{../figures/all_images/114.png}
		\label{fig:1}
	\end{subfigure}
	\begin{subfigure}{0.19\textwidth}
		\centering
		\includegraphics[width=\textwidth]{../figures/all_images/144.png}
		\label{fig:1}
	\end{subfigure}
	\begin{subfigure}{0.19\textwidth}
		\centering
		\includegraphics[width=\textwidth]{../figures/all_images/148.png}
		\label{fig:1}
	\end{subfigure}
	\begin{subfigure}{0.19\textwidth}
		\centering
		\includegraphics[width=\textwidth]{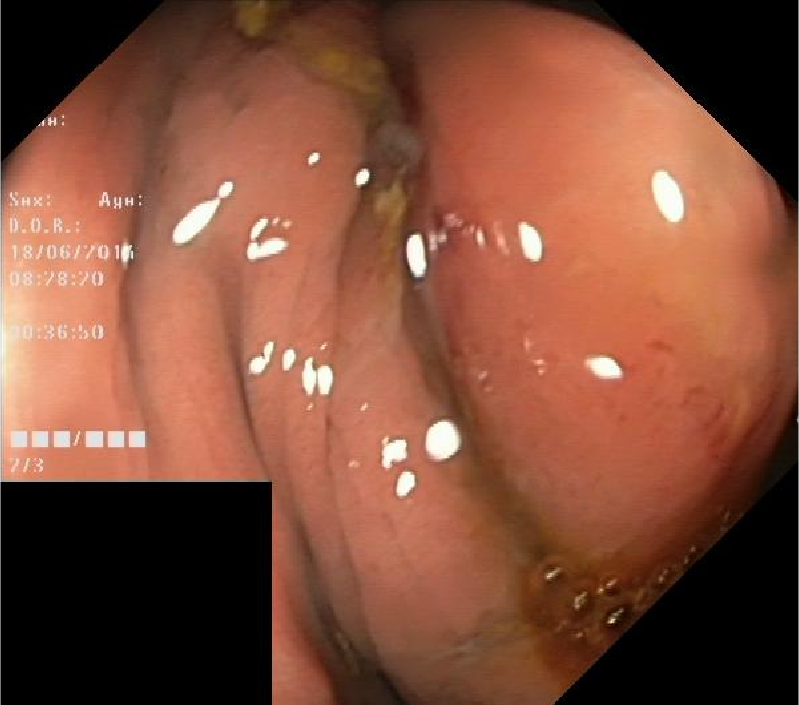}
		\label{fig:1}
	\end{subfigure}
	\vspace{-0.35cm}
	\\
	\begin{subfigure}{0.19\textwidth}
		\centering
		\includegraphics[width=\textwidth]{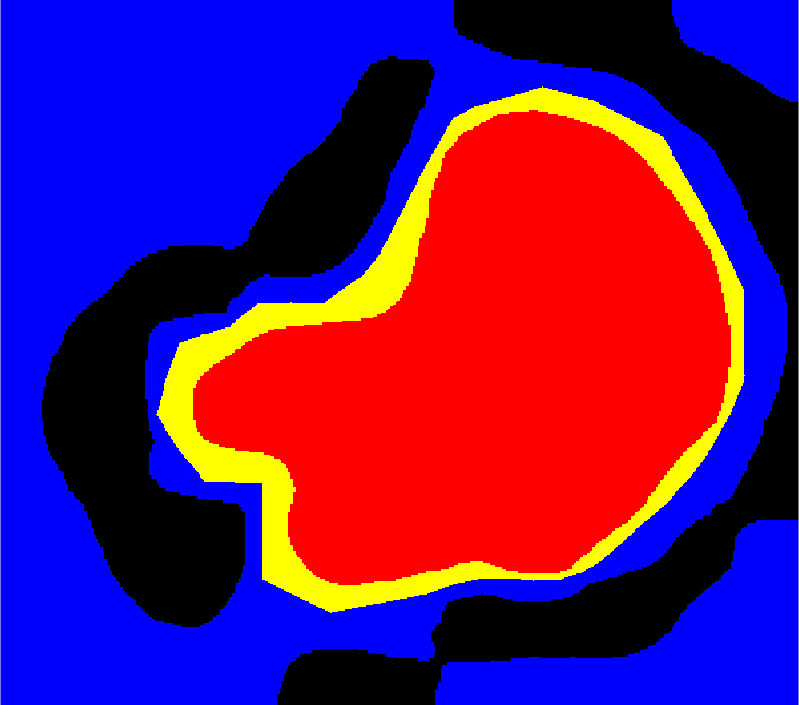}
		\label{fig:1}
	\end{subfigure}
	\begin{subfigure}{0.19\textwidth}
		\centering
		\includegraphics[width=\textwidth]{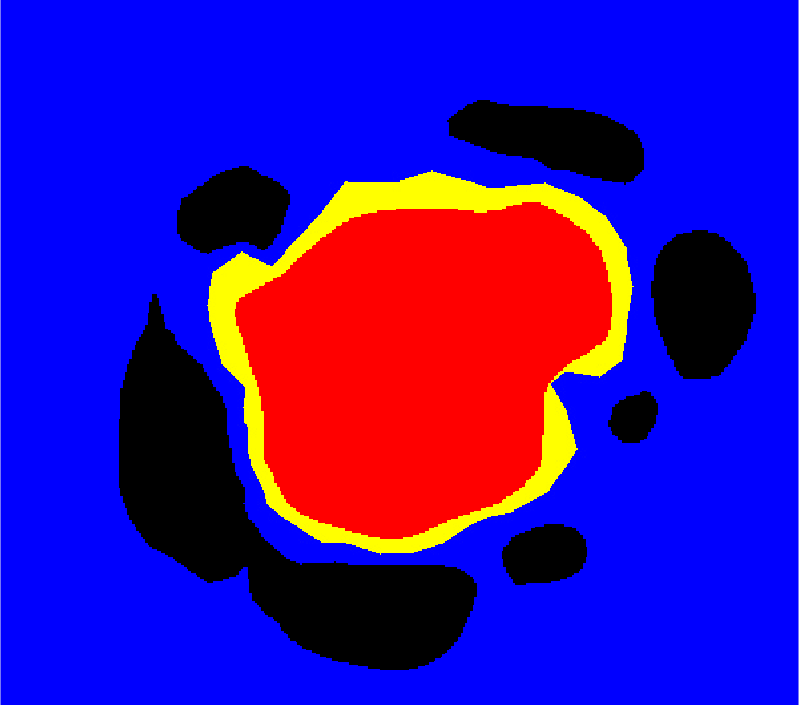}
		\label{fig:1}
	\end{subfigure}
	\begin{subfigure}{0.19\textwidth}
		\centering
		\includegraphics[width=\textwidth]{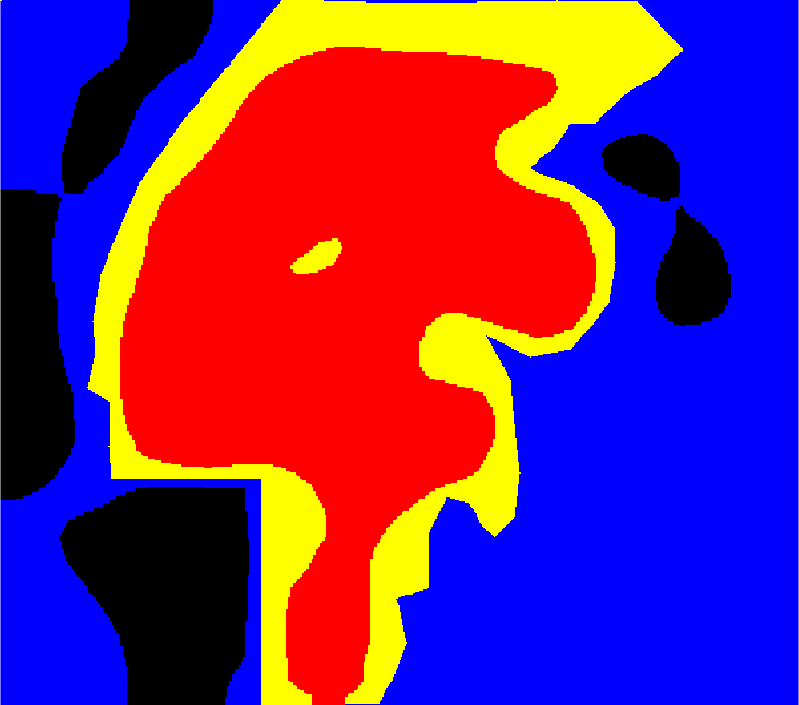}
		\label{fig:1}
	\end{subfigure}
	\begin{subfigure}{0.19\textwidth}
		\centering
		\includegraphics[width=\textwidth]{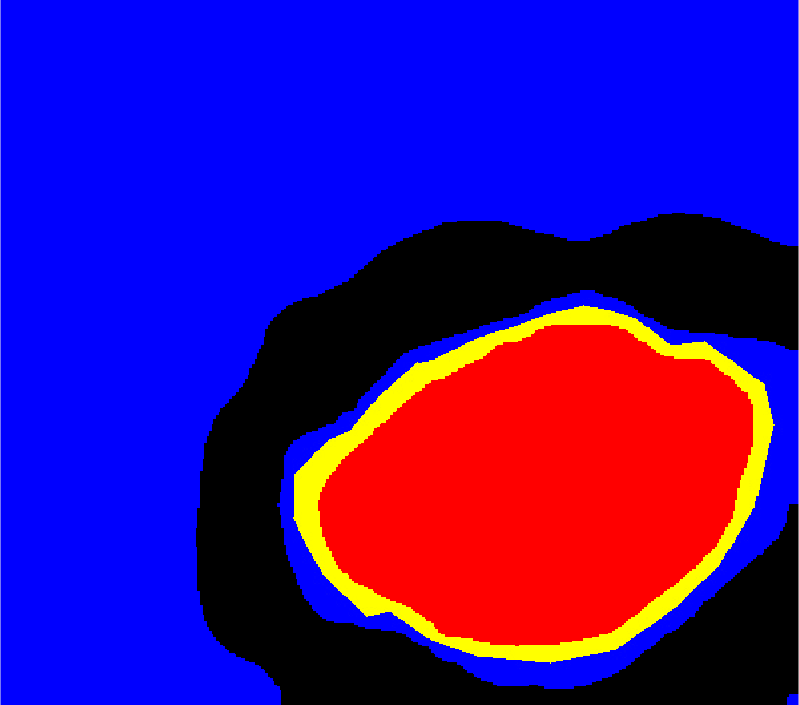}
		\label{fig:1}
	\end{subfigure}
	\begin{subfigure}{0.19\textwidth}
		\centering
		\includegraphics[width=\textwidth]{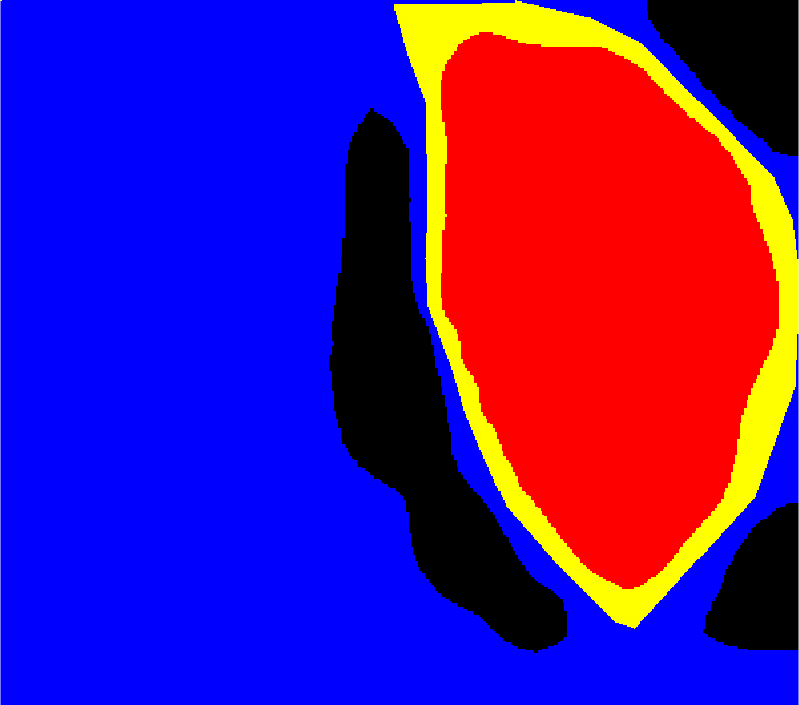}
		\label{fig:1}
	\end{subfigure}
	\vspace{-0.35cm}
	\\
		\begin{subfigure}{0.19\textwidth}
		\centering
		\includegraphics[width=\textwidth]{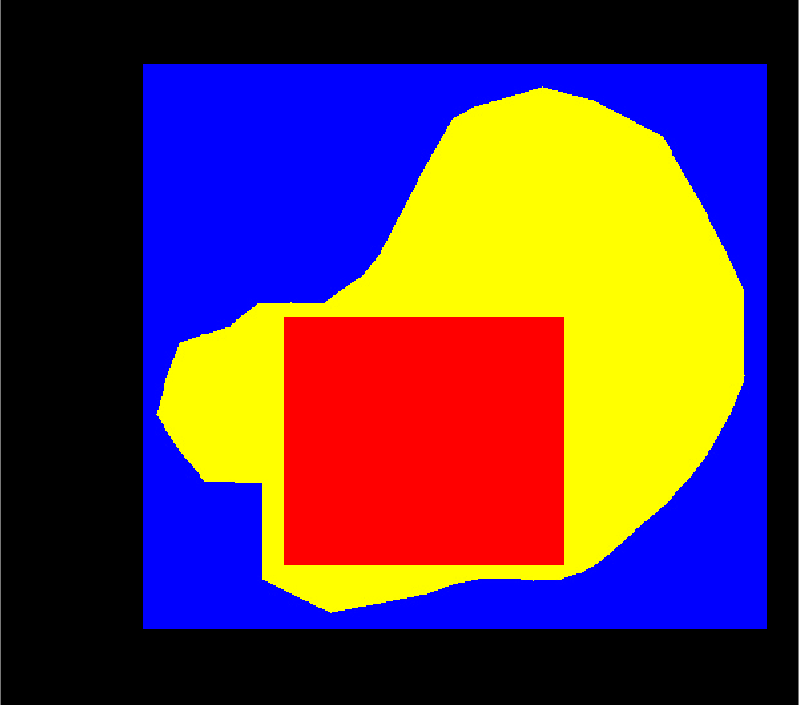}
		\label{fig:1}
	\end{subfigure}
	\begin{subfigure}{0.19\textwidth}
		\centering
		\includegraphics[width=\textwidth]{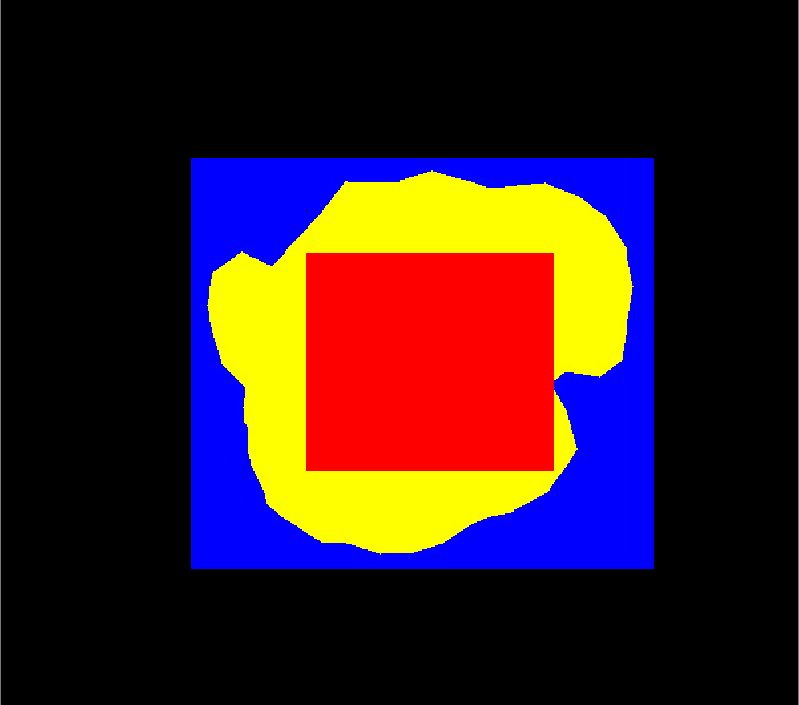}
		\label{fig:1}
	\end{subfigure}
	\begin{subfigure}{0.19\textwidth}
		\centering
		\includegraphics[width=\textwidth]{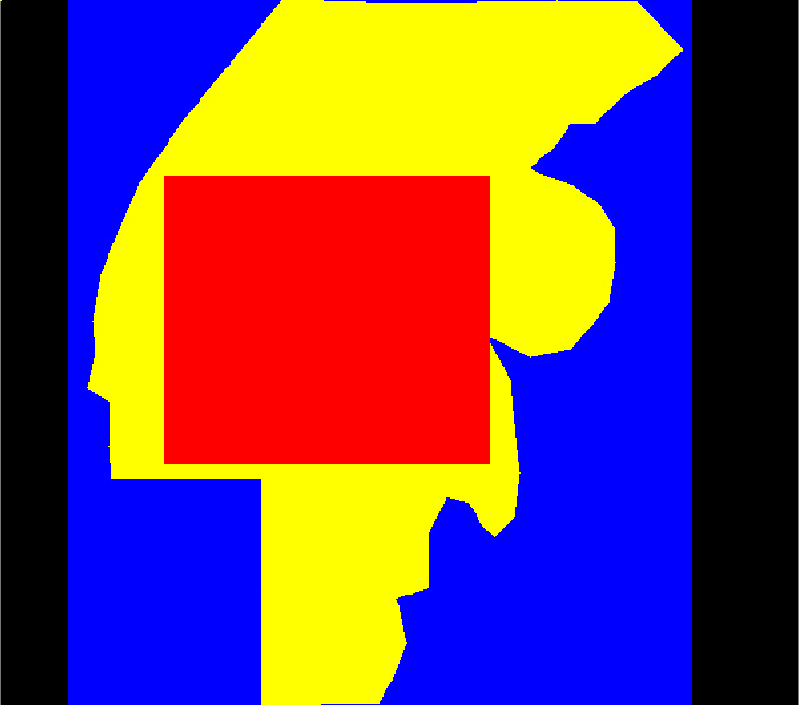}
		\label{fig:1}
	\end{subfigure}
	\begin{subfigure}{0.19\textwidth}
		\centering
		\includegraphics[width=\textwidth]{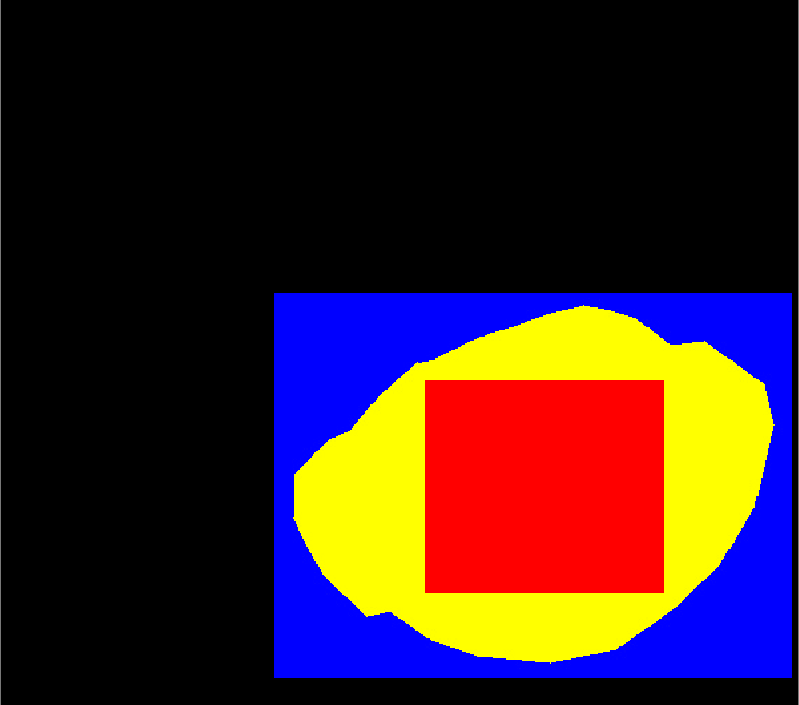}
		\label{fig:1}
	\end{subfigure}
	\begin{subfigure}{0.19\textwidth}
		\centering
		\includegraphics[width=\textwidth]{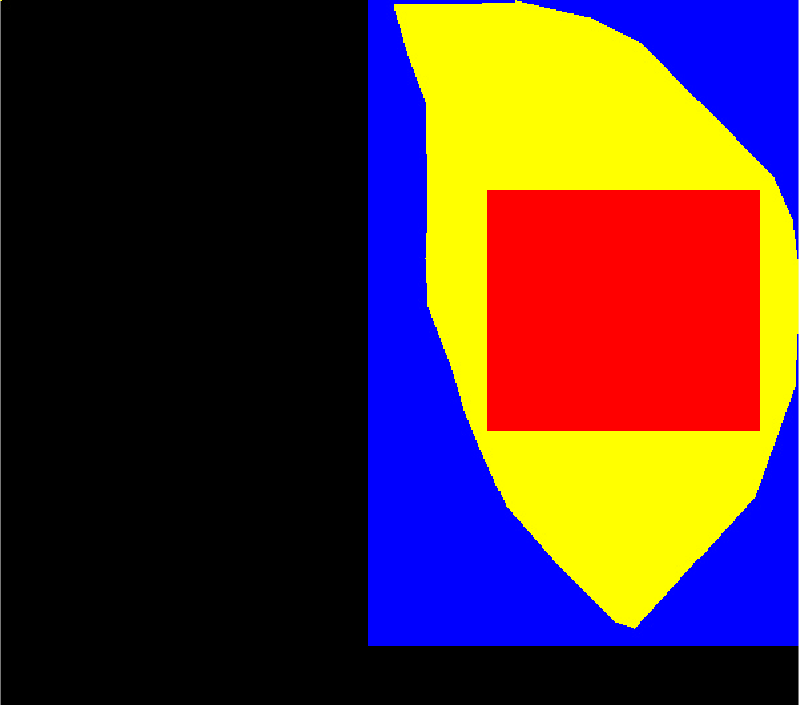}
		\label{fig:1}
	\end{subfigure}
	\vspace{-0.35cm}
	\\
	\begin{subfigure}{0.19\textwidth}
		\centering
		\includegraphics[width=\textwidth]{../figures/all_images/7.png}
		\label{fig:1}
	\end{subfigure}
	\begin{subfigure}{0.19\textwidth}
		\centering
		\includegraphics[width=\textwidth]{../figures/all_images/211.png}
		\label{fig:1}
	\end{subfigure}
	\begin{subfigure}{0.19\textwidth}
		\centering
		\includegraphics[width=\textwidth]{../figures/all_images/1062.png}
		\label{fig:1}
	\end{subfigure}
	\begin{subfigure}{0.19\textwidth}
		\centering
		\includegraphics[width=\textwidth]{../figures/all_images/398.png}
		\label{fig:1}
	\end{subfigure}
	\begin{subfigure}{0.19\textwidth}
		\centering
		\includegraphics[width=\textwidth]{../figures/all_images/269.png}
		\label{fig:1}
	\end{subfigure}
	\vspace{-0.35cm}
	\\
	\begin{subfigure}{0.19\textwidth}
		\centering
		\includegraphics[width=\textwidth]{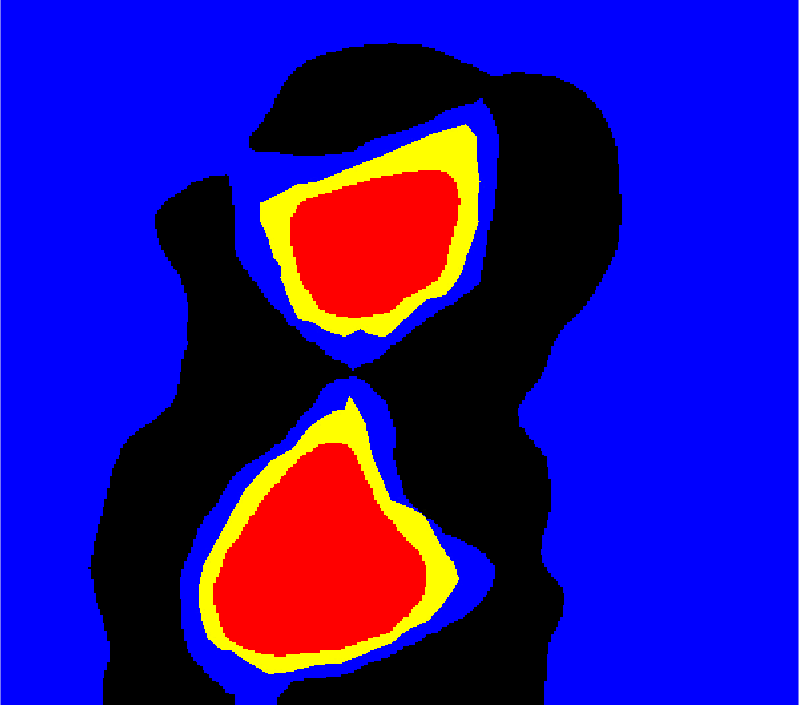}
		\label{fig:1}
	\end{subfigure}
	\begin{subfigure}{0.19\textwidth}
		\centering
		\includegraphics[width=\textwidth]{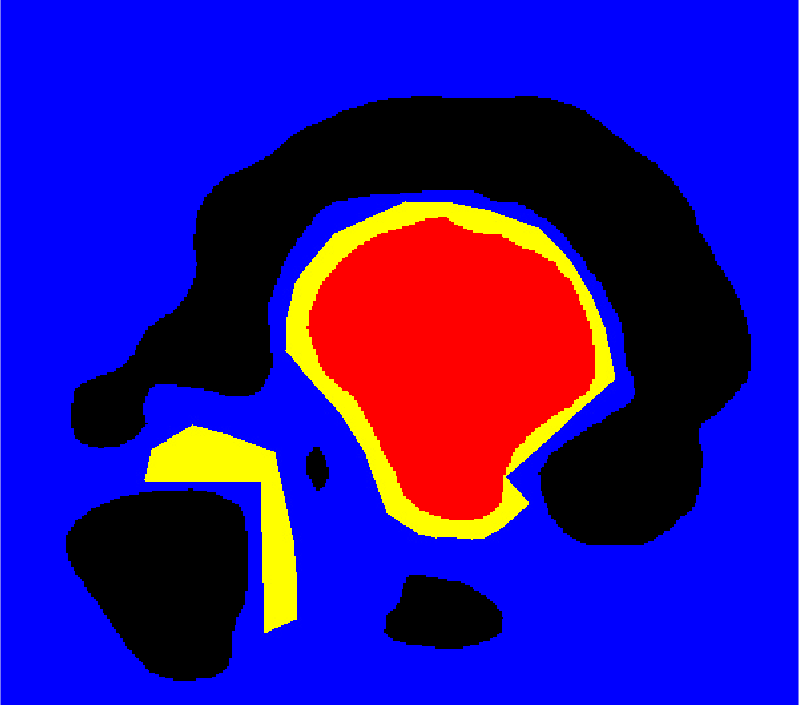}
		\label{fig:1}
	\end{subfigure}
	\begin{subfigure}{0.19\textwidth}
		\centering
		\includegraphics[width=\textwidth]{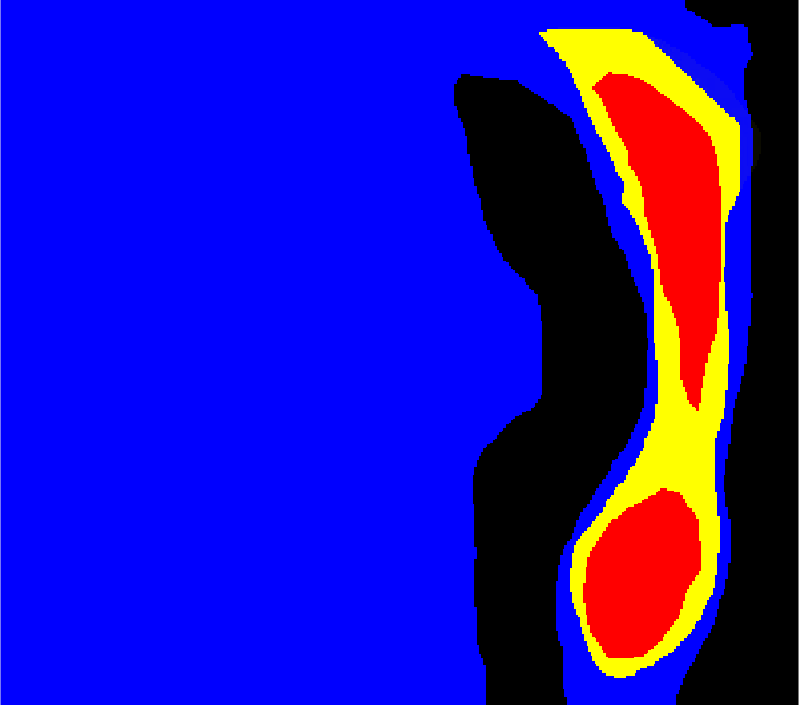}
		\label{fig:1}
	\end{subfigure}
	\begin{subfigure}{0.19\textwidth}
		\centering
		\includegraphics[width=\textwidth]{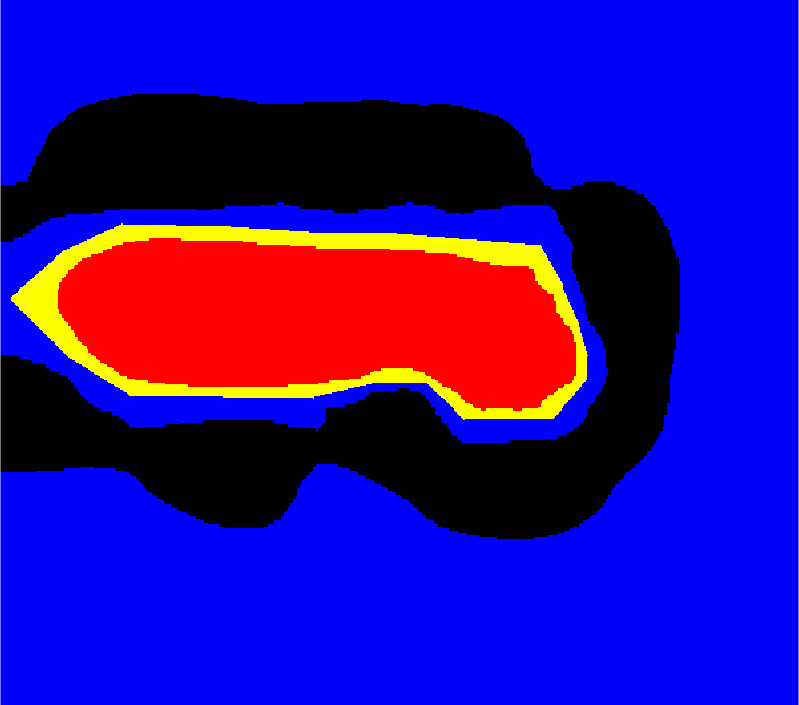}
		\label{fig:1}
	\end{subfigure}
	\begin{subfigure}{0.19\textwidth}
		\centering
		\includegraphics[width=\textwidth]{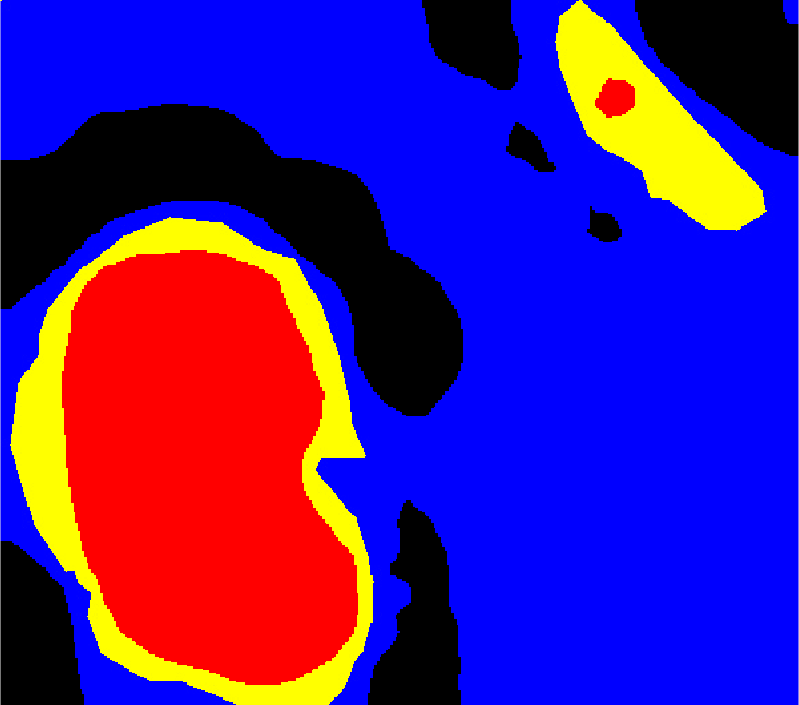}
		\label{fig:1}
	\end{subfigure}
	\vspace{-0.35cm}
	\\
	\begin{subfigure}{0.19\textwidth}
		\centering
		\includegraphics[width=\textwidth]{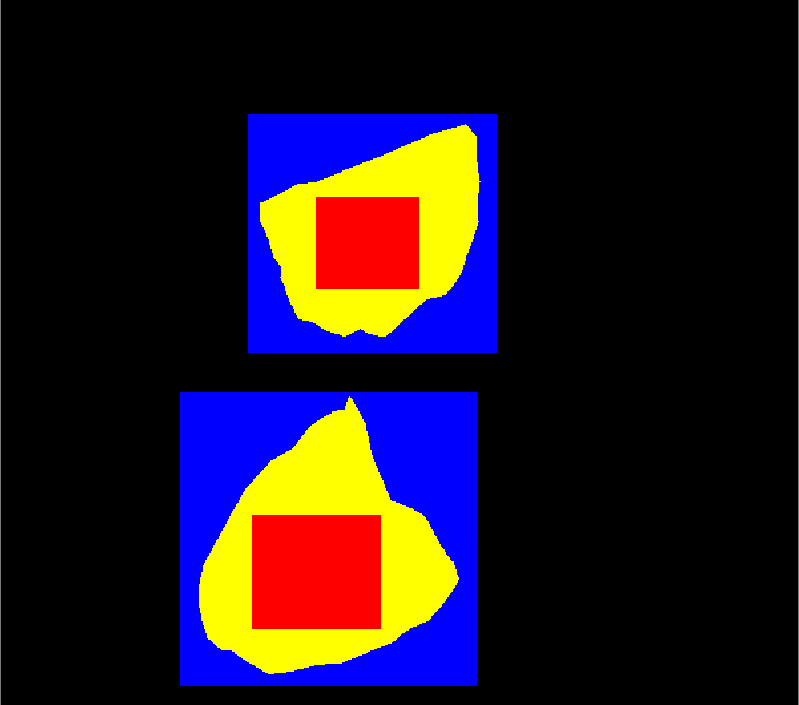}
		\label{fig:1}
	\end{subfigure}
	\begin{subfigure}{0.19\textwidth}
		\centering
		\includegraphics[width=\textwidth]{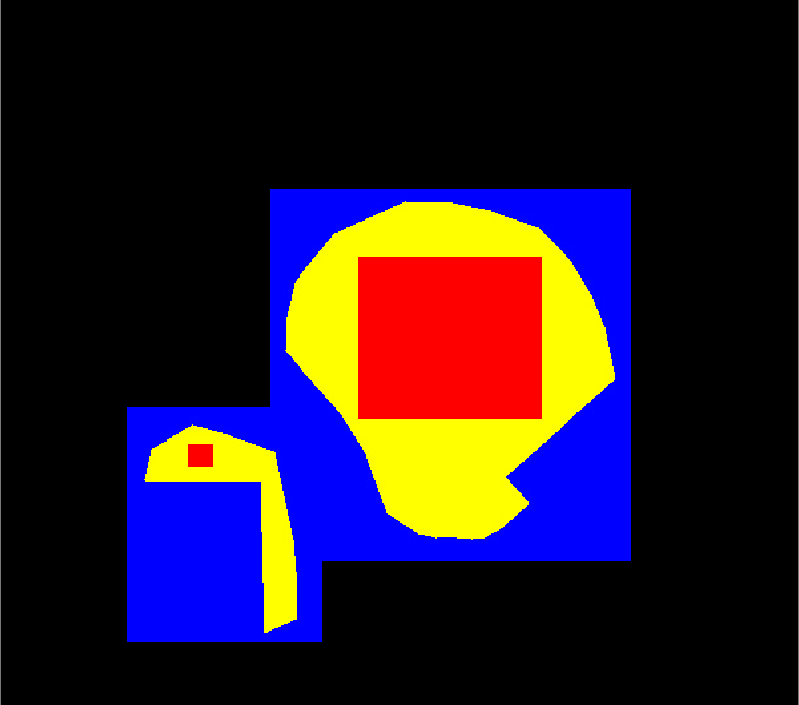}
		\label{fig:1}
	\end{subfigure}
	\begin{subfigure}{0.19\textwidth}
		\centering
		\includegraphics[width=\textwidth]{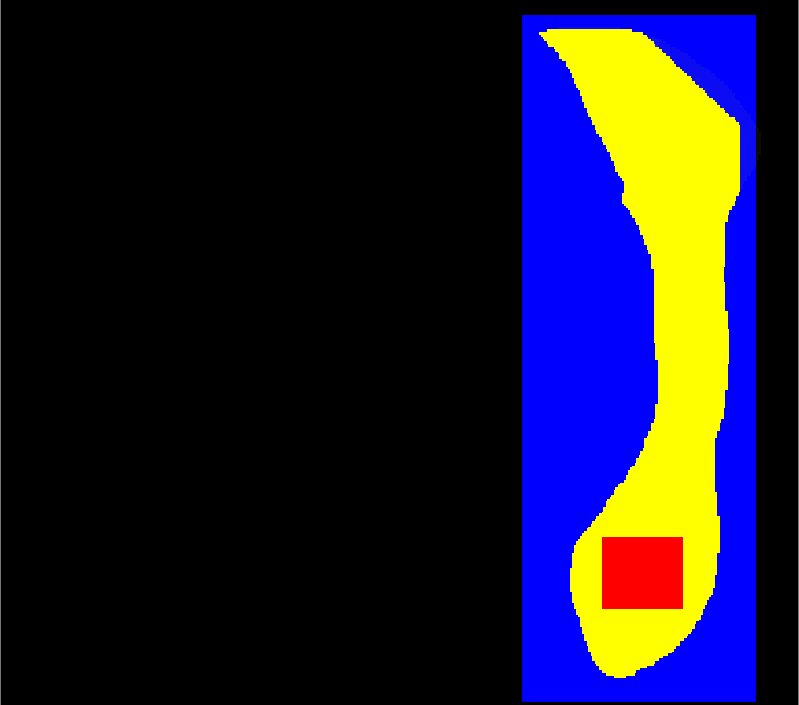}
		\label{fig:1}
	\end{subfigure}
	\begin{subfigure}{0.19\textwidth}
		\centering
		\includegraphics[width=\textwidth]{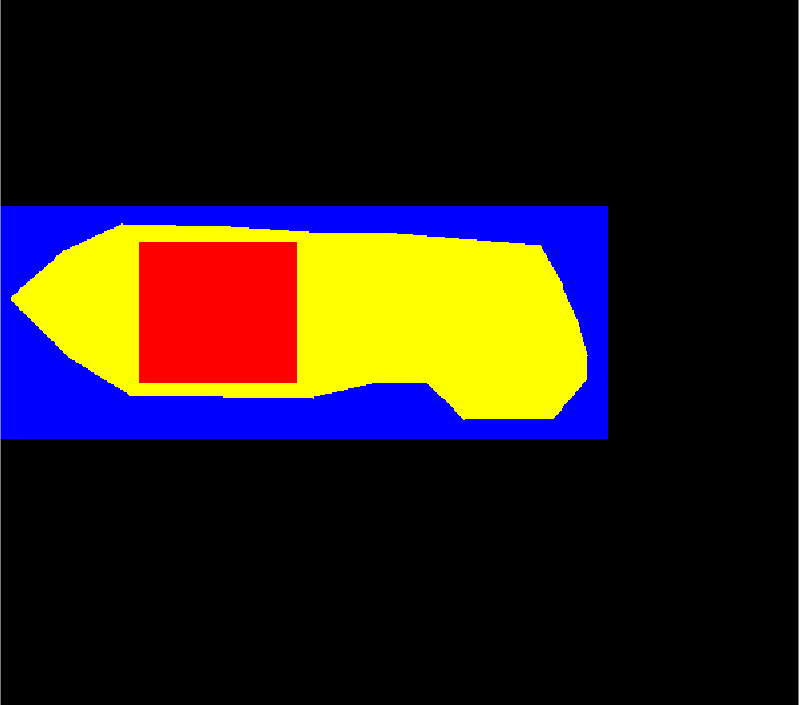}
		\label{fig:1}
	\end{subfigure}
	\begin{subfigure}{0.19\textwidth}
		\centering
		\includegraphics[width=\textwidth]{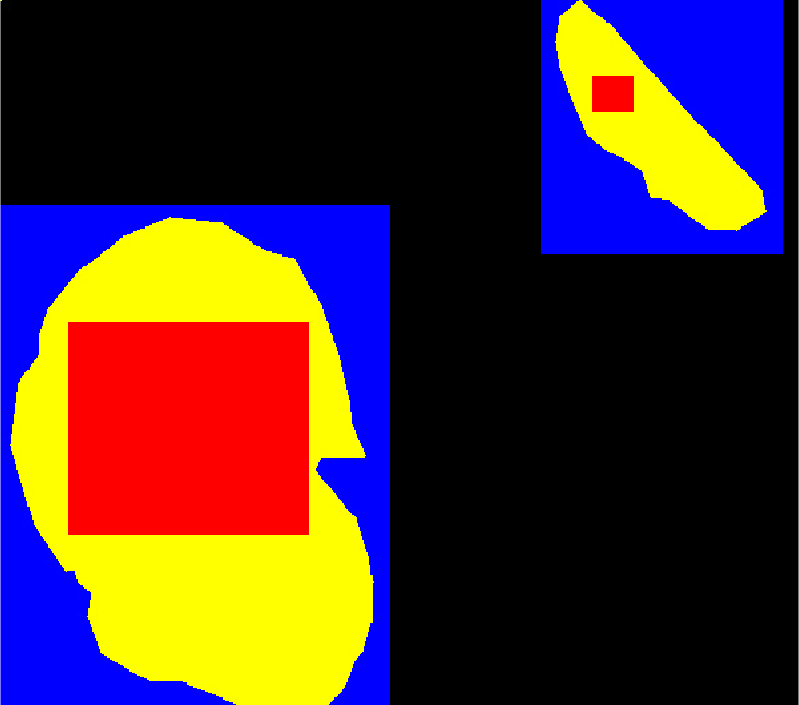}
		\label{fig:1}
	\end{subfigure}
	\label{fig:grid}
	\caption{Conformal confidence sets for the polyps data examples from Figure \ref{fig:res} for alternative scores. In each set of panels the confidence obtained from using the original scores are shown in the middle row and those obtained from the bounding box scores are shown in the bottom row. As observed on the learning dataset the outer sets obtained when using the original scores are very large and uninformative.}\label{fig:polpysex}
\end{figure}
\newpage
\subsection{Additional validition figures}
\begin{figure}[h!]
	\begin{subfigure}{0.19\textwidth}
		\centering
		\includegraphics[width=\textwidth]{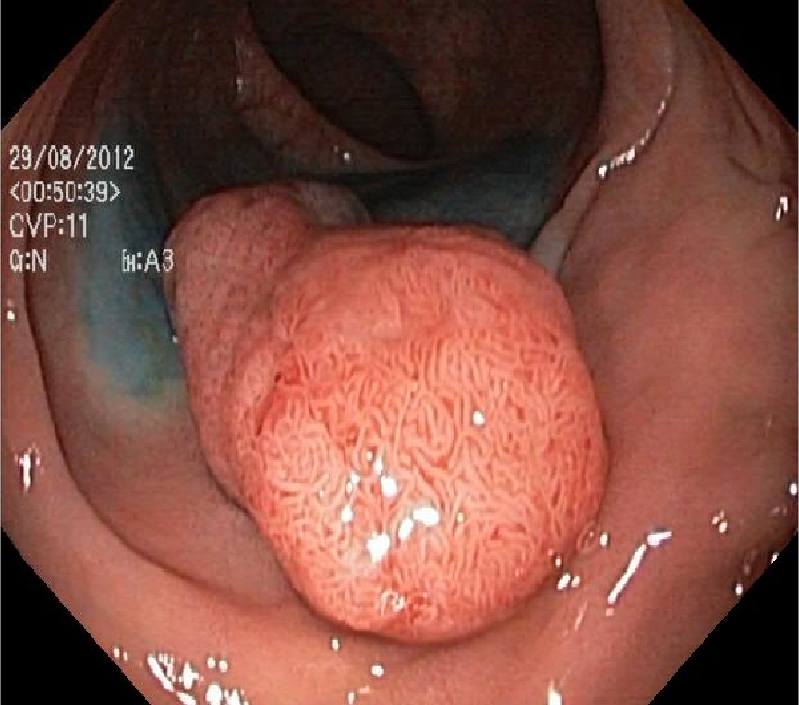}
		\label{fig:1}
	\end{subfigure}
	\begin{subfigure}{0.19\textwidth}
		\centering
		\includegraphics[width=\textwidth]{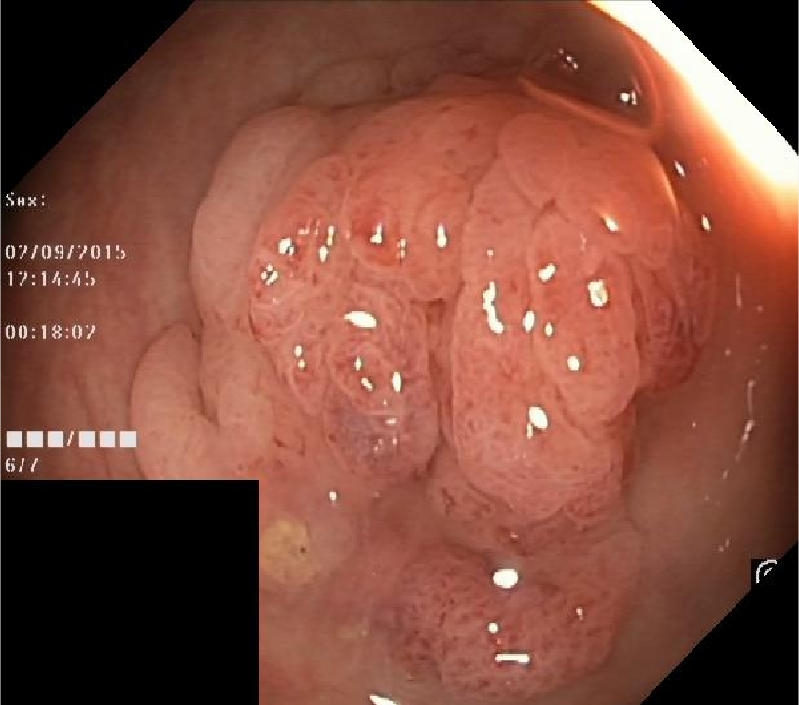}
		\label{fig:1}
	\end{subfigure}
	\begin{subfigure}{0.19\textwidth}
		\centering
		\includegraphics[width=\textwidth]{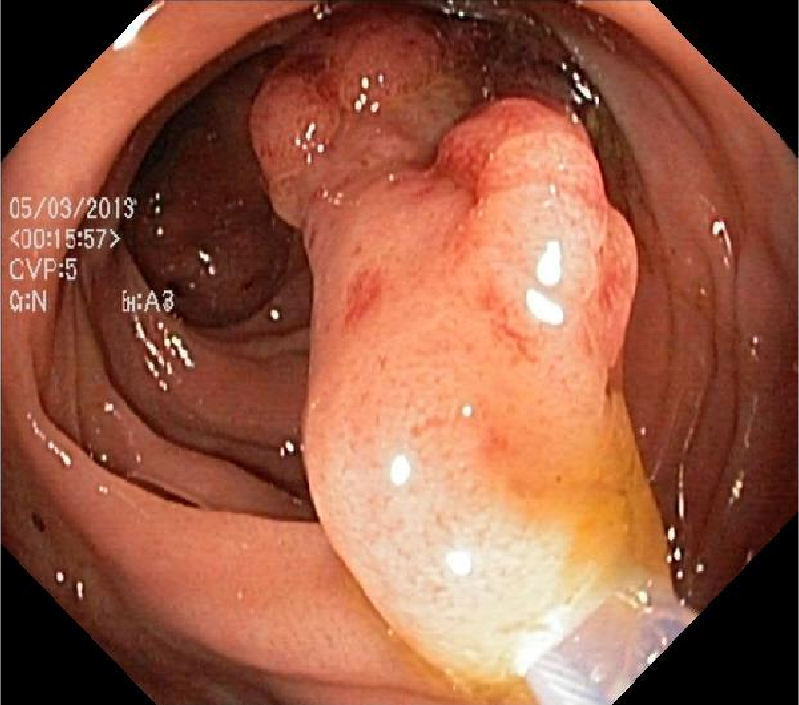}
		\label{fig:1}
	\end{subfigure}
	\begin{subfigure}{0.19\textwidth}
		\centering
		\includegraphics[width=\textwidth]{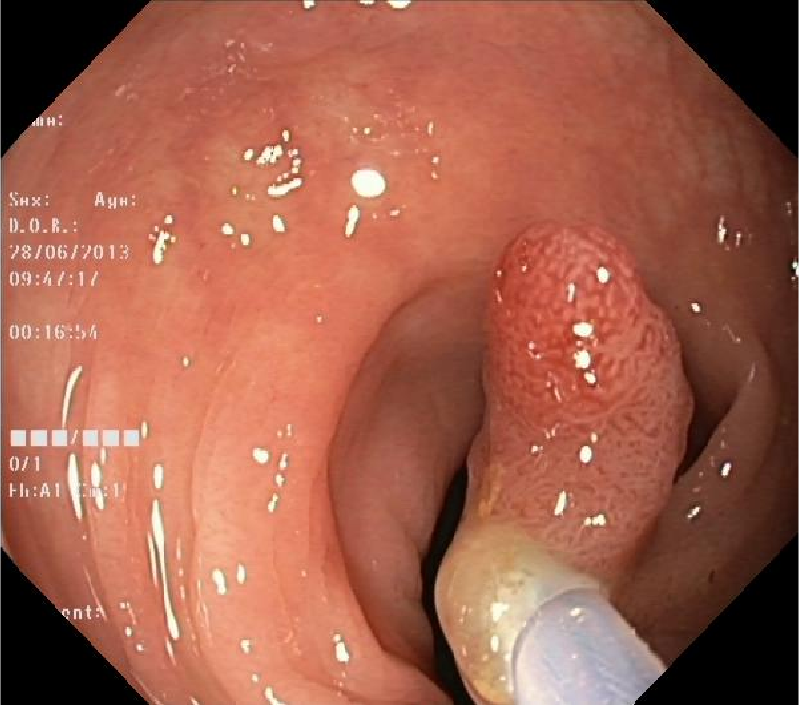}
		\label{fig:1}
	\end{subfigure}
	\begin{subfigure}{0.19\textwidth}
		\centering
		\includegraphics[width=\textwidth]{../figures/all_images/848.png}
		\label{fig:1}
	\end{subfigure}
	\vspace{-0.35cm}
	\\
	\begin{subfigure}{0.19\textwidth}
		\centering
		\includegraphics[width=\textwidth]{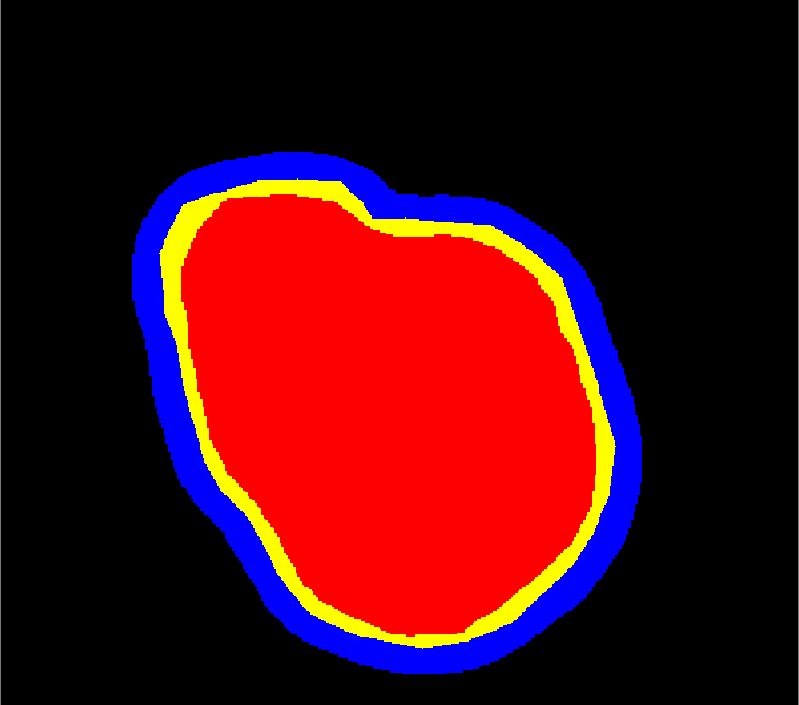}
		\label{fig:1}
	\end{subfigure}
	\begin{subfigure}{0.19\textwidth}
		\centering
		\includegraphics[width=\textwidth]{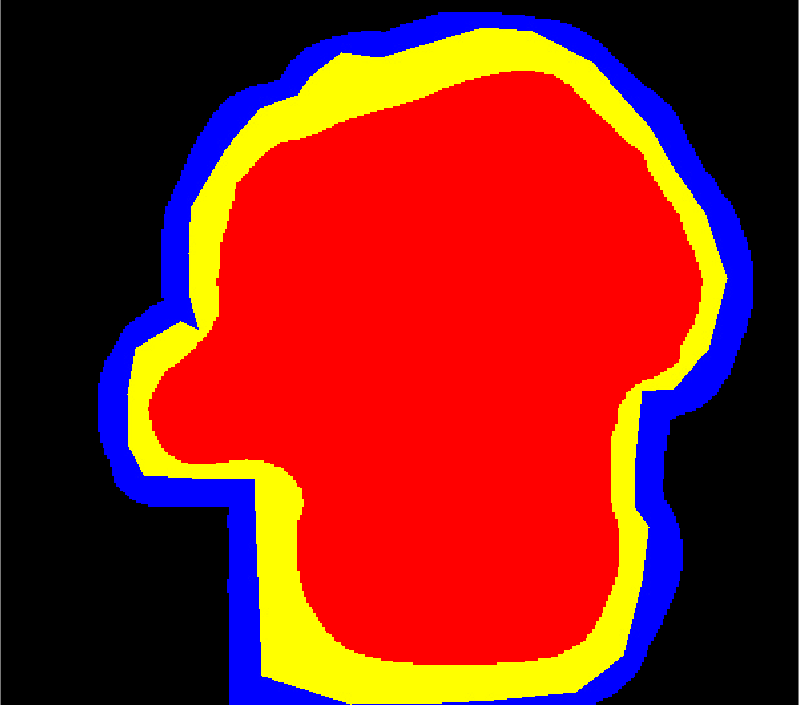}
		\label{fig:1}
	\end{subfigure}
	\begin{subfigure}{0.19\textwidth}
		\centering
		\includegraphics[width=\textwidth]{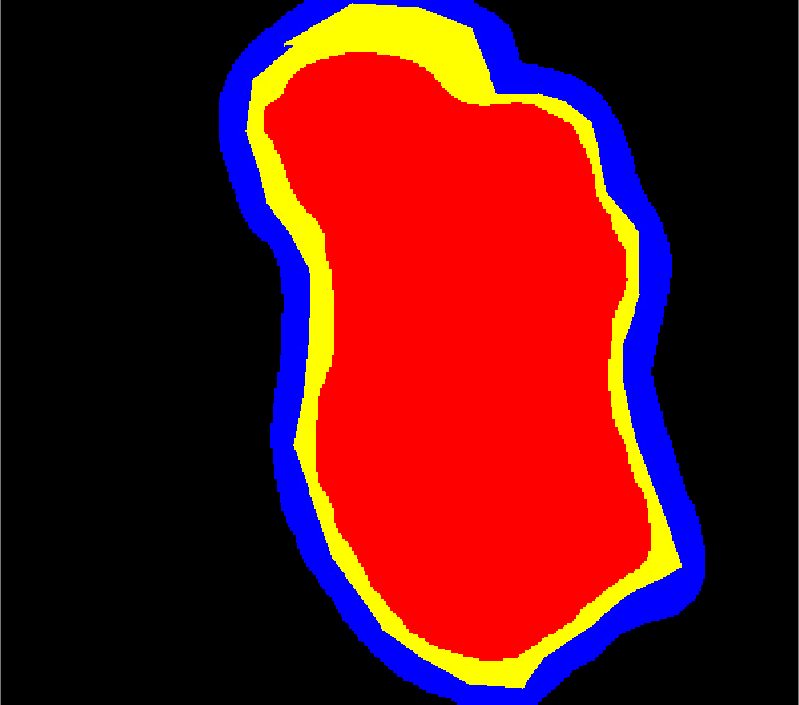}
		\label{fig:1}
	\end{subfigure}
	\begin{subfigure}{0.19\textwidth}
		\centering
		\includegraphics[width=\textwidth]{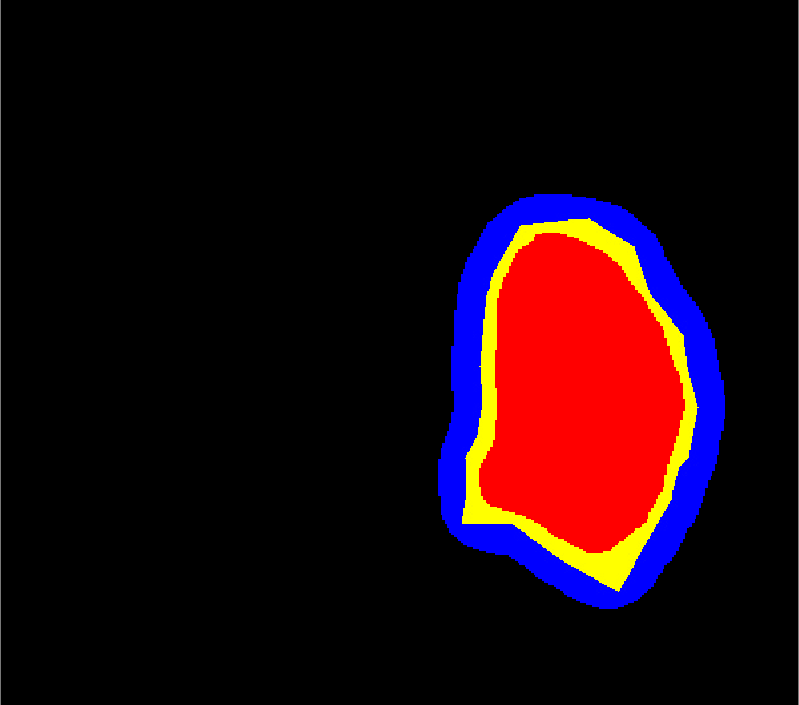}
		\label{fig:1}
	\end{subfigure}
	\begin{subfigure}{0.19\textwidth}
		\centering
		\includegraphics[width=\textwidth]{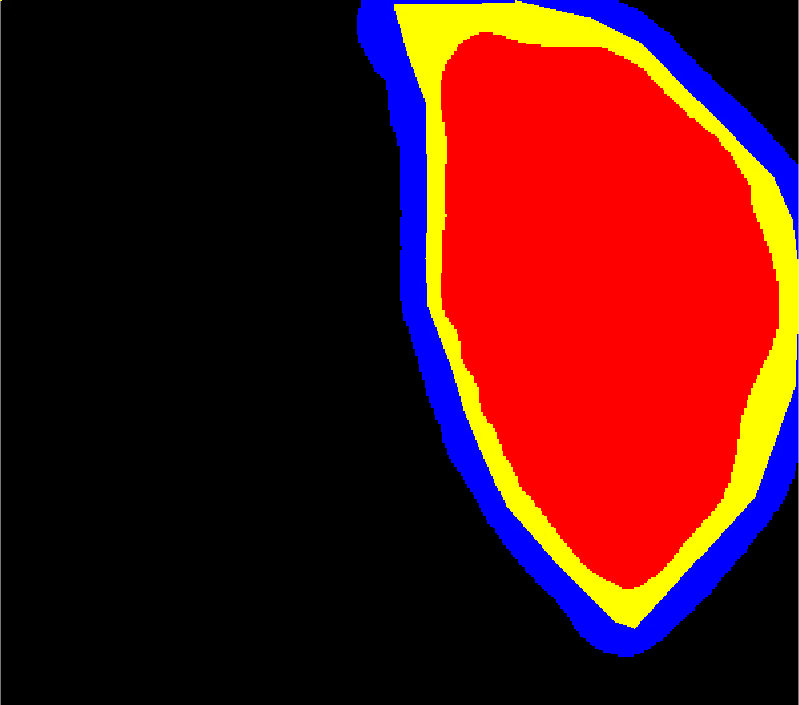}
		\label{fig:1}
	\end{subfigure}
	\vspace{-0.35cm}
	\\
	\begin{subfigure}{0.19\textwidth}
		\centering
		\includegraphics[width=\textwidth]{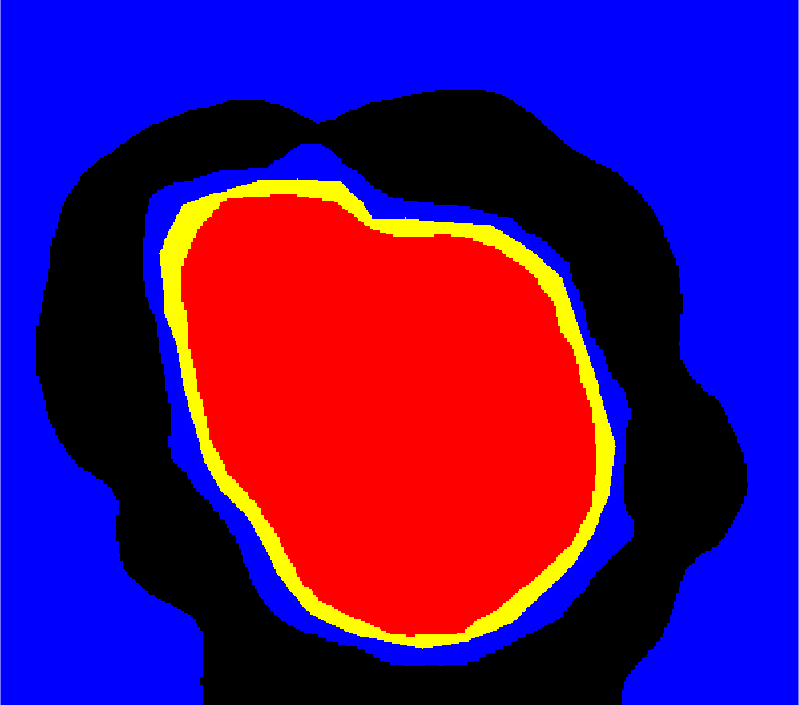}
		\label{fig:1}
	\end{subfigure}
	\begin{subfigure}{0.19\textwidth}
		\centering
		\includegraphics[width=\textwidth]{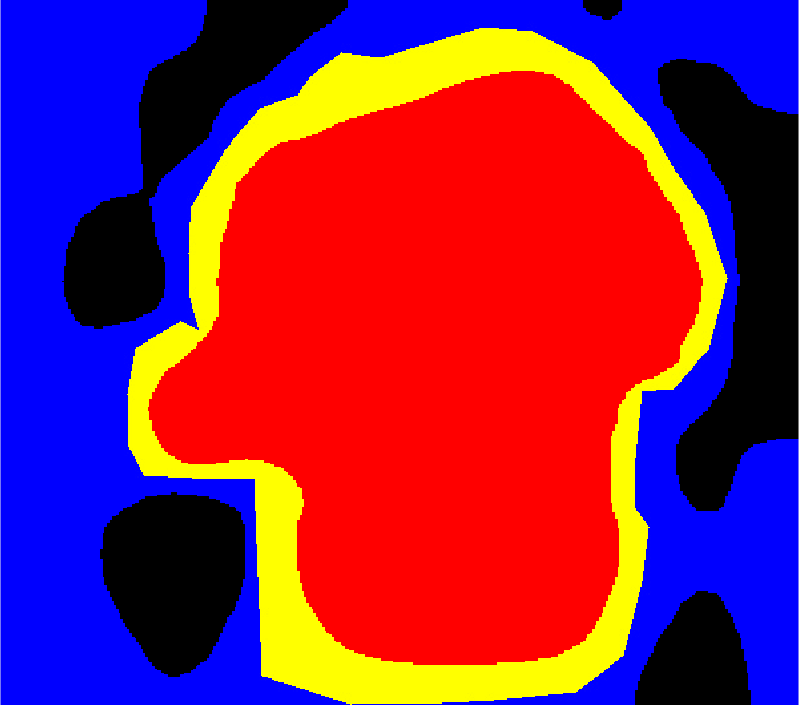}
		\label{fig:1}
	\end{subfigure}
	\begin{subfigure}{0.19\textwidth}
		\centering
		\includegraphics[width=\textwidth]{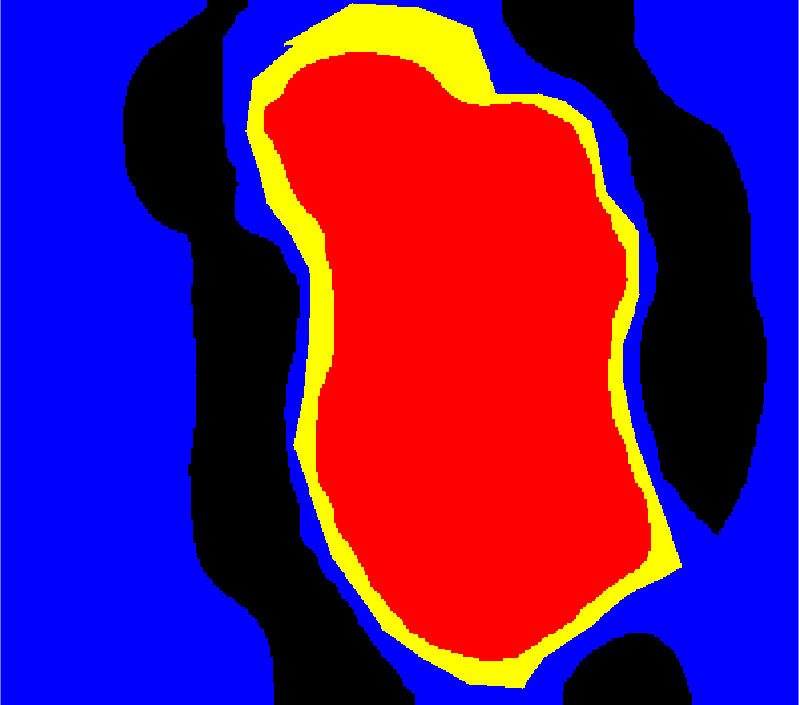}
		\label{fig:1}
	\end{subfigure}
	\begin{subfigure}{0.19\textwidth}
		\centering
		\includegraphics[width=\textwidth]{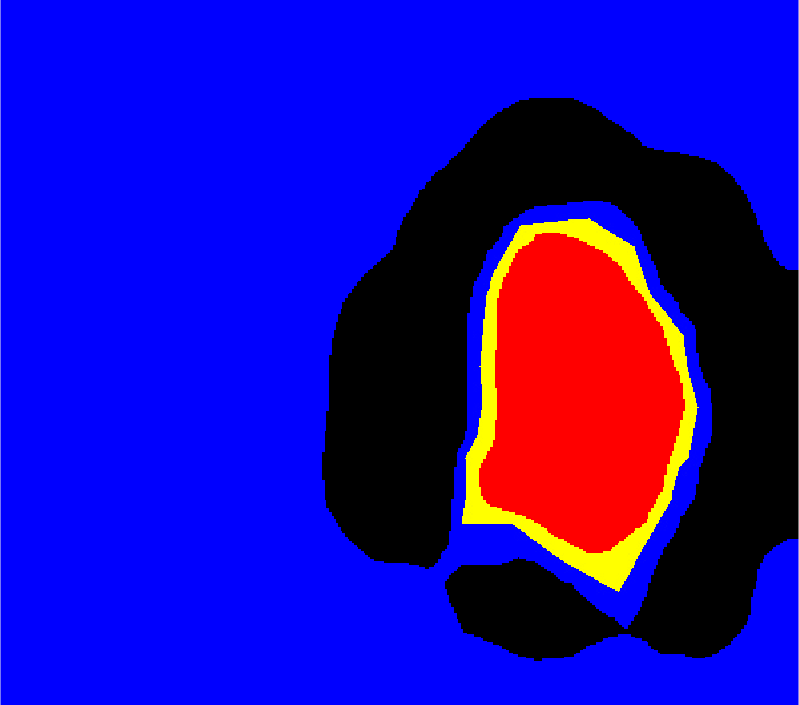}
		\label{fig:1}
	\end{subfigure}
	\begin{subfigure}{0.19\textwidth}
		\centering
		\includegraphics[width=\textwidth]{../figures/validation/val_crs_orig_90/848.png}
		\label{fig:1}
	\end{subfigure}
		\vspace{-0.35cm}
	\\
	\begin{subfigure}{0.19\textwidth}
		\centering
		\includegraphics[width=\textwidth]{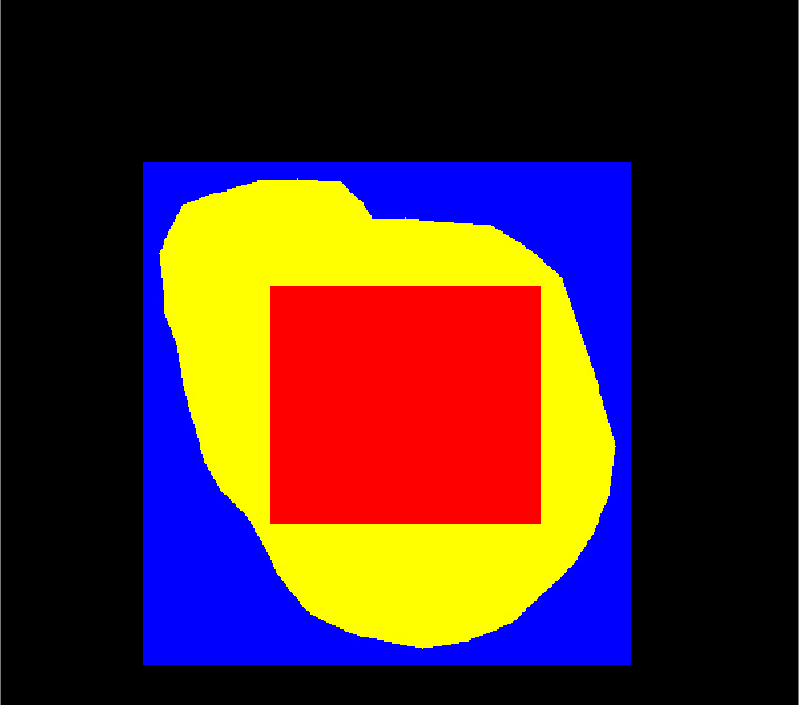}
		\label{fig:1}
	\end{subfigure}
	\begin{subfigure}{0.19\textwidth}
		\centering
		\includegraphics[width=\textwidth]{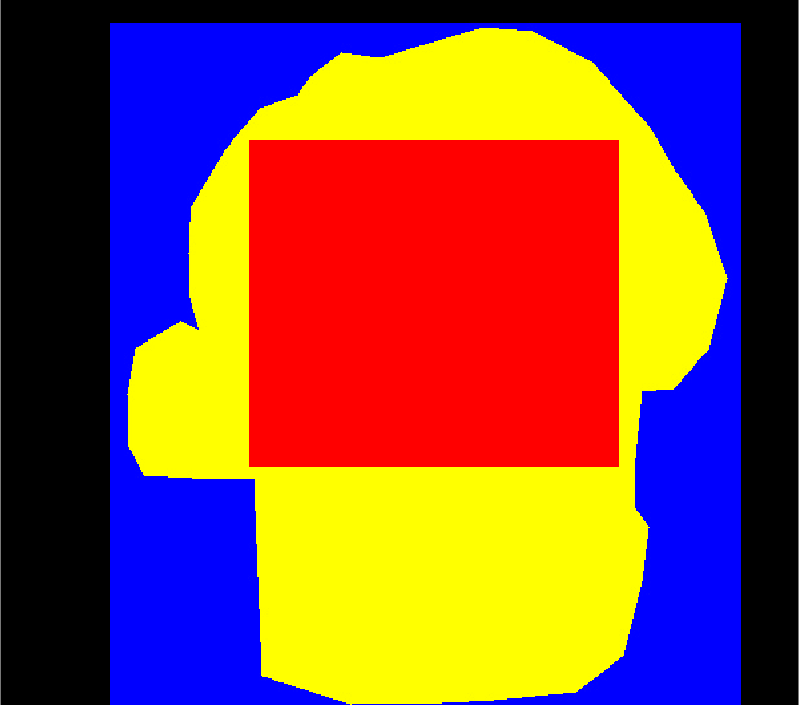}
		\label{fig:1}
	\end{subfigure}
	\begin{subfigure}{0.19\textwidth}
		\centering
		\includegraphics[width=\textwidth]{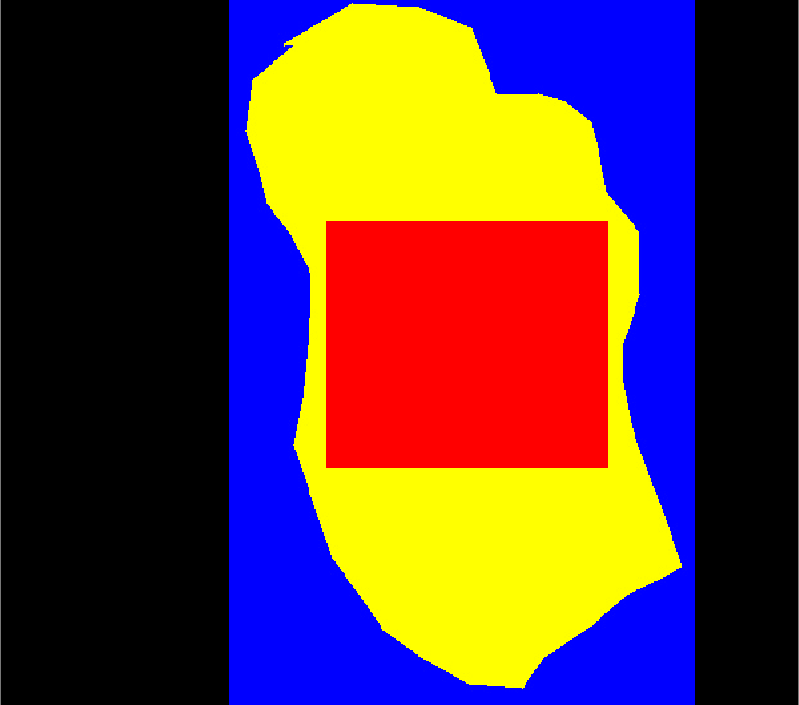}
		\label{fig:1}
	\end{subfigure}
	\begin{subfigure}{0.19\textwidth}
		\centering
		\includegraphics[width=\textwidth]{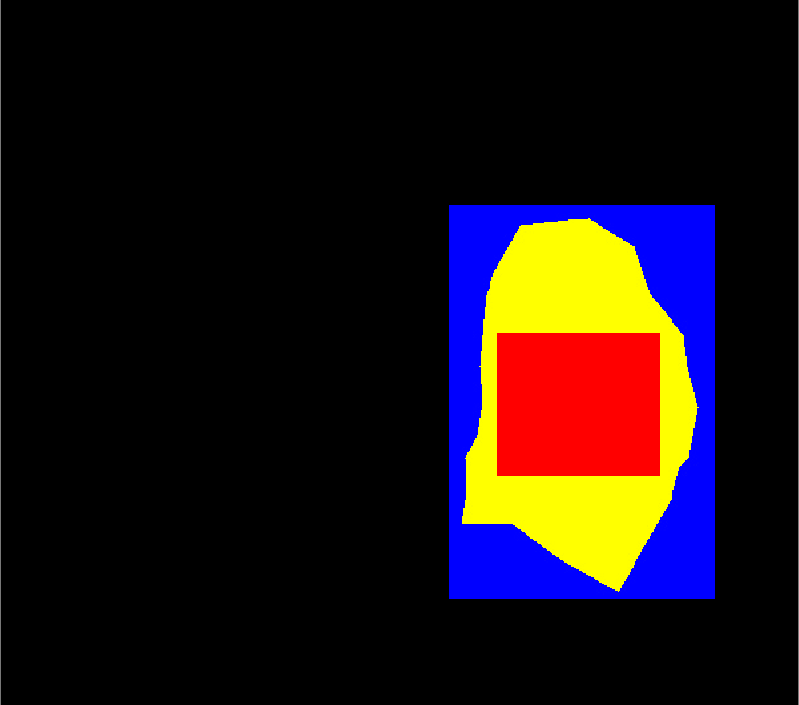}
		\label{fig:1}
	\end{subfigure}
	\begin{subfigure}{0.19\textwidth}
		\centering
		\includegraphics[width=\textwidth]{../figures/validation/val_crs_bt_90/848.png}
		\label{fig:1}
	\end{subfigure}
	\vspace{-0.35cm}
	\\
	\begin{subfigure}{0.19\textwidth}
		\centering
		\includegraphics[width=\textwidth]{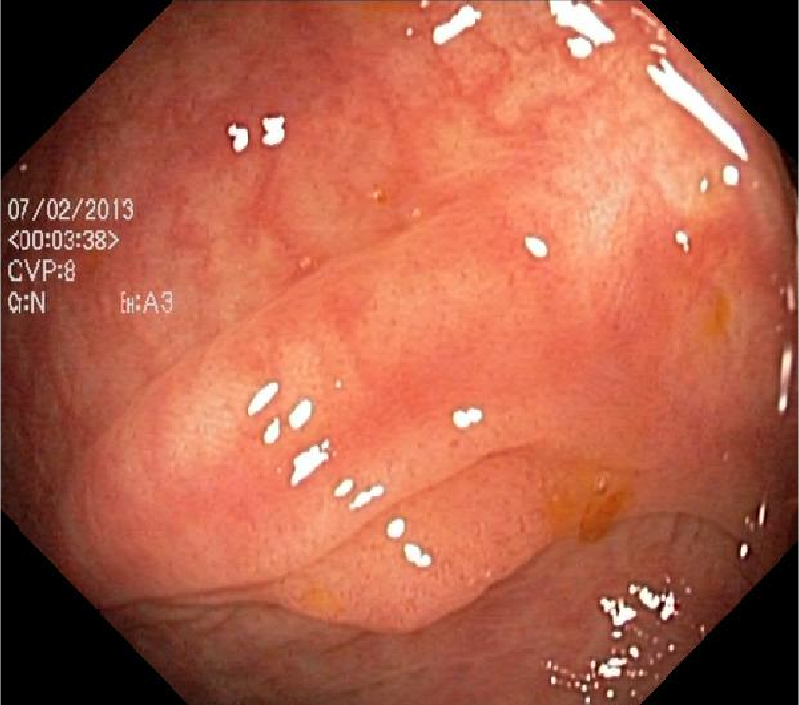}
		\label{fig:1}
	\end{subfigure}
	\begin{subfigure}{0.19\textwidth}
		\centering
		\includegraphics[width=\textwidth]{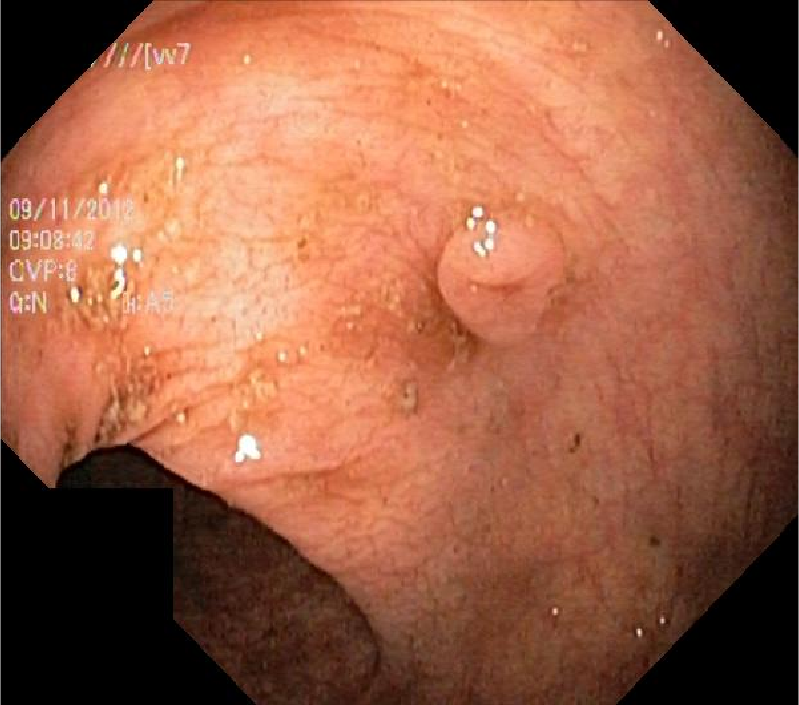}
		\label{fig:1}
	\end{subfigure}
	\begin{subfigure}{0.19\textwidth}
		\centering
		\includegraphics[width=\textwidth]{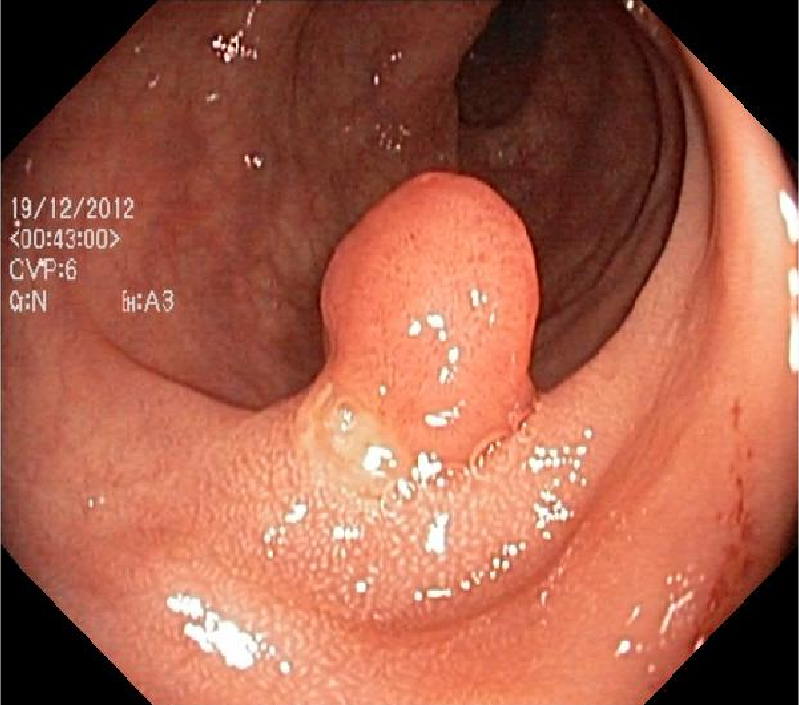}
		\label{fig:1}
	\end{subfigure}
	\begin{subfigure}{0.19\textwidth}
		\centering
		\includegraphics[width=\textwidth]{../figures/all_images/483.png}
		\label{fig:1}
	\end{subfigure}
	\begin{subfigure}{0.19\textwidth}
		\centering
		\includegraphics[width=\textwidth]{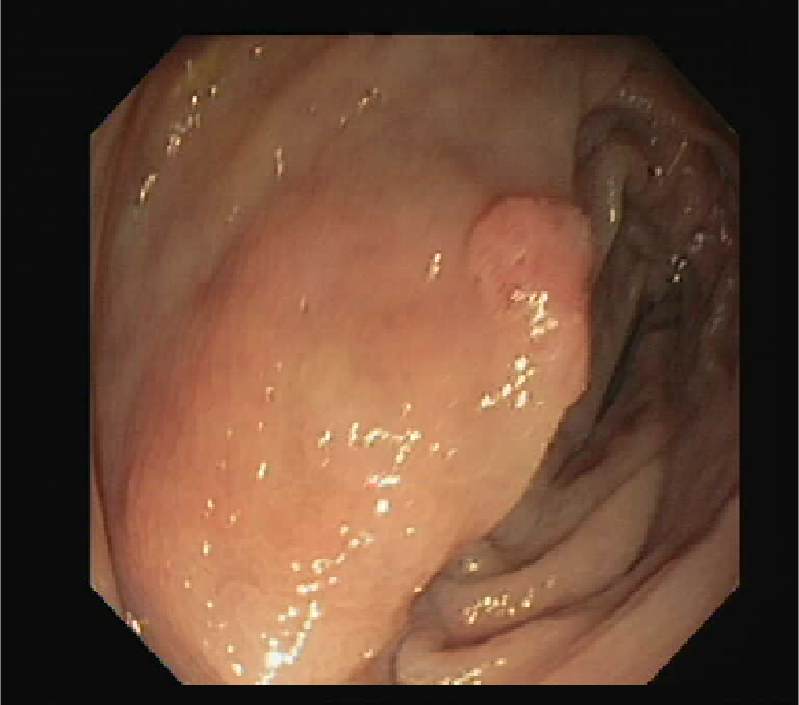}
		\label{fig:1}
	\end{subfigure}
	\vspace{-0.35cm}
	\\
	\begin{subfigure}{0.19\textwidth}
		\centering
		\includegraphics[width=\textwidth]{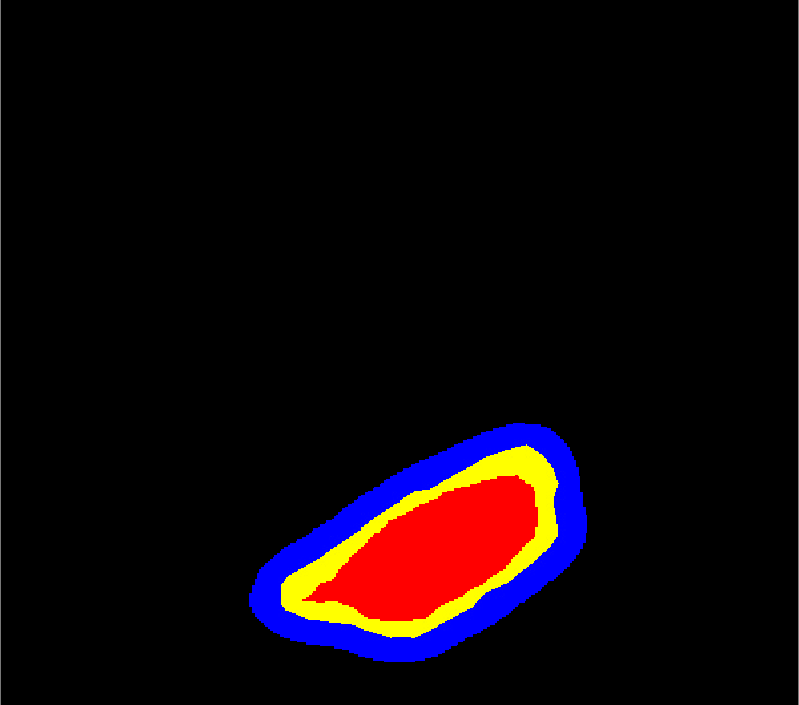}
		\label{fig:1}
	\end{subfigure}
	\begin{subfigure}{0.19\textwidth}
		\centering
		\includegraphics[width=\textwidth]{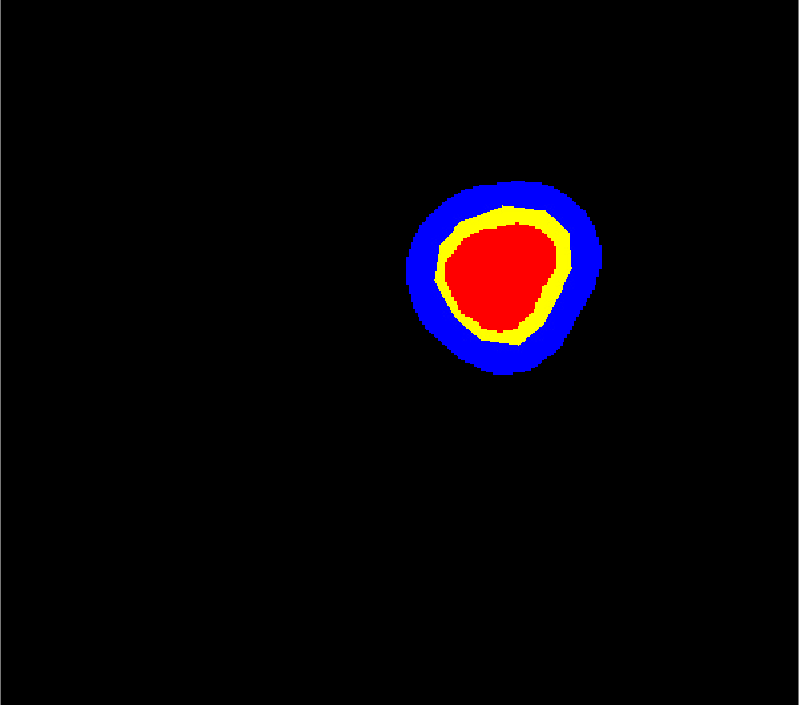}
		\label{fig:1}
	\end{subfigure}
	\begin{subfigure}{0.19\textwidth}
		\centering
		\includegraphics[width=\textwidth]{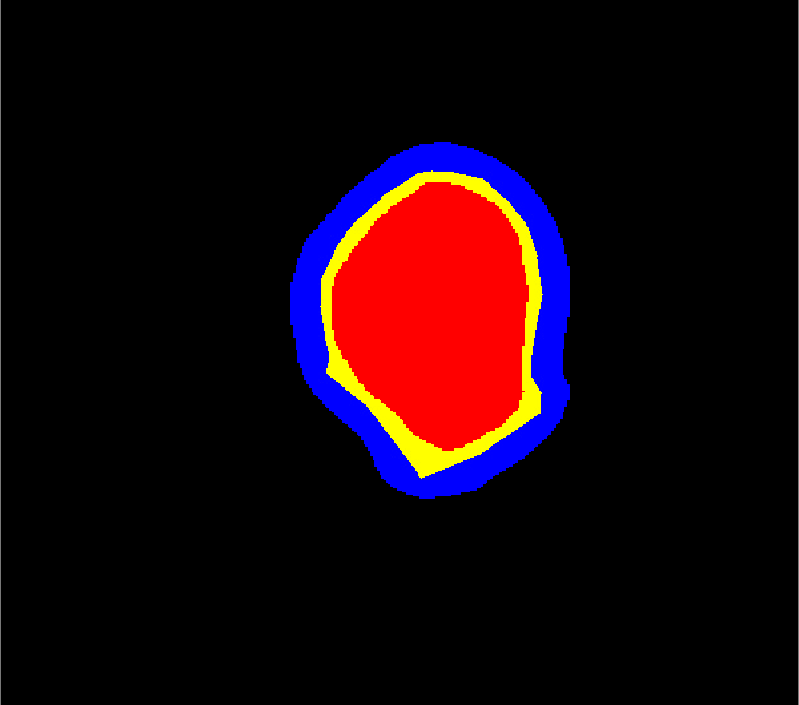}
		\label{fig:1}
	\end{subfigure}
	\begin{subfigure}{0.19\textwidth}
		\centering
		\includegraphics[width=\textwidth]{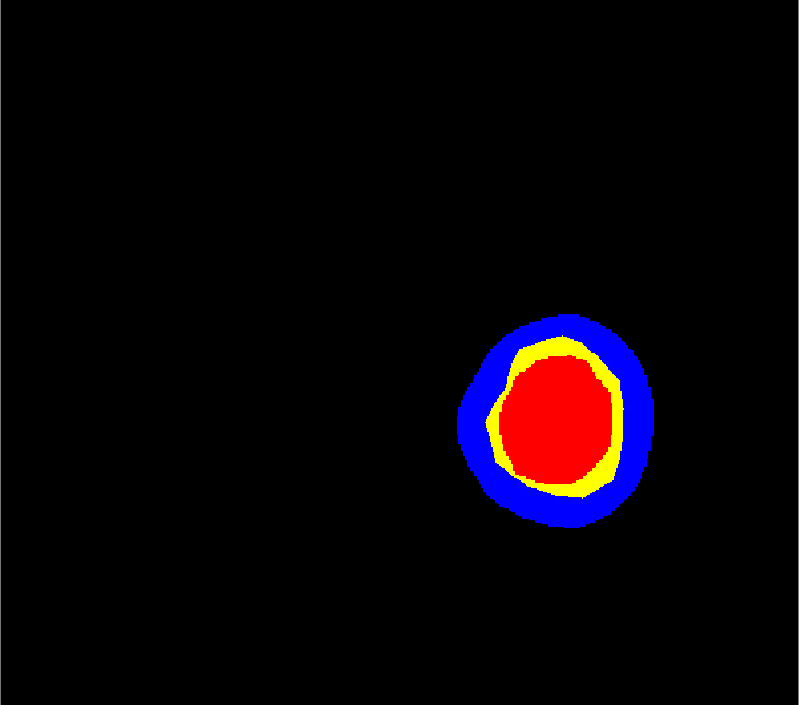}
		\label{fig:1}
	\end{subfigure}
	\begin{subfigure}{0.19\textwidth}
		\centering
		\includegraphics[width=\textwidth]{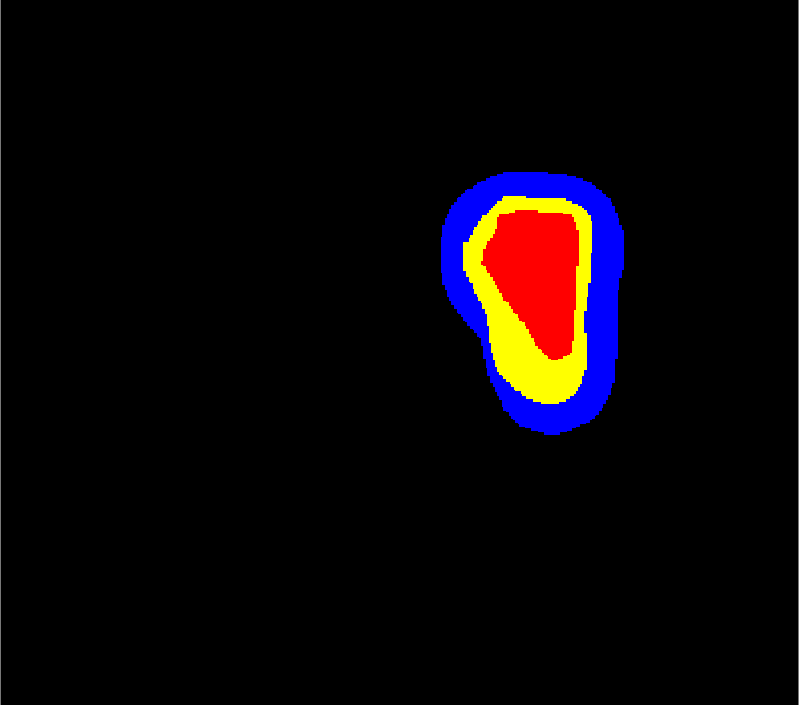}
		\label{fig:1}
	\end{subfigure}
		\vspace{-0.35cm}
	\\
	\begin{subfigure}{0.19\textwidth}
		\centering
		\includegraphics[width=\textwidth]{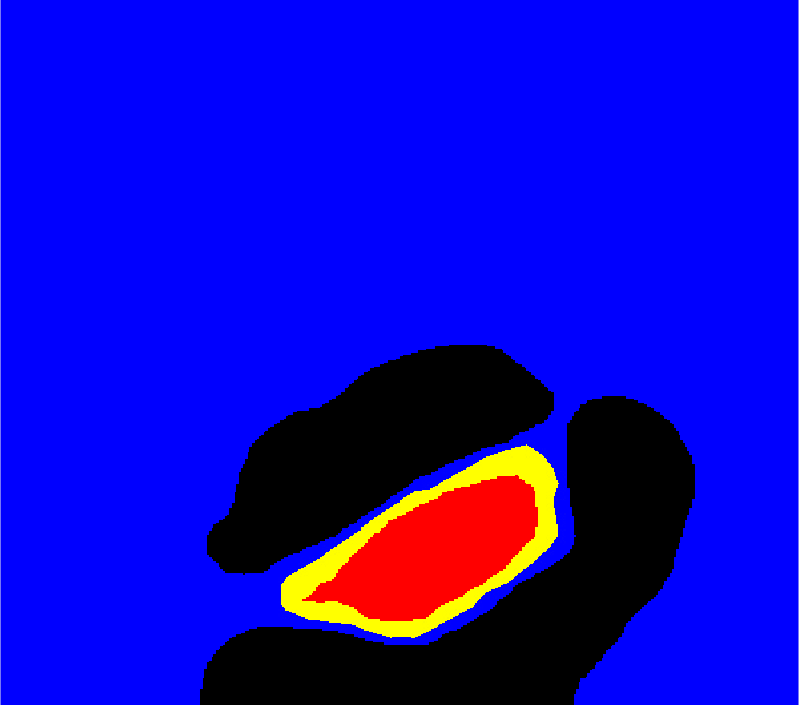}
		\label{fig:1}
	\end{subfigure}
	\begin{subfigure}{0.19\textwidth}
		\centering
		\includegraphics[width=\textwidth]{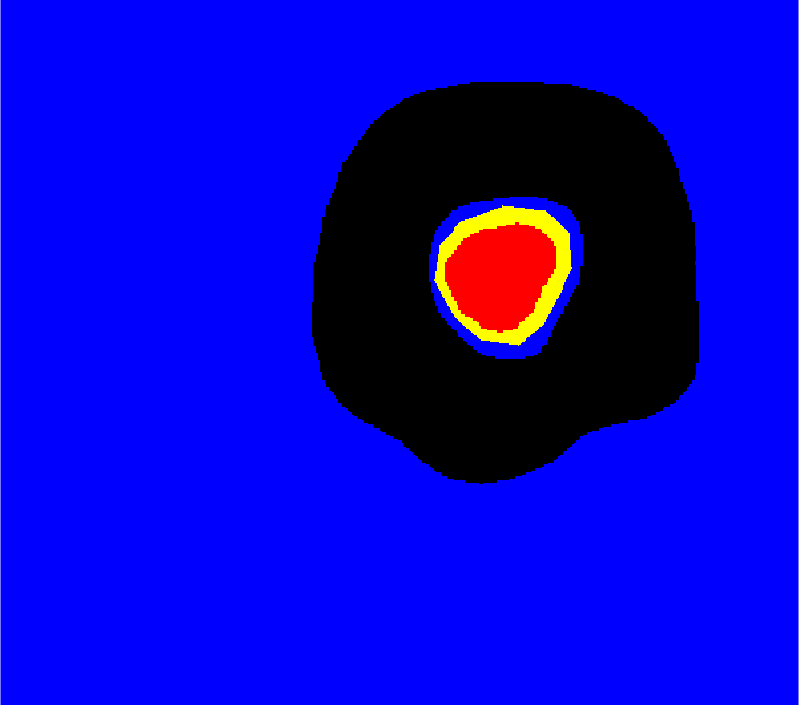}
		\label{fig:1}
	\end{subfigure}
	\begin{subfigure}{0.19\textwidth}
		\centering
		\includegraphics[width=\textwidth]{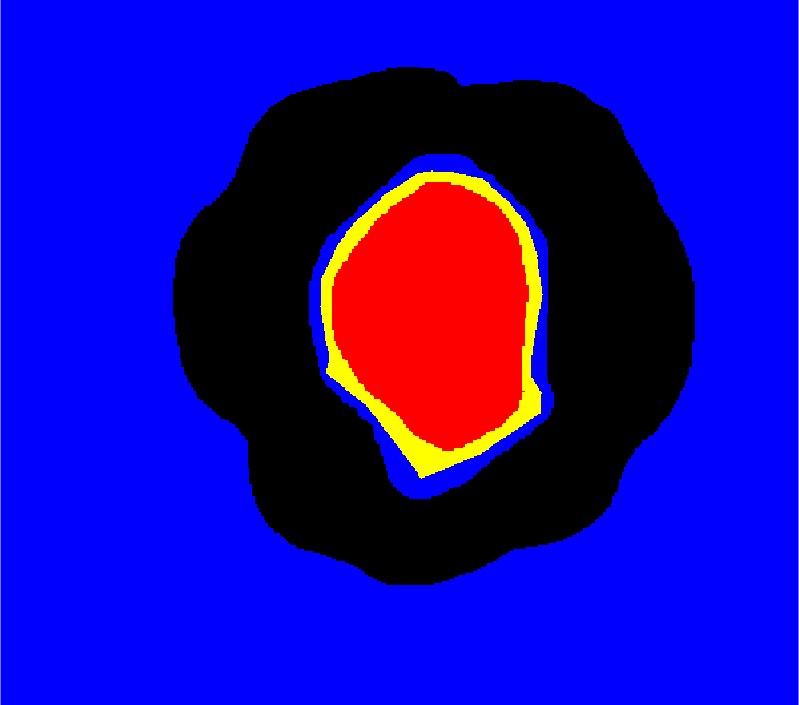}
		\label{fig:1}
	\end{subfigure}
	\begin{subfigure}{0.19\textwidth}
		\centering
		\includegraphics[width=\textwidth]{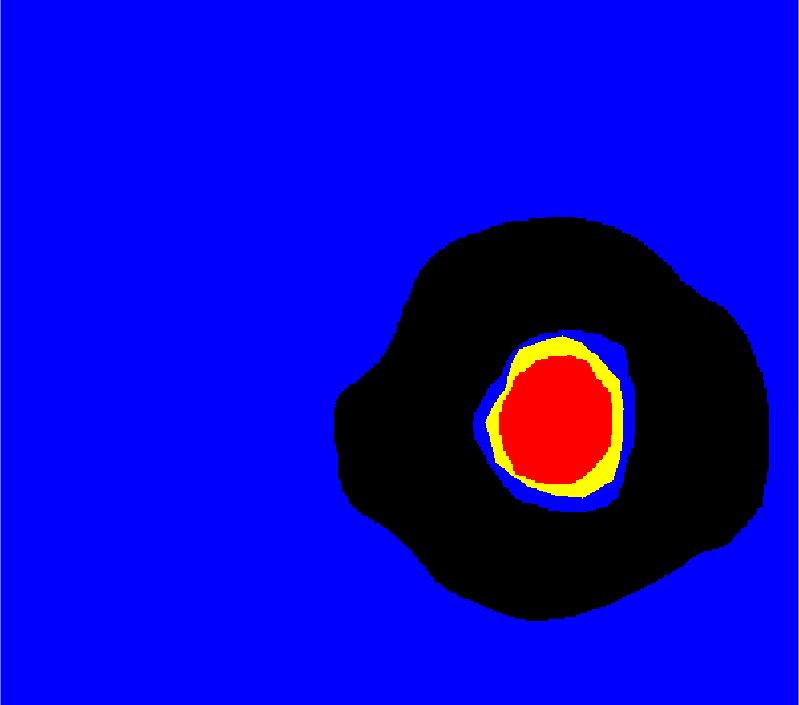}
		\label{fig:1}
	\end{subfigure}
	\begin{subfigure}{0.19\textwidth}
		\centering
		\includegraphics[width=\textwidth]{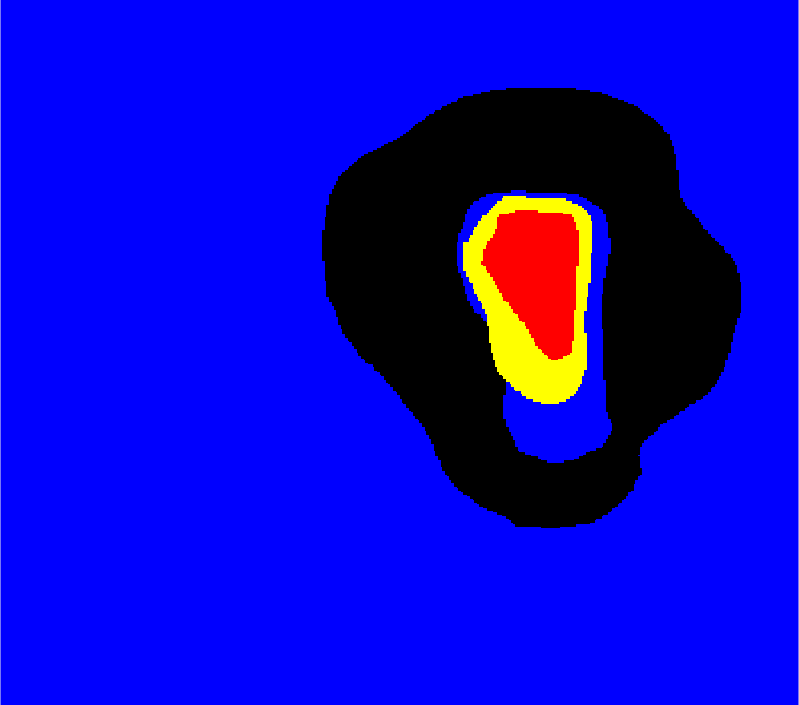}
		\label{fig:1}
	\end{subfigure}
			\vspace{-0.35cm}
	\\
	\begin{subfigure}{0.19\textwidth}
		\centering
		\includegraphics[width=\textwidth]{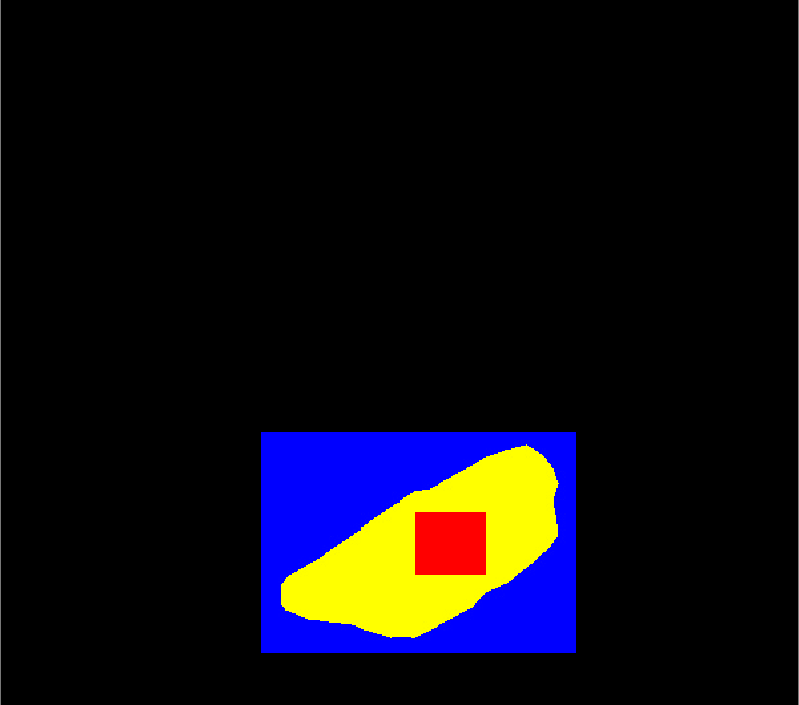}
		\label{fig:1}
	\end{subfigure}
	\begin{subfigure}{0.19\textwidth}
		\centering
		\includegraphics[width=\textwidth]{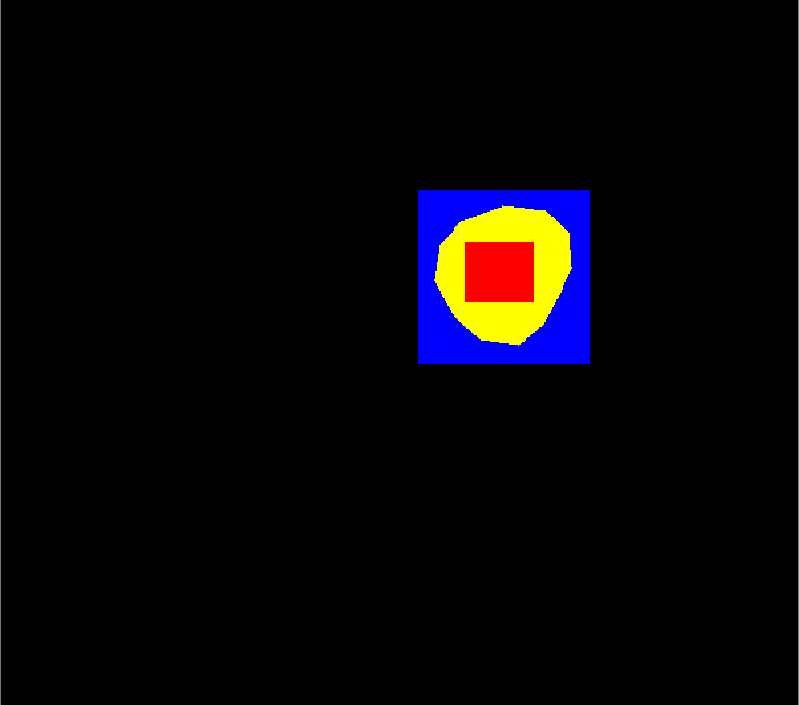}
		\label{fig:1}
	\end{subfigure}
	\begin{subfigure}{0.19\textwidth}
		\centering
		\includegraphics[width=\textwidth]{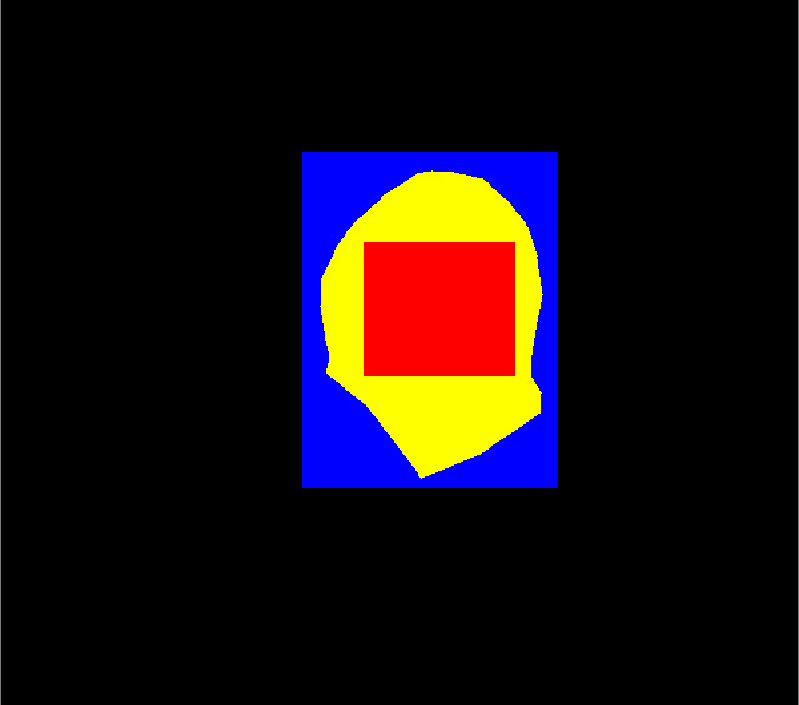}
		\label{fig:1}
	\end{subfigure}
	\begin{subfigure}{0.19\textwidth}
		\centering
		\includegraphics[width=\textwidth]{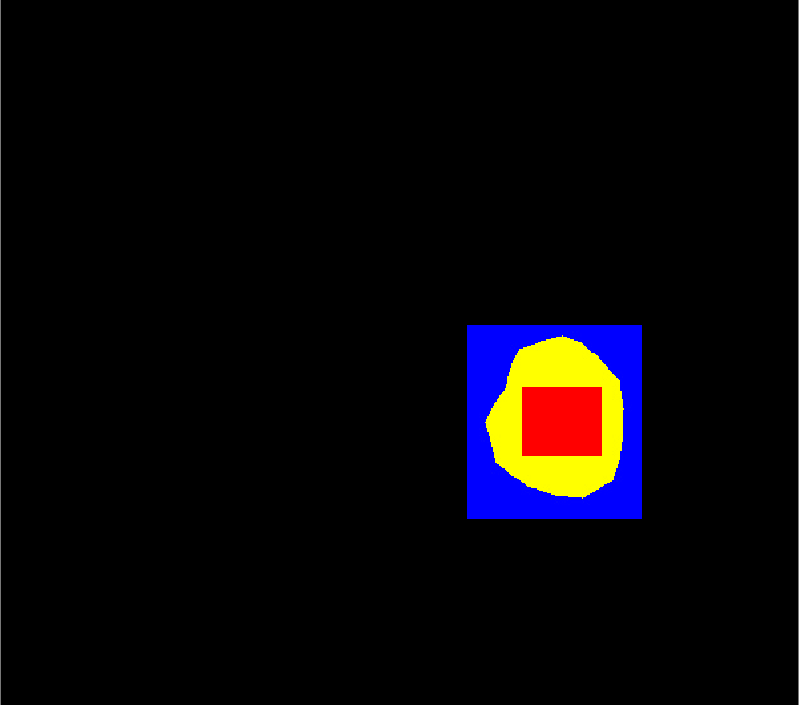}
		\label{fig:1}
	\end{subfigure}
	\begin{subfigure}{0.19\textwidth}
		\centering
		\includegraphics[width=\textwidth]{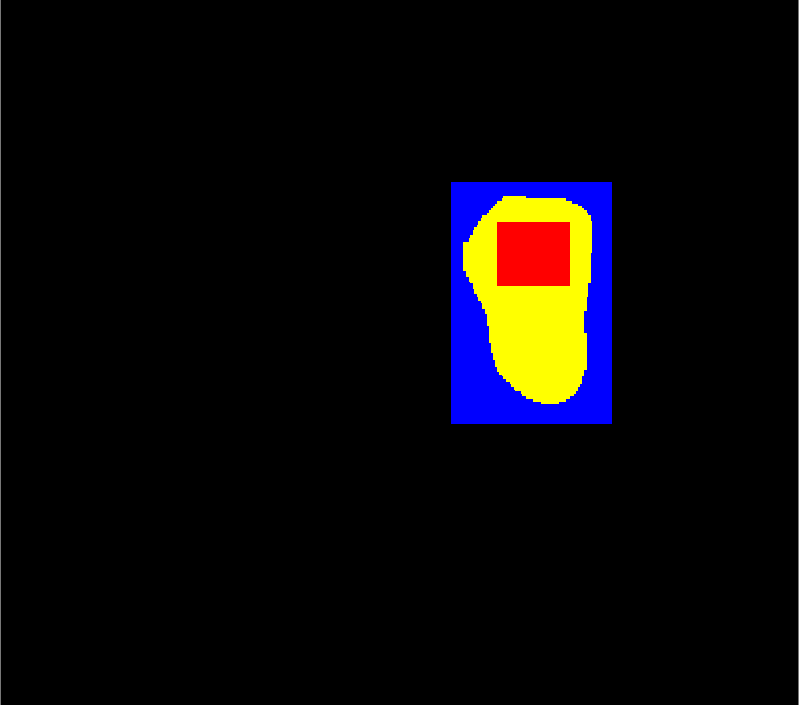}
		\label{fig:1}
	\end{subfigure}
	\label{fig:grid}
	\caption{Additional validition examples. In each example, after the original images, the rows are (from top to bottom) the combination of the original and distance transformed scores, then the original scores and finally the bounding box scores.}\label{fig:polpysex2}
\end{figure}

\newpage
\subsection{Confidence sets for the bounding boxes}
\begin{figure}[h!]
	\begin{subfigure}{0.19\textwidth}
		\centering
		\includegraphics[width=\textwidth]{../figures/all_images/61.png}
		\label{fig:1}
	\end{subfigure}
	\begin{subfigure}{0.19\textwidth}
		\centering
		\includegraphics[width=\textwidth]{../figures/all_images/114.png}
		\label{fig:1}
	\end{subfigure}
	\begin{subfigure}{0.19\textwidth}
		\centering
		\includegraphics[width=\textwidth]{../figures/all_images/144.png}
		\label{fig:1}
	\end{subfigure}
	\begin{subfigure}{0.19\textwidth}
		\centering
		\includegraphics[width=\textwidth]{../figures/all_images/148.png}
		\label{fig:1}
	\end{subfigure}
	\begin{subfigure}{0.19\textwidth}
		\centering
		\includegraphics[width=\textwidth]{../figures/all_images/251.png}
		\label{fig:1}
	\end{subfigure}
	\vspace{-0.35cm}
	\\
	\begin{subfigure}{0.19\textwidth}
		\centering
		\includegraphics[width=\textwidth]{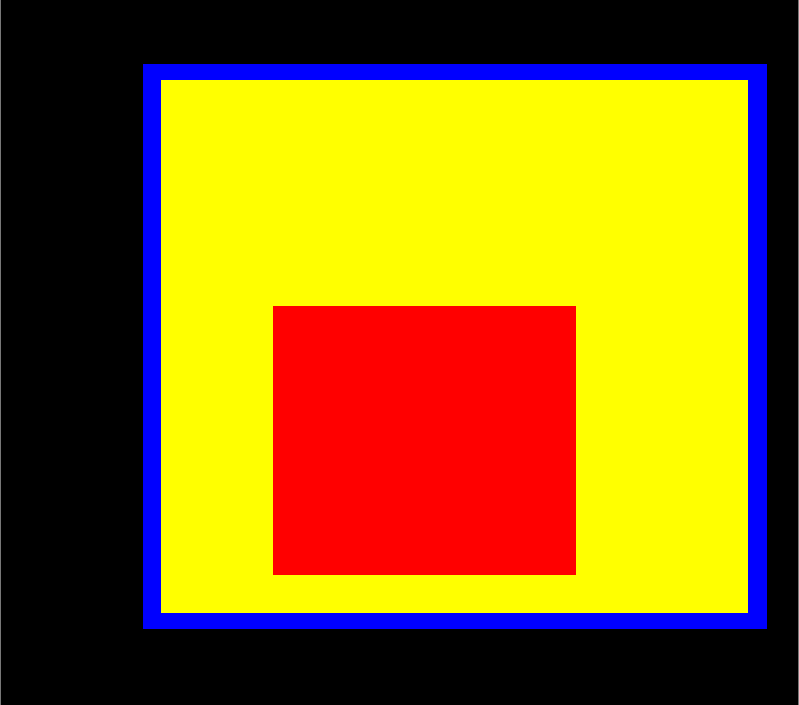}
		\label{fig:1}
	\end{subfigure}
	\begin{subfigure}{0.19\textwidth}
		\centering
		\includegraphics[width=\textwidth]{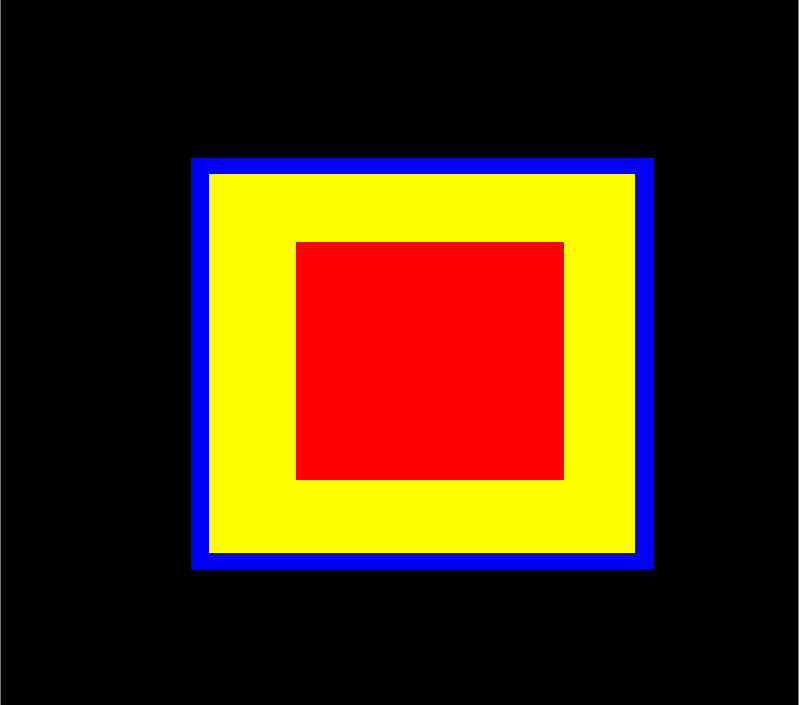}
		\label{fig:1}
	\end{subfigure}
	\begin{subfigure}{0.19\textwidth}
		\centering
		\includegraphics[width=\textwidth]{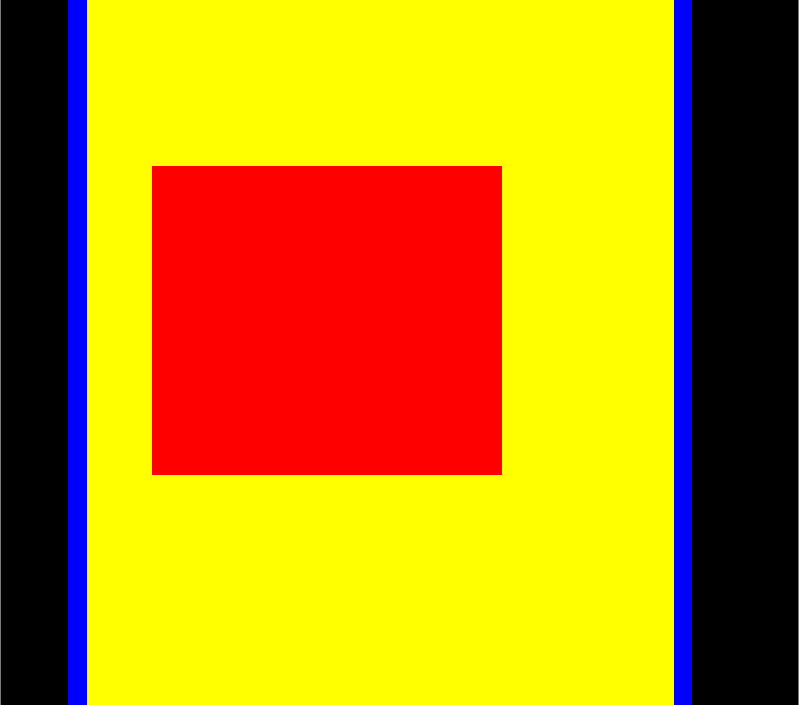}
		\label{fig:1}
	\end{subfigure}
	\begin{subfigure}{0.19\textwidth}
		\centering
		\includegraphics[width=\textwidth]{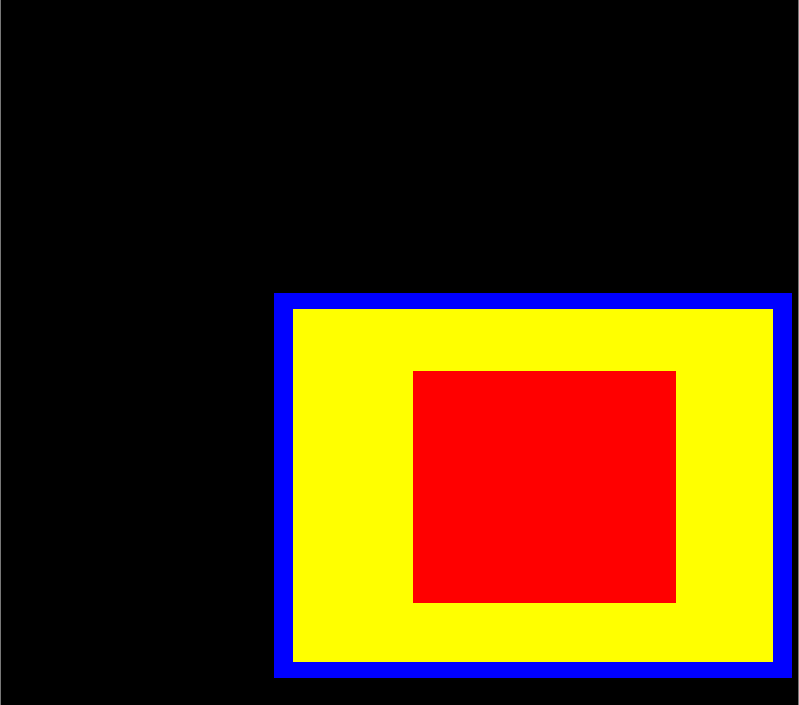}
		\label{fig:1}
	\end{subfigure}
	\begin{subfigure}{0.19\textwidth}
		\centering
		\includegraphics[width=\textwidth]{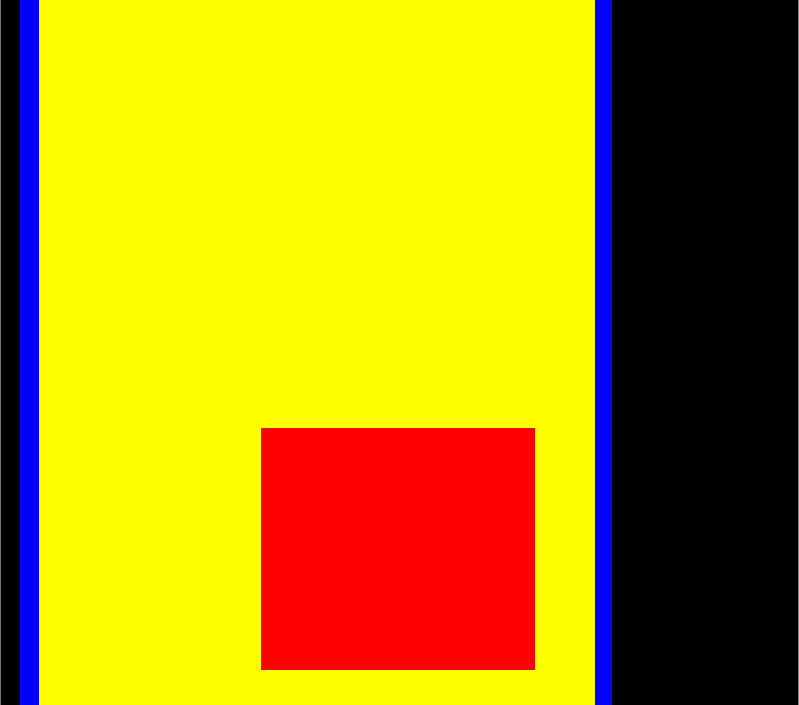}
		\label{fig:1}
	\end{subfigure}
	\vspace{-0.35cm}
	\\
	\begin{subfigure}{0.19\textwidth}
		\centering
		\includegraphics[width=\textwidth]{../figures/all_images/7.png}
		\label{fig:1}
	\end{subfigure}
	\begin{subfigure}{0.19\textwidth}
		\centering
		\includegraphics[width=\textwidth]{../figures/all_images/211.png}
		\label{fig:1}
	\end{subfigure}
	\begin{subfigure}{0.19\textwidth}
		\centering
		\includegraphics[width=\textwidth]{../figures/all_images/1062.png}
		\label{fig:1}
	\end{subfigure}
	\begin{subfigure}{0.19\textwidth}
		\centering
		\includegraphics[width=\textwidth]{../figures/all_images/398.png}
		\label{fig:1}
	\end{subfigure}
	\begin{subfigure}{0.19\textwidth}
		\centering
		\includegraphics[width=\textwidth]{../figures/all_images/269.png}
		\label{fig:1}
	\end{subfigure}
	\vspace{-0.35cm}
	\\
	\begin{subfigure}{0.19\textwidth}
		\centering
		\includegraphics[width=\textwidth]{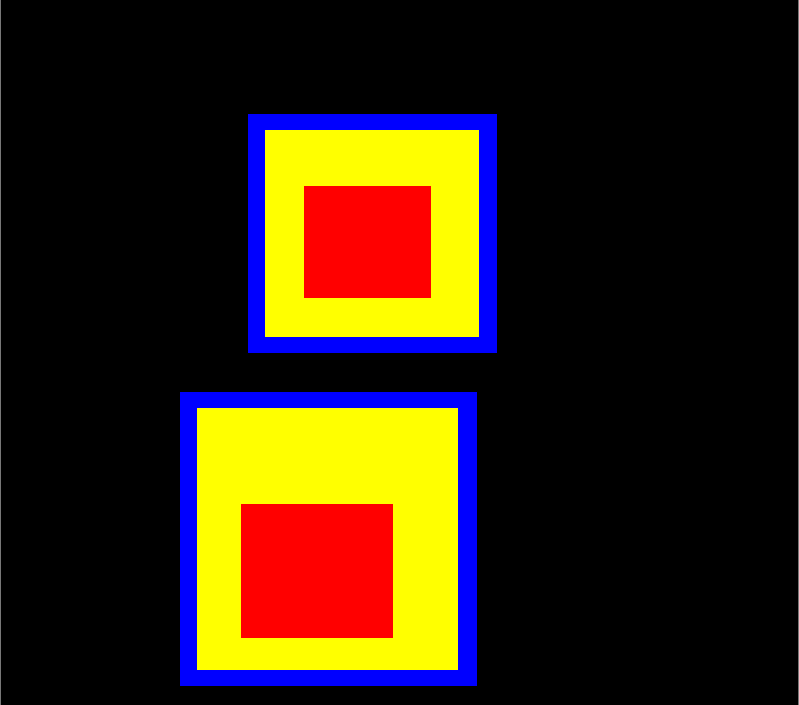}
		\label{fig:1}
	\end{subfigure}
	\begin{subfigure}{0.19\textwidth}
		\centering
		\includegraphics[width=\textwidth]{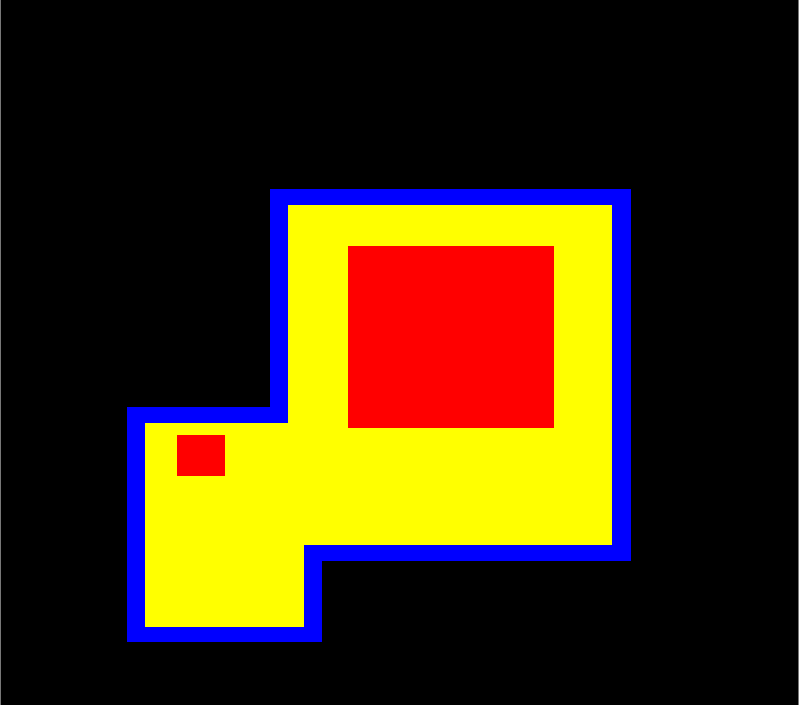}
		\label{fig:1}
	\end{subfigure}
	\begin{subfigure}{0.19\textwidth}
		\centering
		\includegraphics[width=\textwidth]{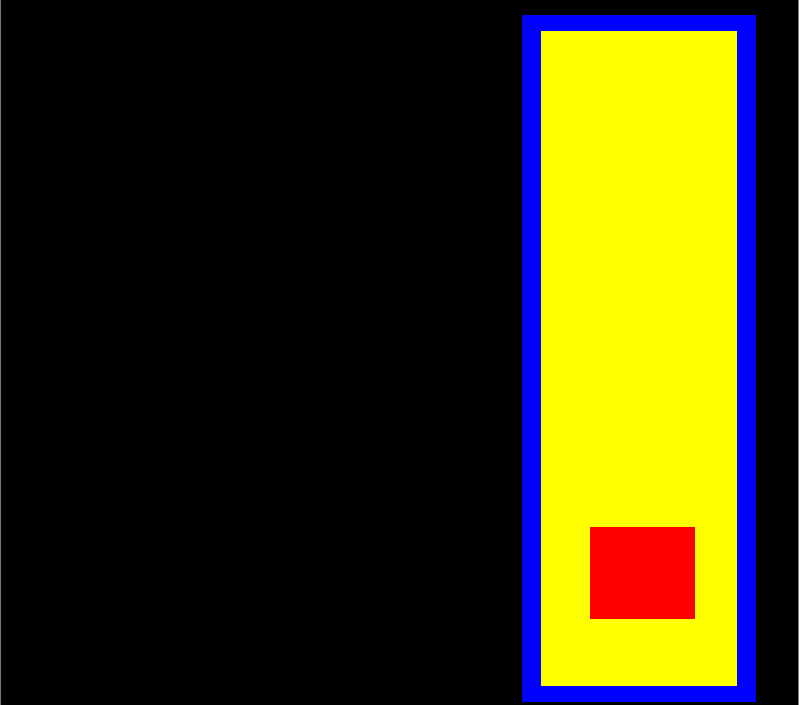}
		\label{fig:1}
	\end{subfigure}
	\begin{subfigure}{0.19\textwidth}
		\centering
		\includegraphics[width=\textwidth]{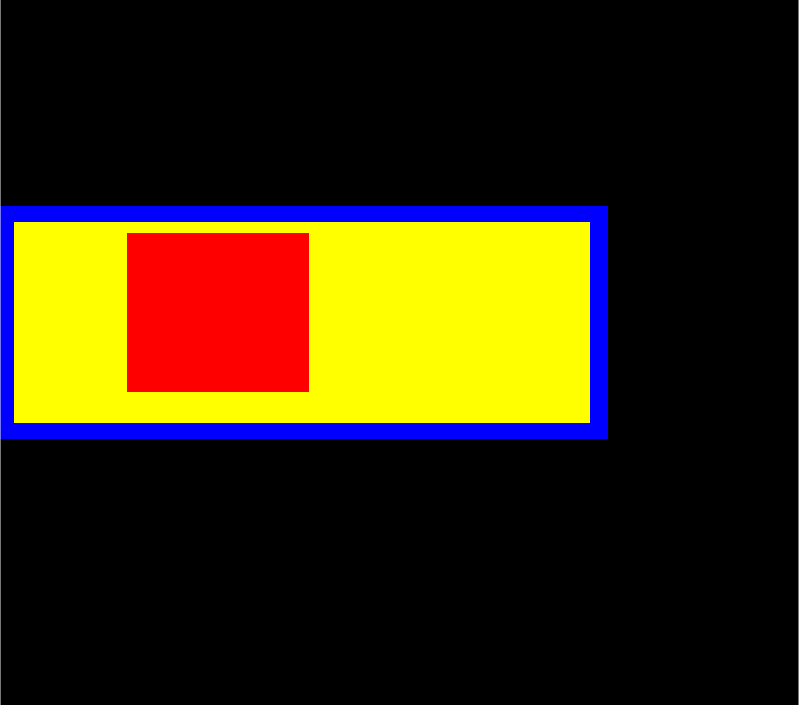}
		\label{fig:1}
	\end{subfigure}
	\begin{subfigure}{0.19\textwidth}
		\centering
		\includegraphics[width=\textwidth]{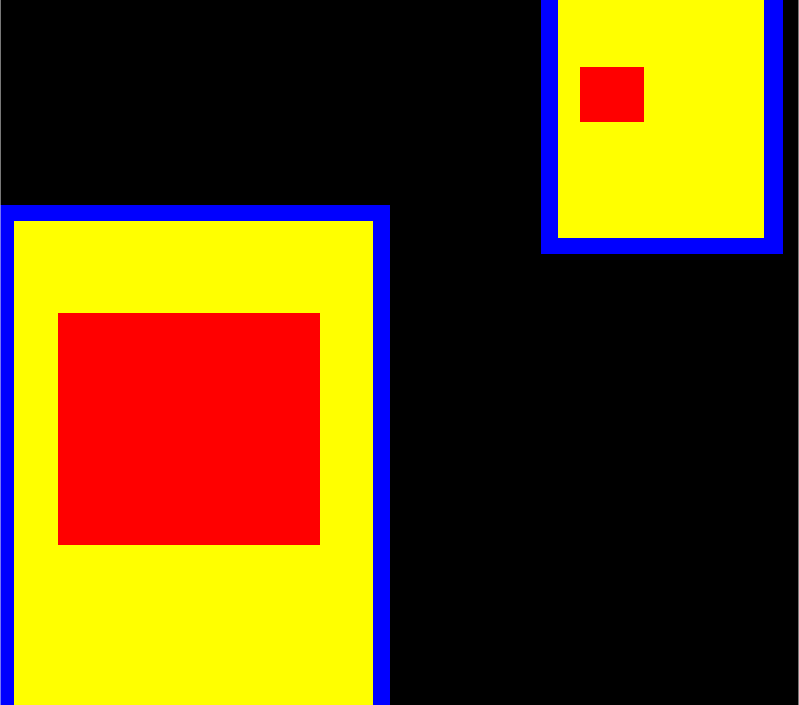}
		\label{fig:1}
	\end{subfigure}
	\label{fig:grid}
	\caption{Conformal confidence sets for the boundary boxes themselves using the approach introduced in Section \ref{AA:BBtheory}. The ground truth outer bounding boxes are shown in yellow.}\label{fig:resbb}
\end{figure}

\subsection{Joint 90\% confidence regions}
\begin{figure}[h!]
	\begin{subfigure}{0.19\textwidth}
		\centering
		\includegraphics[width=\textwidth]{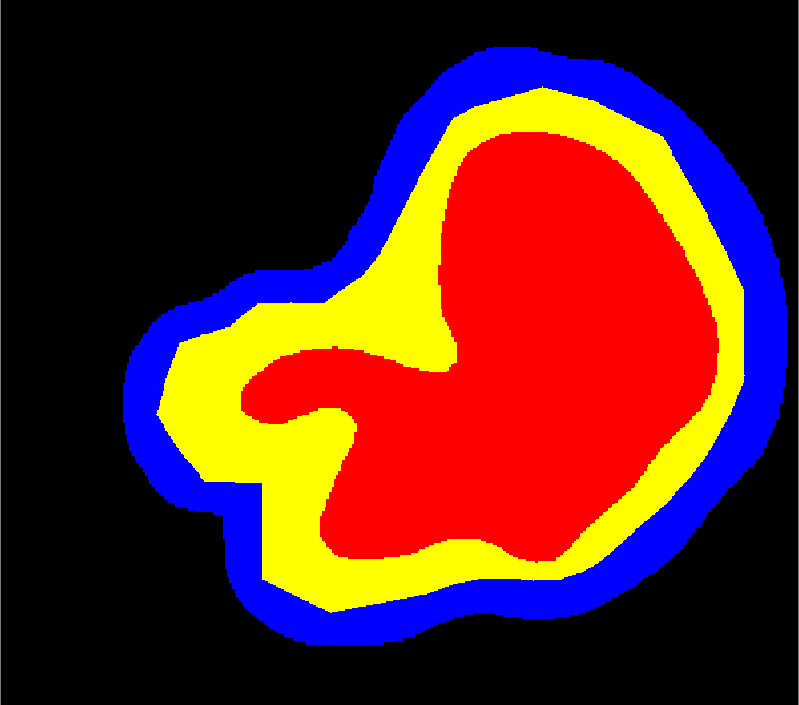}
		\label{fig:1}
	\end{subfigure}
	\begin{subfigure}{0.19\textwidth}
		\centering
		\includegraphics[width=\textwidth]{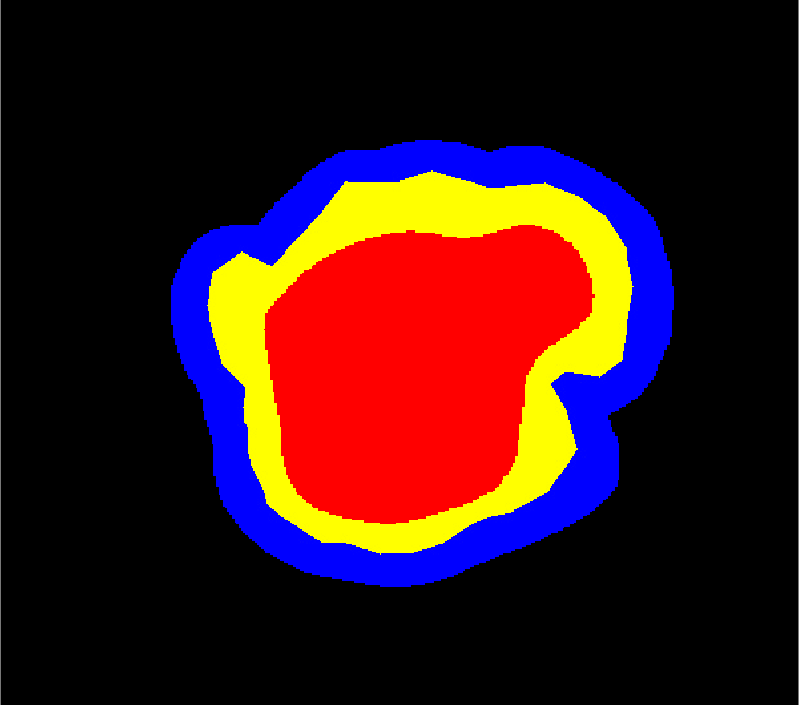}
		\label{fig:1}
	\end{subfigure}
	\begin{subfigure}{0.19\textwidth}
		\centering
		\includegraphics[width=\textwidth]{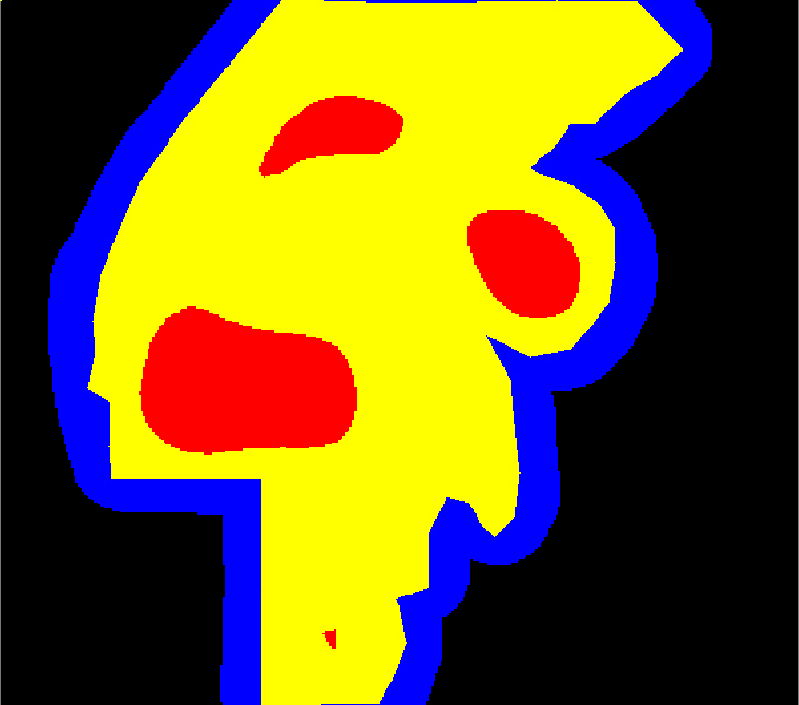}
		\label{fig:1}
	\end{subfigure}
	\begin{subfigure}{0.19\textwidth}
		\centering
		\includegraphics[width=\textwidth]{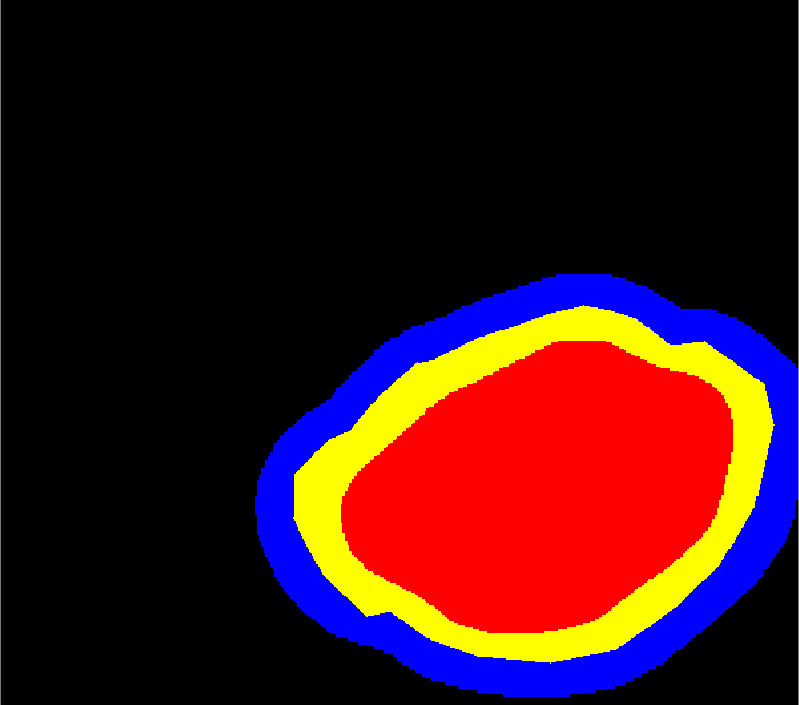}
		\label{fig:1}
	\end{subfigure}
	\begin{subfigure}{0.19\textwidth}
		\centering
		\includegraphics[width=\textwidth]{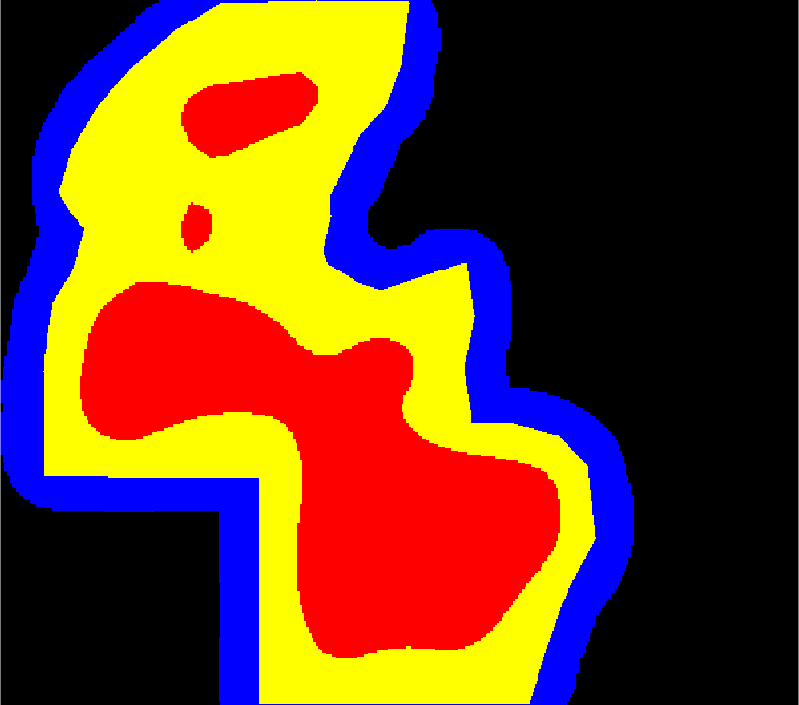}
		\label{fig:1}
	\end{subfigure}
	\vspace{-0.35cm}
	\\
	\begin{subfigure}{0.19\textwidth}
		\centering
		\includegraphics[width=\textwidth]{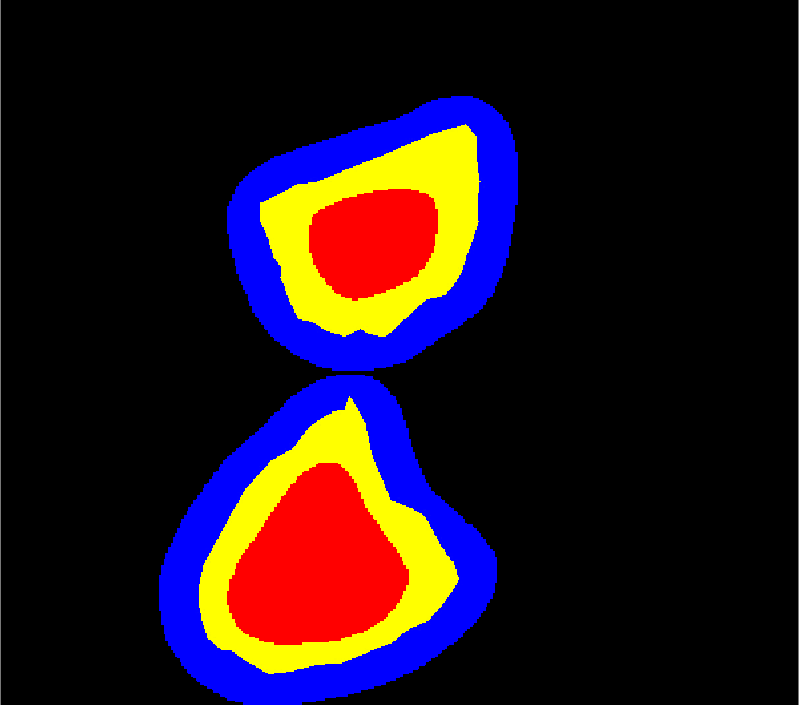}
		\label{fig:1}
	\end{subfigure}
	\begin{subfigure}{0.19\textwidth}
		\centering
		\includegraphics[width=\textwidth]{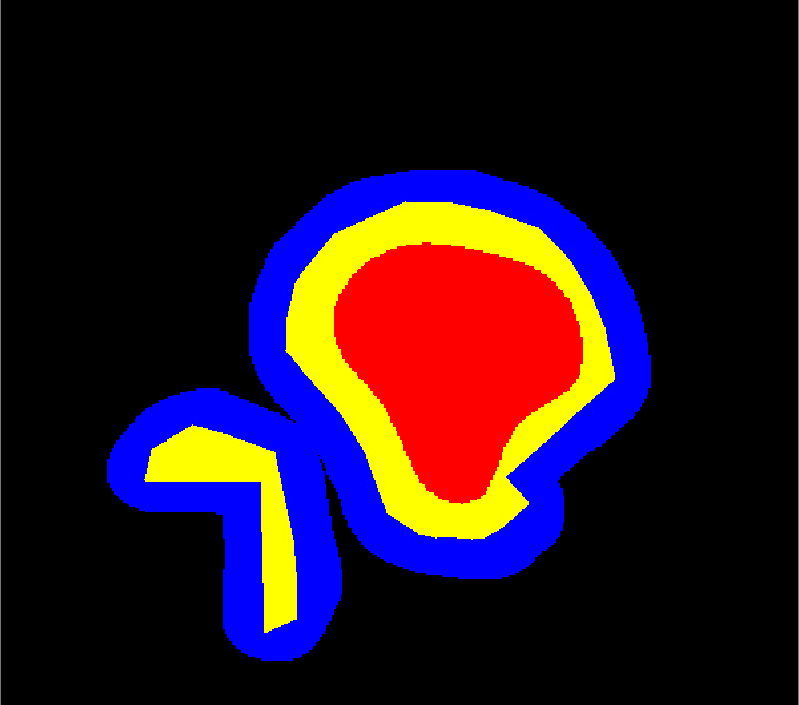}
		\label{fig:1}
	\end{subfigure}
	\begin{subfigure}{0.19\textwidth}
		\centering
		\includegraphics[width=\textwidth]{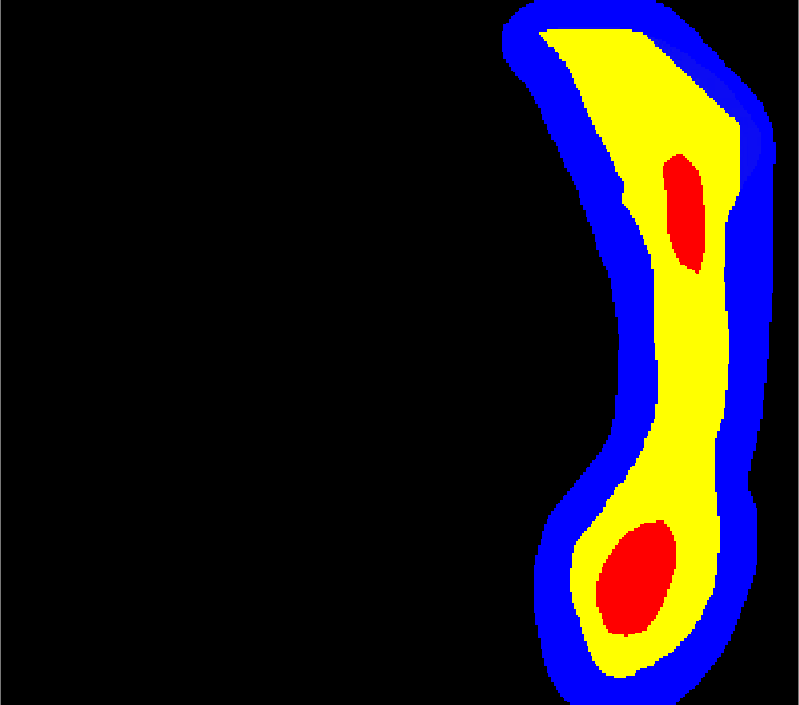}
		\label{fig:1}
	\end{subfigure}
	\begin{subfigure}{0.19\textwidth}
		\centering
		\includegraphics[width=\textwidth]{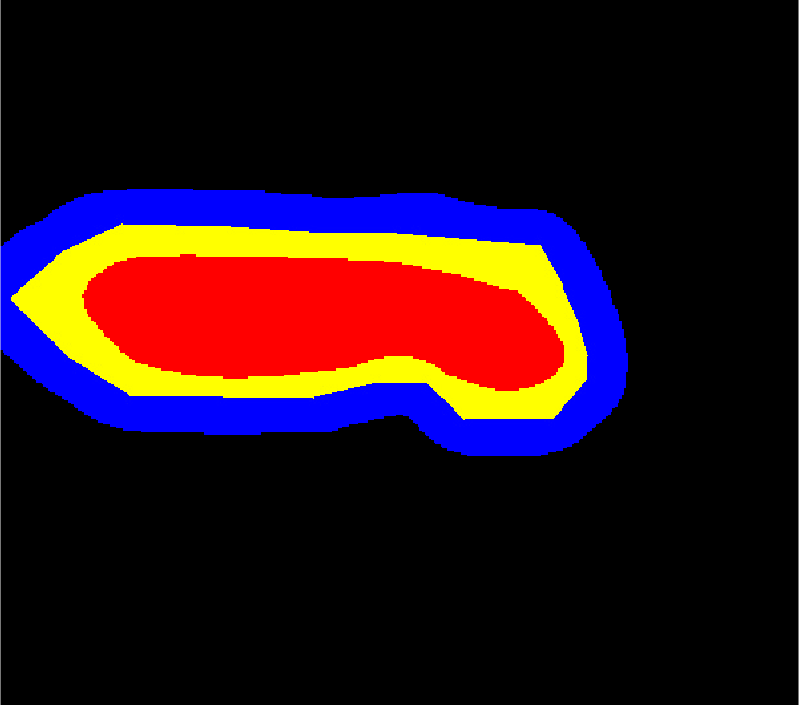}
		\label{fig:1}
	\end{subfigure}
	\begin{subfigure}{0.19\textwidth}
		\centering
		\includegraphics[width=\textwidth]{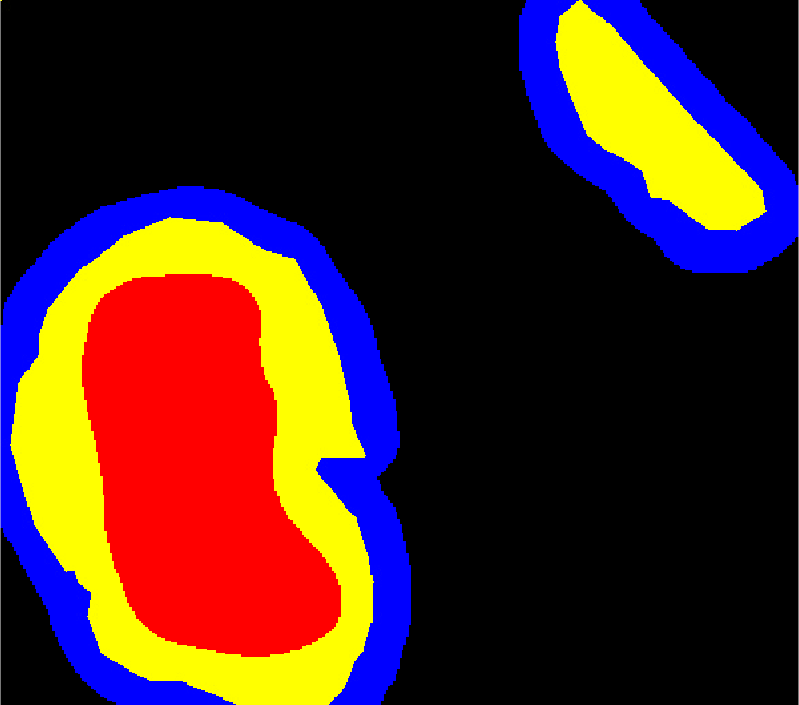}
		\label{fig:1}
	\end{subfigure}
	\label{fig:grid}
	\caption{Joint 90\% conformal confidence sets obtained using Corollary \ref{cor:weighting}, with $\alpha_1 = 0.02$ and $\alpha_2 = 0.08$, for the polpys images in Figure \ref{fig:res}.}\label{fig:joint}
\end{figure}
\newpage
\subsection{Marginal 80 \% Confidence regions}
\begin{figure}[h!]
	\begin{subfigure}{0.19\textwidth}
		\centering
		\includegraphics[width=\textwidth]{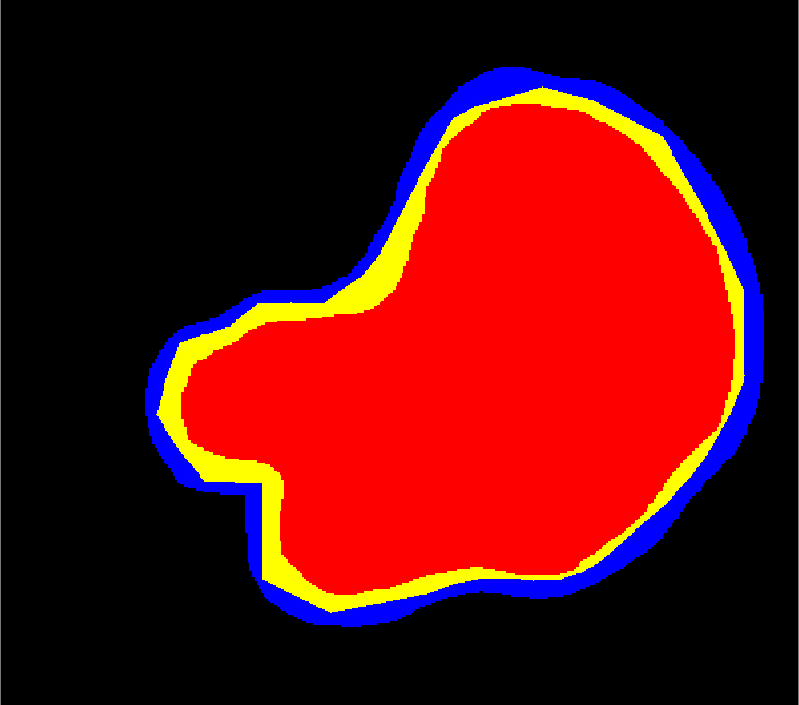}
		\label{fig:1}
	\end{subfigure}
	\begin{subfigure}{0.19\textwidth}
		\centering
		\includegraphics[width=\textwidth]{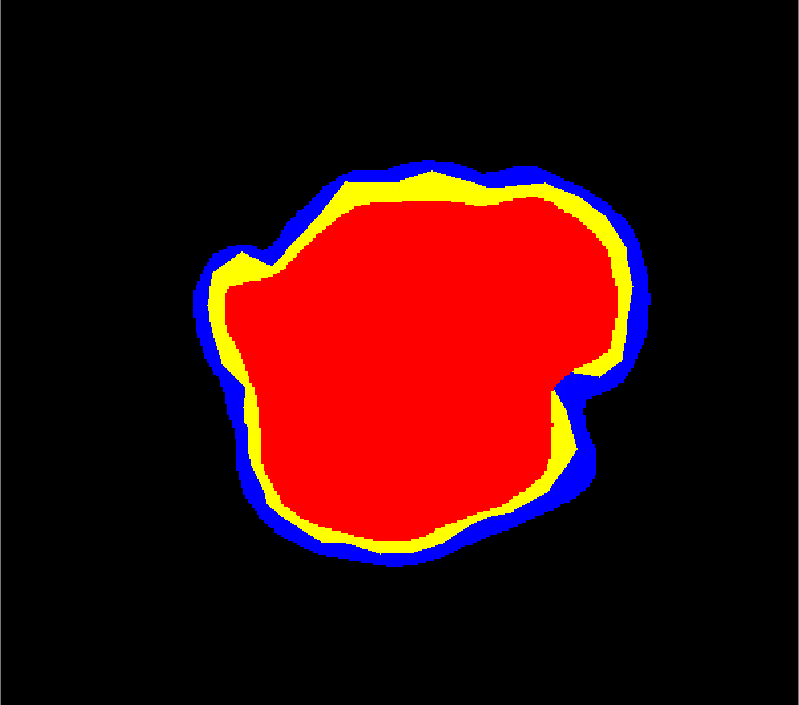}
		\label{fig:1}
	\end{subfigure}
	\begin{subfigure}{0.19\textwidth}
		\centering
		\includegraphics[width=\textwidth]{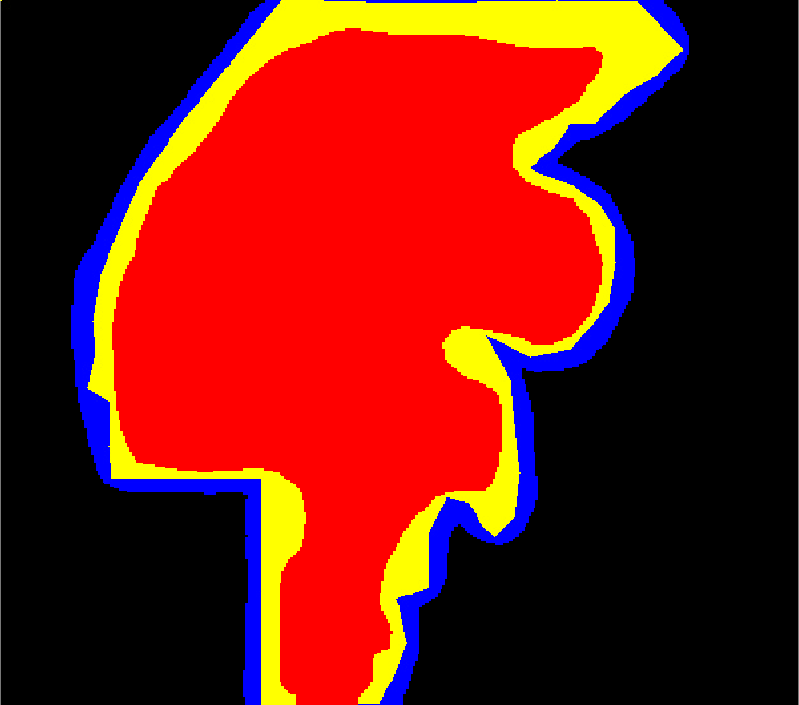}
		\label{fig:1}
	\end{subfigure}
	\begin{subfigure}{0.19\textwidth}
		\centering
		\includegraphics[width=\textwidth]{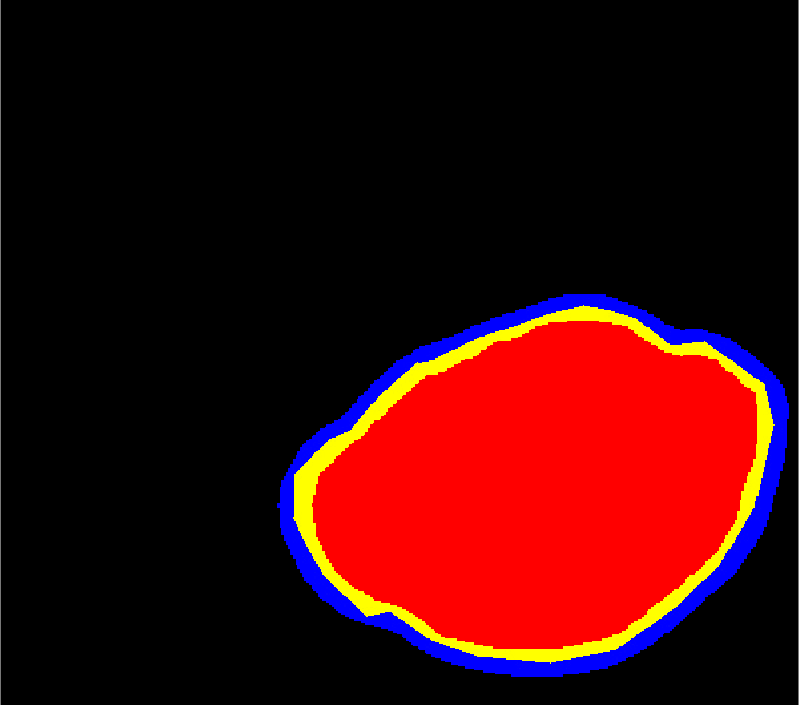}
		\label{fig:1}
	\end{subfigure}
	\begin{subfigure}{0.19\textwidth}
		\centering
		\includegraphics[width=\textwidth]{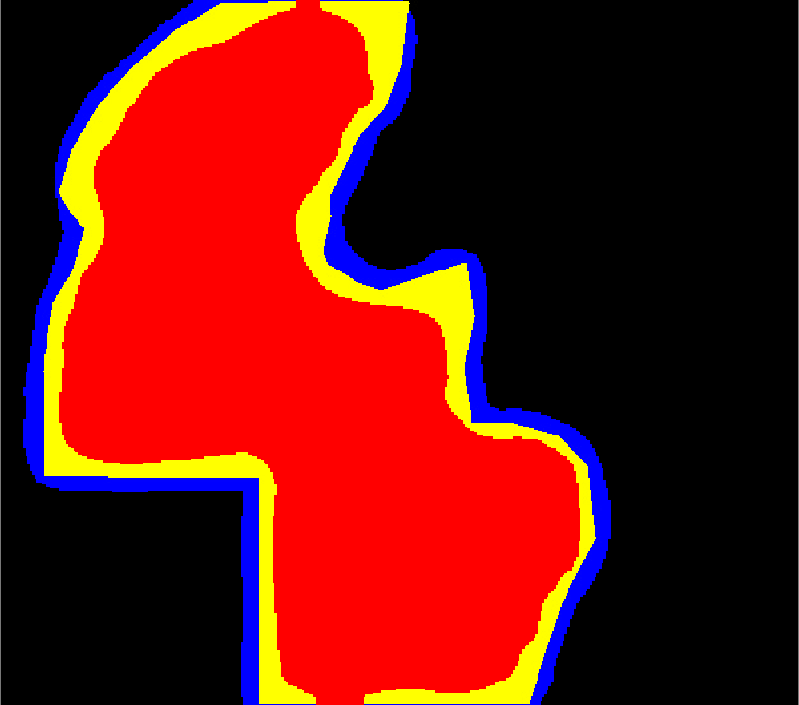}
		\label{fig:1}
	\end{subfigure}
	\vspace{-0.35cm}
	\\
	\begin{subfigure}{0.19\textwidth}
		\centering
		\includegraphics[width=\textwidth]{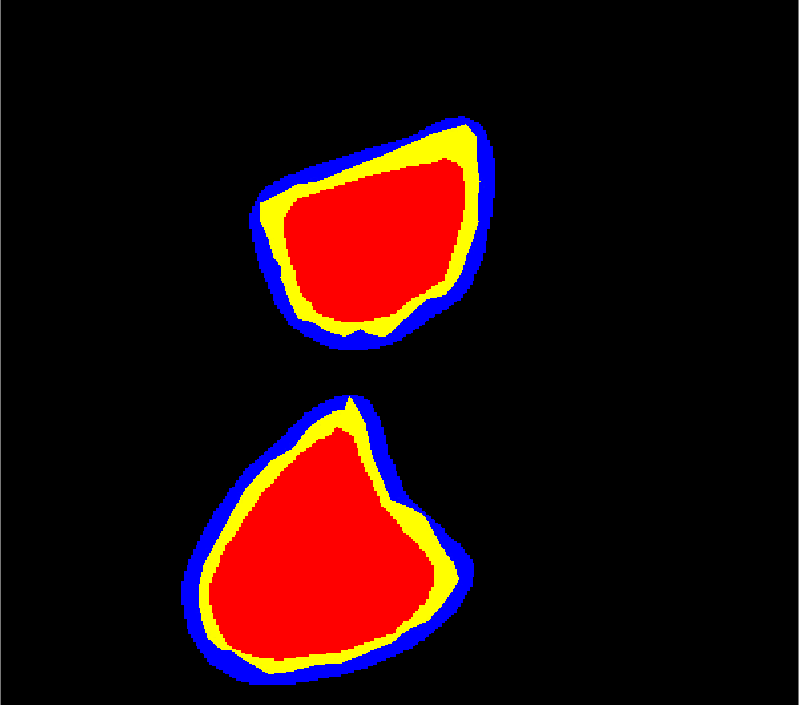}
		\label{fig:1}
	\end{subfigure}
	\begin{subfigure}{0.19\textwidth}
		\centering
		\includegraphics[width=\textwidth]{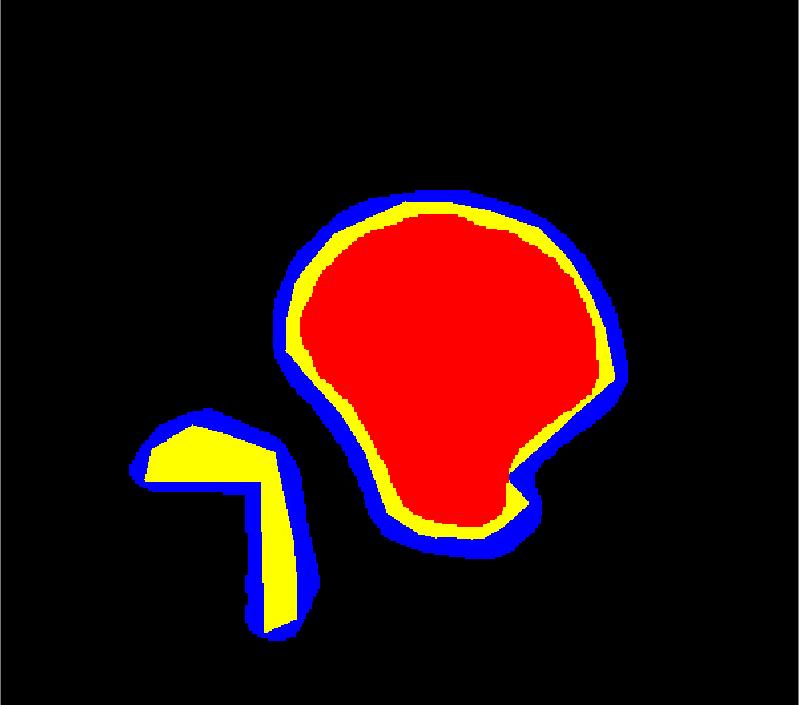}
		\label{fig:1}
	\end{subfigure}
	\begin{subfigure}{0.19\textwidth}
		\centering
		\includegraphics[width=\textwidth]{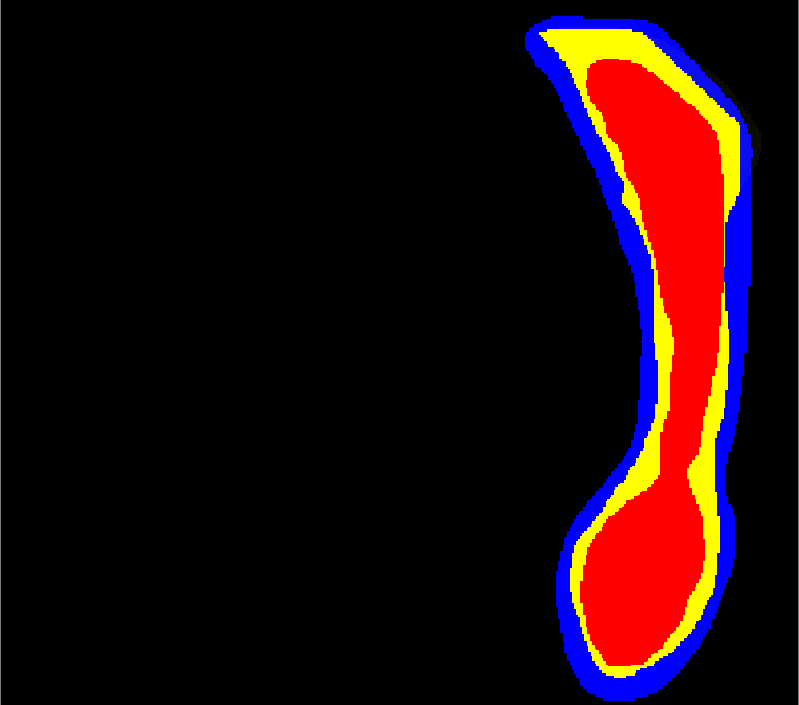}
		\label{fig:1}
	\end{subfigure}
	\begin{subfigure}{0.19\textwidth}
		\centering
		\includegraphics[width=\textwidth]{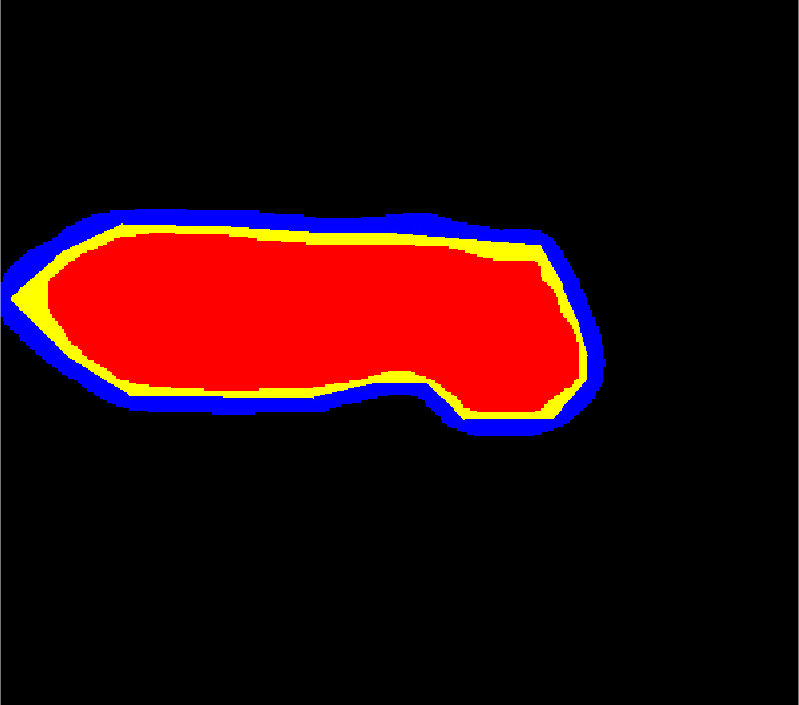}
		\label{fig:1}
	\end{subfigure}
	\begin{subfigure}{0.19\textwidth}
		\centering
		\includegraphics[width=\textwidth]{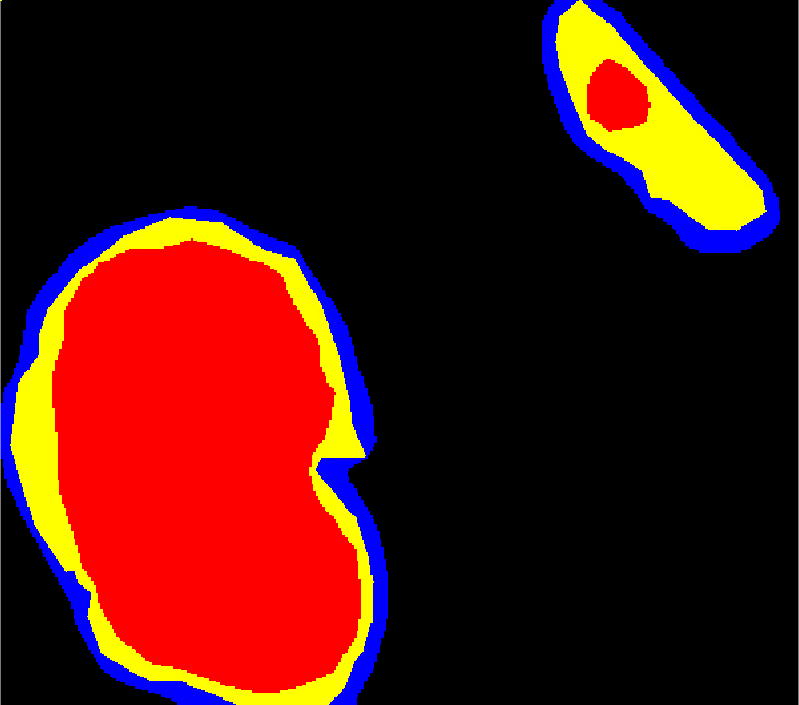}
		\label{fig:1}
	\end{subfigure}
	\label{fig:grid}
	\caption{Marginal 80\% conformal confidence sets obtained for the polpys images in Figure \ref{fig:res}.}\label{fig:joint2}
\end{figure}
\subsection{Marginal 95 \% Confidence regions}

\begin{figure}[h!]
	\begin{subfigure}{0.19\textwidth}
		\centering
		\includegraphics[width=\textwidth]{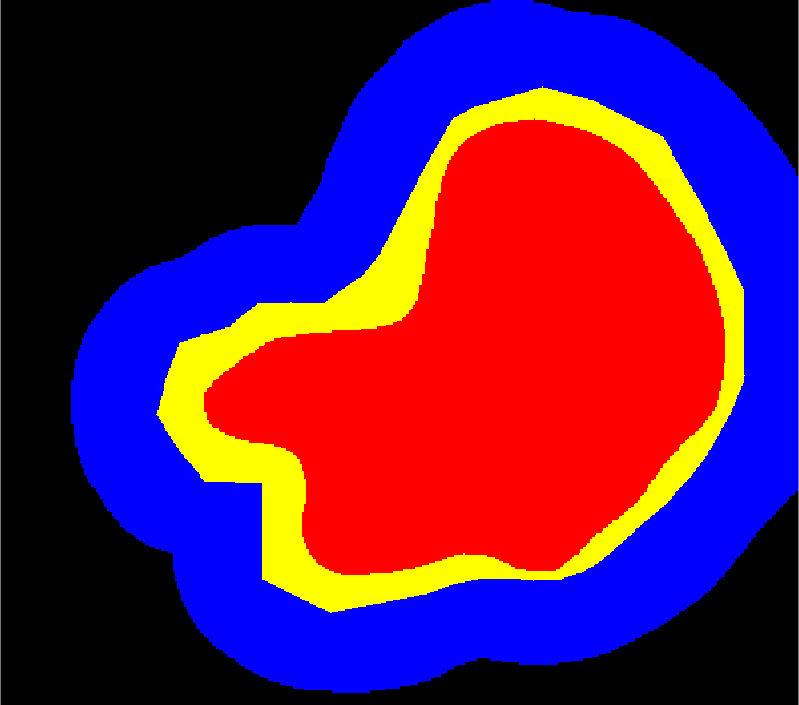}
		\label{fig:1}
	\end{subfigure}
	\begin{subfigure}{0.19\textwidth}
		\centering
		\includegraphics[width=\textwidth]{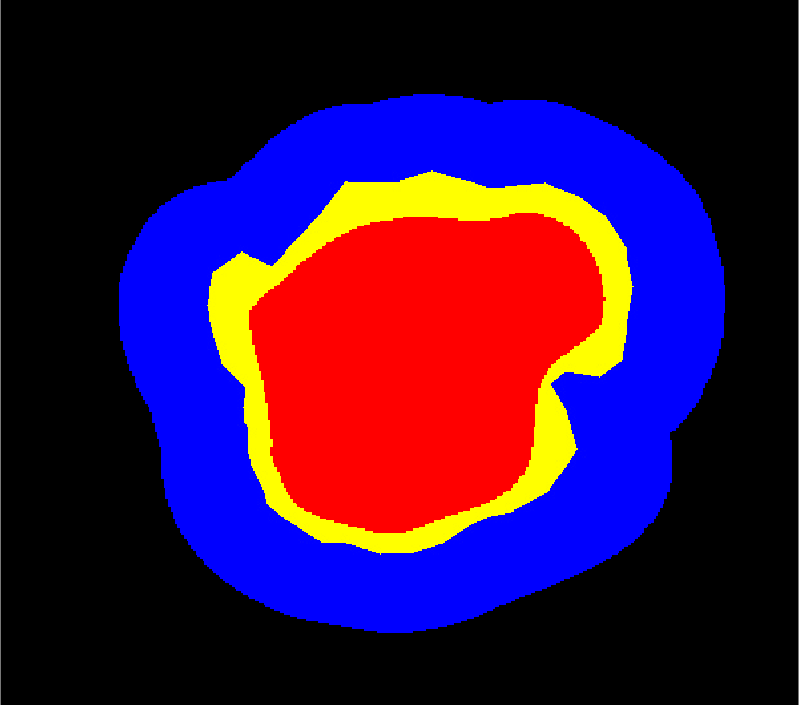}
		\label{fig:1}
	\end{subfigure}
	\begin{subfigure}{0.19\textwidth}
		\centering
		\includegraphics[width=\textwidth]{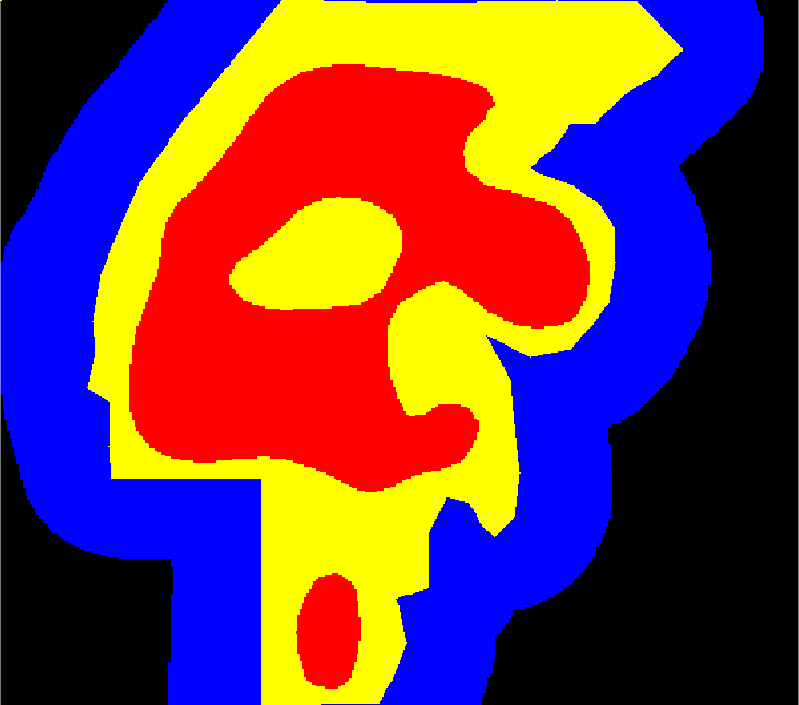}
		\label{fig:1}
	\end{subfigure}
	\begin{subfigure}{0.19\textwidth}
		\centering
		\includegraphics[width=\textwidth]{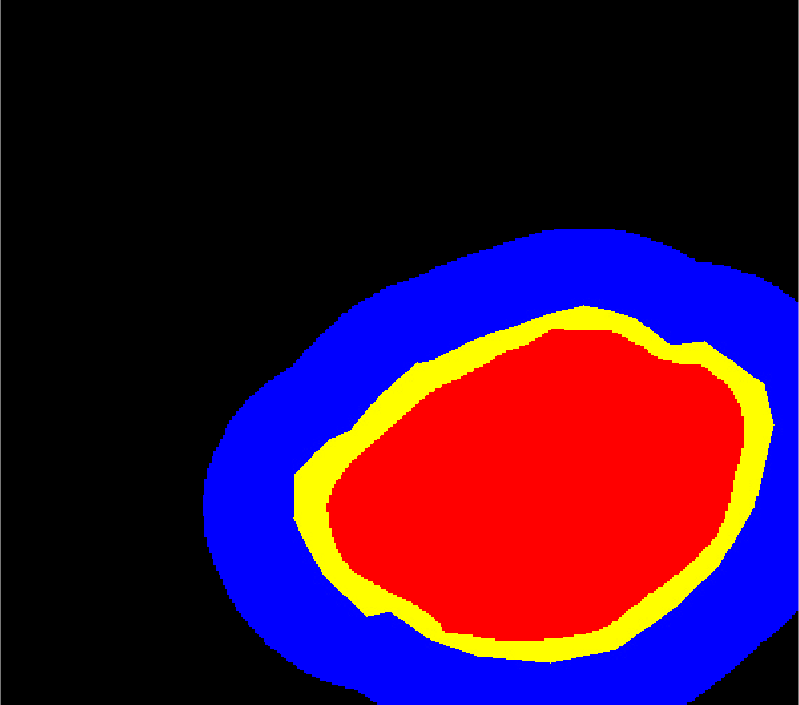}
		\label{fig:1}
	\end{subfigure}
	\begin{subfigure}{0.19\textwidth}
		\centering
		\includegraphics[width=\textwidth]{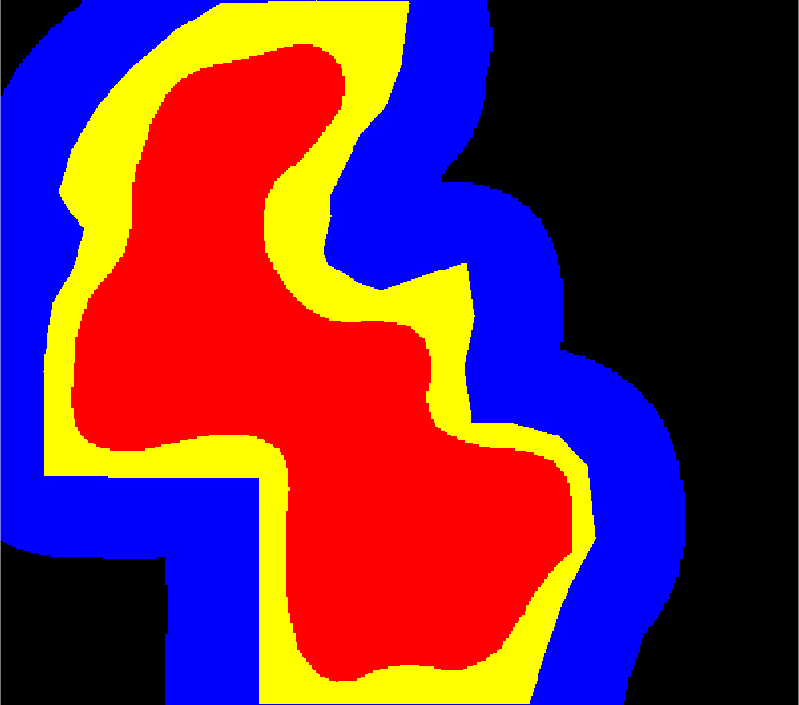}
		\label{fig:1}
	\end{subfigure}
	\vspace{-0.35cm}
	\\
	\begin{subfigure}{0.19\textwidth}
		\centering
		\includegraphics[width=\textwidth]{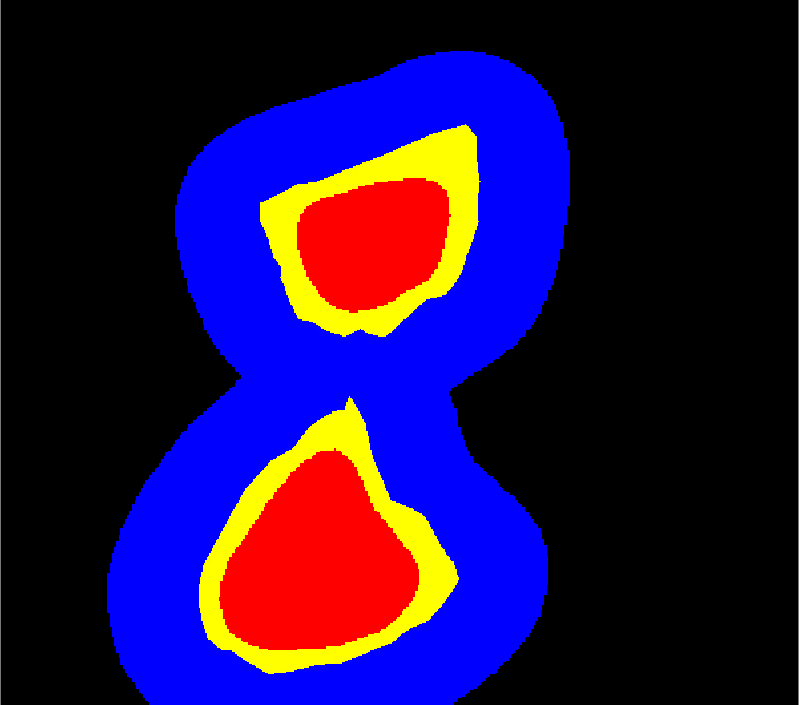}
		\label{fig:1}
	\end{subfigure}
	\begin{subfigure}{0.19\textwidth}
		\centering
		\includegraphics[width=\textwidth]{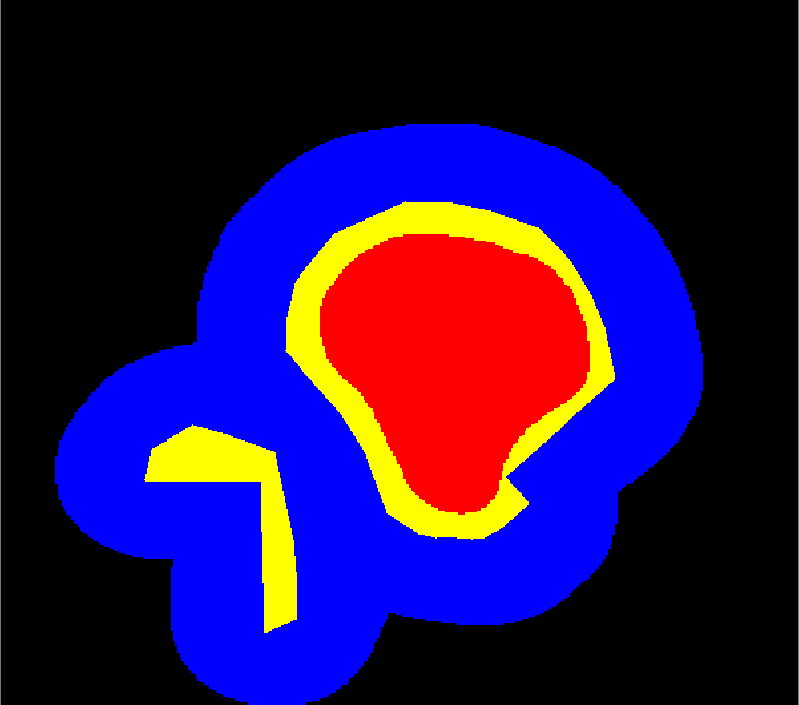}
		\label{fig:1}
	\end{subfigure}
	\begin{subfigure}{0.19\textwidth}
		\centering
		\includegraphics[width=\textwidth]{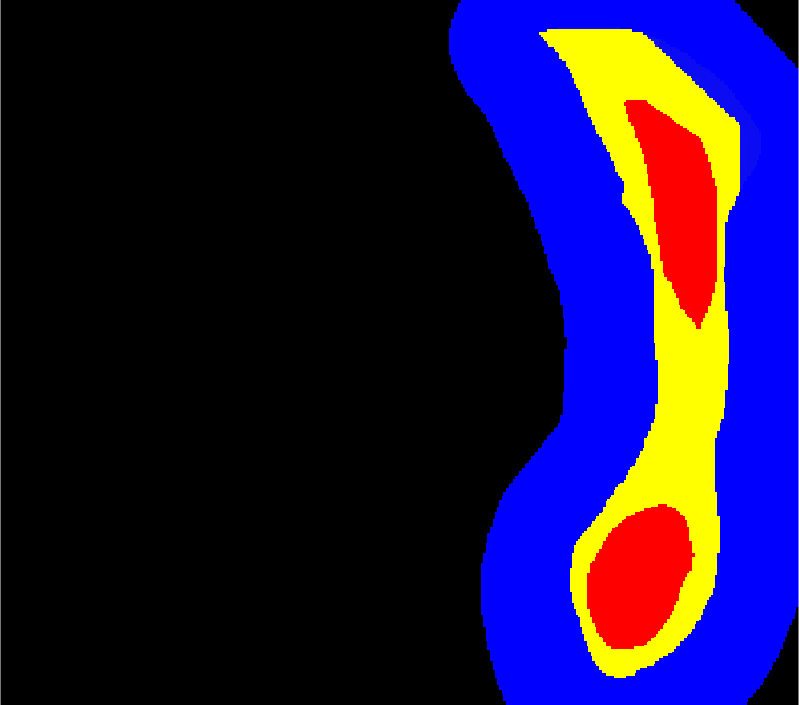}
		\label{fig:1}
	\end{subfigure}
	\begin{subfigure}{0.19\textwidth}
		\centering
		\includegraphics[width=\textwidth]{../figures/validation/joint_90/398.png}
		\label{fig:1}
	\end{subfigure}
	\begin{subfigure}{0.19\textwidth}
		\centering
		\includegraphics[width=\textwidth]{../figures/validation/joint_90/269.png}
		\label{fig:1}
	\end{subfigure}
	\label{fig:grid}
	\caption{Marginal 95\% conformal confidence sets obtained using for the polpys images in Figure \ref{fig:res}. These sets are also joint 90\% confidence sets with equally weighted $\alpha_1 = \alpha_2 = 0.05$. The influence of the weighting scheme can therefore examined by comparing to Figure \ref{fig:joint}.}\label{fig:joint3}
\end{figure}
\newpage
\subsection{Histograms of the coverage}
\begin{figure}[h!]
	\begin{center}
		\includegraphics[width=0.32\textwidth]{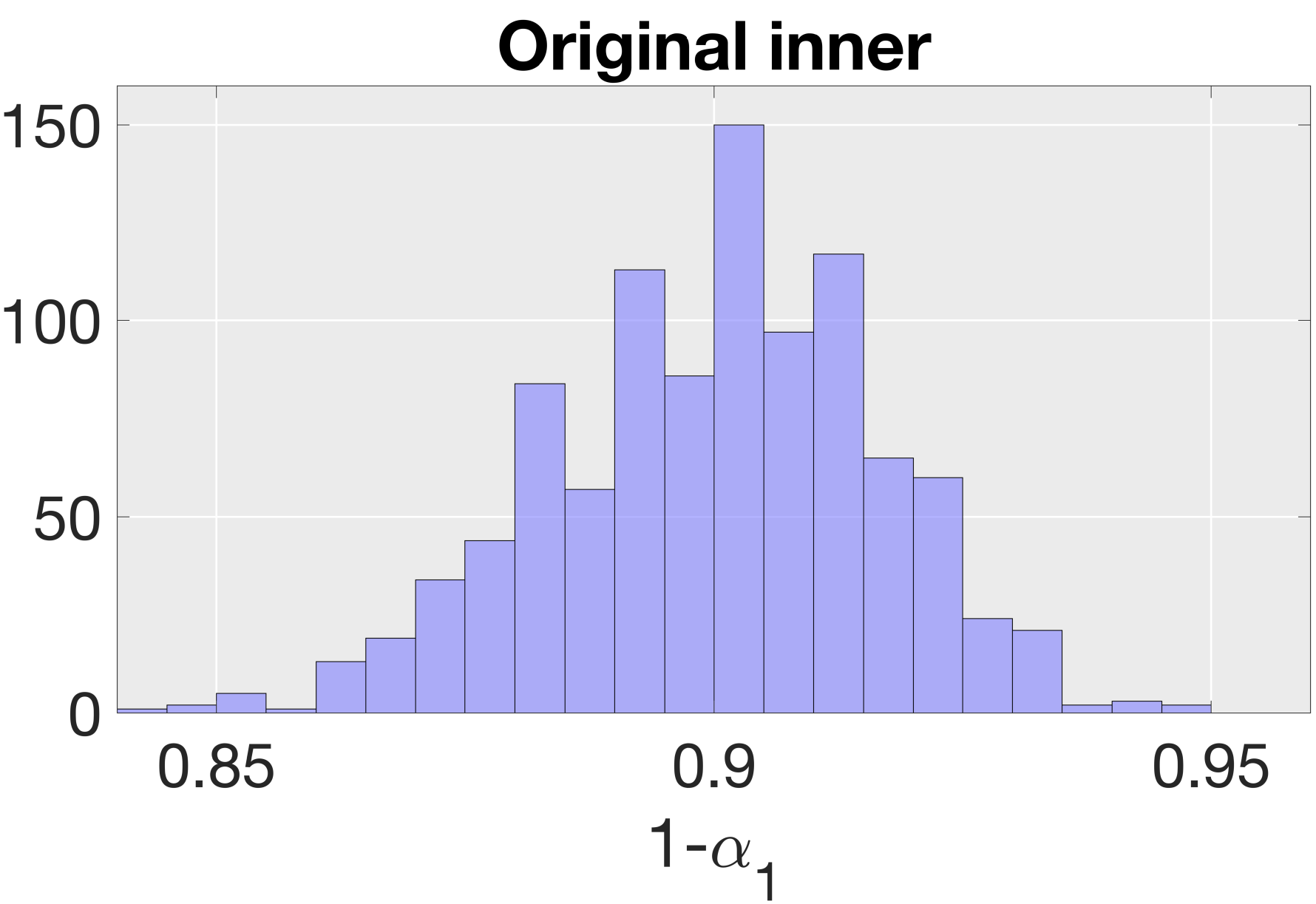}
		\includegraphics[width=0.32\textwidth]{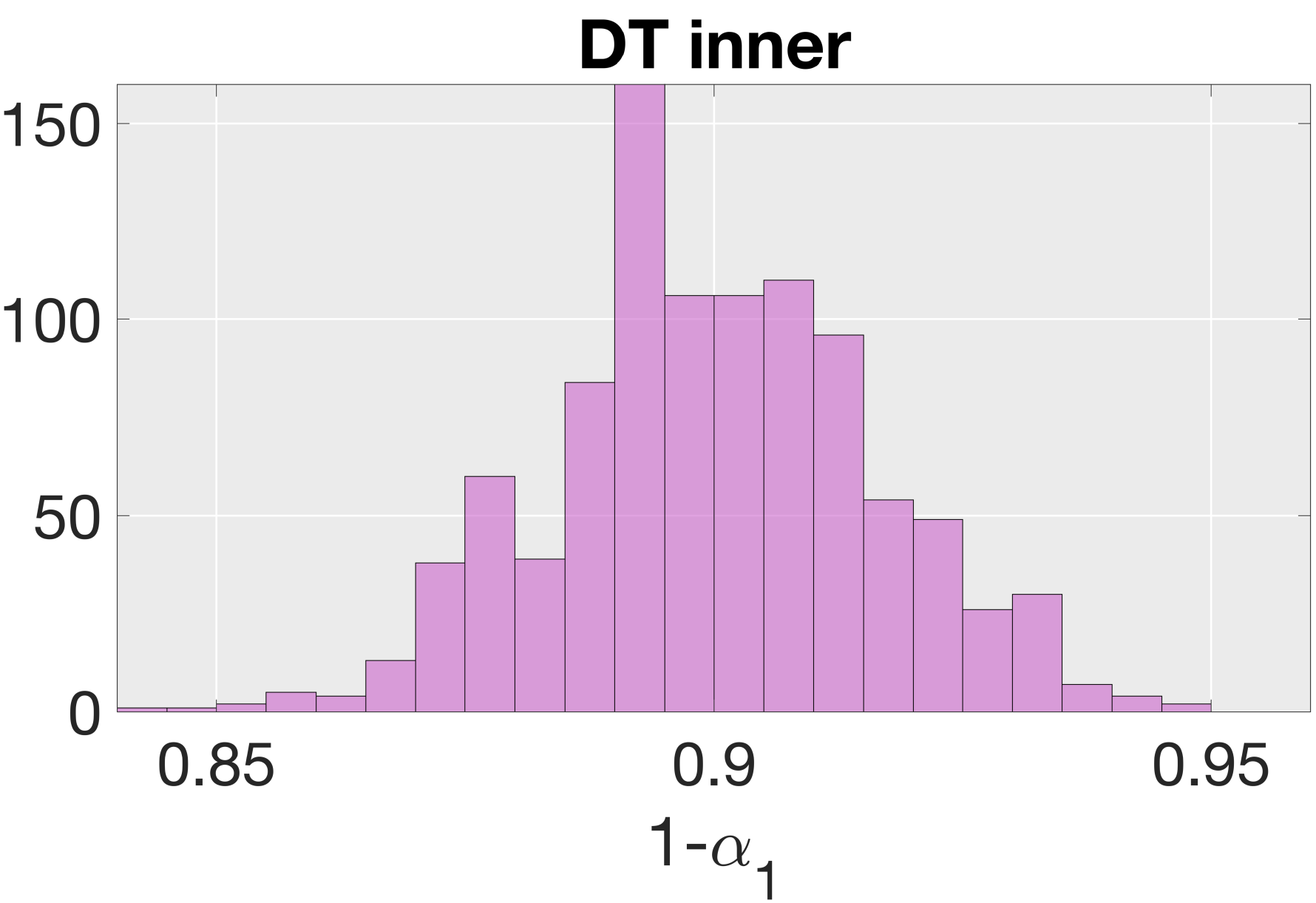}
		\includegraphics[width=0.32\textwidth]{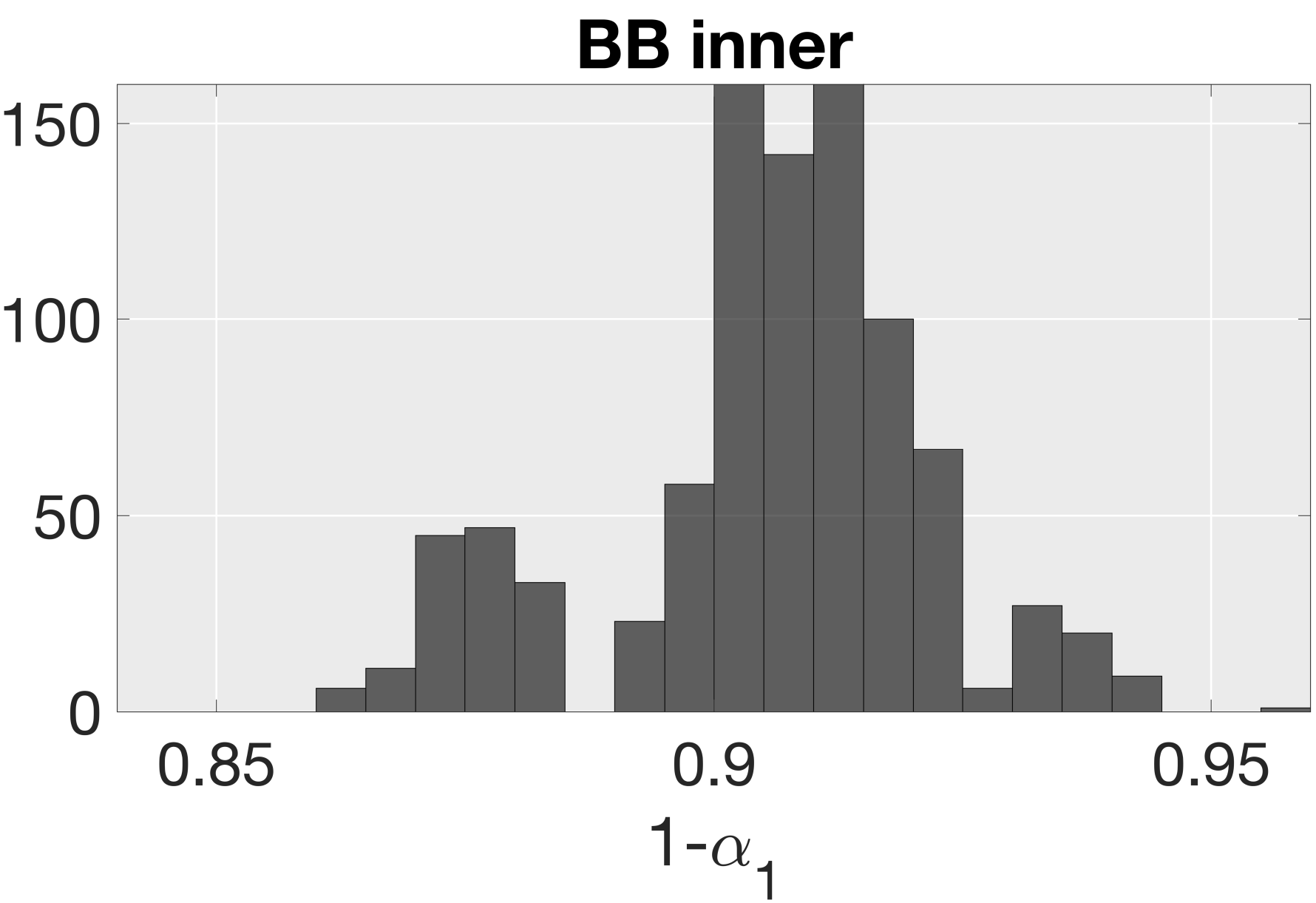}\\
		\includegraphics[width=0.32\textwidth]{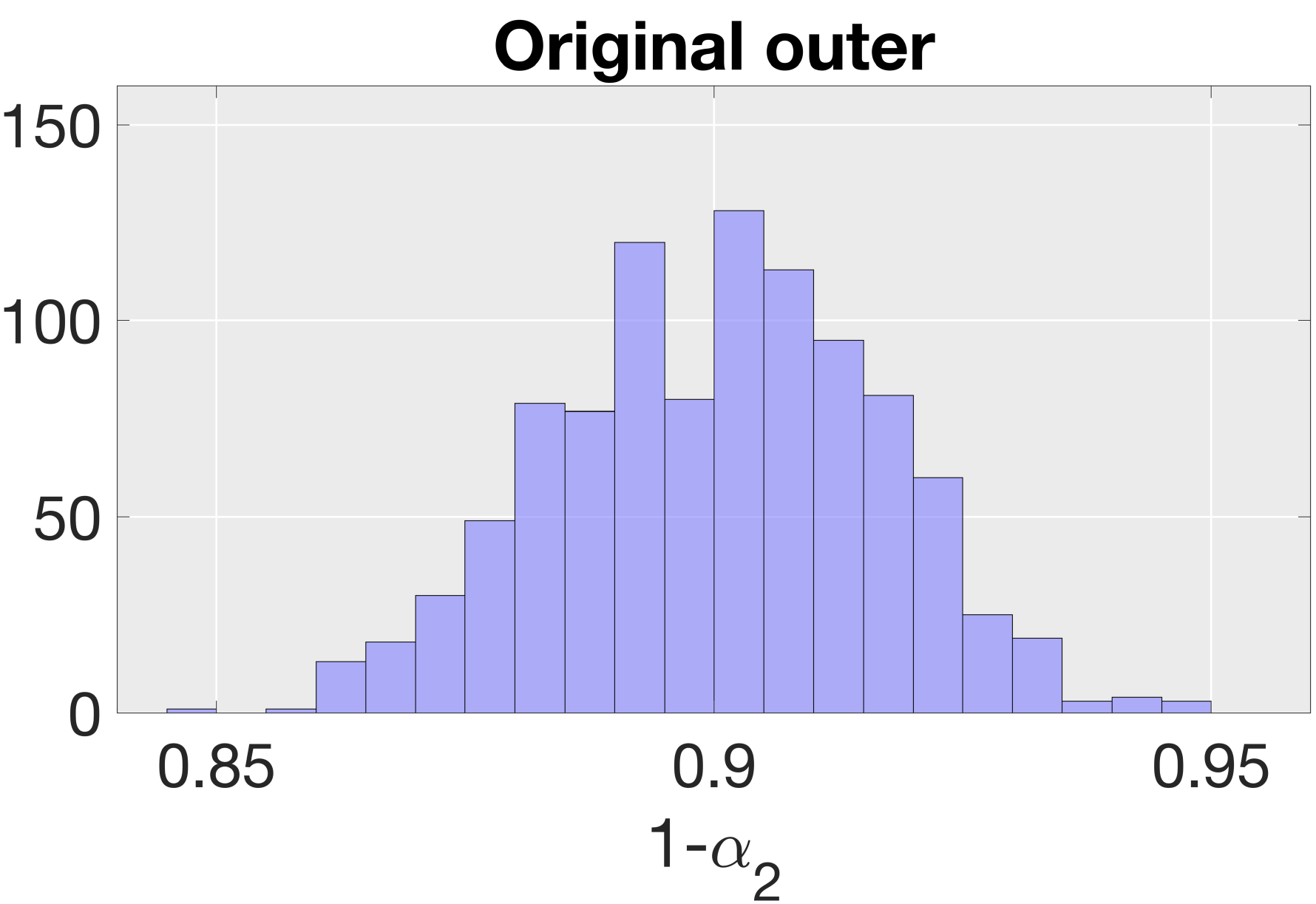}
		\includegraphics[width=0.32\textwidth]{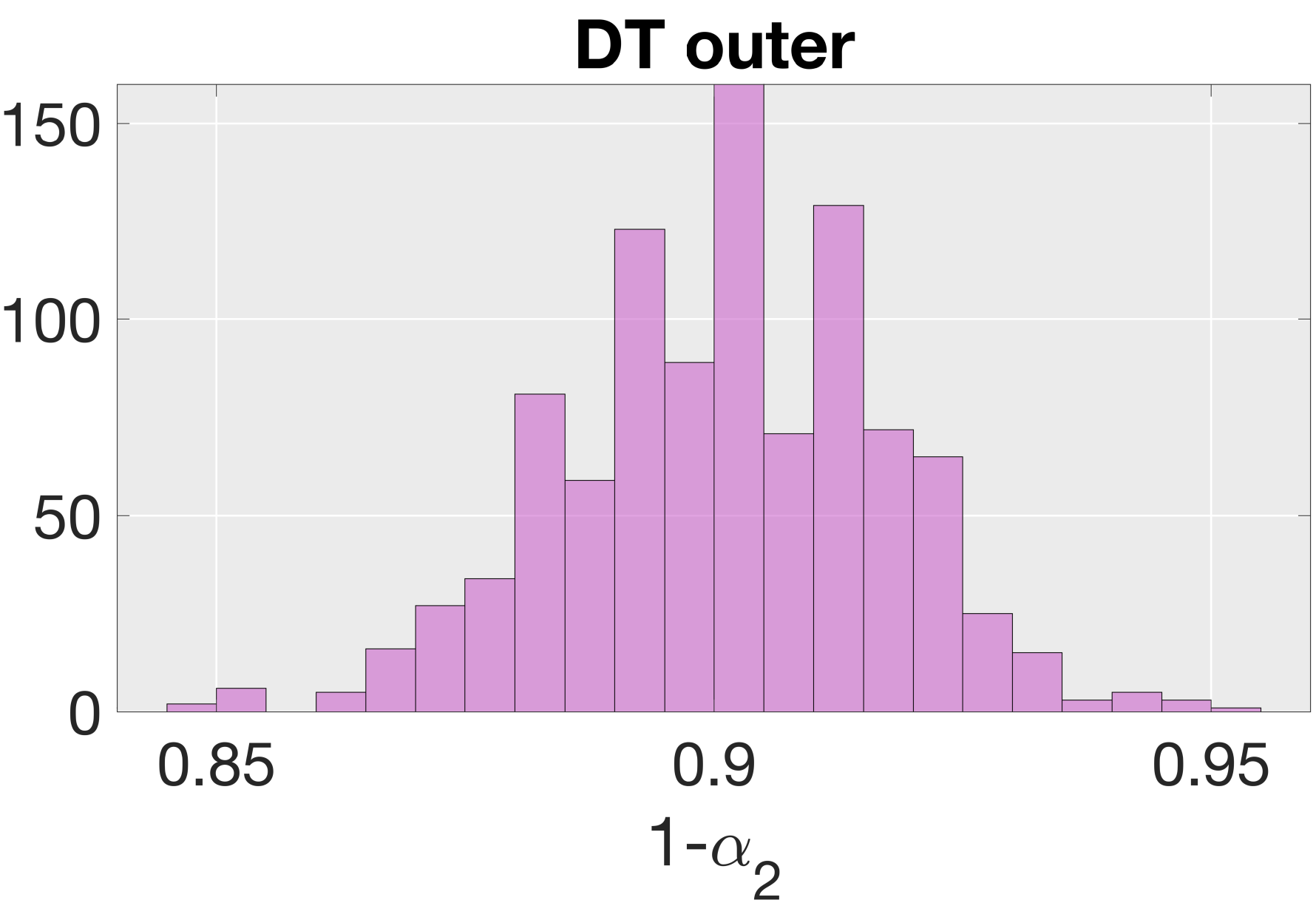}
		\includegraphics[width=0.32\textwidth]{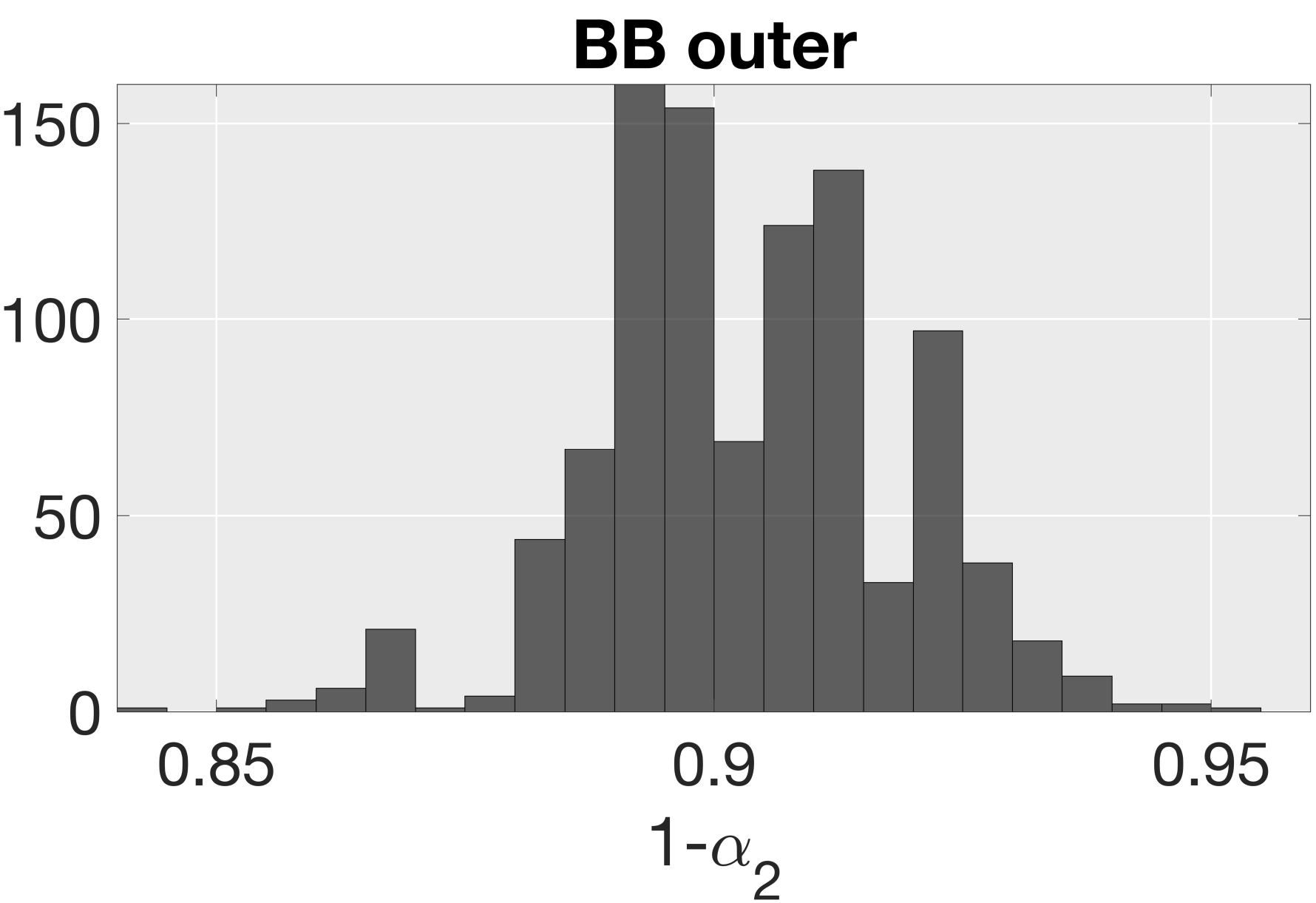}
	\end{center}
	\caption{Histograms of the coverage rates obtained across each of the validation resamples for 90\% inner and outer marginal confidence sets. We plot the results for the original scores, distance transformed scores (DT) and boundary box scores (BB) from left to right. The bounding box scores are discontinuous which is the cause of the discreteness of the rightmost histograms.}\label{fig:valhist}
\end{figure}

\end{document}